%% file: main_arxiv.tex
\documentclass[11pt]{article}
\pdfoutput=1

\usepackage[margin=1in]{geometry}
\usepackage[T1]{fontenc}
\usepackage[utf8]{inputenc}
\usepackage{lmodern}
\usepackage{amsmath,amssymb,amsfonts,amsthm,amscd,mathtools,bm,mathrsfs}
\usepackage{microtype,nicefrac,pifont,setspace}
\usepackage{graphicx,float,subfigure}
\usepackage{booktabs,longtable,tabularx,array,makecell,multirow}
\usepackage{enumitem,wrapfig}
\usepackage{xcolor,url,todonotes}
\usepackage{algorithm}
\usepackage{algpseudocode}
\usepackage[toc,page]{appendix}
\usepackage{authblk}
\usepackage[authoryear,round]{natbib}
\usepackage[linktocpage,pagebackref]{hyperref}
\usepackage{cleveref}
\allowdisplaybreaks

\input{math_commands.tex}

\graphicspath{{./}}

\newcolumntype{L}{>{\raggedright\arraybackslash}X}
\newcolumntype{C}{>{\centering\arraybackslash}X}

\hypersetup{
  colorlinks=true,
  linkcolor=blue,
  anchorcolor=blue,
  citecolor=blue,
  urlcolor=blue,
  hypertexnames=false
}
\setcitestyle{authoryear,open={(},close={)}}
\renewcommand*{\backref}[1]{\ifx#1\relax\else Page #1\fi}
\renewcommand*{\backrefalt}[4]{%
  \ifcase #1 \footnotesize{(Not cited.)}%
  \or \footnotesize{(Cited on page~#2.)}%
  \else \footnotesize{(Cited on pages~#2.)}%
  \fi
}

\newcommand{\Poly}{\operatorname{Poly}}

\newtheorem{proposition}{Proposition}
\newtheorem{theorem}{Theorem}
\newtheorem{lemma}{Lemma}
\newtheorem{remark}{Remark}
\newtheorem{assumption}{Assumption}

\newtheorem{corollary}{Corollary}
\newtheorem{example}{Example}

\newcommand{\epspsiemp}{\varepsilon_{\mathrm{pred, train}}}
\newcommand{\epspsigen}{\varepsilon_{\mathrm{pred, gen}}}
\newcommand{\predorac}{\widehat{\mathbb P}_{Y\mid C=c,A=a}}
\newcommand{\eotsol}{\widehat \Gamma_{\mathrm{EOT}}}
\newcommand{\rev}[1]{#1}
\newenvironment{revision}{\begingroup}{\endgroup}

\makeatletter
\renewcommand{\paragraph}{%
  \@startsection{paragraph}{4}{\z@}{1.25ex \@plus 1ex \@minus .2ex}{-1em}%
  {\normalfont\normalsize\bfseries}%
}
\makeatother

\title{\bf Counterfactual Generation via Flow Matching:\\ Coupling-Sensitive End-to-End Rates}
\author[1]{Yunrui Guan\thanks{This work was done during an internship at Amazon. Email: \texttt{yguan34@jh.edu}}}
\author[2,3]{Krishnakumar Balasubramanian}
\author[3]{Shiva Prasad Kasiviswanathan}

\affil[1]{Department of Applied Mathematics and Statistics, Johns Hopkins University.}
\affil[2]{Department of Statistics, University of California, Davis.}
\affil[3]{Amazon}

\date{}

\begin{document}
\maketitle

\begin{abstract}
Counterfactual generation seeks to sample outcomes under a hypothetical
intervention or decision using observational data collected under the factual
assignment mechanism.  We develop a flow-matching approach that combines a
sample-split, doubly robust training objective with a learned coupling between
observed source outcomes and target outcomes drawn from a fitted conditional
outcome model.  To enable finite-step generation, we leverage a
score-corrected stochastic sampler based on a Gaussian-smoothed interpolation.
Our main theoretical contribution is a coupling-sensitive KL bound for
constant-step Euler discretization: the error is controlled by moments of the
source--target displacement under the chosen coupling, rather than by global
uniform regularity of the velocity field, and has near-linear dependence on
the ambient dimension.  We also establish finite-sample non-parametric guarantees for the learned velocity and score fields when both the conditional outcome model and
the source--target coupling are estimated from data.  These bounds separate
approximation, coupling-replacement, nuisance-estimation, generalization, and
Monte Carlo errors and, combined with the sampler analysis, yield an
end-to-end guarantee for counterfactual generation.  Experiments on synthetic
and semi-synthetic image benchmarks support the coupling-dependent theory and
show that, at finite discretization budgets, the stochastic sampler can
outperform the corresponding deterministic ODE sampler.
\end{abstract}

\tableofcontents
\clearpage

\input{Intro}

\subsection*{AI Use Statement}
In this work, we used Generative AI for formulate mathematical claims (e.g., finding counter examples),
provide critical ingredients for proving mathematical claims (e.g., carry out some calculation that were checked by the authors),
assist in the writing of proofs,
design or provide feedback on research methodology or experiments (e.g., asking AI whether a specific theoretical claim can be improved toward a specific direction), implement methods (e.g., write code), literature search (e.g., ask if we missed any existing literature), language editing, proof auditing, and
\LaTeX{} assistance.
The authors independently checked the mathematical
arguments, citations, code, and experimental claims and take responsibility
for the final manuscript.

\subsection*{Acknowledgement}
We thank the authors of~\cite{dance2026debiased} for sharing their code with us before its public release.

\bibliographystyle{abbrvnat}
\bibliography{main_arxiv}

\clearpage
\noindent\rule{\textwidth}{0.11pt}
\begin{center}
\vspace{7pt}
{\Large Appendix}
\end{center}
\noindent\rule{\textwidth}{0.1pt}
\appendix

\input{Appendix_Experiments}
\input{Appendix_Additional}
\input{Appendix_Start}
\input{Proof_Stochastic_Sampler}

\input{sect_Q}
\input{Field_Error_Common}
\input{Vel_new_Q}

\input{Score_Error}

\input{Proof_End_to_End}

\end{document}

%% file: math_commands.tex
\usepackage{amsmath,amsfonts,bm}

\def\eqref#1{Equation~(\ref{#1})}

\def\1{\bm{1}}

\DeclareMathAlphabet{\mathsfit}{\encodingdefault}{\sfdefault}{m}{sl}
\SetMathAlphabet{\mathsfit}{bold}{\encodingdefault}{\sfdefault}{bx}{n}



%% file: Intro.tex
\section{Introduction}
\label{sec:introduction}
Counterfactual generation seeks to model the distribution of outcomes that
would arise if a treatment were assigned throughout a target population.
The full distribution reveals heterogeneity, variability, and tail behavior
that an average alone cannot capture. With observational data, however,
treatment assignment may depend on background covariates. Individuals who
receive treatment \(a\) can therefore differ systematically from the target
population, so a generator trained directly on their outcomes may reproduce
this selection bias.

Concretely, we observe data \((C,A,Y)\), where \(C\) denotes background
covariates, \(A\) is the treatment, and \(Y\in\mathbb R^d\) is the outcome.
Fix a treatment value \(a\). The observed and target outcome distributions
average the same conditional outcome law over different covariate
distributions:
\[
\underbrace{
    \mathbb P_{Y\mid A=a}
    =
    \int \mathbb P_{Y\mid C=c,A=a}\,
    d\mathbb P_{C\mid A=a}(c)
}_{\text{Observed outcome distribution}}
\qquad\text{and}\qquad
\underbrace{
    \int \mathbb P_{Y\mid C=c,A=a}\,
    d\mathbb P_C(c)
}_{\text{Target outcome distribution}}.
\]
The observed law weights covariates as they occur among individuals who
received treatment \(a\); the target instead retains the covariate
distribution of the full population. Under overlap, this target is identified
from the observed-data law. If consistency and conditional exchangeability
also hold, it coincides with the interventional outcome distribution under
treatment \(a\) \citep{hernan2020causal, dance2026debiased}. 

This structure suggests transporting observed outcomes to the target rather
than generating the target from an unrelated reference distribution. Given a
coupled source--target pair \((Y_0,Y_1)\), flow matching learns a
time-dependent velocity field by regressing the displacement \(Y_1-Y_0\)
along paths between the paired endpoints. The coupling determines which
source and target outcomes are connected. Structured choices, such as
optimal-transport couplings, can shorten these displacements and facilitate
flow learning
\citep{pmlr-v202-pooladian23a,tong2024improving,liu2023flow}.

Counterfactual flow matching nevertheless poses three linked problems: the
target-dependent training loss must be estimated from observational data, the
source--target coupling must itself be learned, and the learned dynamics must
be simulated accurately with finitely many steps.

Recent work primarily addresses the first problem through doubly robust
training. DoubleGen \citep{luedtke2025doublegen} combines a prediction oracle
\(\predorac\) for \(\mathbb P_{Y\mid C=c,A=a}\) with an estimate of the
inverse propensity
\(\alpha_a(c)=1/\mathbb P(A=a\mid C=c)\). Oracle samples provide conditional
training losses, while observed treatment-\(a\) outcomes provide
inverse-propensity-weighted residual corrections. This construction trains a
generator directly for the counterfactual target. In a complementary
approach, \citet{dance2026debiased} use
\(\mathbb P_{Y\mid A=a}\) as the flow source and derive a debiased objective
that accounts for its dependence on the observed-data law. They also extend
the construction to learned optimal-transport couplings.
Appendix~\ref{app:comparison_dance} gives a detailed comparison with their
method.

Our focus is the joint effect of coupling estimation, field estimation, and
finite-step simulation. Existing theory relates flow straightness to
finite-step Wasserstein error \citep{bansal2026on,liu2023flow}, but
straightness is a property of the induced flow and cannot be read directly
from the selected coupling. Other analyses rely on global Lipschitz bounds
\citep{zhou2025an,liu2026finitetime,guan2026total}, which reflect worst-case
regularity over the state space. Neither approach yields a guarantee
controlled directly by the source--target displacement under the chosen
coupling. Appendix~\ref{sect_existing_work_appendix} provides a broader
comparison with related finite-step sampling results.

Examples~\ref{example_fixed_sigma_ode_kl} and~\ref{example_no_sigma} in
Section~\ref{sec:sde_sampler} isolate the two resulting obstacles.
Even with fixed Gaussian smoothing, Example~\ref{example_fixed_sigma_ode_kl}
constructs couplings whose displacement approaches zero while the velocity
derivative diverges and the Euler ODE sampler has infinite KL error. Without
smoothing, Example~\ref{example_no_sigma} shows that arbitrarily small
displacement does not control the integrated squared derivative of the score.
Together, the examples motivate stochastic sampling, positive smoothing, and
regularity control averaged along the interpolation rather than uniformly over
the state space.

We address these obstacles with a Gaussian-smoothed interpolation and a
score-corrected SDE. With exact velocity and score fields, the SDE preserves
the interpolation marginals and reaches the smoothed target. We estimate the
propensity model, prediction oracle, and entropic optimal-transport coupling
on an auxiliary data split, then train the velocity and score fields on an
independent split using doubly robust objectives. Our analysis combines
displacement-dependent discretization control with finite-sample field bounds,
separating smoothing, coupling, estimation, and numerical errors in an
end-to-end guarantee. Our main contributions are as follows.
\begin{itemize}[noitemsep, leftmargin=0.15in]
    \item \textbf{Score-Corrected Stochastic Sampling.}
    We analyze a score-corrected SDE for a Gaussian-smoothed interpolation
    with constant smoothing. The two counterexamples delineate the respective
    roles of stochasticity and positive smoothing in obtaining
    coupling-dependent finite-step control.

    \item \textbf{Coupling-Sensitive Discretization Theory.}
    For constant-step Euler discretization, we prove a KL error bound controlled
    by
    \begin{align}\label{eq:d5def}
    D_5=\bigl(\mathbb E\|Y_1-Y_0\|^5\bigr)^{1/5},
    \end{align}
    rather than by global uniform regularity. The fifth moment is the highest displacement
moment required by our derivative bounds; its empirical analogue is a coupling diagnostic. With exact fields and the prescribed noise tuning, Euler-step complexity is near-linear in
dimension for fixed smoothing, $D_5$, and target accuracy.

    \item \textbf{Finite-Sample Learning and End-to-End Guarantees.}
    We derive finite-sample bounds for sample-split, doubly robust velocity and
    score estimators when the conditional outcome model, propensity model, and
    source--target coupling are learned from data. The bounds separate neural
    approximation, coupling replacement, nuisance estimation, generalization,
    and Monte Carlo errors. Combined with the sampling analysis, they yield an
    end-to-end \(W_1\) guarantee relative to the unsmoothed counterfactual
    target.

    \item \textbf{Empirical Validation.}
    Experiments on synthetic and semi-synthetic image benchmarks support the
    predicted dependence on coupling quality and show that, at finite
    discretization budgets, the score-corrected SDE can outperform the
    corresponding deterministic ODE sampler.
\end{itemize}

\section{Problem Setup and Method}
\label{sec:setup}

This section develops our counterfactual flow-matching procedure. We first
define the Gaussian-smoothed interpolation and its population fields, then
construct the learned source--target coupling and doubly robust field
objectives, and finally introduce the score-corrected SDE sampler.

\subsection{Gaussian-Smoothed Flow Matching}

Given a coupled source--target pair \((Y_0,Y_1)\),
let \(I_d\) denote the \(d\times d\) identity matrix, and let
\(G\sim\mathcal N(0,I_d)\) be independent of \((Y_0,Y_1)\).
For \(t\in[0,1]\),
define
\[
S_t=Y_0+t(Y_1-Y_0),
\qquad
X_t=S_t+\sigma G,
\]
where \(\sigma>0\) is constant in time. This Gaussian-smoothed
interpolation is a special case of the stochastic interpolant framework
\citep{albergo2025stochastic,liu2026finitetime}. Write \(q_t\) for the
density of \(X_t\). Its endpoint distributions are those of
\(Y_0+\sigma G\) and \(Y_1+\sigma G\); thus, the terminal population law
is a Gaussian-smoothed approximation to the target.

For \(t\sim\mathrm{Unif}[0,1]\), independently of \((Y_0,Y_1,G)\),
the population flow-matching loss and its minimizer are
\[
\mathcal L_{\mathrm{FM}}(v)
=
\mathbb E\!\left[
\left\|v(X_t,t)-(Y_1-Y_0)\right\|^2
\right],
\qquad
v_t(x):=v^\star(x,t)
=
\mathbb E[Y_1-Y_0\mid X_t=x].
\]
Gaussian smoothing gives the score as the population minimizer
of the score-matching loss:
\[
\mathcal L_{\mathrm{score}}(s)
=
\mathbb E\!\left[
\left\|s(X_t,t)+\frac{G}{\sigma}\right\|^2
\right],
\qquad
s_t(x):=s^\star(x,t)
=
\nabla_x\log q_t(x)
=
\mathbb E\!\left[-\frac{G}{\sigma}\,\middle|\,X_t=x\right].
\]

Appendix~\ref{App_score_loss} verifies the score identity. Under suitable
regularity conditions, the ODE with velocity \(v_t\), initialized from
\(q_0\), has marginal density \(q_t\). Section~\ref{sec:sde_sampler}
introduces stochastic dynamics that preserve the same marginal path.

\subsection{Sample Splitting and Learned Coupling}

For a fixed treatment value \(a\), we observe \(2n\) independent samples
from the observed-data law and split them into $\mathcal D_{\mathrm{tr}}
=
\{(C_i,A_i,Y_i)\}_{i=1}^n,$ and $\mathcal D_{\mathrm{aux}}
=
\{(C_i,A_i,Y_i)\}_{i=n+1}^{2n}.$ The auxiliary split is used to estimate the propensity
\(\widehat{\mathbb P}(A=a\mid C=c)\), prediction oracle \(\predorac\),
and coupling kernel \(\widehat Q_n(dy_0\mid y_1)\). The velocity and
score fields are then trained on \(\mathcal D_{\mathrm{tr}}\).
Conditioning on the auxiliary split fixes these estimated objects during
the training analysis.

The empirical source distribution is
\[
\widehat{\mathbb P}_{Y\mid A=a}^{\mathrm{aux}}
=
\frac{
\sum_{i=n+1}^{2n}\mathbf 1\{A_i=a\}\delta_{Y_i}
}{
\sum_{i=n+1}^{2n}\mathbf 1\{A_i=a\}
}.
\]
This distribution is defined when the auxiliary split contains at least
one treatment-\(a\) observation. The sampler is initialized from its
Gaussian convolution: $\widehat{\mathbb P}_{Y\mid A=a}^{\mathrm{aux}}
*
\mathcal N(0,\sigma^2I_d).$

We approximate the target distribution using \(\predorac\), implemented
as a conditional flow trained on \(\mathcal D_{\mathrm{aux}}\); see
Appendix~\ref{App_Prediction_Oracle_Accuracy}. For each auxiliary
covariate \(C_i\), we draw \(M_Q\) independent oracle samples
\(\widehat Y_{i,1}^{Q},\ldots,\widehat Y_{i,M_Q}^{Q}\), yielding
\[
\widehat{\mathbb P}_{Y_1}^{Q}
=
\frac{1}{nM_Q}
\sum_{i=n+1}^{2n}
\sum_{r=1}^{M_Q}
\delta_{\widehat Y_{i,r}^{Q}}.
\]
We construct an entropic optimal-transport (EOT) coupling
with regularization parameter \(\varepsilon_{\mathrm{OT}}>0\) between
\(\widehat{\mathbb P}_{Y\mid A=a}^{\mathrm{aux}}\) and
\(\widehat{\mathbb P}_{Y_1}^{Q}\). Disintegration with respect to the
target coordinate gives the kernel
\(\widehat Q_n(dy_0\mid y_1)\); its extension to target queries outside
the finite Monte Carlo sample is described in
Appendix~\ref{Sect_EOT}. The prediction-based target is used only to
construct this coupling kernel. The field objectives use the doubly
robust correction described next.

\subsection{Doubly Robust Field Training}

With the auxiliary objects fixed, a training loss is constructed by
pairing each target endpoint \(y_1\) with
\(y_0\sim\widehat Q_n(\cdot\mid y_1)\), and then sampling the
interpolation time and Gaussian noise.
For each training covariate, let \(M\geq1\) be the number of
independent prediction-oracle draws used to approximate the plug-in loss.
Let
\(\widehat\ell_i^{\mathrm{aux}}(v)\) denote the average flow-matching
loss using oracle-generated outcomes at \(C_i\), and let
\(\widehat\ell_i^{\mathrm{obs}}(v)\) denote the corresponding loss using
the observed outcome \(Y_i\) when \(A_i=a\). For a minibatch
\(\mathcal B\subseteq\{1,\ldots,n\}\), the loss is 
\[
\widehat L_{\mathcal B,\mathrm{DR}}(v)
=
\frac{1}{|\mathcal B|}
\sum_{i\in\mathcal B}
\left[
\widehat\ell_i^{\mathrm{aux}}(v)
+
\frac{\mathbf 1\{A_i=a\}}
{\widehat{\mathbb P}(A=a\mid C=C_i)}
\left\{
\widehat\ell_i^{\mathrm{obs}}(v)
-
\widehat\ell_i^{\mathrm{aux}}(v)
\right\}
\right].
\]
The first term is the prediction-oracle plug-in loss. Observed
treatment-\(a\) outcomes provide an inverse-propensity-weighted
correction to its residual error. We train the score using the same
construction with \(\mathcal L_{\mathrm{FM}}\) replaced by
\(\mathcal L_{\mathrm{score}}\). Algorithm~\ref{alg:minibatch_dr_flow_training}
gives the complete minibatch procedure.

\subsection{Score-Corrected SDE Sampler}
\label{sec:sde_sampler}

Let \(N\geq1\), \(h=1/N\), and recall the coupling-displacement moment
\eqref{eq:d5def}. We seek a finite-step error bound controlled by \(D_5\),
rather than by global uniform regularity, which need not follow from small
coupling displacement
\citep{zhou2025an,liu2026finitetime,guan2026total,
guan2025mirror,stephanovitch2026lipschitz}.
Examples~\ref{example_fixed_sigma_ode_kl} and
\ref{example_no_sigma} isolate the two resulting obstacles: stochasticity
addresses the deterministic Euler obstruction, whereas positive smoothing
provides the score regularity needed for displacement-dependent control.

Let \(\widehat X_N^{\mathrm{ODE}}\) denote the terminal constant-step
Euler iterate using the exact velocity \(v_t\), initialized from \(q_0\).
The following example shows that 
arbitrarily small $D_5$ guarantees neither small uniform velocity regularity nor finite KL error for deterministic Euler sampling. 

\begin{example}
\label{example_fixed_sigma_ode_kl}
Let \(p\in(0,1)\), and let \(B\sim\operatorname{Bernoulli}(p)\) and
\(U\sim\mathcal N(0,1)\) be independent. Define
\[
(Y_0,Y_1)
=
\begin{cases}
(U,U), & B=0,\\
(2U,2(1-\lambda)U), & B=1,
\end{cases}
\qquad
\lambda=\frac{4+\sigma^2}{4h}.
\]
As \(p\to0\),
\[
D_5\to0,\qquad \sup_{t\in[0,1],\,x\in\mathbb R}|\partial_xv_t(x)|\to\infty, \qquad \text{but} \qquad \int_0^1\mathbb E|\partial_xv_t(X_t)|^2\,dt\to0.
\]
Moreover, for every \(p\in(0,1)\), we have $\operatorname{KL}(
\mathcal L(X_1)\| \mathcal L(\widehat X_N^{\mathrm{ODE}})
)
=+\infty.$
\end{example}

Appendix~\ref{Appendix_Proof_Example} proves the claim. It motivates
controlling derivatives averaged along the interpolation rather than uniformly
over the state space. Following the It\^{o}-based approach of
\citet{xu2026polynomial,benton2024nearly}, we therefore use a stochastic
sampler for which this integrated regularity is sufficient.

For \(\tau>0\), the exact-field sampling process is
\[
dZ_t
=
\left(
v_t(Z_t)+\frac{\tau^2}{2}s_t(Z_t)
\right)dt
+
\tau\,dB_t,
\qquad
Z_0\sim q_0,
\]
where \(B_t\) is a standard Brownian motion. Under the conditions of
Lemma~\ref{lem:exact_sampling_conditions}, \(Z_t\) has density \(q_t\) for
every \(t\in[0,1]\). Thus the score correction preserves the interpolation's
marginal path while introducing stochasticity. Applying It\^{o}'s formula to
\(v_t\) and \(s_t\) along \(Z_t\) reduces the analysis to integrated derivative
bounds, which Gaussian smoothing controls through moments of \(Y_1-Y_0\).

Initialize the learned sampler according to $\widehat Z_0
\sim
\widehat{\mathbb P}_{Y\mid A=a}^{\mathrm{aux}}
*
\mathcal N(0,\sigma^2I_d).$ For \(t_k=kh\), the learned Euler sampler is
\[
\widehat Z_{k+1}
=
\widehat Z_k
+
h\left(
\widehat v(\widehat Z_k,t_k)
+
\frac{\tau^2}{2}\widehat s(\widehat Z_k,t_k)
\right)
+
    \tau\sqrt h\,G_k,
    \qquad
    G_k\sim\mathcal N(0,I_d),
\]
where the \(G_k\) are mutually independent and independent of the
initialization. Since \(t_N=Nh=1\), the terminal Euler output
\(\widehat Z_N\) is compared with \(Z_{t_N}=Z_1\).
Algorithm~\ref{alg:score_corrected_sde_sampler} gives the complete procedure.
Setting \(\tau=0\) recovers deterministic Euler sampling, although our KL
analysis requires \(\tau>0\). Appendix~\ref{sect_existing_work_appendix}
discusses related stochastic samplers and finite-step comparisons.

Gaussian smoothing is essential for controlling the required score
regularity through \(D_5\). The following example shows that this
control fails in general when \(\sigma=0\).

\begin{example}
\label{example_no_sigma}
Fix \(\varepsilon>0\), set \(\delta=2^{1/5}\varepsilon\), and let
\(Y_0\sim\mathcal N(0,1)\). For each integer \(m\geq0\), define
\[
K_m=\frac{(2m+1)\pi}{\delta},
\qquad
p_m(x)=\frac14+\frac18\sin(K_mx).
\]
Conditional on \(Y_0=x\), let
\[
Y_1
=
\begin{cases}
x+\delta, & \text{with probability }p_m(x),\\
x-\delta, & \text{with probability }p_m(x),\\
x, & \text{with probability }1-2p_m(x).
\end{cases}
\]
For the unsmoothed interpolation
\(Y_t=(1-t)Y_0+tY_1\), let \(q_t^{(m)}\) be its density and
\(s_t^{(m)}=\partial_x\log q_t^{(m)}\); the dependence of \(Y_t\) on
\(m\) is suppressed in the notation.. Every
coupling in this sequence satisfies
\(D_5=\varepsilon\), but
\[
\int_0^1
\mathbb E\!\left[
|\partial_xs_t^{(m)}(Y_t)|^2
\right]dt
\geq
\frac{35}{32768}K_m^4
\left(1-e^{-2K_m^2}\right)^4
\longrightarrow+\infty.
\]
Thus, even arbitrarily small \(D_5\) does not uniformly control the
integrated score regularity without smoothing.
\end{example}

\section{Coupling-Dependent Sampling and End-to-End Guarantees}

This section proves a coupling-dependent finite-step sampling bound, derives
finite-sample guarantees for the doubly robust velocity and score estimators,
and combines them into an end-to-end sample-complexity result. Unless otherwise
stated, \(\lesssim\) suppresses universal numerical constants and may also hide
logarithmic factors in the displayed parameters. A complete notation list is
provided in Appendix~\ref{sec:notation_summary}.

\subsection{Coupling-Dependent Finite-Step Sampling}

For \(t_k=kh\), where \(h=1/N\) (so \(t_N=1\)), write
\(\widehat v_{t_k}(x)=\widehat v(x,t_k)\) and
\(\widehat s_{t_k}(x)=\widehat s(x,t_k)\).  
We use the \hyperref[cond:standing_sampling]{Standard Sampling Conditions} in Appendix~\ref{Appendix_Sampling_Proof}.
For $\sigma,\tau>0$ and $D_5<\infty$, Lemma~\ref{lem:exact_sampling_conditions} provides a jointly continuous exact drift and a weak solution $Z$ with $\mathcal L(Z_t)=q_t$; throughout the sampling analysis, we use this solution.
The following  Assumption~\ref{Assumption_Integrated_Estimation_Error} defines \(\varepsilon_{\mathrm{vel}}^2\) and \(\varepsilon_{\mathrm{score}}^2\) as the grid-averaged mean-squared differences between the learned and population fields evaluated at \(Z_{t_k}\), and assumes both are finite. This is standard in related sampling analyses \citep{liu2025finitesde,liu2026finitetime,li2024unified,benton2024nearly}. Section~\ref{sec:finite_sample_field_estimation} later provides finite-sample conditions that verify this assumption.  
\begin{assumption}[Integrated Estimation Error]
\label{Assumption_Integrated_Estimation_Error}
Assume
\begin{align*}
\varepsilon_{\mathrm{vel}}^2
:=h\sum_{k=0}^{N-1}
\mathbb E\!\left[
\|v_{t_k}(Z_{t_k})-\widehat v_{t_k}(Z_{t_k})\|^2
\right]<\infty,\quad
\varepsilon_{\mathrm{score}}^2
:=h\sum_{k=0}^{N-1}
\mathbb E\!\left[
\|s_{t_k}(Z_{t_k})-\widehat s_{t_k}(Z_{t_k})\|^2
\right]<\infty.
\end{align*}
\end{assumption}

\begin{theorem}[KL Sampling Bound]
\label{Theorem_Sampling_Error}

Let \(N,d\geq1\), \(h=1/N\), \(\sigma,\tau>0\), and
\(D_5\in(0,\infty)\). Under \hyperref[cond:standing_sampling]{Standard Sampling Conditions} and
Assumption~\ref{Assumption_Integrated_Estimation_Error},
define
\begin{align*}
A_h&=\frac{4(9d+51)}{3}\frac{h^2D_5^4}{\sigma^2}+2\varepsilon_{\mathrm{vel}}^2,\qquad
B_h=\frac{124}{3}\frac{h^2(d+6)^3D_5^2}{\sigma^4}+\frac12\varepsilon_{\mathrm{score}}^2.
\end{align*}
Then
\begin{align}
\label{eq:explicit_KL_sampling_bound}
\operatorname{KL}\bigl(\mathcal L(Z_{t_N})\,\|\,\mathcal L(\widehat Z_N)\bigr)
\leq\frac{A_h}{\tau^2}+B_h\tau^2+\frac{hd(d+3)}{\sigma^4}\tau^4+\frac{68h^2(d+6)^3}{3\sigma^6}\tau^6+\frac{2h(d+1)D_5^2}{\sigma^2}.
\end{align}
The first two terms are balanced by $\tau=(A_h/B_h)^{1/4}$.
\end{theorem}

\begin{remark}[Coupling-Dependent Scaling]
In the exact-field case
\(\varepsilon_{\mathrm{vel}}=\varepsilon_{\mathrm{score}}=0\), choosing
\(\tau^2=\sigma D_5/d\) gives
\begin{align}
\label{eq_KL_D5_Order}
\operatorname{KL}\bigl(\mathcal L(Z_{t_N})\,\|\,\mathcal L(\widehat Z_N)\bigr)
\lesssim\frac{hd}{\sigma^2}D_5^2+\frac{h^2d^2}{\sigma^3}D_5^3.
\end{align}
Thus, for any \(\varepsilon_{\mathrm{target}}\in(0,1)\),
\begin{align*}
    h\lesssim \frac{1}{dD_5^2/(\varepsilon_{\mathrm{target}}^2\sigma^2)+dD_5/\sigma}
\end{align*}
gives KL error of order \(\varepsilon_{\mathrm{target}}^2\), corresponding to
\(\widetilde{\mathcal O}\big(d\big[\frac{D_5^2}{\varepsilon_{\mathrm{target}}^2\sigma^2} + \frac{D_5}{\sigma}\big] \big)\) Euler steps.
\end{remark}

Recall that Example~\ref{example_no_sigma} motivates the importance of smoothing. 
Compared with bounds for constant-step Euler discretizations of stochastic-interpolant SDE and ODE samplers in
\citep{liu2025finitesde,liu2026finitetime}, our result improves the dimension dependence from quadratic to near-linear; Appendix
\ref{sect_existing_work_appendix} gives a detailed comparison. This dimension comparison holds with smoothing, displacement moment,
and target accuracy fixed.

\subsection[Finite-Sample Field Estimation and End-to-End Guarantee]
{Finite-Sample Field Estimation and End-to-End Guarantee}
\label{sec:finite_sample_field_estimation}

We now give sufficient conditions for
Assumption~\ref{Assumption_Integrated_Estimation_Error}. Below, \(d\) is fixed, so implicit constants may depend on \(d\) in
addition to the logarithmic factors suppressed by \(\lesssim\).

\begin{assumption}[Clipping, Overlap, and Bounded Support]
\label{ass:clipping_bounded_support}
We work with outcomes clipped to \([-1,1]^d\) and assume that, for every
\(c\in\operatorname{supp}(\mathbb P_C)\),
\[
\begin{gathered}
\mathbb P_{Y\mid C=c,A=a}([-1,1]^d)=\predorac([-1,1]^d)=1,\\[-2pt]
\mathbb P(A=a\mid C=c)\geq\pi_{\min}>0,\qquad
0\leq\widehat\alpha(c)\leq2\pi_{\min}^{-1}.
\end{gathered}
\]
Here \(\widehat\alpha(c)=1/\widehat{\mathbb P}(A=a\mid C=c)\). Thus every
endpoint pair satisfies \(Y_0,Y_1\in[-1,1]^d\) almost surely, and
\(\widehat Q_n([-1,1]^d\mid y)=1\) for every \(y\in\mathbb R^d\).
\end{assumption}

\begin{assumption}[Propensity Accuracy]
\label{ass:alpha_l2_accuracy}
Let \(\alpha_a(c)=1/\mathbb P(A=a\mid C=c)\). Conditional on the
auxiliary sample used to fit \(\widehat\alpha\), assume 
\begin{align*}
    \left\|\widehat\alpha/\alpha_a-1\right\|_{L^2(\mathbb P_C)}
\leq\varepsilon_\alpha.
\end{align*}
\end{assumption}
Assumption~\ref{ass:clipping_bounded_support} bounds outcomes,
oracle-generated samples, and inverse-propensity weights while enforcing
overlap. Assumption~\ref{ass:alpha_l2_accuracy} requires relative
\(L^2(\mathbb P_C)\) accuracy of the estimated inverse propensity.

\begin{assumption}[Prediction-Oracle Accuracy]
\label{ass:prediction_oracle_accuracy}
Define
\(\varepsilon_{\mathrm{pred}}
=[(\epspsiemp)^2+\epspsigen^2]^{1/2}\), and assume
\[
\left[\mathbb E_C
W_1\!\left(\mathbb P_{Y\mid C=c,A=a},\predorac\right)^2\right]^{1/2}
\leq\varepsilon_{\mathrm{pred}}.
\]
\end{assumption}

Assumption~\ref{ass:prediction_oracle_accuracy} controls the oracle's average
conditional \(W_1\) error over population covariates. Appendix
\ref{App_Prediction_Oracle_Accuracy} derives this requirement from empirical-fit
and generalization conditions and describes conditional-flow training. Let $\mu_1
:=
\int\mathbb P_{Y\mid C=c,A=a}\,d\mathbb P_C(c).$
For any coupling kernel \(Q(dy_0\mid y_1)\), define the joint endpoint law 
\begin{align*}
    \Gamma_{\mu_1,Q}(dy_0,dy_1)
:=
Q(dy_0\mid y_1)\mu_1(dy_1).
\end{align*}
To emphasize the role of coupling explicitly, under \((Y_0,Y_1)\sim\Gamma_{\mu_1,Q}\), we let \(q_{t,Q}\) denote the density of
\(X_t=Y_0+t(Y_1-Y_0)+\sigma G\), and define
\[
v_Q^\star(x,t)
:=
\mathbb E[Y_1-Y_0\mid X_t=x],
\qquad
s_Q^\star(x,t)
:=
\nabla_x\log q_{t,Q}(x).
\]

\phantomsection
\label{def:population_comparison_coupling}
\par\noindent\textbf{Population Coupling.}
Let \(\mathcal E_a=\{\sum_{i=n+1}^{2n}\mathbf 1\{A_i=a\}\geq1\}\).
On \(\mathcal E_a\), optimally couple
\(Y_1\sim\mu_1\) and
\(\widehat Y_1\sim\widehat{\mathbb P}_{Y_1}^{Q}\) in \(W_1\), and then draw
\(Y_0\sim\widehat Q_n(\cdot\mid\widehat Y_1)\). We define
\(Q(\cdot\mid Y_1)\) as the resulting conditional law of \(Y_0\) given
\(Y_1\). Thus \(\Gamma_{\mu_1,Q}\) has target marginal \(\mu_1\) and source
marginal \(\widehat{\mathbb P}_{Y\mid A=a}^{\mathrm{aux}}\).
Lemma~\ref{lem:existence_population_valid_Q} formalizes this construction
and proves the coupling-replacement bound used below.

Let \(\mathcal V_n\) denote the clipped coordinatewise sparse-ReLU
velocity-field class with input-clipping radius \(R\) and approximation
tolerance \(\varepsilon\), as defined in Appendix~\ref{Appendix_Vel_Error},
and let \(\mathcal S_n\) denote the score-field class induced by clipped
sparse-ReLU denoisers through \(s_u(x,t)=\{u(x,t)-x\}/\sigma^2\), as defined
in Appendix~\ref{Appendix_Score_Error}. The next proposition gives explicit
finite-sample bounds for the velocity- and score-field components of
Assumption~\ref{Assumption_Integrated_Estimation_Error}. If Assumption~\ref{ass:nonempty_treatment_a_auxiliary_sample}, which ensures that the auxiliary split contains at least one treatment-\(a\) observation, is added to the hypotheses of Proposition~\ref{thm:final_velocity_L2_bound_DR}, then
  \(\mathcal E_a\) holds by assumption and the term \(e^{-n\pi_{\min}}\) can be omitted from the failure probability.

\begin{proposition}[Doubly Robust Velocity and Score Field Error Bound]
\label{thm:final_velocity_L2_bound_DR}
\label{thm:final_score_L2_bound_DR}
Let \(n\geq2\), \(M,M_Q\geq1\), and \(0<\sigma\leq1\).
Under Assumptions~\ref{ass:clipping_bounded_support},
\ref{ass:alpha_l2_accuracy},
\ref{ass:prediction_oracle_accuracy}, let \(Q\) be the
\hyperref[def:population_comparison_coupling]{Population Coupling}
defined above.
Let \(\epspsiemp \le \varepsilon_{\mathrm{pred}}\) denote the empirical generator accuracy for prediction oracle, formally defined in Assumption~\ref{ass:psi_empirical_training_W1_accuracy}.  Define $$\mathcal{A} :=\frac{\varepsilon_{\mathrm{OT}}^{-1}}{\sigma^2}
\left[n^{-1/(d+2)}+\epspsiemp\right]+\frac{1+\varepsilon_{\mathrm{OT}}^{-1}}{\sigma^2}
\varepsilon_\alpha\varepsilon_{\mathrm{pred}}.$$ Then, with probability at least
\(1-\mathcal O(n^{-1})-e^{-n\pi_{\min}}\) over the data and Monte Carlo
draws, \(\mathcal E_a\) occurs and the following bound holds:
\begin{enumerate}[leftmargin=0.1in]
    \item For velocity field, Let \(\widehat v_M\) be the exact empirical minimizer over \(\mathcal V_n\), with
\(\varepsilon=n^{-1/(d+5)}\) and
\(R=1+\sigma[2\log(32d^2/\varepsilon^2)]^{1/2}\). \begin{align}
\label{eq:final_velocity_L2_bound_DR_theorem}
\frac1N\sum_{k=0}^{N-1}
\mathbb E\!\left[
\|\widehat v_M(X_{t_k},t_k)-v_Q^\star(X_{t_k},t_k)\|^2
\right]
\lesssim{}& \mathcal{A} +
\frac{1+2\pi_{\min}^{-1}}{\sigma^{d+1}}
n^{-2/(d+5)}.
\end{align}
\item For score field, let \(\widehat s_M\) be the exact empirical minimizer over
\(\mathcal S_n\), with \(\varepsilon=n^{-1/(d+5)}\) and
\(R=1+\sqrt{2}\sigma
[\log(8d^2/(\sigma^4\varepsilon^2))]^{1/2}\). \begin{align}
\label{eq:final_score_L2_bound_DR_theorem}
\frac1N\sum_{k=0}^{N-1}
\mathbb E\!\left[
\|\widehat s_M(X_{t_k},t_k)-s_Q^\star(X_{t_k},t_k)\|^2
\right]
\lesssim{}&
\sigma^{-4} \mathcal{A}+ \frac{1+2\pi_{\min}^{-1}}{\sigma^{2d+6}}
n^{-2/(d+5)}.
\end{align}
\end{enumerate}

\end{proposition}

In Proposition~\ref{thm:final_velocity_L2_bound_DR}, the finite-\(M\) and finite-\(M_Q\)
  errors are dominated by the displayed \(n\)-dependent terms, up to
  logarithmic factors.

\noindent\textbf{Verification of Assumption~\ref{Assumption_Integrated_Estimation_Error}.}
Set
\(\widehat v_{t_k}=\widehat v_M(\cdot,t_k)\) and
\(\widehat s_{t_k}=\widehat s_M(\cdot,t_k)\). For the interpolation induced by
\(Q\), \(v_t=v_Q^\star(\cdot,t)\) and \(s_t=s_Q^\star(\cdot,t)\), while the
interpolation variable \(X_{t_k}\) and the exact SDE state \(Z_{t_k}\) both have
density \(q_{t_k}=q_{t_k,Q}\). Hence the expectations in the two field bounds equal the
corresponding expectations in
Assumption~\ref{Assumption_Integrated_Estimation_Error}. Since \(h=1/N\), the
two displayed grid averages are exactly
\(\varepsilon_{\mathrm{vel}}^2\) and
\(\varepsilon_{\mathrm{score}}^2\), respectively.

In each field bound, the three terms respectively capture coupling replacement,
doubly robust nuisance error, and final-field estimation. The first vanishes
for a coupling fixed independently of the prediction-based target, while the
second is small when either nuisance model is accurate. Substituting both
bounds into Theorem~\ref{Theorem_Sampling_Error} controls finite-step KL error
relative to the smoothed target. Corollary
\ref{cor:counterfactual_W1_from_KL} then adds the smoothing bias and converts
this control into a \(W_1\) guarantee for the original target, yielding the
following theorem.

\begin{theorem}[End-to-End Sample Complexity With Negligible Nuisance Error]
\label{thm:end_to_end_sample_complexity}
Fix \(d\geq1\), \(\pi_{\min}>0\), and
\(\varepsilon_{\mathrm{OT}}>0\). Assume the conditions of
Proposition~\ref{thm:final_velocity_L2_bound_DR}. Let
\(\mu_1=\int\mathbb P_{Y\mid C=c,A=a}\,d\mathbb P_C(c)\) and 
let \(\overline D_5>0\) be a deterministic coupling-displacement envelope such that, almost surely on the event where \(Q\) is defined,
\begin{align*}
    \left(\mathbb E_{\Gamma_{\mu_1,Q}}\|Y_1-Y_0\|^5\right)^{1/5}\leq \overline D_5.
\end{align*}
Suppose \(\epspsiemp\) and
\(\varepsilon_\alpha\varepsilon_{\mathrm{pred}}\) are negligible relative
to the explicit \(n\)-dependent terms. For
\(\varepsilon_{\mathrm{target}}\in(0,1)\), choose \(0<\sigma\leq1\) with
\(\sigma=\Theta(\varepsilon_{\mathrm{target}})\), set
\(\tau^2=\sigma\overline D_5/d\), and choose \(N\geq1\) so that
\begin{align*}
    h=1/N\leq
\frac{ c_h}{d\overline D_5^2/(\varepsilon_{\mathrm{target}}^4\sigma^2)
+d\overline D_5/\sigma},
\end{align*}
where \(c_h\leq2\) is a scaling
constant. Then, up to logarithmic factors, it suffices to take
\begin{align*}
n\gtrsim \max\Bigg\{&
\left[
\max\left\{
\frac{2+\varepsilon_{\mathrm{OT}}^{-1}}{\overline{D_5}}
\varepsilon_{\mathrm{target}}^{-7},
\frac{\overline{D_5}}{\varepsilon_{\mathrm{OT}}}
\varepsilon_{\mathrm{target}}^{-9}
\right\}
\right]^{d+2}, \\ &
\left[
(1+2\pi_{\min}^{-1})
\max\left\{
\overline{D_5}^{-1}\varepsilon_{\mathrm{target}}^{-(d+6)},
\overline{D_5}\varepsilon_{\mathrm{target}}^{-(2d+9)}
\right\}
\right]^{\frac{(d+5)}{2}}
\Bigg\}.
\end{align*}
This yields \(W_1(\mu_1,\mathcal L(\widehat Z_N))
\lesssim\varepsilon_{\mathrm{target}}\) with probability at least
\(1-\mathcal O(n^{-1})-e^{-n\pi_{\min}}\).
\end{theorem}
This gives an end-to-end guarantee for counterfactual flow matching.
Under Assumption \ref{ass:clipping_bounded_support}, we may clip the final generated sample
coordinatewise to $[-1,1]$. This post-processing won't increase its $W_1$
distance to the target and gives a sharper bound
in $W_1$; see Corollary~\ref{cor:terminal_projection_W1}.
  Although Theorem~\ref{Theorem_Sampling_Error} has near-linear dimension
  dependence, the end-to-end sample complexity still suffers from the curse of
  dimensionality, as expected under ambient-dimensional nonparametric assumptions
  and consistent with related flow-matching and diffusion results
  \citep{zhou2025an,su2025flow,wibisono2024optimal,cai2025minimax}.

\section{Experiments}

We first isolate the effects of coupling and discretization in
two dimensions, then evaluate finite-step generation on two semi-synthetic
image benchmarks.

\subsection{Angular Mode Reallocation}

We use a two-dimensional mixture of eight Gaussians arranged
on a circle to isolate coupling effects. The observed source places total mass
\(3/4\) on the axis-aligned modes and \(1/4\) on the diagonal modes, while the
counterfactual target reverses these masses. We compare IID and entropic
optimal-transport couplings across regularization levels and use the resulting
exact empirical velocity and score fields as references for the learned
models. Figures~\ref{Fig_vel_Gaussian} and~\ref{Fig_score_Gaussian} show that
smaller \(D_5\) generally corresponds to smaller field-estimation error.
Appendix~\ref{app:angular_experiment_details} gives the complete mixture
construction and experimental protocol.
The remaining panels isolate finite-step sampling error using
exact fields. Figure~\ref{Fig_KL_D5_Gaussian} shows how the \(8\)-step KL error
varies with \(D_5\) when
\(\tau=0.5\sqrt{\sigma D_5/d}\). For fixed \(h\), \(\sigma\), and \(d\),
Theorem~\ref{Theorem_Sampling_Error} predicts the coupling-dependent order
\(\operatorname{KL}(\mathcal L(Z_{t_N})\|\mathcal L(\widehat Z_N))
\lesssim D_5^2+D_5^3\), consistent with the empirical trend.
Figure~\ref{Fig_Gaussian_covariate} gives a complementary mode-mass view: with
\(4\) Euler steps, both the ODE and SDE samplers substantially misallocate
mass, whereas with \(64\) steps both closely match the target and assign little
mass outside the mode regions.

\subsection{Image Benchmarks}
\label{sec:image_experiments}

We evaluate finite-step counterfactual generation on two
semi-synthetic colored-image benchmarks: MNIST digits \(1\) and \(6\), adapted
from \citet{dance2026debiased}, and Fashion-MNIST sneakers and ankle boots. In
both tasks, \(C\sim\operatorname{Unif}(0,1)\) controls image color, \(A\)
determines the class, and treatment assignment follows
\(\mathbb P(A=1\mid C=c)
=\exp\{5(c-\tfrac12)\}/(1+\exp\{5(c-\tfrac12)\})\).
This selection mechanism biases the observed \(A=1\) images toward larger
values of \(C\), whereas \(C\) remains uniform under
\(\operatorname{do}(A=1)\); the clean counterfactual target can therefore be
sampled directly. We vary the SDE noise level \(\tau\), compare with the
corresponding \(\tau=0\) ODE and a pure \(\sigma=0\) ODE baseline, and use
\(50\) Euler steps throughout. We report sliced Wasserstein-2 (SW2) distance to
the clean target. Appendix~\ref{app:image_benchmark_details} gives the complete
data construction, training, and evaluation protocols.

\textbf{Digit ColorMNIST.}
Here \(A=0\) denotes digit \(1\), \(A=1\) denotes digit \(6\),
and the observational sample size is \(n=20{,}000\). We evaluate smoothing
levels \(\sigma\in\{0.01,0.015\}\); smaller \(\sigma\) reduces smoothing bias
and places the smoothed target closer to the clean target.
Figure~\ref{Fig_Mnist_1_6} shows that, for both values of \(\sigma\), some
\(\tau>0\) improves on the corresponding \(\tau=0\) ODE, and the best results
also outperform the pure \(\sigma=0\) baseline. For \(\sigma=0.01\),
performance is best near \(\tau=0.075\). Along the corresponding SDE
trajectory, the color-distribution KS statistic decreases from \(0.266\) at
\(t=0\) to \(0.079\) at \(t=1\).

\textbf{Fashion-MNIST.}\label{sec:shoe_experiment}
Here \(A=0\) denotes sneakers and \(A=1\) denotes ankle boots.
We evaluate \(\sigma\in\{0.005,0.01\}\). As shown in
Figure~\ref{fig:shoe}, the benefit of stochastic sampling depends strongly on
the smoothing level. For \(\sigma=0.005\), performance improves with \(\tau\),
is best near \(\tau=0.02\), and clearly outperforms the pure ODE baseline
before deteriorating for larger \(\tau\). For \(\sigma=0.01\), varying
\(\tau\) yields no clear improvement over the pure ODE baseline
within the reported variation,
possibly because of the larger smoothing error. Along the
\(\sigma=0.005\) trajectory, the red--blue color-statistic KS value decreases
from \(0.260\) at \(t=0\) to \(0.063\) at \(t=1\). Overall,
\(\sigma=0.005\) performs substantially better, consistent with its smaller
smoothing error. Qualitative samples are provided in
Appendix~\ref{app:image_benchmark_details}. \vspace{-0.1in}

\input{experiment_figures.tex}

\section{Conclusion}

We developed a counterfactual flow-matching framework that combines doubly robust field estimation, learned source--target couplings, Gaussian smoothing, and score-corrected stochastic sampling. 
We prove a coupling-sensitive finite-step guarantee in which the Euler sampling error depends explicitly on the source--target displacement under the chosen coupling. Separately, our finite-sample analysis separates coupling replacement, nuisance estimation, and statistical errors. Together, these results yield an end-to-end guarantee relative to the original counterfactual target. 
Synthetic and image experiments show that stochastic sampling can improve over the corresponding ODE at fixed step budgets, with performance depending on smoothing and SDE
noise.

%% file: experiment_figures.tex
\begingroup
\setlength{\abovecaptionskip}{2pt}
\setlength{\belowcaptionskip}{0pt}
\subfigtopskip=0pt
\subfigbottomskip=0pt
\subfigcapskip=0pt
\subfigcaptopadj=0pt

\newcommand{\histgridpanel}[5][0 0 0 0]{%
    \begin{minipage}[b]{0.49\linewidth}
    \centering
    \subfigure[#2]{%
        \includegraphics[width=\linewidth,height=#3,keepaspectratio,trim=#1,clip]{#4}%
        \label{#5}%
    }
    \end{minipage}%
}

\begin{figure}[!t]
\centering
\setcounter{subfigure}{0}
\subfigure[$\varepsilon_{\mathrm{vel}}^{2}$]{
    \includegraphics[width=0.47\linewidth,trim=0 6bp 0 6bp,clip]{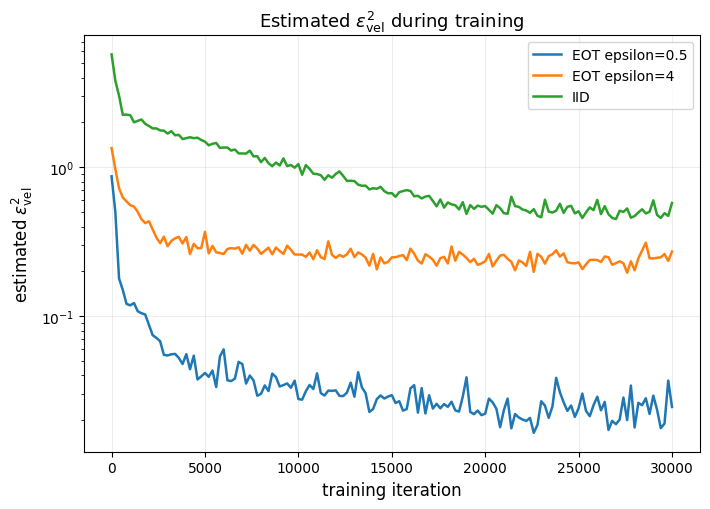}
    \label{Fig_vel_Gaussian}
}\hfill
\subfigure[$\varepsilon_{\mathrm{score}}^{2}$]{
    \includegraphics[width=0.47\linewidth,trim=0 6bp 0 6bp,clip]{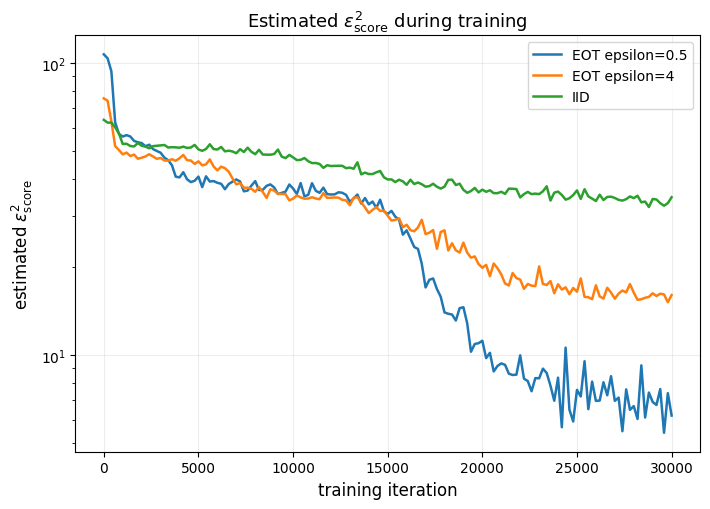}
    \label{Fig_score_Gaussian}
}
\par
\subfigure[KL vs.\ $D_5$]{
    \includegraphics[width=0.47\linewidth,trim=0 6bp 0 6bp,clip]{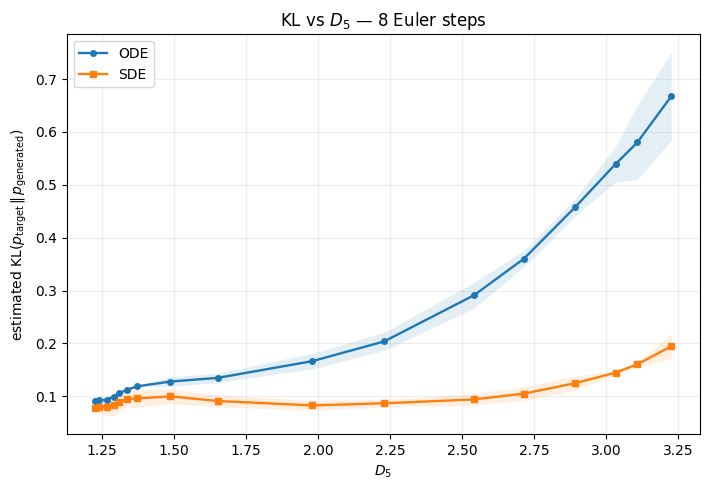}
    \label{Fig_KL_D5_Gaussian}
}\hfill
\subfigure[Mode mass]{
    \includegraphics[width=0.47\linewidth,trim=0 4bp 0 4bp,clip]{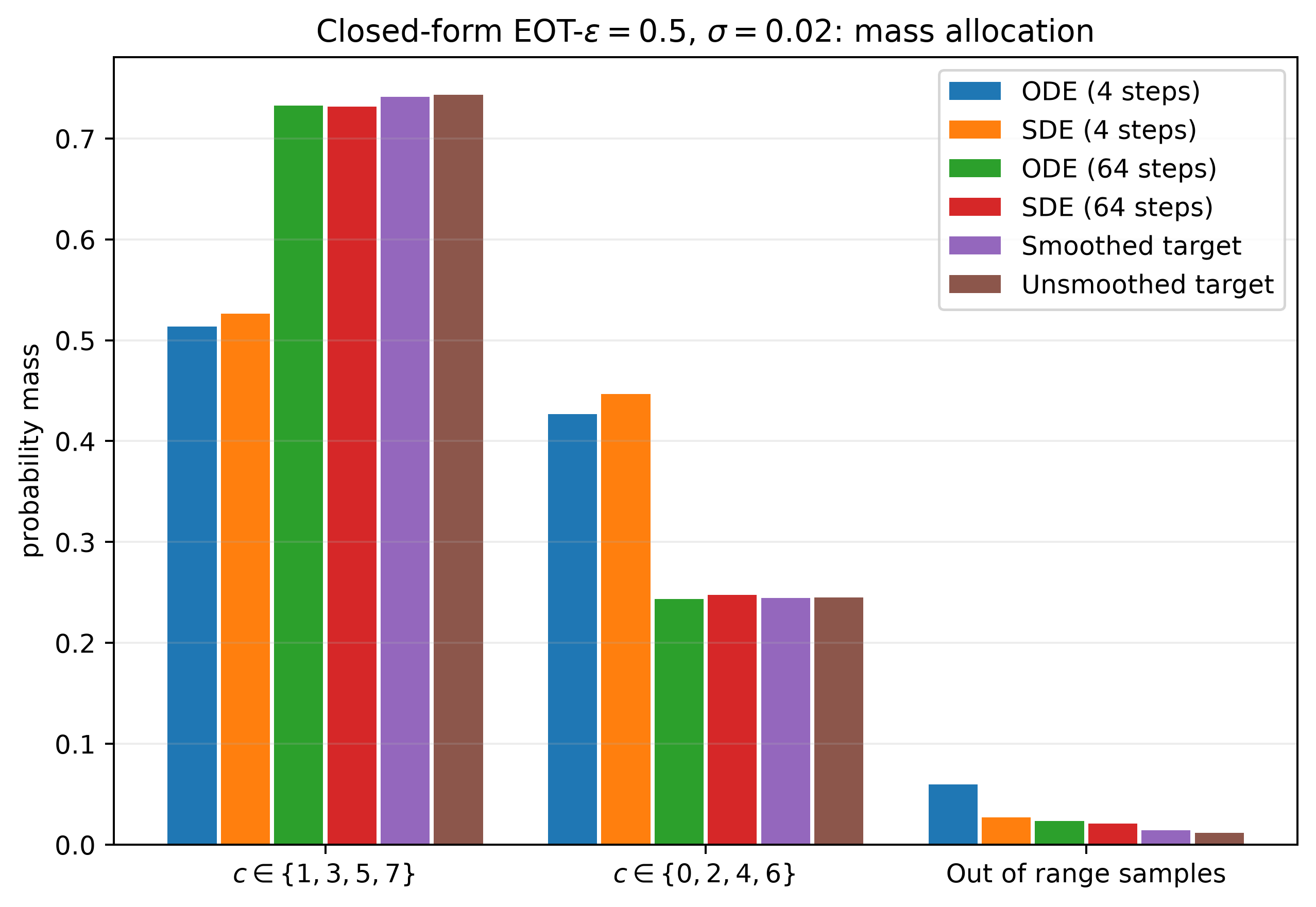}
    \label{Fig_Gaussian_covariate}
}
\caption{\footnotesize
Angular-mode experiment. (a)--(b) Velocity and score errors across couplings
($\sigma=0.12$). (c) Exact-field $8$-step KL versus $D_5$, using
$2{,}000$ samples and $\tau=0.5\sqrt{\sigma D_5/d}$.
(d) Mode-mass allocation for $4$ and $64$ steps
(EOT $\varepsilon_{\mathrm{OT}}=0.5$, $\sigma=0.02$):
$4$-step ODE/SDE misallocate mass, while $64$-step ODE/SDE match the target modes.}
\label{Fig_eps_Gaussian}
\end{figure}

\begin{figure}[!t]
\centering
\setcounter{subfigure}{0}

\histgridpanel[0 6bp 0 6bp]{SW2}{4.4cm}
    {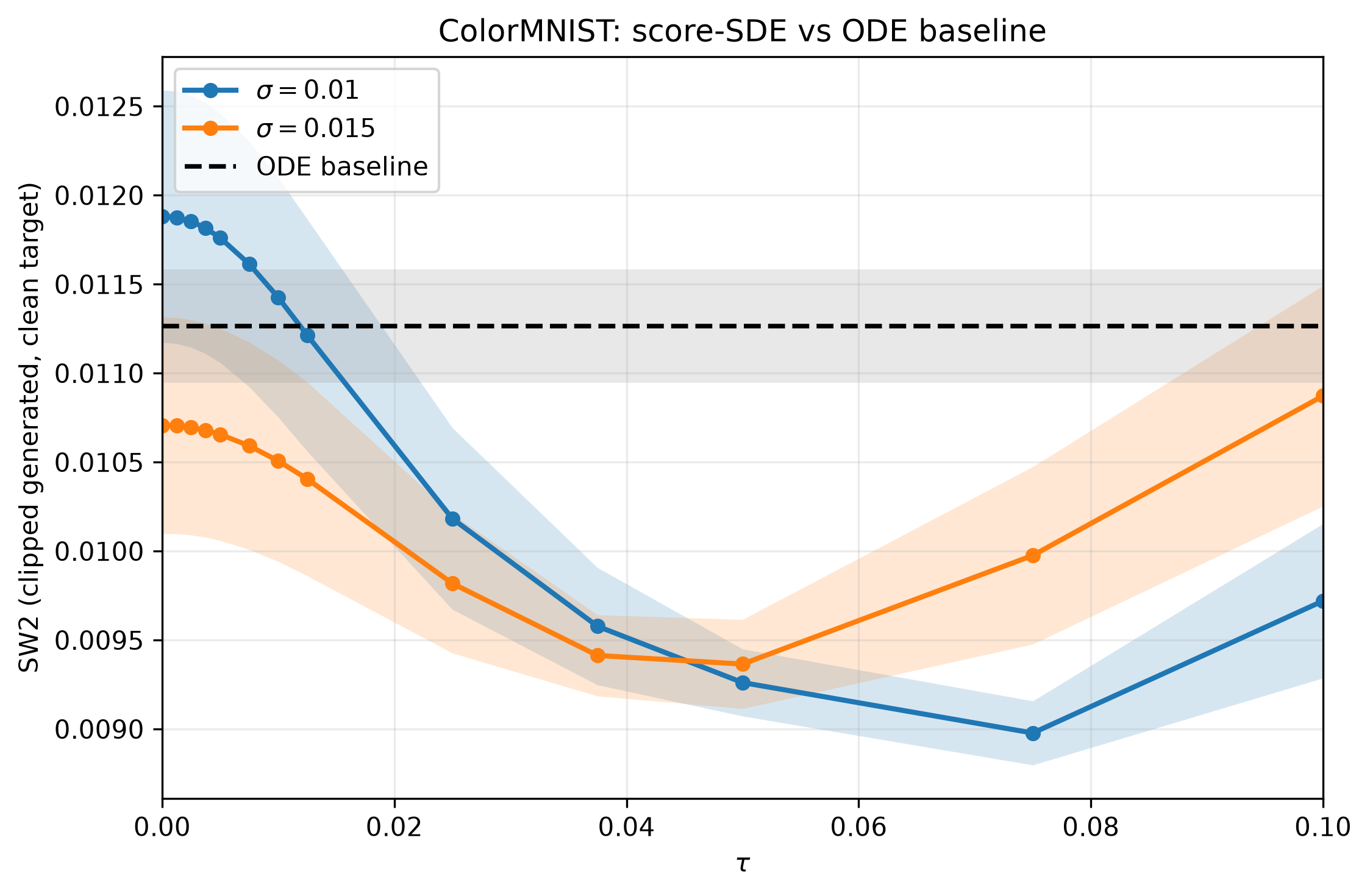}
    {Fig_Mnist_Eval}
\hfill
\histgridpanel[0 10bp 0 10bp]{$t=0$}{4.5cm}
    {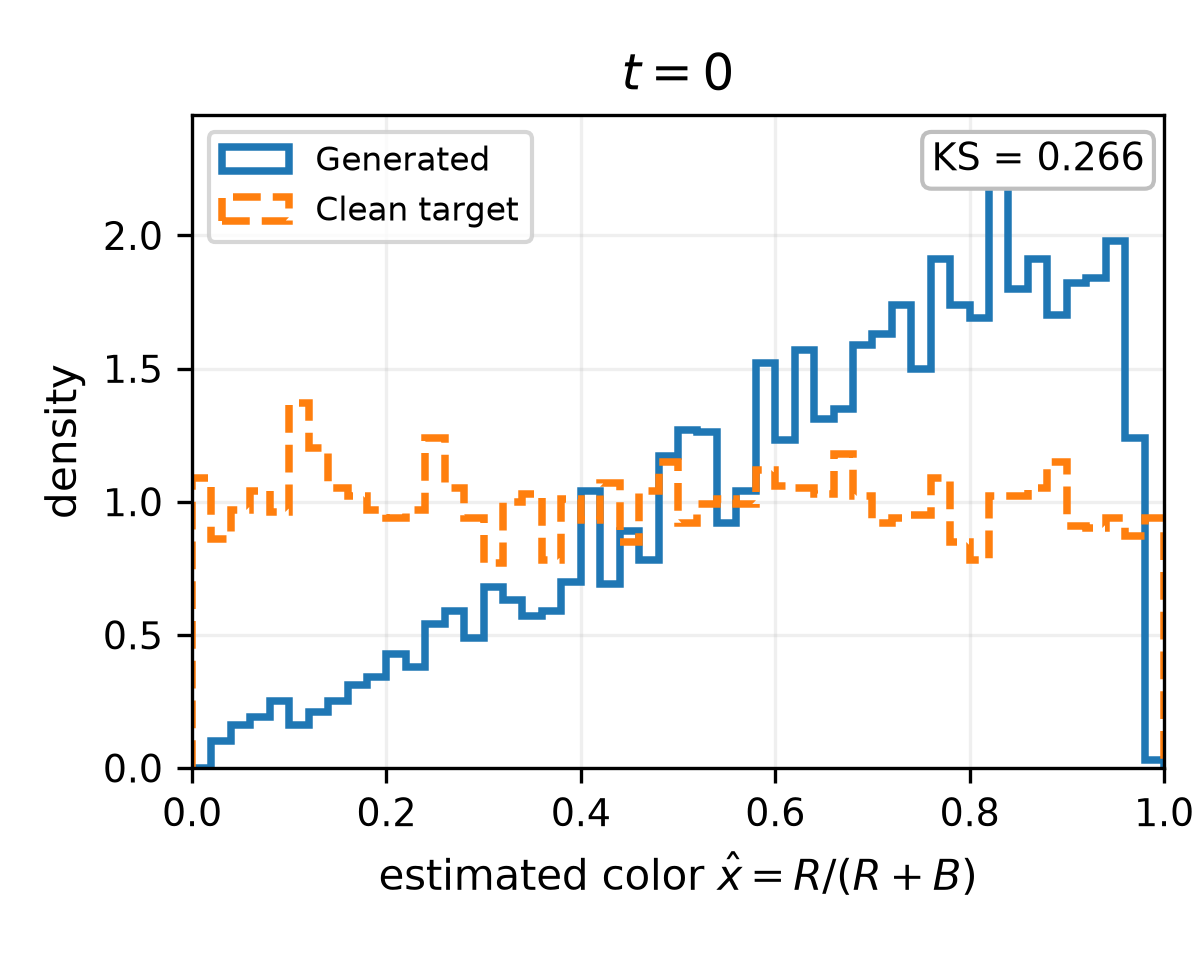}
    {Fig_Mnist_hist}
\par
\histgridpanel[0 10bp 0 10bp]{$t=0.5$}{4.5cm}
    {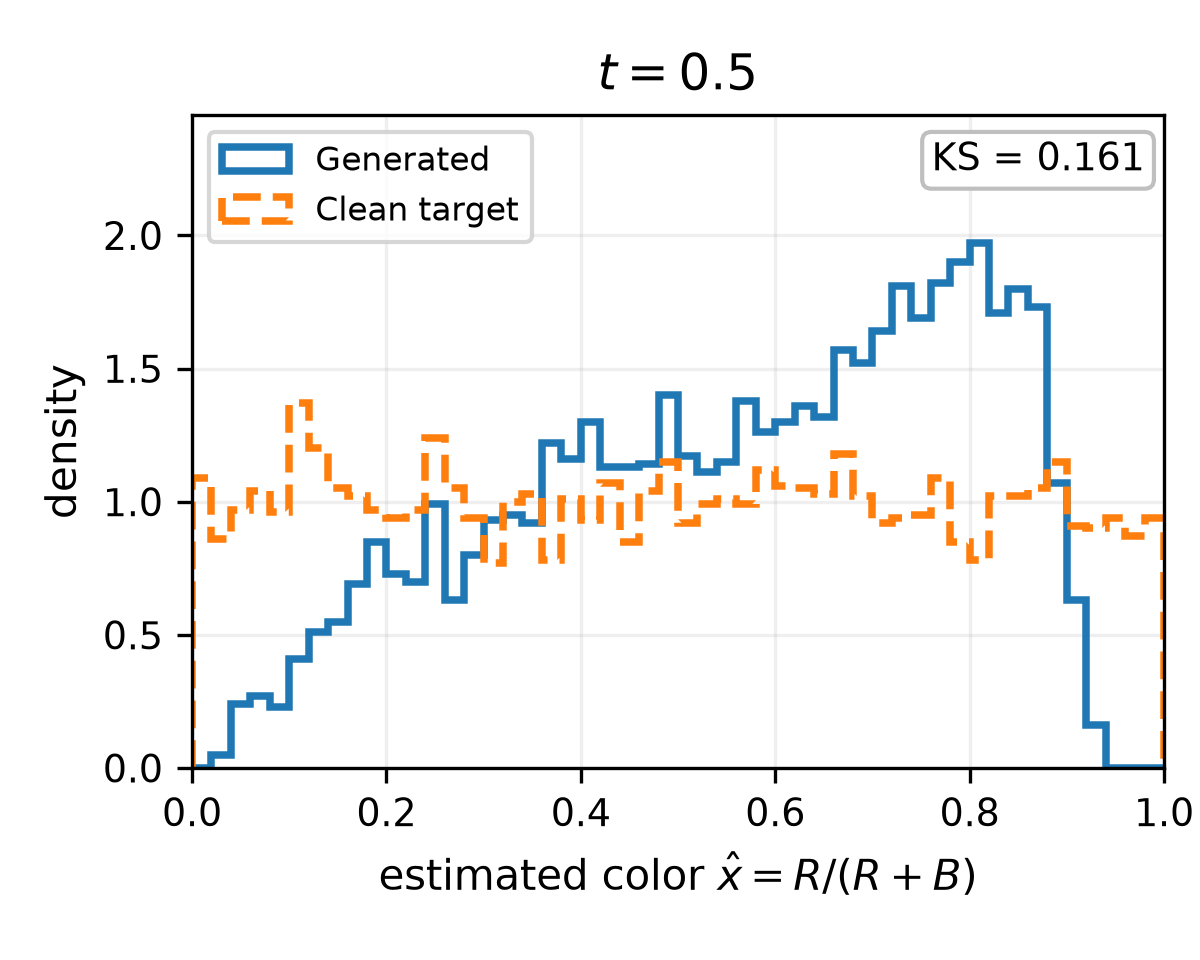}
    {Fig_Mnist_hist_t0p5}
\hfill
\histgridpanel[0 10bp 0 10bp]{$t=1$}{4.5cm}
    {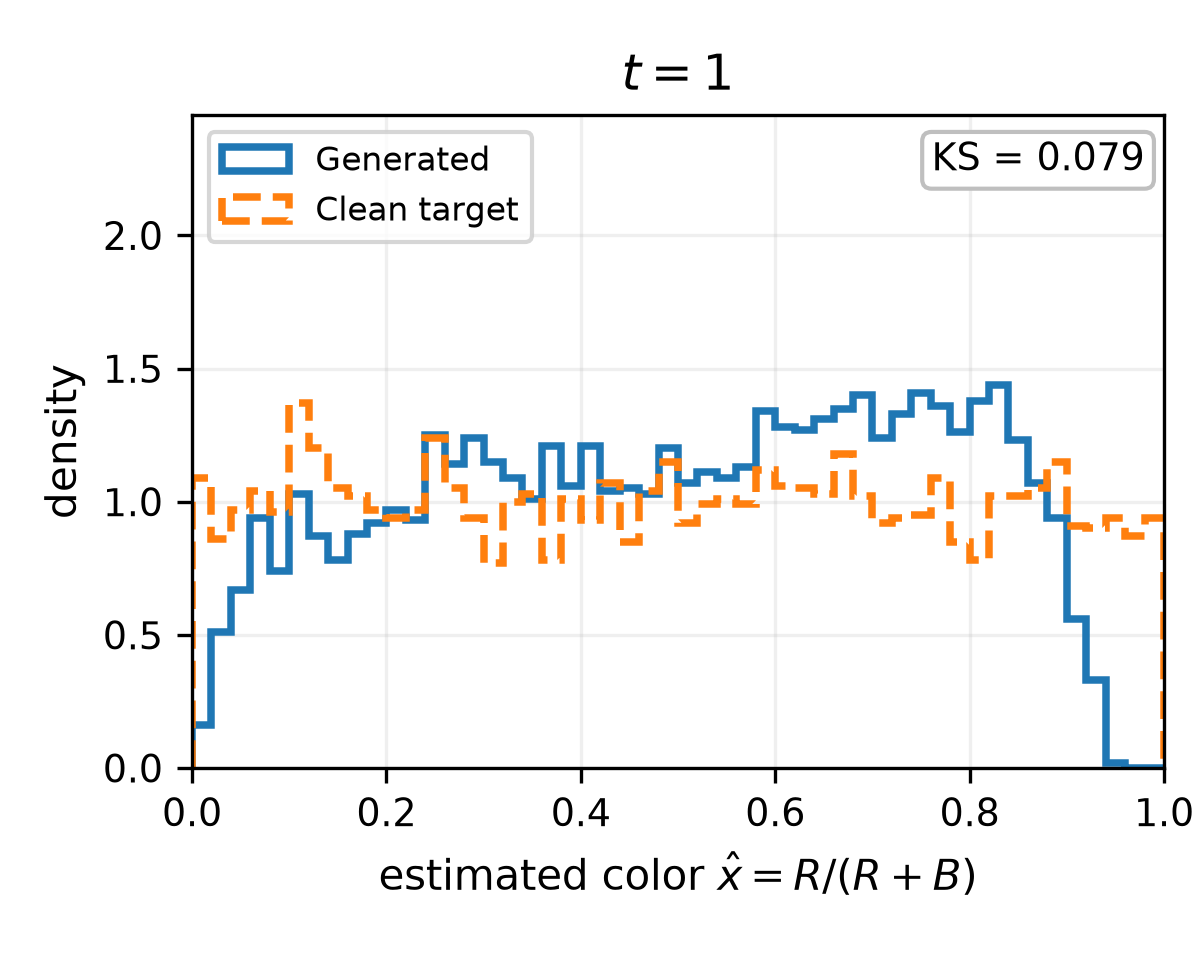}
    {Fig_Mnist_hist_t1}

\caption{\footnotesize
ColorMNIST 1--6. (a) Clean-target SW2 for $\sigma\in\{0.01,0.015\}$
($5{,}000$ samples, five seeds; dashed: pure ODE).
(b)--(d) Color evolution along the $\sigma=0.01$ SDE at $t=0,0.5,1$
($5{,}000$ samples; dashed: clean target).}
\label{Fig_Mnist_1_6}

\setcounter{subfigure}{0}

\histgridpanel[0 2bp 0 2bp]{SW2}{4.4cm}
    {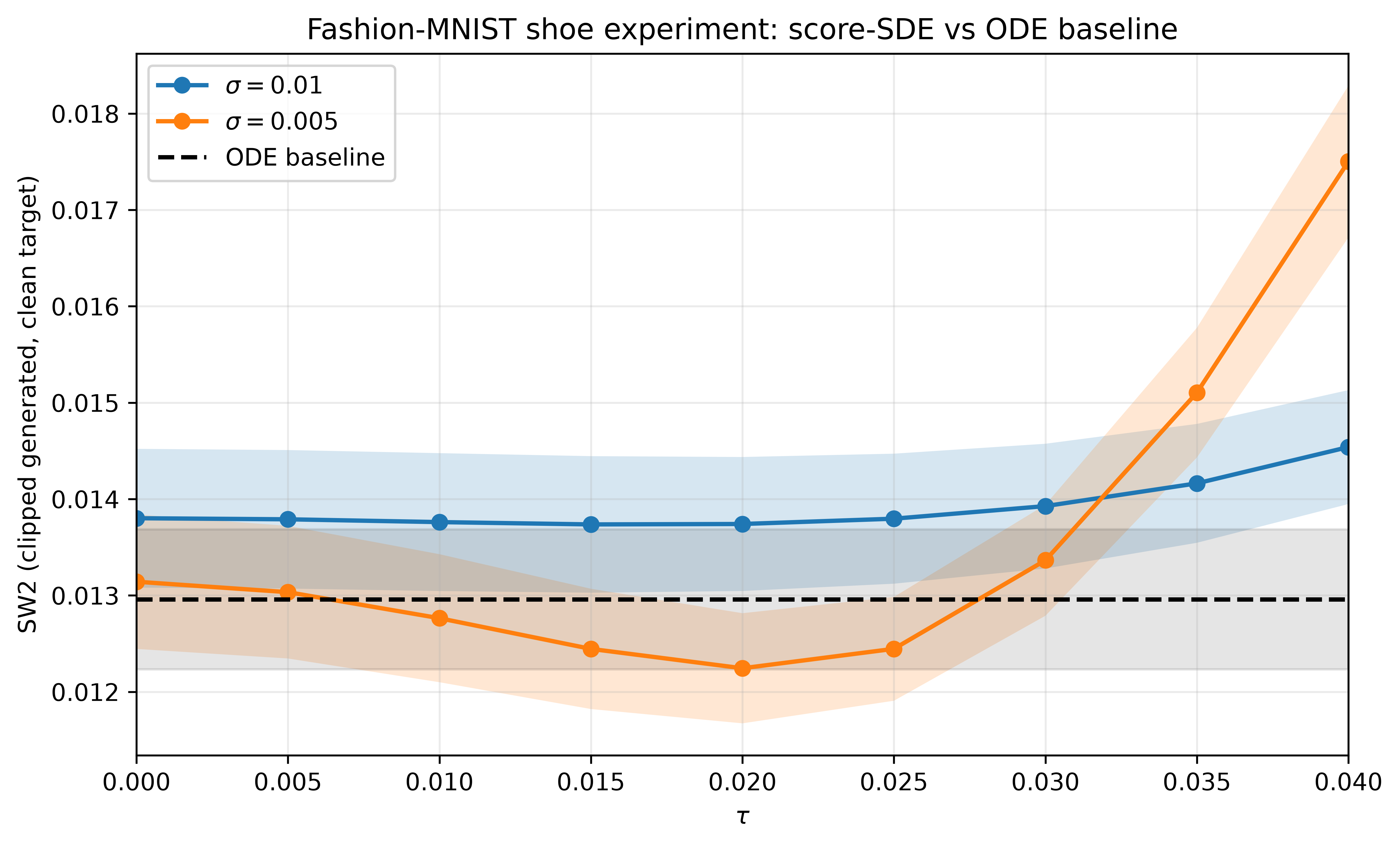}
    {fig:shoe_tau}
\hfill
\histgridpanel[0 10bp 0 10bp]{$t=0$}{4.5cm}
    {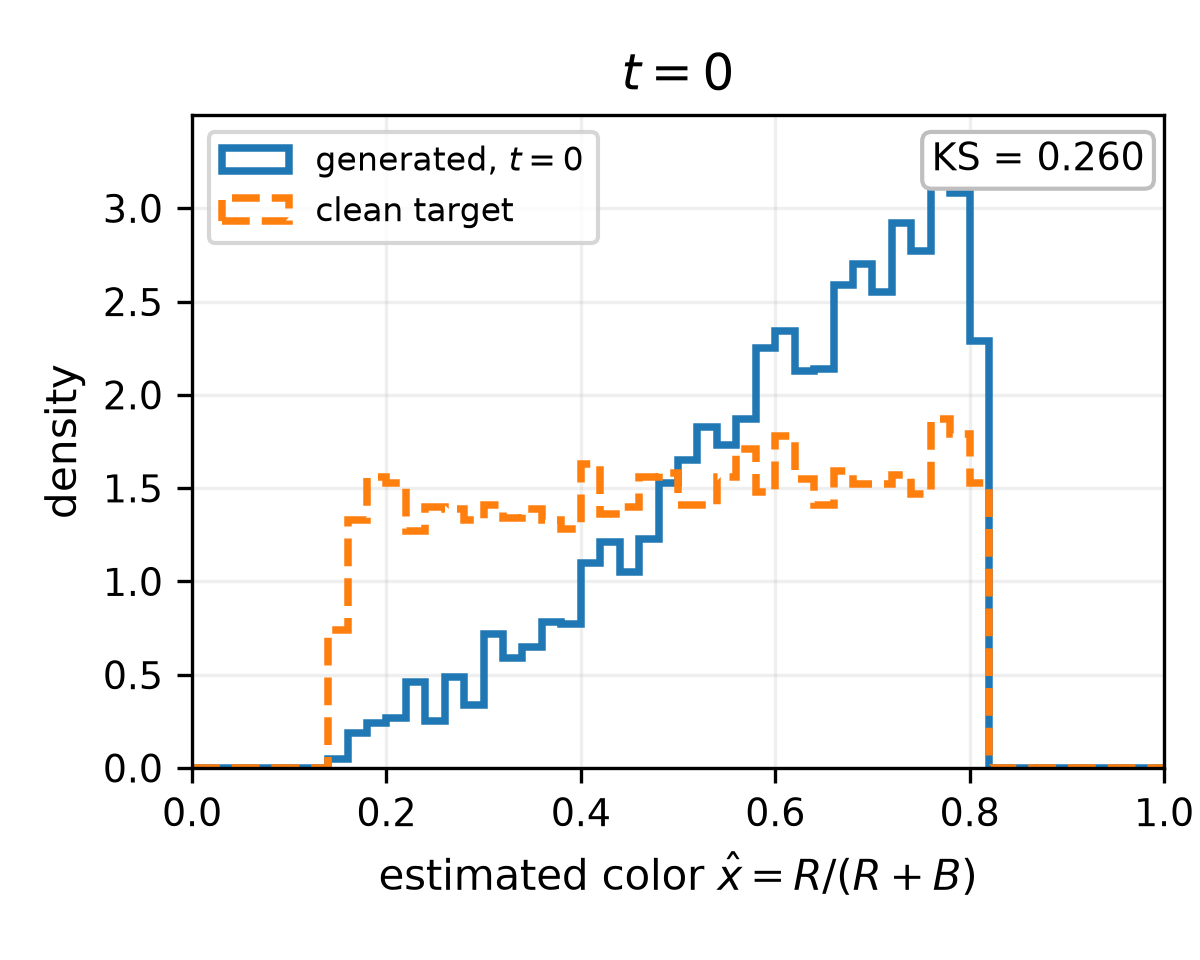}
    {fig:shoe_trajectory}
\par
\histgridpanel[0 10bp 0 10bp]{$t=0.5$}{4.5cm}
    {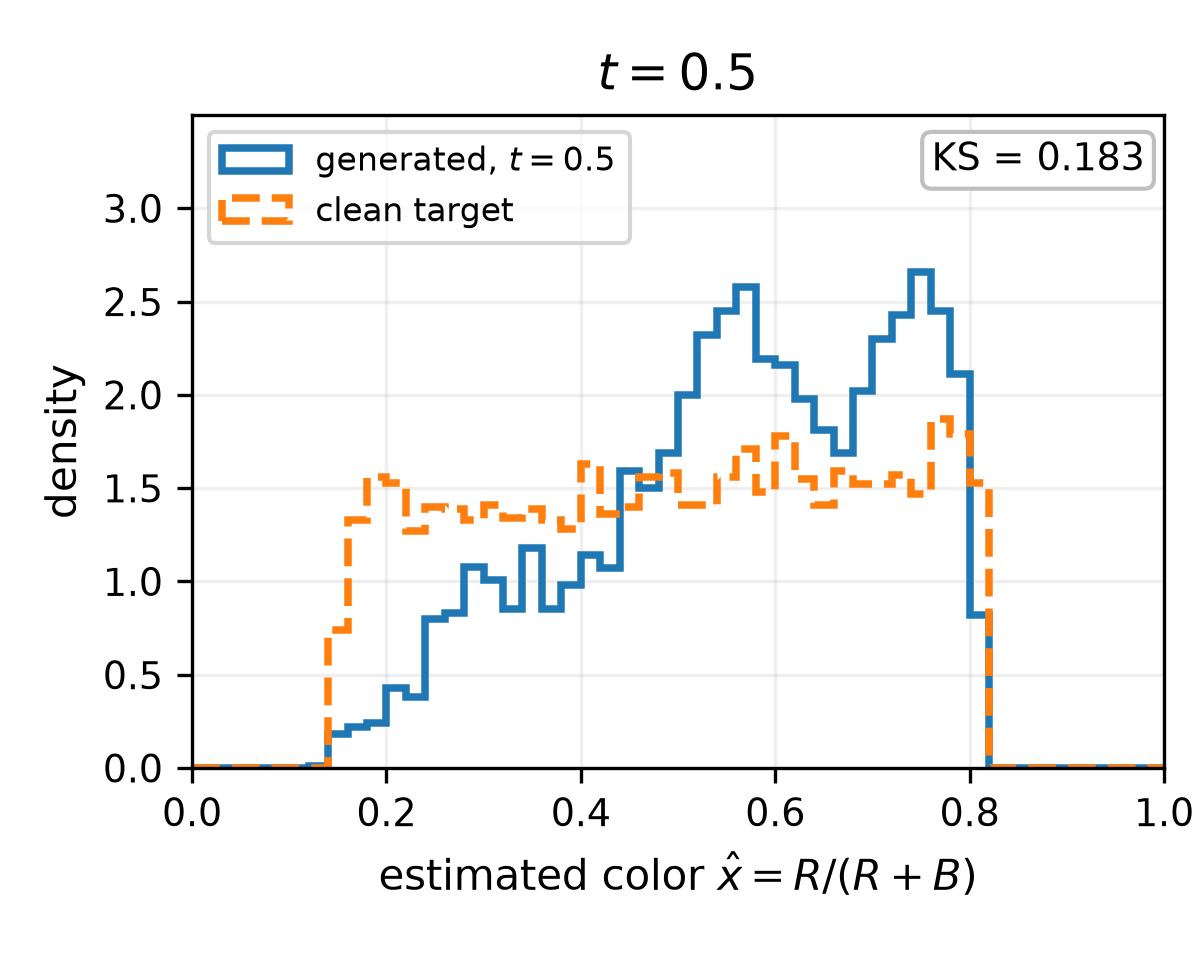}
    {fig:shoe_trajectory_t0p5}
\hfill
\histgridpanel[0 10bp 0 10bp]{$t=1$}{4.5cm}
    {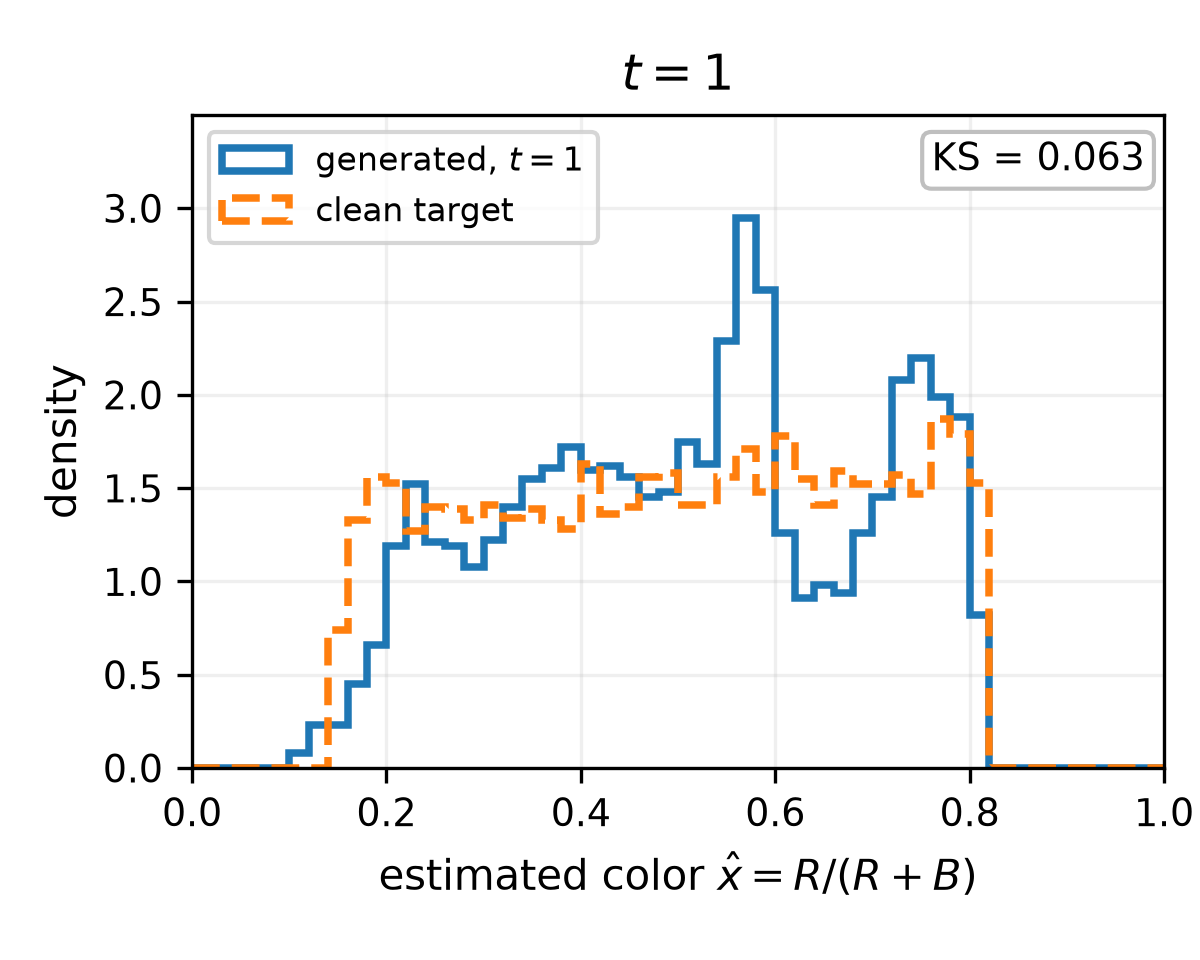}
    {fig:shoe_trajectory_t1}

\caption{\footnotesize
Fashion-MNIST shoes. (a) Clean-target SW2 versus $\tau$ for
$\sigma\in\{0.005,0.01\}$ ($5{,}000$ samples, five seeds; dashed: pure ODE).
(b)--(d) Color evolution along the $\sigma=0.005$ SDE at $t=0,0.5,1$
($5{,}000$ samples; dashed: clean target).}
\label{fig:shoe}
\end{figure}

\endgroup

%% file: Appendix_Experiments.tex
\section{Additional Experiment Details}\label{app:angular_experiment_details}

\subsection{Angular-Mode Experiment Details}

\paragraph{Data-Generating Process.}
{
Let \(C\in\{0,\ldots,7\}\) index eight Gaussian components equally spaced on a
circle, with
\[
\mu_c=2(\cos(c\pi/4),\sin(c\pi/4))^\top,\qquad
P_c=\mathcal N(\mu_c,\sigma_Y^2I_2),\qquad
\sigma_Y=0.08.
\]
The population assigns weight \(1/16\) to each axis-aligned component
\(c\in\{0,2,4,6\}\) and \(3/16\) to each diagonal component
\(c\in\{1,3,5,7\}\). Treatment assignment satisfies
\(\mathbb P(A=1\mid C=c)=0.9\) for axis-aligned components and \(0.1\) for
diagonal components. Bayes' rule therefore gives conditional weights \(3/16\)
and \(1/16\), respectively, among observations with \(A=1\).

Taking \(\mathbb P_{Y\mid C=c,A=1}=P_c\), the observed source and
counterfactual target are
\[
\begin{aligned}
\mathbb P_{Y\mid A=1}
&=\sum_{c\in\{0,2,4,6\}}\frac{3}{16}P_c
+\sum_{c\in\{1,3,5,7\}}\frac{1}{16}P_c,\\
\int\mathbb P_{Y\mid C=c,A=1}\,d\mathbb P_C(c)
&=\sum_{c\in\{0,2,4,6\}}\frac{1}{16}P_c
+\sum_{c\in\{1,3,5,7\}}\frac{3}{16}P_c.
\end{aligned}
\]
Thus, the two distributions share the same components but assign opposite
weights to the axis-aligned and diagonal groups.
}

\paragraph{Couplings and Evaluation.}
{
For fixed empirical source and target samples, we compare IID and entropic
optimal-transport couplings at different regularization levels. Each coupling
assigns mass to source--target pairs; after Gaussian smoothing, these weighted
pairs determine the interpolation law and the exactly computable empirical
fields
\[
v_t(x)=\mathbb E[Y_1-Y_0\mid X_t=x],
\qquad
s_t(x)=\nabla_x\log q_t(x).
\]
These exact fields serve as references for the learned models. The velocity
and score experiments use three-hidden-layer MLPs of width \(128\) and
\(\sigma=0.12\).

To isolate discretization error, we generate \(2{,}000\) exact-field samples
with \(8\) Euler steps and estimate the generated density by kernel density
estimation to compute KL divergence. We set
\(\tau=0.5\sqrt{\sigma D_5/d}\), following the scaling
\(\tau\asymp\sqrt{\sigma D_5/d}\) suggested by
Theorem~\ref{Theorem_Sampling_Error}. The mode-mass experiment uses the
closed-form EOT coupling with \(\varepsilon_{\mathrm{OT}}=0.5\) and
\(\sigma=0.02\), and compares \(4\)- and \(64\)-step ODE and SDE samplers.
}

\subsection{Image-Benchmark Details and Additional Results}
\label{app:image_benchmark_details}

\paragraph{Digit ColorMNIST.}
{ 
Following \citet{dance2026debiased}, we draw
\(C\sim\operatorname{Unif}(0,1)\) and color each MNIST image using RGB value
\((C,0,1-C)\), with the original image intensity controlling brightness.
Treatment \(A=0\) selects digit \(1\), \(A=1\) selects digit \(6\), and
\[
\mathbb P(A=1\mid C=c)
=\frac{\exp\{5(c-\tfrac12)\}}{1+\exp\{5(c-\tfrac12)\}}.
\]
Conditional on \(A\), an image is sampled with replacement from the
corresponding MNIST class and colored according to \(C\). We generate
\(n=20{,}000\) observational samples and target \(a=1\). Because \(C\) remains
uniform under \(\operatorname{do}(A=1)\), the clean counterfactual target can
be sampled directly.
}

{ 
We train smoothed models with \(\sigma\in\{0.01,0.015\}\) for \(1{,}000\)
epochs and evaluate the SDE over a range of \(\tau\), with \(\tau=0\)
corresponding to the ODE for the same smoothed model. The pure ODE baseline is
the same doubly robust flow trained with \(\sigma=0\); its best checkpoint
through epoch \(1{,}000\) occurs at epoch \(700\). All samplers use \(50\)
Euler steps. Generated images are clipped to \([0,1]\), and SW2 is computed
against the clean target using \(5{,}000\) samples over five random seeds. The
baseline and smoothed models use the same fitted propensity model,
conditional-flow prediction oracle, and conditional-sample reservoir. For
\(\sigma=0.01\), the epoch-\(1000\) model performs best near
\(\tau=0.075\).
}

\paragraph{Fashion-MNIST.}
{ 
We use \(C\sim\operatorname{Unif}(0,1)\), with \(A=0\) for
sneakers and \(A=1\) for ankle boots, and use the same treatment-assignment
probability as in the digit experiment. Given \(C=c\), the grayscale image is
colored using
\[
(0.15+0.75c,\;0.10+0.55\sin^2(\pi c),\;0.85-0.65c).
\]
Consequently, observed ankle boots tend to have larger \(C\), whereas the
clean \(\operatorname{do}(A=1)\) target contains the full population color
range.
}

{ 
We train models with \(\sigma=0.01\) and \(\sigma=0.005\), together with the
pure ODE baseline \((\sigma=0)\), for \(400\) epochs. Checkpoints are saved
every \(20\) epochs. For \(\sigma=0.01\), the selected velocity and score
checkpoints are epochs \(100\) and \(220\); for \(\sigma=0.005\), they are
epochs \(100\) and \(120\). The selected pure ODE checkpoint is epoch \(140\).
We vary \(\tau\), and every evaluation uses \(50\) Euler steps, \(5{,}000\)
generated samples, \(256\) random projection directions, and five evaluation
seeds. The color-distribution plots use the red--blue statistic
\(\hat{x}=R/(R+B)\).
}

Figure~\ref{fig:shoe_samples} shows images before and after
transport. The source contains observed ankle boots with color-selection bias;
the generated samples remain realistic ankle boots while covering colors
closer to the clean target population.

Figures~\ref{fig:shirt} and~\ref{fig:shirt_samples} provide
complementary results for the shirt experiment: the former reports sampling
performance and color evolution, while the latter shows that the generated
samples preserve shirt structure while changing their colors toward the
target distribution.
\begingroup
\setlength{\abovecaptionskip}{2pt}
\setlength{\belowcaptionskip}{0pt}
\subfigtopskip=0pt
\subfigbottomskip=0pt
\subfigcapskip=0pt
\subfigcaptopadj=0pt

\newcommand{\shirtgridpanel}[5][0 0 0 0]{%
    \begin{minipage}[b]{0.49\linewidth}
    \centering
    \subfigure[#2]{%
        \includegraphics[width=\linewidth,height=#3,keepaspectratio,trim=#1,clip]{#4}%
        \label{#5}%
    }
    \end{minipage}%
}

\begin{figure}[!htbp]
\centering
\setcounter{subfigure}{0}

\shirtgridpanel{Performance evaluation}{4.4cm}
    {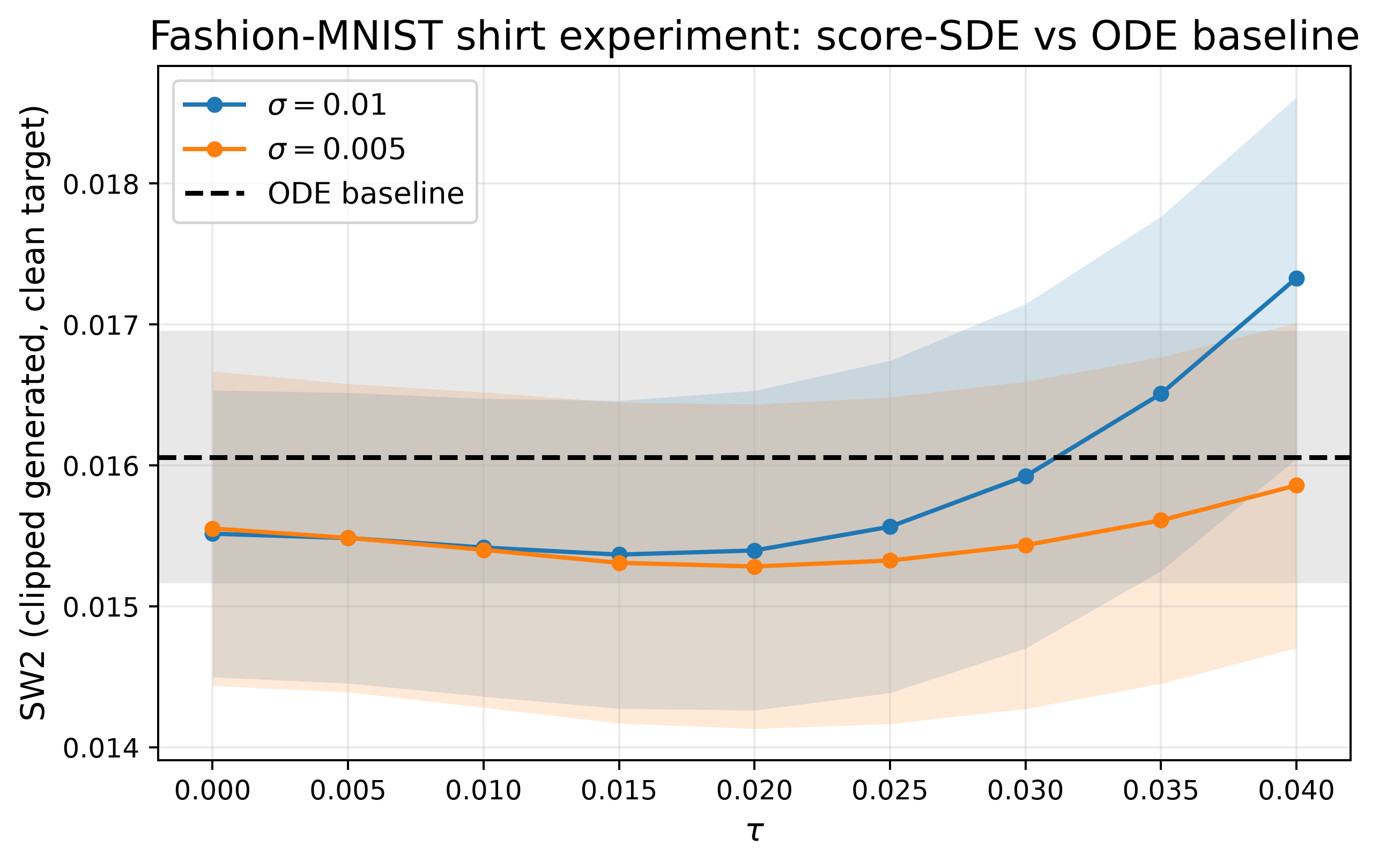}
    {fig:shirt_tau}
\hfill
\shirtgridpanel[0 10bp 0 10bp]{$t=0$}{4.5cm}
    {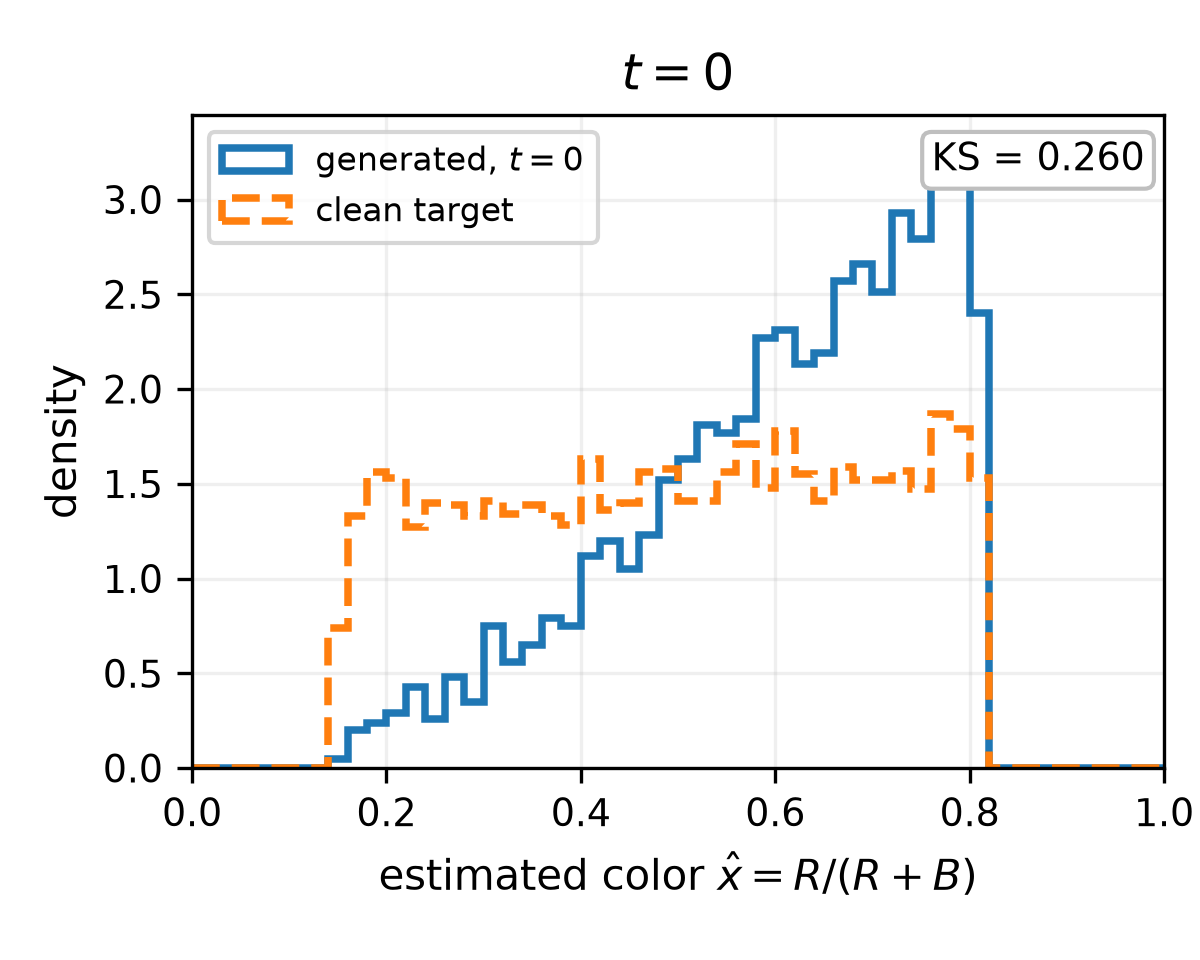}
    {fig:shirt_trajectory}
\par
\shirtgridpanel[0 10bp 0 10bp]{$t=0.5$}{4.5cm}
    {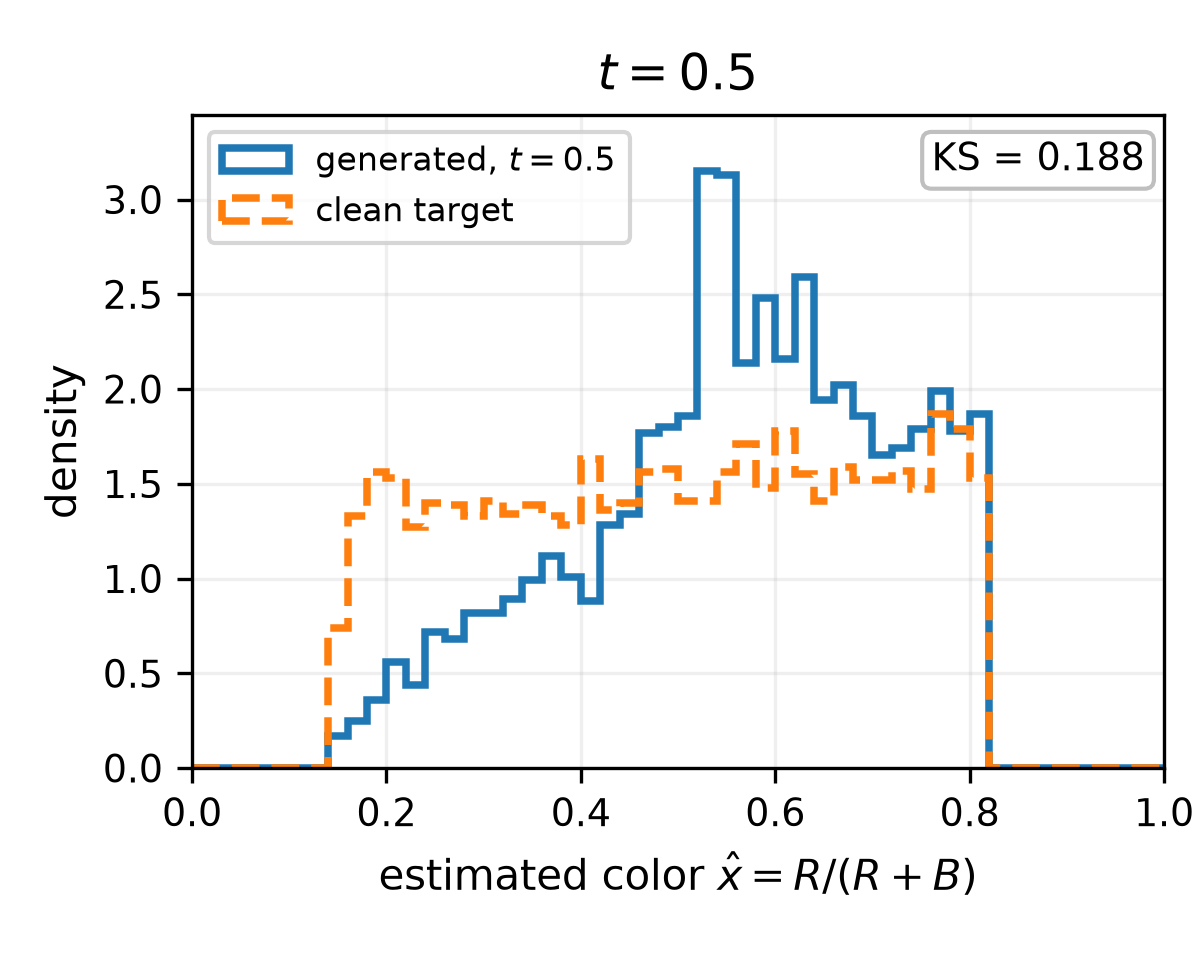}
    {fig:shirt_trajectory_t0p5}
\hfill
\shirtgridpanel[0 10bp 0 10bp]{$t=1$}{4.5cm}
    {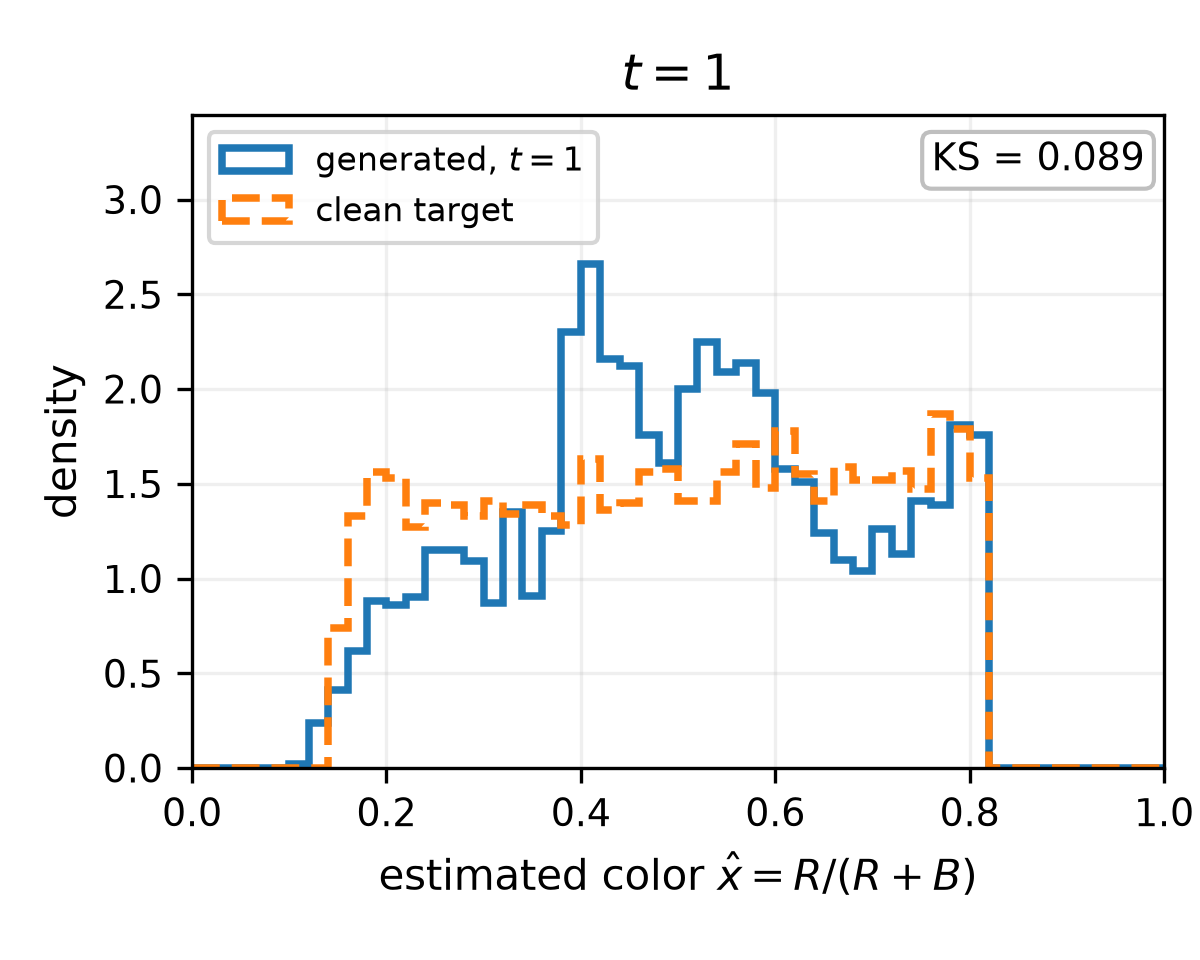}
    {fig:shirt_trajectory_t1}

\caption{\footnotesize
Fashion-MNIST shirts experiment.
(a) Clean-target SW2 as a function of $\tau$ for the two smoothing levels;
shaded regions show variation across five evaluation seeds, and the dashed
line denotes the pure ODE baseline.
(b)--(d) Color distributions at $t=0,0.5,1$ along the $\sigma=0.005$ SDE
trajectory; dashed distributions denote the clean target.}
\label{fig:shirt}
\end{figure}

\endgroup

\begin{figure}[!ht]
    \centering
    \includegraphics[width=0.95\linewidth]{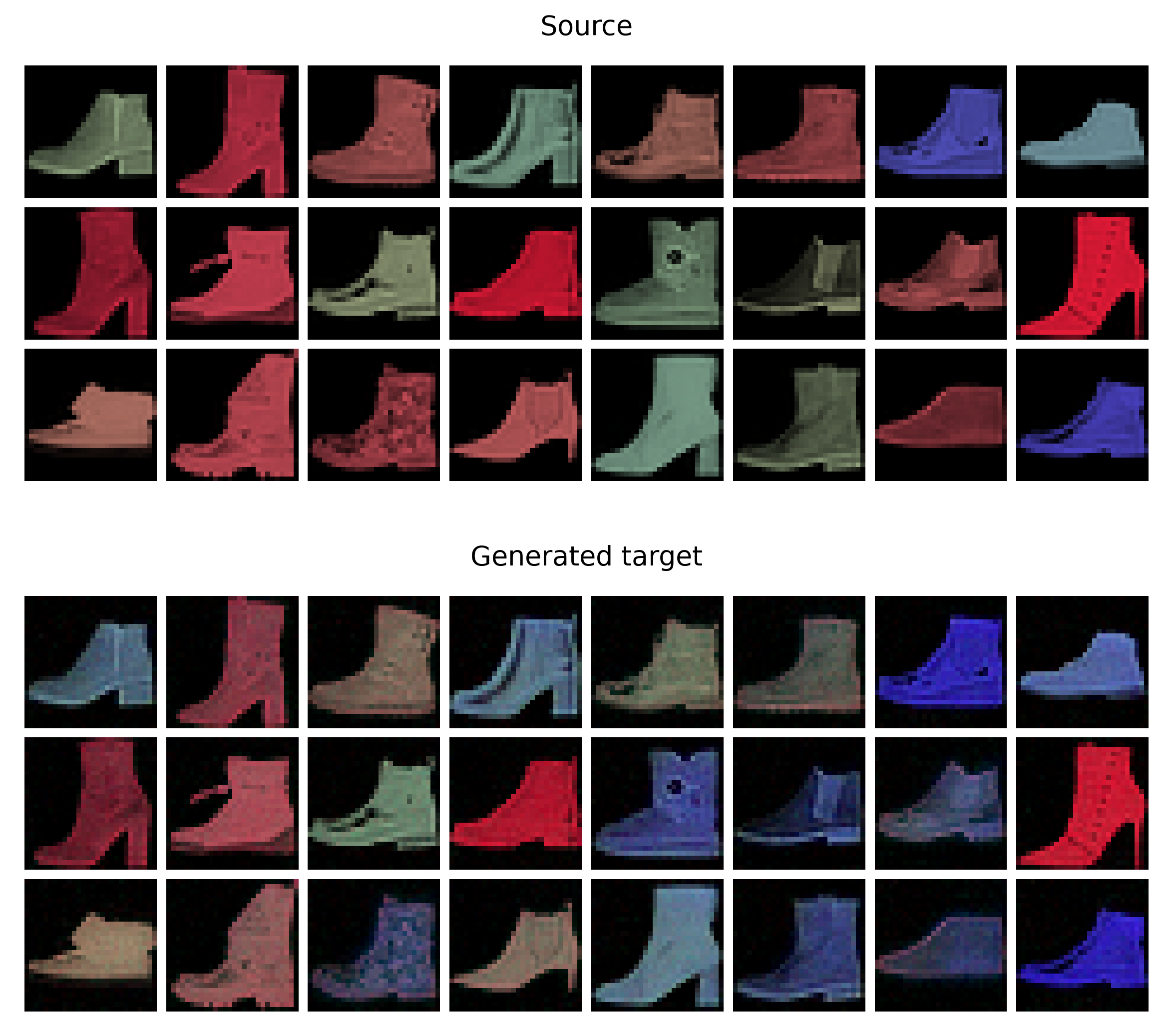}
    \caption{
    Qualitative examples from the observed source distribution (top),
    samples generated by the $\sigma=0.005$ SDE (bottom). The generated samples preserve the boots shape and structure while showing
clear, coherent color changes.
    }
    \label{fig:shoe_samples}
\end{figure}

\begin{figure}[!ht]
    \centering
    \includegraphics[width=0.95\linewidth]{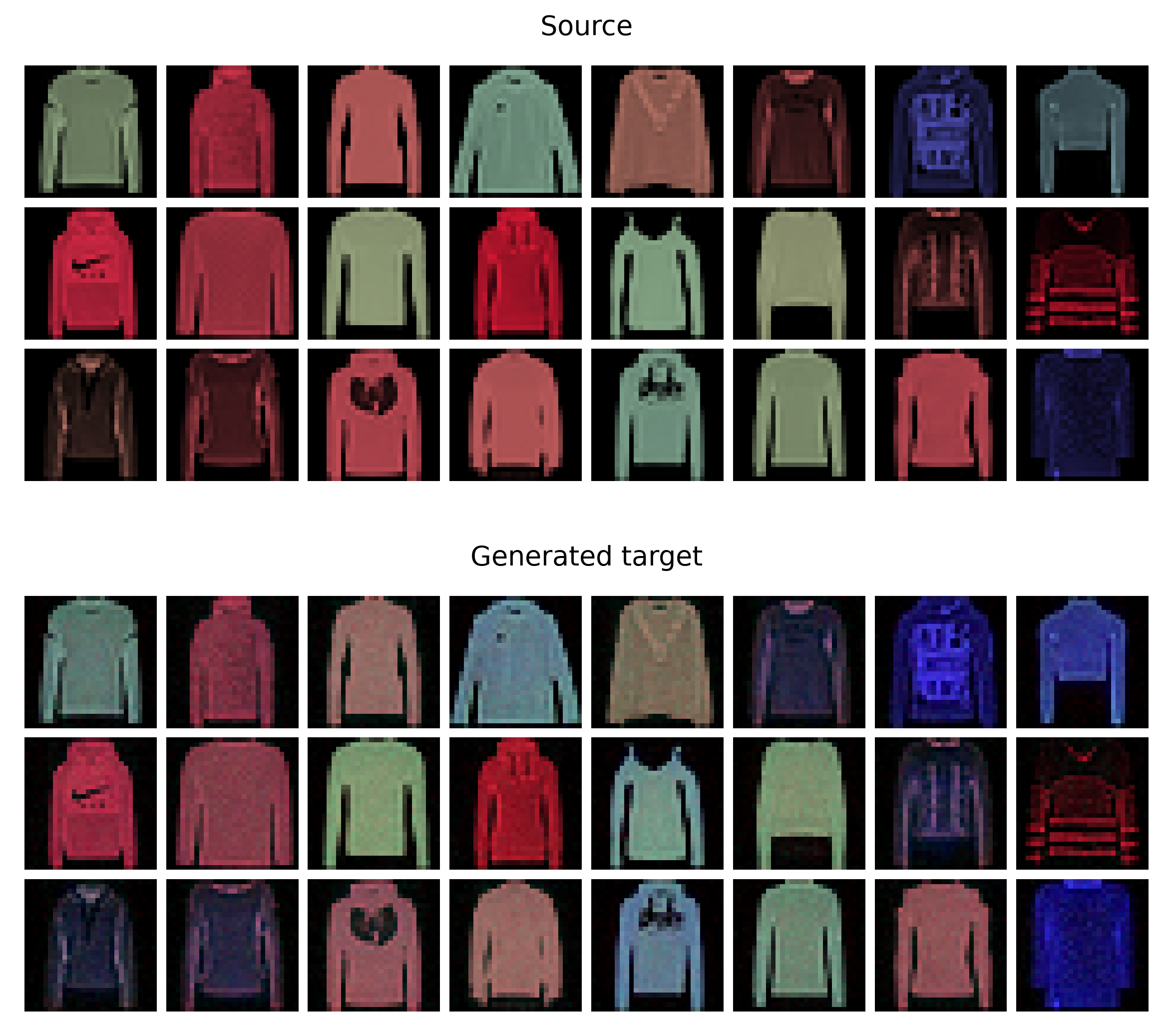}

    \caption{
    Qualitative examples from the observed source distribution (top),
    samples generated by the $\sigma=0.005$ SDE (bottom). The generated samples preserve the shirt shape and structure while showing clear, coherent color changes.
    }
    \label{fig:shirt_samples}
\end{figure}

%% file: Appendix_Additional.tex
\section{Additional Related-Work Details}
\label{sect_existing_work_appendix}

The closest comparisons are
\citep{liu2025finitesde,liu2026finitetime}. For stochastic sampling,
\citep{liu2025finitesde} considers a more general stochastic interpolant with
time-varying smoothing \(\gamma_t\). Under a uniform Euler discretization,
their complexity has leading dimension dependence
\(\widetilde{\mathcal O}(d^2)\) and polynomial dependence on the endpoint
cutoffs \(t_0\) and \(1-t_N\). These cutoff parameters are introduced to
avoid the singular behavior near the two endpoints and play a role analogous
to the fixed smoothing level \(\sigma\) in our analysis. In comparison, our
uniform-step result achieves near-linear \(\widetilde{\mathcal O}(d)\)
dimension dependence with leading smoothing dependence of order
\(1/\sigma^2\). This comparison holds with displacement moment and target
accuracy fixed, using \(\tau^2=\sigma D_5/d\) and controlled field error;
\(D_5\) itself can grow with dimension.
\citet{liu2025finitesde} further considers a specially designed nonuniform
time schedule, which improves the polynomial endpoint dependence (on
\(1/t_0\) and \(1/(1-t_N)\)) to logarithmic dependence while retaining a
leading \(\widetilde{\mathcal O}(d^2)\) dimension dependence. Thus, compared
with their uniform discretization, our result improves the dimension
dependence, while their nonuniform schedule achieves sharper endpoint
dependence through a more specialized discretization.
Although their bounds contain source--target displacement moments, they also
retain coupling-independent score terms and do not develop a
coupling-quality interpretation analogous to our \(D_5\) control.
The companion ODE analysis \citep{liu2026finitetime} likewise gives a
\(\widetilde{\mathcal O}(d^2)\) dimension dependence for Euler
discretization.

Diffusion models admit sharper guarantees by exploiting the specific
Gaussian noising structure of the diffusion probability path.
Under weak distributional assumptions, \citep{benton2024nearly}
and \citep{li2025d} obtain near-linear dependence on \(d\)
for reverse-diffusion sampling.
More recent results establish sublinear dependence under additional
regularity. \citet{jiao2025optimal} introduce a
non-uniform regularity parameter \(L\) and obtain dimension dependence
\(\widetilde{\mathcal O}(\sqrt d\min\{\sqrt d,L\})\) for the standard DDPM
sampler. They showed \(L\) can be small for some specific class of targets,
but in general the bound could still be linear in \(d\).

We remark that related works such as
\citep{zhang2025sublinear,li2025improved} obtain sublinear dimension
dependence using randomized midpoint schemes, rather than the Euler
discretization considered here. Specifically, \citep{zhang2025sublinear}
obtain
\(\widetilde{\mathcal O}(\widetilde\beta_0\sqrt{d+M_2^2})\), which reduces
to \(\widetilde{\mathcal O}(\sqrt d)\) under bounded regularity constants,
while \citep{li2025improved} obtain
\(\widetilde{\mathcal O}(L d^{1/3})\). These improved dimension rates depend
explicitly on score regularity constants such as
\(\widetilde\beta_0\), \(M_2\), and \(L\).

Table~\ref{tab:sampling_related_work} compares our sampling bound with the
most relevant existing results, highlighting the dimension dependence,
assumptions, and whether the bound reflects the quality of the source--target
coupling.

Finally, we mention two additional related works. \citet{yang2025elucidating} extends the analysis of \citep{chen2023the} to the flow-matching setting using a similar predictor--corrector framework, under assumptions including \(L_2\) estimation error, bounded support, and uniform regularity of the learned vector field. \citet{su2025flow} also establishes a KL-divergence error bound for flow matching with a deterministic ODE sampler. However, their analysis relies on strong global-in-space regularity assumptions on the density path and velocity fields and does not analyze numerical ODE discretization error.

\begin{table}[h]
\centering
\caption{Comparison with related finite-step sampling results.
Dimension dependence suppresses logarithmic factors and holds accuracy
and the other problem parameters fixed; for our result this includes
\(\sigma,D_5\), with \(\tau^2=\sigma D_5/d\) and controlled field error.``Partial'' means that source--target displacement enters the bound, but coupling-independent score terms remain.}
\label{tab:sampling_related_work}
\renewcommand{\arraystretch}{1.15}

\begin{tabular}{@{}llllll@{}}
\toprule
\textbf{Work}
& \makecell{\textbf{Setting /}\\\textbf{Method}}
& \textbf{Metric}
& \textbf{Dim.}
& \textbf{Assumption}
& \makecell{\textbf{Coupling}\\\textbf{control}}
\\
\midrule

\multicolumn{6}{@{}l}{\textit{Stochastic interpolants and flow matching}}\\
\midrule

\makecell[l]{\textbf{This work}\\Main theorem}
& \makecell[l]{constant-\\smoothing SI\\SDE, Euler}
& KL
& $\widetilde{\mathcal O}(d)$
& $L_2$ error
& $D_5$
\\
\midrule

\makecell[l]{\citep{liu2025finitesde}\\Thm.~4.3; Sec.~5}
& \makecell[l]{SI\\SDE, Euler}
& KL
& $\widetilde{\mathcal O}(d^2)$
& $L_2$ error
& Partial
\\
\midrule

\makecell[l]{\citep{liu2026finitetime}\\Thm.~4.5}
& \makecell[l]{SI\\ODE, Euler}
& TV
& $\widetilde{\mathcal O}(d^2)$
& \makecell[l]{drift/div.\ errors;\\unif.\ reg.}
& No
\\

\midrule
\multicolumn{6}{@{}l}{\textit{Diffusion-model baselines}}\\
\midrule

\makecell[l]{\citep{benton2024nearly}\\Cor.~1}
& \makecell[l]{diffusion\\exp. integrator}
& KL
& $\widetilde{\mathcal O}(d)$
& $L_2$ error
& --
\\
\midrule

\makecell[l]{\citep{li2025d}\\Thm.~1}
& DDPM
& TV
& $\widetilde{\mathcal O}(d)$
& $L_2$ error
& --
\\
\midrule

\makecell[l]{\citep{jiao2025optimal}\\Thm.~1}
& DDPM
& KL
& \makecell[l]{$\widetilde{\mathcal O}(\sqrt d$ \\
  $\min\{\sqrt d,L\})$}
& \makecell[l]{weighted \(L_2\) error;\\nonuniform reg.}
& --
\\

\bottomrule
\end{tabular}
\end{table}

\subsection[Comparison With Dance et al. (2026)]
{Comparison With Dance et al.\ (2026)}
\label{app:comparison_dance}

The main difference from Algorithm 3 of \citep{dance2026debiased} is that our
training loss is a direct minibatch implementation of the doubly robust loss.
That is, for a fixed treatment \(a\), the plug-in term is evaluated for every
sample in the minibatch, while the observed correction is included only for
samples with \(A_i=a\):
\[
\frac{1}{|\mathcal B|}
\sum_{i\in\mathcal B}
\left[
\widehat \ell_i^{a}(v)
+
\frac{\mathbf 1\{A_i=a\}}{\widehat\pi_a(C_i)}
\left\{
\ell_i^{\mathrm{obs},a}(v)
-
\widehat \ell_i^{a}(v)
\right\}
\right].
\]
Note that the plug-in term is averaged
over the full covariate distribution and the inverse-propensity term corrects
only the observed treatment-\(a\) samples.

In contrast, Algorithm 3 of \citep{dance2026debiased} forms, for each
treatment \(a\), a subset
\(\mathcal B_a
=\{i\in\mathcal B:A_i=a\}\)
and accumulates a loss only over \(i\in\mathcal B_a\). This is
not a direct implementation of the standard DR loss: the plug-in term is then
averaged over the covariate distribution of the observed treatment-\(a\)
subsample rather than over the full minibatch \footnote{This description refers to Algorithm 3 as stated in the first arXiv version of \citep{dance2026debiased}; the discrepancy may simply reflect a typo or pseudocode-level issue rather than the intended implementation.}.

A minor difference is how the coupling is handled.
\citep{dance2026debiased} recomputes the coupling within each minibatch,
whereas our analysis estimates a single kernel \(\widehat Q_n\) and uses it
for the full dataset of size \(n\). Using one fixed global kernel gives a
well-defined empirical loss associated with a specific coupling.

Moreover, we remark that \citep{dance2026debiased} implements an ODE sampler, whereas we consider an SDE sampler.

%% file: Appendix_Start.tex
\section{Reader's Guide and Proof Roadmap}

This section lists the subclaims needed for the end-to-end result and points
to the proof of each one.

\paragraph{Appendix Organization.}
Appendix~\ref{app:angular_experiment_details} provides experimental details and additional results.
Appendix~\ref{sect_existing_work_appendix} gives additional
related-work details first. We then collect the notation and algorithmic
details, give the proofs of Examples~\ref{example_fixed_sigma_ode_kl} and
\ref{example_no_sigma}, and present the proof of
Theorem~\ref{Theorem_Sampling_Error}, the coupling and field-estimation
analyses, and the proof of Theorem~\ref{thm:end_to_end_sample_complexity}.

\paragraph{Motivating Examples.}
The proofs of Examples~\ref{example_fixed_sigma_ode_kl} and
\ref{example_no_sigma} are given in
Appendices~\ref{Appendix_Proof_Example} and
\ref{Appendix_Proof_Example_No_Smoothing}, respectively.

\paragraph{Sampling Theorem.}
Theorem~\ref{Theorem_Sampling_Error} follows from the following chain.
\begin{enumerate}[leftmargin=2.2em]
\item \emph{Marginal Correctness of the Exact SDE.}
Lemma~\ref{lem:exact_sampling_conditions} shows that \(q_t\) solves the
Fokker--Planck equation of the score-corrected SDE and satisfies the
integrability and continuity conditions required by the superposition
principle. It therefore yields a weak solution \(Z\) with \(Z_t\sim q_t\)
for every \(t\in[0,1]\).

\item \emph{Girsanov KL bound for the Euler process.}
Lemma~\ref{Lemma_Girsanov_Euler_KL} localizes the exact-to-Euler drift
perturbation, applies Girsanov before the stopping time, and passes to the
unstopped process by weak convergence and lower semicontinuity of KL.
Proposition~\ref{Prop_Pathwise_Drift_Finiteness} verifies the pathwise
integrability condition when the drifts and sample paths are continuous.

\item \emph{Posterior Derivative Formulas.}
Lemmas~\ref{Lemma_Score_Identity} and \ref{Lemma_Derivative_Identity} express
the score, velocity Jacobian, velocity Hessian, score Jacobian, and score
Laplacian as conditional moments of \(D=Y_1-Y_0\) and the Gaussian variable
\(G\).

\item \emph{Integrated Derivative Bounds.}
Lemmas~\ref{Lemma_Integrated_Score_Bounds} and
\ref{Lemma_Integrated_Derivative_Bounds} convert the posterior identities
into \(L^2(q_t)\) bounds using conditional variance contraction,
Cauchy--Schwarz, and explicit Gaussian moments.

\item \emph{Euler Increment Bounds.}
Lemmas~\ref{Lemma_Vel_regularity_sum} and
\ref{Lemma_Score_regularity_sum} apply It\^{o}'s formula along the exact SDE.
The score proof uses the continuity equation to cancel
\(D_xs_t\,v_t\), leaving only the Hessian, mixed velocity--score, and pure
score terms controlled by the integrated derivative bounds.

\item \emph{Collection of Constants.}
The proof of Theorem~\ref{Theorem_Sampling_Error} substitutes the Euler
increment bounds into the Girsanov reduction and records the explicit
coefficients \(A_h\) and \(B_h\), together with the explicit \(\tau^4\),
\(\tau^6\), and displacement terms. Corollary
\ref{cor:counterfactual_W1_from_KL} then adds the Gaussian smoothing error and
converts KL to \(W_1\) using a maximal coupling, Pinsker's inequality, and a
second-moment bound for the Euler recursion.
\end{enumerate}


\paragraph{Field-Estimation Propositions.}
The proof of Proposition~\ref{thm:final_velocity_L2_bound_DR} first derives
Proposition~\ref{prop:shared_DR_field_bound} and then verifies its hypotheses
for the velocity and score losses.
\begin{enumerate}[leftmargin=2.2em]
\item \emph{Coupling Construction and Replacement Error.}
Lemma~\ref{Lemma_EOT_Stability} proves Lipschitz stability of the EOT
conditional kernel.  Lemma~\ref{lem:existence_population_valid_Q} constructs
an ideal kernel \(Q\) with the required source and target marginals, and
Corollary~\ref{Cor_Coupling_Replacement} controls the joint-law discrepancy
between \(Q\) and \(\widehat Q_n\).

\item \emph{Exact Excess-Risk Identities.}
For both fields, conditional projection identifies excess population loss
with the grid-averaged squared \(L^2(q_{t_k})\) field error.  The score proof
uses the conditional-mean parametrization
\(s_u(x,t)=\{u(x,t)-x\}/\sigma^2\), for which the Gaussian term cancels from
the loss.

\item \emph{Approximation.}
Propositions~\ref{prop:relu_l2_approximation_noisy_oracle} and
\ref{prop:relu_l2_approximation_noisy_score_oracle} construct network
approximations uniformly on a bounded cube and control the Gaussian tail
outside the cube. 

\item \emph{Loss Perturbations.}
The coupling-replacement error is controlled using
Corollary~\ref{Cor_Coupling_Replacement}, together with the endpoint
Lipschitz bound for the loss. For the doubly robust error, we first write the
population bias as
\(\mathbb E_C[(1-\widehat\alpha/\alpha_a)(m_v-\widehat m_v)]\).
We then control \(m_v-\widehat m_v\) through the conditional-generator
\(W_1\) error and the stability of \(\widehat Q_n\), and apply
Cauchy--Schwarz to obtain the final product bound.

\item \emph{Statistical Generalization.}
Lemma~\ref{lem:finite_cover_empirical_process} controls the
empirical-to-population difference for any bounded loss class from a proper
finite cover.  The two field analyses supply their own network covering
numbers and loss-to-network transfer constants.

\item \emph{Monte Carlo Approximation.} 
We first decompose the implementation error according to the structure of the doubly robust loss and define separate observed-correction and auxiliary Monte Carlo classes. 
For each class, the corresponding error can be written as the difference between a conditional expectation and its realized Monte Carlo
average.  Lemma~\ref{lem:finite_cover_empirical_process} is then applied
conditionally to each array.
A key step is to condition on the training observations and fitted auxiliary
objects, which gives the conditional independence required in
Proposition~\ref{prop:shared_DR_field_bound}.

\item \emph{End-to-End Substitution.}
Combining the preceding velocity- and score-field steps gives the integrated errors
\(\varepsilon_{\mathrm{vel}}^2\) and
\(\varepsilon_{\mathrm{score}}^2\).  Substitution into the explicit sampling
theorem, followed by the Wasserstein conversion, proves
Theorem~\ref{thm:end_to_end_sample_complexity}.
\end{enumerate}

\paragraph{Technical Novelty.}
Our proof makes two main contributions. First, for fixed Gaussian smoothing, we obtain near-linear rather than quadratic dimension dependence in the number of uniform Euler steps, with $D_5$, $\sigma$, and target accuracy held fixed. We bound velocity derivatives and mixed velocity--score terms using displacement moments, and keep the noise factors in the score terms explicit. With exact fields, choosing $\tau^2=\sigma D_5/d$ then reduces the dimension dependence and makes every discretization term shrink with $D_5$, including the pure-score terms. This gives both improved dimension dependence and full coupling-sensitive control of the discretization error.

Second, we give an end-to-end analysis of counterfactual generation with a learned source--target coupling, rather than only independent pairing. We analyze a doubly robust loss that combines prediction-based losses with weighted corrections from observed outcomes, rather than the usual flow-matching loss. Because the training loss itself is approximated using finitely many random draws from the prediction oracle and the coupling, we bound the error caused by this Monte Carlo approximation. We also separate this error from the finite-data error and the Monte Carlo error in the target samples used to learn the coupling. Combining these learning bounds with the Euler and smoothing errors gives a guarantee for the final output distribution.

\subsection{Summary of Notation}
\label{sec:notation_summary}

Table~\ref{tab:notation_summary} summarizes the notation used throughout the paper.

{
\footnotesize
\setlength{\tabcolsep}{2pt}
\renewcommand{\arraystretch}{0.96}
\setlength{\LTpre}{3pt}
\setlength{\LTpost}{3pt}
\begin{longtable}{
@{}p{0.13\linewidth}p{0.32\linewidth}
@{\hspace{0.025\linewidth}}
p{0.13\linewidth}p{0.32\linewidth}@{}}
\caption{Summary of notation used throughout the paper.}
\label{tab:notation_summary}\\
\hline
Notation & Meaning & Notation & Meaning \\
\hline
\endfirsthead
\endhead
$(C,A,Y)$ & Observed covariate, treatment, and outcome.
& $\mathcal V_n,\mathcal S_n$
& Clipped sparse-ReLU velocity class and score class induced by clipped sparse-ReLU denoisers. \\
$a$ & Fixed target treatment value.
& $R,\varepsilon$
& Network input-clipping radius and approximation tolerance; field-specific choices appear in Proposition~\ref{thm:final_velocity_L2_bound_DR}. \\
$\mathbb P_{Y\mid C=c,A=a}$ & Conditional outcome law at covariate value \(c\) under treatment \(a\).
& $\varepsilon_{\mathrm{vel}},\varepsilon_{\mathrm{score}}$
& Integrated estimation errors of the learned velocity and score fields. \\
$\mu_1$
& Target law \(\int\mathbb P_{Y\mid C=c,A=a}\,d\mathbb P_C(c)\).
& $\tau$ & Noise level of the score-corrected SDE sampler. \\
$\mathbb P_{Y\mid A=a}$ & Observed outcome law among samples with treatment \(a\).
& $N,h$ & Number of sampling steps and Euler step size, with \(h=1/N\). \\
$\mathcal D_{\mathrm{tr}},\mathcal D_{\mathrm{aux}}$ & Training and auxiliary splits, each containing \(n\) observations.
& $t_k$ & Sampling grid point \(t_k=kh\). \\
$\widehat{\mathbb P}_{Y\mid C=c,A=a}$ & Learned conditional outcome law, or prediction oracle.
& $M$ & Monte Carlo size in the final doubly robust training objective. \\
$\widehat{\mathbb P}_{Y\mid A=a}^{\mathrm{aux}}$ & Empirical source law formed from treatment-\(a\) outcomes in \(\mathcal D_{\mathrm{aux}}\).
& $M_Q$ & Prediction-oracle samples per auxiliary covariate used to construct \(\widehat{\mathbb P}_{Y_1}^{Q}\). \\
$\widehat{\mathbb P}_{Y_1}^{Q}$ & Finite-sample fitted-target approximation based on the prediction oracle.
& $\alpha_a(c),\widehat\alpha(c)$
& True and estimated inverse propensities for treatment \(a\). \\
$\widehat Q_n(dy_0\mid y_1)$ & Estimated coupling kernel between the empirical source and fitted target.
& $\pi_{\min}$ & Lower bound on \(\mathbb P(A=a\mid C=c)\). \\
$Q(dy_0\mid y_1),\Gamma_{\mu_1,Q}$
& Population-valid kernel and joint endpoint law \(\Gamma_{\mu_1,Q}=Q\mu_1\).
& $\varepsilon_\alpha$ & Estimation error of the inverse propensity model. \\
$(Y_0,Y_1)$ & Source--target pair under the specified coupling.
& $\epspsiemp$ & Empirical prediction-oracle error on auxiliary covariates. \\
$D_5$ & Coupling displacement moment \((\mathbb E\|Y_1-Y_0\|^5)^{1/5}\).
& $\epspsigen$ & Prediction-oracle generalization error. \\
$S_t$ & Linear interpolation \(Y_0+t(Y_1-Y_0)\).
& $\varepsilon_{\mathrm{OT}}$ & Entropic regularization parameter for EOT. \\
$X_t$ & Smoothed interpolation \(S_t+\sigma G\), \(G\sim\mathcal N(0,I_d)\).
& $\mathcal E_a$
& Event that \(\mathcal D_{\mathrm{aux}}\) contains at least one observation with \(A_i=a\). \\
$\sigma$ & Gaussian smoothing level.
& $\operatorname{supp}(P)$
& Topological support: \(x\) such that \(P(B(x,r))>0\) for every \(r>0\). \\
$v_t,s_t$ & Population velocity and score fields for the smoothed interpolation.
& $d$ & Outcome dimension. \\
$q_{t,Q},v_Q^\star,s_Q^\star$
& Density, exact velocity, and exact score induced by \(\Gamma_{\mu_1,Q}\).
& $W_1,W_{1,\rho}$ & Wasserstein distances on the outcome and coupled endpoint spaces. \\
$\widehat v_M,\widehat s_M$ & Learned doubly robust velocity and score fields.
& $\operatorname{KL}$ & Kullback--Leibler divergence. \\
\hline
\multicolumn{4}{@{}p{0.97\linewidth}@{}}{\textbf{Conditioning Convention.}
In the finite-sample analysis, we first condition on
\(\mathcal D_{\mathrm{aux}}\) and the prediction-oracle target reservoir.
The fitted nuisance models, \(\widehat Q_n\), \(Q\),
\(\Gamma_{\mu_1,Q}\), \(q_{t,Q}\), \(v_Q^\star\), \(s_Q^\star\), and the
corresponding \(D_5\) are then interpreted conditionally. Unconditional
high-probability statements include the corresponding auxiliary-sample events
in their failure probabilities.}\\
\end{longtable}
}

\section[Algorithms and Auxiliary-Model Details]
{Algorithms and Auxiliary-Model Details}

\subsection{Prediction-Oracle Accuracy}
\label{App_Prediction_Oracle_Accuracy}

We give sufficient conditions for Assumption~\ref{ass:prediction_oracle_accuracy}. The first controls the prediction-oracle error on the auxiliary covariates used to construct the coupling.

\begin{assumption}[Empirical Generator Accuracy on Auxiliary Covariates]
\label{ass:psi_empirical_training_W1_accuracy}
The learned conditional outcome model satisfies
\begin{align*}
\left[
\frac{1}{n}
\sum_{i=n+1}^{2n}
W_1
\left(
\mathbb P_{Y\mid C=C^i,A=a},
\widehat{\mathbb P}_{Y\mid C=C^i,A=a}
\right)^2
\right]^{1/2}
&\leq
\epspsiemp.
\end{align*}
\end{assumption}

This assumption measures how well the learned conditional outcome model fits
the conditional laws at the empirical auxiliary covariates. It is not a population-\(C\) quantity.

\begin{assumption}[Generator Generalization Gap]
\label{ass:psi_generator_generalization_gap}
Let \(\Psi_n\) denote the class for prediction oracle, and let
\(\widehat{\mathbb P}^{\psi}_{Y\mid C=c,A=a}\) be the conditional law generated
by \(\psi\). Assume
\begin{align*}
&\sup_{\psi\in\Psi_n}
\left|
\mathbb E_C
\left[
W_1
\left(
\mathbb P_{Y\mid C=c,A=a},
\widehat{\mathbb P}^{\psi}_{Y\mid C=c,A=a}
\right)^2
\right]
\right.
\\
&\hspace{3cm}
\left.
-
\frac{1}{n}
\sum_{i=n+1}^{2n}
W_1
\left(
\mathbb P_{Y\mid C=C^i,A=a},
\widehat{\mathbb P}^{\psi}_{Y\mid C=C^i,A=a}
\right)^2
\right|
\leq
\epspsigen^2.
\end{align*}
\end{assumption}

Together, these two assumptions imply the population prediction-oracle accuracy used in the main text.

\begin{proposition}[Population Generator Accuracy From Empirical Accuracy]
\label{prop:psi_population_accuracy_from_empirical_accuracy}
Under Assumptions~\ref{ass:psi_empirical_training_W1_accuracy} and~\ref{ass:psi_generator_generalization_gap},
\begin{align*}
\bigg[\mathbb E_C W_1(\mathbb P_{Y\mid C=c,A=a},\predorac)^2\bigg]^{1/2}
&\leq [(\epspsiemp)^2+\epspsigen^2]^{1/2}.
\end{align*}
Hence Assumption~\ref{ass:prediction_oracle_accuracy} holds.
\end{proposition}

\begin{proof}
Since Assumption~\ref{ass:psi_generator_generalization_gap} holds uniformly over $\psi\in\Psi_n$, it also applies to the learned prediction oracle. Therefore,
\begin{align*}
\mathbb E_C W_1(\mathbb P_{Y\mid C=c,A=a},\predorac)^2
&\leq \frac{1}{n}\sum_{i=n+1}^{2n}W_1(\mathbb P_{Y\mid C=C_i,A=a},\widehat{\mathbb P}_{Y\mid C=C_i,A=a})^2+\epspsigen^2 \\ &
\leq (\epspsiemp)^2+\epspsigen^2.
\end{align*}
Taking square roots proves the claim.
\end{proof}

The prediction oracle $\predorac$ can be obtained by training an auxiliary conditional flow model on $\mathcal D_{\mathrm{aux}}$. For each observation $(Y_i,C_i,A_i)$, conditional on $(C_i,A_i)$, $Y_i$ is a draw from $\mathbb P_{Y\mid C=C_i,A=A_i}$. Let $P_0$ be a source distribution, for example ${\mathcal N}(0,I_d)$. The population flow-matching loss is
\begin{align*}
\mathcal L_{\mathrm{aux}}(u)
&= \mathbb E_{C,A,Y}\mathbb E_{Z\sim P_0,\,t\sim\operatorname{Unif}[0,1]}
[\|u(X_t,t,C,A)-(Y-Z)\|^2], \qquad X_t=(1-t)Z+tY.
\end{align*}
The auxiliary observations $(C_i,A_i,Y_i)$, $i=n+1,\ldots,2n$, provide samples for training this model. After training, fixing $A=a$ and conditioning on $C=c$ gives the prediction oracle $\widehat{\mathbb P}_{Y\mid C=c,A=a}$.

\subsection[Score-Field Training Objective]
{Score-Field Training Objective}
\label{App_score_loss}

We now explain how the score field in the score-corrected SDE is learned. For each $t\in[0, 1]$, let $q_t$ denote the density of $X_t$ and define $s_t(x)=\nabla_x\log q_t(x)$. The corresponding population score-matching loss is
\begin{align*}
\mathcal L_{\mathrm{score,true}}(\phi)
&=
\int_0^1
\mathbb E\left[
\left\|
s_\phi(X_t, t)-s_t(X_t)
\right\|^2
\right]dt.
\end{align*}
This loss cannot be evaluated directly because $s_t$ is unknown.
By the Gaussian score identity proved in Lemma \ref{Lemma_Score_Identity}, $s_t(X_t)=\mathbb E[-G/\sigma\mid X_t]$. Therefore, $-G/\sigma=(S_t-X_t)/\sigma^2$ can be used as a regression target for the score.
We define the denoising score-matching loss by
\begin{align*}
\mathcal L_{\mathrm{score}}(\phi)
= \int_0^1
\mathbb E_{(Y_0,Y_1)\sim\gamma}
\mathbb E_{G\sim\mathcal N(0,I_d)}
\left[
\left\|
s_\phi(X_t, t)
+
\frac{G}{\sigma}
\right\|^2
\right]dt
=
\int_0^1
\mathbb E
\left[
\left\|
s_\phi(X_t, t)
-
\frac{S_t-X_t}{\sigma^2}
\right\|^2
\right]dt.
\end{align*}

The equivalence follows from the conditional Pythagorean identity.  For each
fixed \(t\), set \(U=-G/\sigma\).  Since
\(s_t(X_t)=\mathbb E[U\mid X_t]\), every square-integrable candidate
\(s_\phi(t,\cdot)\) satisfies
\begin{align*}
\mathbb E\|s_\phi(X_t, t)-U\|^2
&=
\mathbb E\|s_\phi(X_t, t)-s_t(X_t)\|^2
+
\mathbb E\|U-s_t(X_t)\|^2.
\end{align*}
The second term is independent of \(\phi\).  Integrating over \(t\) proves
that the denoising and true-score losses have the same minimizers.  Since
\(X_t-S_t=\sigma G\), the implementable objective may equivalently be
written as
\begin{align*}
\mathcal L_{\mathrm{score}}(\phi)
&=
\int_0^1
\mathbb E
\left[
\left\|
s_\phi(X_t, t)-\frac{S_t-X_t}{\sigma^2}
\right\|^2
\right]dt.
\end{align*}

\subsection{Algorithms}

Algorithm~\ref{alg:score_corrected_sde_sampler} gives the
finite-step score-corrected sampler, while
Algorithm~\ref{alg:minibatch_dr_flow_training} gives the doubly robust
procedure used to train its velocity and score fields.

\begin{algorithm}[t]
\caption{Score-Corrected SDE Sampler}
\label{alg:score_corrected_sde_sampler}
\begin{algorithmic}[1]
\Require Learned velocity field \(\widehat v\), learned score field \(\widehat s\), empirical source distribution \(\widehat{\mathbb P}_{Y\mid A=a}^{\mathrm{aux}}\), smoothing level \(\sigma>0\), SDE noise level \(\tau>0\), number of time steps \(N\).
\State Set \(h=1/N\) and \(t_k=kh\), \(k=0,\ldots,N\).
\State Draw \(Y_0^{\mathrm{init}}\sim \widehat{\mathbb P}_{Y\mid A=a}^{\mathrm{aux}}\) and \(G^{\mathrm{init}}\sim {\mathcal N}(0,I_d)\).
\State Initialize $\widehat Z_0 = Y_0^{\mathrm{init}} + \sigma G^{\mathrm{init}}$.
\For{\(k=0,\ldots,N-1\)}
    \State Draw \(G_k\sim {\mathcal N}(0,I_d)\).
    \State Update $\widehat Z_{k+1} = \widehat Z_k  + h (\widehat v(\widehat Z_k,t_k) +     \frac{\tau^2}{2}     \widehat s(\widehat Z_k,t_k)) + \tau\sqrt h\,G_k$.
\EndFor
\State \Return \(\widehat Z_N\).
\end{algorithmic}
\end{algorithm}

\begin{algorithm}[t]
\caption{Minibatch Doubly Robust Flow Training for a Fixed Treatment   \(a\)}
\label{alg:minibatch_dr_flow_training}
\begin{algorithmic}[1]
\Require Training data \(\{(C_i,A_i,Y_i)\}_{i=1}^n\) for the final flow model, and training data \(\{(C_i,A_i,Y_i)\}_{i=n+1}^{2n}\) for the following models: propensity estimator \(\widehat{\mathbb P}(A=a\mid C=c)\), conditional outcome sampler \(\predorac\), estimated coupling kernel \(\widehat Q_n(dy_0\mid y_1)\).
\Require Vector-field class \(\mathcal V_n\),  score-field class \(\mathcal S_n\), noise level \(\sigma>0\), Monte Carlo size \(M\).

\State Initialize \(v\in\mathcal V_n\) and \(s\in\mathcal S_n\).

\For{\(i\in \{1, ..., n\}\)}
    \For{\(k=1,\ldots,M\)}
        \State Draw \(\widehat Y_{i,k} \sim \widehat{\mathbb P}_{Y\mid C=C_{i},A=a}\).
        \State Draw \(\widehat Y_{0,i,k}\sim\widehat Q_n(\cdot\mid \widehat Y_{i,k})\), \(t_{i,k}\sim \operatorname{Unif}\{t_0, ..., t_{N-1}\}\), and \(G_{i,k}\sim {\mathcal N}(0,I_d)\).
    \EndFor
    \If{\(A_i=a\)}
        \State Draw \(\widehat Y_{0,i}^{\mathrm{obs}}\sim\widehat Q_n(\cdot\mid Y_i)\), \(t_i^{\mathrm{obs}}\sim \operatorname{Unif}\{t_0, ..., t_{N-1}\}\), and \(G_i^{\mathrm{obs}}\sim {\mathcal N}(0,I_d)\).
    \EndIf
\EndFor
\For{each training iteration}
    \State Draw a minibatch \(\mathcal B\subseteq\{1,\ldots,n\}\).

    \State Compute the minibatch doubly robust flow loss
    \begin{align*}
    \widehat L_{\mathcal B,\mathrm{DR}}(v)
    &=
    \frac{1}{|\mathcal B|}
    \sum_{i\in\mathcal B}
    \Bigg[
    \frac{\mathbf 1\{A_i=a\}}{\widehat{\mathbb P}(A=a\mid C=C_i)}
    \left\{
    \widehat\ell_i^{\mathrm{obs}}(v)
    -
    \frac{1}{M}
    \sum_{k=1}^{M}
    \widehat\ell_{i,k}^{\mathrm{plug}}(v)
    \right\}
    +
    \frac{1}{M}
    \sum_{k=1}^{M}
    \widehat\ell_{i,k}^{\mathrm{plug}}(v)
    \Bigg],
    \end{align*}
    where
    \begin{align*}
    \widehat\ell_i^{\mathrm{obs}}(v)
    &=
    \left\|
    v\left(
    \widehat Y_{0,i}^{\mathrm{obs}}
    +
    t_i^{\mathrm{obs}}
    (Y_i-\widehat Y_{0,i}^{\mathrm{obs}})
    +
    \sigma G_i^{\mathrm{obs}},
    t_i^{\mathrm{obs}}
    \right)
    -
    (Y_i-\widehat Y_{0,i}^{\mathrm{obs}})
    \right\|^2,
    \\
    \widehat\ell_{i,k}^{\mathrm{plug}}(v)
    &=
    \left\|
    v\left(
    \widehat Y_{0,i,k}
    +
    t_{i,k}
    (\widehat Y_{i,k}-\widehat Y_{0,i,k})
    +
    \sigma G_{i,k},
    t_{i,k}
    \right)
    -
    (\widehat Y_{i,k}-\widehat Y_{0,i,k})
    \right\|^2.
    \end{align*}

    \State Update \(v\) by one optimizer step using \(\widehat L_{\mathcal B,\mathrm{DR}}(v)\).

    \State Compute the minibatch doubly robust score loss
    \begin{align*}
    \widehat L_{\mathcal B,\mathrm{DR}}^{\mathrm{score}}(s)
    &=
    \frac{1}{|\mathcal B|}
    \sum_{i\in\mathcal B}
    \Bigg[
    \frac{\mathbf 1\{A_i=a\}}{\widehat{\mathbb P}(A=a\mid C=C_i)}
    \left\{
    \widehat r_i^{\mathrm{obs}}(s)
    -
    \frac{1}{M}
    \sum_{k=1}^{M}
    \widehat r_{i,k}^{\mathrm{plug}}(s)
    \right\}
    +
    \frac{1}{M}
    \sum_{k=1}^{M}
    \widehat r_{i,k}^{\mathrm{plug}}(s)
    \Bigg],
    \end{align*}
    where
    \begin{align*}
    \widehat r_i^{\mathrm{obs}}(s)
    &=
    \left\|
    s\left(
    \widehat Y_{0,i}^{\mathrm{obs}}  +  t_i^{\mathrm{obs}}  \left(  Y_i-\widehat Y_{0,i}^{\mathrm{obs}} \right) + \sigma G_i^{\mathrm{obs}},
    t_i^{\mathrm{obs}}
    \right)
    +
    \frac{G_i^{\mathrm{obs}}}{\sigma}
    \right\|^2,
    \\
    \widehat r_{i,k}^{\mathrm{plug}}(s)
    &=
    \left\|
    s\left(
    \widehat Y_{0,i,k} +  t_{i,k} \left( \widehat Y_{i,k}-\widehat Y_{0,i,k} \right) + \sigma G_{i,k},
    t_{i,k}
    \right)
    +
    \frac{G_{i,k}}{\sigma}
    \right\|^2.
    \end{align*}

    \State Update \(s\) by one optimizer step using \(\widehat L_{\mathcal B,\mathrm{DR}}^{\mathrm{score}}(s)\).
\EndFor

\State \Return $v$ and $s$.
\end{algorithmic}
\end{algorithm}


\clearpage
\section[Proof of Example \ref{example_fixed_sigma_ode_kl}]
{Proof of Example~\ref{example_fixed_sigma_ode_kl}}
\label{Appendix_Proof_Example}

\begin{proof}[Proof of Example \ref{example_fixed_sigma_ode_kl}]
In this proof, write \(\widehat Z_k\) for the Euler ODE iterates, so
\(\widehat Z_N=\widehat X_N^{\mathrm{ODE}}\) in the example.
Starting from \(\widehat Z_0\overset{d}{=}X_0\), the recursion is
\begin{align*}
\widehat Z_{k+1}=\widehat Z_k+h v_{kh}(\widehat Z_k),\qquad k=0,\ldots,N-1.
\end{align*}
Here $N$ is the number of Euler steps, and $h = 1/N$ is the step size.

Write $D=Y_1-Y_0$. We have $v_t(x)=\mathbb E[D\mid X_t=x]$.
Let $\varphi_V$ denote the density of \(\mathcal N(0,V)\). Conditional on $B=0$ we have \(X_0\sim\mathcal N(0,1+\sigma^2)\) and conditional on $B=1$ we have \(X_0\sim\mathcal N(0,4+\sigma^2)\).
Bayes' rule implies
\begin{align*}
\mathbb P(B=1\mid X_0=x)
=\frac{p\varphi_{4+\sigma^2}(x)}
{(1-p)\varphi_{1+\sigma^2}(x)+p\varphi_{4+\sigma^2}(x)}.
\end{align*}

Since $D=0$ on $\{B=0\}$,
\begin{align*}
v_0(x)&=\mathbb P(B=1\mid X_0=x) \mathbb E[D\mid X_0=x,B=1] + \mathbb P(B=0\mid X_0=x) \mathbb E[D\mid X_0=x,B=0] \\
&=\mathbb P(B=1\mid X_0=x) \mathbb E[D\mid X_0=x,B=1].
\end{align*}

We have that 
\begin{align*}
    \mathbb E[U\mid 2U+\sigma G=x]
    = \int_{-\infty}^{\infty} u P_{u \mid 2u+\sigma G = x}(u) du.
\end{align*}
To compute $P_{u \mid 2u+\sigma G = x}(u)$, we use the formula $p_{X\mid Y = y}(x) = \frac{p_{X, Y}(x, y)}{p_{Y}(y)}$. 
We get 
\begin{align*}
    p_{U\mid 2U + \sigma G = x}(u) = \frac{p_{U, 2U + \sigma G}(u, x)}{p_{2U + \sigma G}(x)}
    \propto p_{U, 2U + \sigma G}(u, x) = \frac{1}{\sigma} p_{U, G}(u, \frac{x-2u}{\sigma}),
\end{align*}
where we used change of variable formula. 
Then we have 
\begin{align*}
    p_{U\mid 2U + \sigma G = x}(u)
    &\propto \frac{1}{\sigma} \exp(-u^2/2 - (x-2u)^2/(2 \sigma^2)) 
    \propto \exp(- \frac{(u - \frac{2x}{\sigma^2+4})^2 }{2\sigma^2 /(\sigma^2 + 4)}).
\end{align*}
Hence $\mathbb E[U\mid 2U+\sigma G=x]=\frac{2}{4+\sigma^2}x$.
Using the definition of $\lambda$,
\begin{align*}
\mathbb E[D\mid X_0=x,B=1]
=-2\lambda\frac{2}{4+\sigma^2}x
=-\frac{x}{h},
\end{align*}
so we have $v_0 (x) = -\frac{\mathbb P(B=1\mid X_0=x)}{h}x.$

We next examine the uniform spatial regularity of the population velocity. By the expression above for $\mathbb P(B=1\mid X_0=x)$,
\begin{align*}
\mathbb P(B=1\mid X_0=x)
&=\frac{p\varphi_{4+\sigma^2}(x)}
{(1-p)\varphi_{1+\sigma^2}(x)+p\varphi_{4+\sigma^2}(x)} =\left(
1+\frac{1-p}{p}\frac{\varphi_{1+\sigma^2}(x)}{\varphi_{4+\sigma^2}(x)}
\right)^{-1}\\
&=\left(
1+\frac{1-p}{p}\sqrt{\frac{4+\sigma^2}{1+\sigma^2}}
\exp\left\{-\frac{x^2}{2(1+\sigma^2)}+\frac{x^2}{2(4+\sigma^2)}\right\}
\right)^{-1}\\
&=\left(
1+\frac{1-p}{p}\sqrt{\frac{4+\sigma^2}{1+\sigma^2}}
\exp\left\{-\frac{3x^2}{2(1+\sigma^2)(4+\sigma^2)}\right\}
\right)^{-1}.
\end{align*}
Therefore,
\begin{align*}
v_0(x)
=-\frac{x}{h}
\left(
1+\frac{1-p}{p}\sqrt{\frac{4+\sigma^2}{1+\sigma^2}}
\exp\left\{-\frac{3x^2}{2(1+\sigma^2)(4+\sigma^2)}\right\}
\right)^{-1}.
\end{align*}
Differentiating with respect to $x$ gives
\begin{align*}
\partial_x v_0(x)
={}&-\frac{1}{h}
\left(
1+\frac{1-p}{p}\sqrt{\frac{4+\sigma^2}{1+\sigma^2}}
\exp\left\{-\frac{3x^2}{2(1+\sigma^2)(4+\sigma^2)}\right\}
\right)^{-1}\\
&-\frac{3x^2}{h(1+\sigma^2)(4+\sigma^2)}
\frac{
\frac{1-p}{p}\sqrt{\frac{4+\sigma^2}{1+\sigma^2}}
\exp\left\{-\frac{3x^2}{2(1+\sigma^2)(4+\sigma^2)}\right\}
}{
\left(
1+\frac{1-p}{p}\sqrt{\frac{4+\sigma^2}{1+\sigma^2}}
\exp\left\{-\frac{3x^2}{2(1+\sigma^2)(4+\sigma^2)}\right\}
\right)^2
}.
\end{align*}
For all sufficiently small $p$, we have
\begin{align*}
\frac{1-p}{p}\sqrt{\frac{4+\sigma^2}{1+\sigma^2}}>1.
\end{align*}
Choose $x>0$ such that $x^2=\frac{2(1+\sigma^2)(4+\sigma^2)}{3}
\log\left(
\frac{1-p}{p}\sqrt{\frac{4+\sigma^2}{1+\sigma^2}}
\right)$, i.e., $\frac{1-p}{p}\sqrt{\frac{4+\sigma^2}{1+\sigma^2}}
\exp\left\{-\frac{3x^2}{2(1+\sigma^2)(4+\sigma^2)}\right\}=1$.
Substituting this value of $x$ into the expression for $\partial_xv_0(x)$ gives
\begin{align*}
\partial_xv_0(x)=-\frac{1}{2h}
-\frac{3x^2}{4h(1+\sigma^2)(4+\sigma^2)}=-\frac{1}{2h}
\left[
1+
\log\left(
\frac{1-p}{p}\sqrt{\frac{4+\sigma^2}{1+\sigma^2}}
\right)
\right].
\end{align*}
Consequently,
\begin{align*}
\sup_{t\in[0,1]}\sup_{x\in\mathbb R}|\partial_xv_t(x)| \geq\sup_{x\in\mathbb R}|\partial_xv_0(x)| \geq
\frac{1}{2h}
\left[
1+
\log\left(
\frac{1-p}{p}\sqrt{\frac{4+\sigma^2}{1+\sigma^2}}
\right)
\right].
\end{align*}

By construction, $D=0$ on $\{B=0\}$ and $D=-2\lambda U$ on $\{B=1\}$. Hence $D_5
=\bigl(\mathbb E|D|^5\bigr)^{1/5}
=2\lambda\bigl(\mathbb E|U|^5\bigr)^{1/5}p^{1/5}$. So we conclude that when $p \to 0$, $\sup_{t\in[0,1]}\sup_{x\in\mathbb R}|\partial_xv_t(x)|
\to +\infty$ but $D_5\to0$.

For the integrated derivative claim, keep \(h,\sigma\) fixed, so that
\(\lambda\) is fixed. Since
\(\mathbb E|D|^4=16\lambda^4p\,\mathbb E|U|^4=48\lambda^4p\),
Lemma~\ref{Lemma_Integrated_Derivative_Bounds} with \(d=1\) gives
\[
\int_0^1\mathbb E|\partial_xv_t(X_t)|^2\,dt
\leq\frac{\sqrt3}{\sigma^2}(48\lambda^4p)^{1/2}
=\frac{12\lambda^2}{\sigma^2}\sqrt p\longrightarrow0.
\]

Next, we show the sampling KL error would blow up.
The first Euler update is the deterministic map
\begin{align*}
\widehat Z_1=T_p(\widehat Z_0),
\qquad
T_p(x)=x+h v_0(x)=x(1-\mathbb P(B=1\mid X_0=x)).
\end{align*}
Moreover,
\begin{align*}
1-\mathbb P(B=1\mid X_0=x)
&\leq\frac{1-p}{p}\frac{\varphi_{1+\sigma^2}(x)}{\varphi_{4+\sigma^2}(x)}\\
&=\frac{1-p}{p}\sqrt{\frac{4+\sigma^2}{1+\sigma^2}}
\exp\left\{-\frac{x^2}{2}\left(\frac{1}{1+\sigma^2}-\frac{1}{4+\sigma^2}\right)\right\}.
\end{align*}
Since $T_p$ is continuous, its image is bounded. Thus $\widehat Z_1$ is supported on a compact interval.

For $v_t(x)$, we change the condition $X_0 = 2U + \sigma G = x$ to $X_t = 2U(1-\lambda t) + \sigma G = x$.
\begin{align*}
\mathbb E[D\mid X_t=x,B=1]
=-\frac{4\lambda(1-\lambda t)}{4(1-\lambda t)^2+\sigma^2}x,
\end{align*}
so 
\begin{align*}
v_t(x)
=-\frac{4\lambda(1-\lambda t)}{4(1-\lambda t)^2+\sigma^2}
\frac{p\varphi_{4(1-\lambda t)^2+\sigma^2}(x)}
{(1-p)\varphi_{1+\sigma^2}(x)+p\varphi_{4(1-\lambda t)^2+\sigma^2}(x)}x.
\end{align*}
Note that $|\frac{4\lambda(1-\lambda t)}{4(1-\lambda t)^2+\sigma^2}| \le |\frac{4\lambda(1-\lambda t)}{\sigma^2}|$ and $\frac{p\varphi_{4(1-\lambda t)^2+\sigma^2}(x)}
{(1-p)\varphi_{1+\sigma^2}(x)+p\varphi_{4(1-\lambda t)^2+\sigma^2}(x)} \le 1$ are uniformly bounded.
Consequently,
\begin{align*}
x\mapsto x+h v_{kh}(x)
\end{align*}
is continuous for every $k=0,\ldots,N-1$. Since a continuous image of a compact set is compact, $\widehat Z_N$ is supported on some compact set.

On the other hand, the density \(q_1(x)\) of \(X_1\) is strictly positive for every $x\in\mathbb R$. Therefore, $\mathcal L(X_1)$ is not absolutely continuous with respect to $\mathcal L(\widehat Z_N)$, which proves
\begin{align*}
\mathrm{KL}\bigl(\mathcal L(X_1)\,\|\,\mathcal L(\widehat Z_N)\bigr)=+\infty.
\end{align*}
The result on integrated regularity follows from Lemma \ref{Lemma_Integrated_Derivative_Bounds}.
\end{proof}

\section[Proof of Example \ref{example_no_sigma}]
{Proof of Example~\ref{example_no_sigma}}
\label{Appendix_Proof_Example_No_Smoothing}

\begin{proof}[Proof of Example \ref{example_no_sigma}]
We first compute the coupling displacement. Conditional on $Y_0=x$, the displacement $Y_1-Y_0$ is equal to $\delta$, $-\delta$, or $0$, with respective probabilities $p_m(x)$, $p_m(x)$, and $1-2p_m(x)$. Therefore
\[
\mathbb E\left[|Y_1-Y_0|^5\mid Y_0=x\right]
=2p_m(x)\delta^5
=\delta^5\left(\frac12+\frac14\sin(K_mx)\right).
\]

Since $Y_0$ is symmetric about zero and $x\mapsto\sin(K_mx)$ is odd, $\mathbb E[\sin(K_mY_0)]=0$.
Hence \(\mathbb E|Y_1-Y_0|^5=\delta^5/2\), and therefore
\[
D_5=\frac{\delta}{2^{1/5}}=\varepsilon.
\]
In particular, $D_5$ is independent of $m$ even though $K_m\to\infty$.

We next verify that the score is well defined despite the absence of smoothing. Let
\[
\phi(x)=\frac{1}{\sqrt{2\pi}}e^{-x^2/2}
\]
denote the standard Gaussian density. Conditional on the three possible values of $Y_1-Y_0$, the interpolation satisfies
\[
Y_t=
\begin{cases}
Y_0+t\delta,&Y_1-Y_0=\delta,\\
Y_0-t\delta,&Y_1-Y_0=-\delta,\\
Y_0,&Y_1-Y_0=0.
\end{cases}
\]
It follows that
\begin{align*}
q_t^{{(m)}}(y)
=\phi(y-t\delta)p_m(y-t\delta)
+\phi(y+t\delta)p_m(y+t\delta)+\phi(y)\left(1-2p_m(y)\right).
\end{align*}
Since $p_m(y)\leq 3/8$, $1-2p_m(y)\geq\frac14$,
and hence $q_t^{{(m)}}(y)\geq\frac14\phi(y)>0$.
Moreover, $q_t^{{(m)}}$ is smooth because it is a finite sum of
products of smooth functions. Thus
$s_t^{{(m)}}=\partial_x\log q_t^{{(m)}}$ and
$\partial_xs_t^{{(m)}}$ are well defined for every
$t\in[0,1]$.

Since
$\partial_xs_t^{{(m)}}(x)
=\partial_x\frac{\bigl(q_t^{{(m)}}\bigr)'(x)}
{q_t^{{(m)}}(x)}
=\frac{\bigl(q_t^{{(m)}}\bigr)''(x)
q_t^{{(m)}}(x)
-\bigl(\bigl(q_t^{{(m)}}\bigr)'(x)\bigr)^2}
{\bigl(q_t^{{(m)}}(x)\bigr)^2}
=\frac{\bigl(q_t^{{(m)}}\bigr)''(x)}
{q_t^{{(m)}}(x)}
-s_t^{{(m)}}(x)^2$,
we have
\begin{align*}
\mathbb E[\partial_xs_t^{{(m)}}(Y_t)]
&=\int_{\mathbb R}\partial_xs_t^{{(m)}}(y)
q_t^{{(m)}}(y)\,dy\\
&=\int_{\mathbb R}\bigl(q_t^{{(m)}}\bigr)''(y)\,dy
-\mathbb E[s_t^{{(m)}}(Y_t)^2]
=-\mathbb E[s_t^{{(m)}}(Y_t)^2],
\end{align*}
because $\bigl(q_t^{{(m)}}\bigr)'(y)\to0$ as
$|y|\to\infty$. Jensen's inequality therefore yields
\[
\mathbb E[|\partial_xs_t^{{(m)}}(Y_t)|^2]
\geq
\left|\mathbb E[\partial_xs_t^{{(m)}}(Y_t)]\right|^2
=
\left(\mathbb E[s_t^{{(m)}}(Y_t)^2]\right)^2.
\]

It therefore remains to lower bound
$\mathbb E[s_t^{{(m)}}(Y_t)^2]$. Rather than differentiating the
explicit expression for $q_t^{{(m)}}$ and attempting to control
the resulting cross terms pointwise, we use a variational consequence of the
score identity. For any smooth bounded function $f$,
\begin{align*}
\mathbb E[f'(Y_t)]
&=\int_{\mathbb R}f'(y)q_t^{{(m)}}(y)\,dy\\
&=-\int_{\mathbb R}f(y)
\bigl(q_t^{{(m)}}\bigr)'(y)\,dy\\
&=-\mathbb E[f(Y_t)s_t^{{(m)}}(Y_t)],
\end{align*}
where the boundary term vanishes because $q_t^{{(m)}}$ has
Gaussian tails. Hence, by Cauchy--Schwarz,
\begin{align*}
|\mathbb E[f'(Y_t)]|^2
&\leq \mathbb E[f(Y_t)^2]
\mathbb E[s_t^{{(m)}}(Y_t)^2],
\end{align*}
and therefore
\begin{align*}
\mathbb E[s_t^{{(m)}}(Y_t)^2]
\geq
\frac{|\mathbb E[f'(Y_t)]|^2}{\mathbb E[f(Y_t)^2]}
=
\frac{\left(\int q_t^{{(m)}}(y)f'(y)\,dy\right)^2}
{\int q_t^{{(m)}}(y)f(y)^2\,dy}.
\end{align*}

We want to choose $f$ so that the lower bound doesn't decrease, as $m$ increases. Take $f(y)=\cos(K_my)$.
Its derivative is $f'(y)=-K_m\sin(K_my)$, 
so the preceding inequality gives
\begin{align*}
\mathbb E[s_t^{{(m)}}(Y_t)^2]
\geq
\frac{K_m^2|\mathbb E[\sin(K_mY_t)]|^2}
{\mathbb E[\cos^2(K_mY_t)]} \geq
K_m^2|\mathbb E[\sin(K_mY_t)]|^2,
\end{align*}
where the last inequality uses $\cos^2(K_mY_t)\leq1$. By construction,
\begin{align*}
\mathbb E[\sin(K_mY_t)\mid Y_0=x]
={}&p_m(x)\sin(K_mx+tK_m\delta)
+p_m(x)\sin(K_mx-tK_m\delta)\\
&+\left(1-2p_m(x)\right)\sin(K_mx).
\end{align*}
Using $\sin(a+b)+\sin(a-b)=2\sin(a)\cos(b)$
we obtain
\begin{align*}
\mathbb E[\sin(K_mY_t)\mid Y_0=x]
&=2p_m(x)\sin(K_mx)\cos(tK_m\delta)
+\left(1-2p_m(x)\right)\sin(K_mx)\\
&=\sin(K_mx)
\left[
1+2p_m(x)\left(\cos(tK_m\delta)-1\right)
\right].
\end{align*}
Taking expectation with respect to $Y_0$ and using $\mathbb E[\sin(K_mY_0)]=0$ gives
\begin{align*}
\mathbb E[\sin(K_mY_t)]
&=\left(\cos(tK_m\delta)-1\right)
\mathbb E\left[2p_m(Y_0)\sin(K_mY_0)\right].
\end{align*}
Recall by definition $2p_m(Y_0)=\frac12+\frac14\sin(K_mY_0)$, so we have
\begin{align*}
\mathbb E\left[2p_m(Y_0)\sin(K_mY_0)\right] =\frac12\mathbb E[\sin(K_mY_0)]
+\frac14\mathbb E[\sin^2(K_mY_0)] =\frac14\mathbb E[\sin^2(K_mY_0)].
\end{align*}
Using $\sin^2(u)=\frac{1-\cos(2u)}{2}$
and $\mathbb E[\cos(aY_0)] =\int_{\mathbb R} \phi(x) \cos(ax) dx=e^{-a^2/2}$,
we obtain
\[
\mathbb E[\sin^2(K_mY_0)]
=\frac{1-e^{-2K_m^2}}{2}.
\]
Consequently,
\begin{align*}
\mathbb E[\sin(K_mY_t)]
=-\frac{\left(1-\cos(tK_m\delta)\right)\left(1-e^{-2K_m^2}\right)}{8}=-\frac{\left(1-\cos((2m+1)\pi t)\right)\left(1-e^{-2K_m^2}\right)}{8},
\end{align*}
where the second equality uses $K_m\delta=(2m+1)\pi$. Substituting this expression back yields
\begin{align*}
\mathbb E[s_t^{{(m)}}(Y_t)^2]
&\geq K_m^2|\mathbb E[\sin(K_mY_t)]|^2\\
&=
\frac{K_m^2(1-e^{-2K_m^2})^2}{64}
\left(1-\cos((2m+1)\pi t)\right)^2.
\end{align*}
It follows that 
\begin{align*}
\mathbb E[|\partial_xs_t^{{(m)}}(Y_t)|^2]
&\geq
\left(
\frac{K_m^2(1-e^{-2K_m^2})^2}{64}
\left(1-\cos((2m+1)\pi t)\right)^2
\right)^2\\
&=
\frac{K_m^4(1-e^{-2K_m^2})^4}{4096}
\left(1-\cos((2m+1)\pi t)\right)^4.
\end{align*}
Integrating over $t\in[0,1]$ therefore gives
\begin{align*}
\int_0^1\mathbb E[|\partial_xs_t^{{(m)}}(Y_t)|^2]\,dt
\geq
\frac{K_m^4(1-e^{-2K_m^2})^4}{4096}
\int_0^1
\left(1-\cos((2m+1)\pi t)\right)^4dt.
\end{align*}
It remains only to check that the oscillatory factor does not disappear after integration in time. Expanding,
\begin{align*}
(1-\cos u)^4
=
1-4\cos u+6\cos^2u-4\cos^3u+\cos^4u.
\end{align*}
With the change of variables $u=(2m+1)\pi t$,
\begin{align*}
\int_0^1
\left(1-\cos((2m+1)\pi t)\right)^4dt
=
\frac{1}{(2m+1)\pi}
\int_0^{(2m+1)\pi}(1-\cos u)^4du.
\end{align*}
Since $2m+1$ is an integer,
\begin{align*}
\int_0^{(2m+1)\pi}\cos u\,du&=0,\\
\int_0^{(2m+1)\pi}\cos^3u\,du&=0,\\
\int_0^{(2m+1)\pi}\cos^2u\,du&=\frac{(2m+1)\pi}{2},\\
\int_0^{(2m+1)\pi}\cos^4u\,du&=\frac{3(2m+1)\pi}{8}.
\end{align*}
Hence
\begin{align*}
\int_0^1
\left(1-\cos((2m+1)\pi t)\right)^4dt
=
1+6\cdot\frac12+\frac38
=
\frac{35}{8}.
\end{align*}
Consequently,
\begin{align*}
\int_0^1\mathbb E[|\partial_xs_t^{{(m)}}(Y_t)|^2]\,dt
\geq
\frac{35}{32768}
K_m^4(1-e^{-2K_m^2})^4.
\end{align*}
Recall that $D_5=\varepsilon$ for every $m$, whereas
\[
K_m=\frac{(2m+1)\pi}{\delta}\longrightarrow\infty.
\]
Therefore,
\begin{align*}
\int_0^1\mathbb E[|\partial_xs_t^{{(m)}}(Y_t)|^2]\,dt
\longrightarrow+\infty
\end{align*}
while $D_5$ remains fixed at the arbitrarily prescribed value $\varepsilon$. Thus, without smoothing, a small coupling displacement does not imply small integrated regularity. 

\end{proof}

%% file: Proof_Stochastic_Sampler.tex
\section[Proof of Theorem \ref{Theorem_Sampling_Error}]
{Proof of Theorem~\ref{Theorem_Sampling_Error}}
\label{Appendix_Sampling_Proof}

\phantomsection
\label{cond:standing_sampling}
\noindent\textbf{Standard Sampling Conditions.}
For the sampling analysis, we require that the exact score-corrected drift
\(b_t\) is jointly continuous on \([0,1]\times\mathbb R^d\), and that the
exact SDE initialized from \(q_0\),
\begin{align*}
dZ_t=b_t(Z_t)\,dt+\tau\,dB_t,
\end{align*}
admits a weak solution on \([0,1]\) satisfying
\(\mathcal L(Z_t)=q_t\) for every \(t\in[0,1]\). 
We also require each grid-frozen learned drift
\(\widehat b_{t_k}\) to be continuous.
The exact-process requirements are verified in
Lemma~\ref{lem:exact_sampling_conditions}.
Continuity of the learned drifts follows from that of the
neural networks.

\begin{lemma}[Sampling well-posedness]
\label{lem:exact_sampling_conditions}
Let $\sigma,\tau>0$ and $D_5<\infty$. Then $b_t(x)=v_t(x)+\frac{\tau^2}{2}s_t(x)$ is jointly continuous on $[0,1]\times\mathbb R^d$, and the exact SDE
\begin{align*}
dZ_t=b_t(Z_t)\,dt+\tau\,dB_t,\qquad Z_0\sim q_0,
\end{align*}
admits a weak solution on $[0,1]$ satisfying $\mathcal L(Z_t)=q_t$ for every $t\in[0,1]$.
\end{lemma}

\begin{proof}
Let $\varphi_\sigma$ be the density of $\mathcal N(0,\sigma^2I_d)$ and recall that $S_t=Y_0+tD$, where $D=Y_1-Y_0$. Then
\begin{align*}
q_t(x)=\mathbb E[\varphi_\sigma(x-S_t)],\qquad
q_t(x)v_t(x)=\mathbb E[D\varphi_\sigma(x-S_t)].
\end{align*}
Fix $(t_n,x_n)\to(t,x)$. Since $S_{t_n}\to S_t$ almost surely,
\begin{align*}
\varphi_\sigma(x_n-S_{t_n})&\to\varphi_\sigma(x-S_t),\\
\nabla_x\varphi_\sigma(x_n-S_{t_n})&\to\nabla_x\varphi_\sigma(x-S_t)
\end{align*}
almost surely. Both $\varphi_\sigma$ and $\nabla_x\varphi_\sigma$ are bounded, so dominated convergence gives joint continuity of $q_t(x)$ and $\nabla_xq_t(x)$. Moreover,
\begin{align*}
D\varphi_\sigma(x_n-S_{t_n})\to D\varphi_\sigma(x-S_t)
\end{align*}
almost surely and
\begin{align*}
\|D\varphi_\sigma(x_n-S_{t_n})\|
\leq \|\varphi_\sigma\|_\infty\|D\|.
\end{align*}
Since $D_5<\infty$ implies $\mathbb E\|D\|<\infty$, dominated convergence also gives joint continuity of $q_t(x)v_t(x)$.

Because $\varphi_\sigma>0$, we have $q_t(x)>0$ for every $(t,x)$. Hence
\begin{align*}
v_t(x)=\frac{\mathbb E[D\varphi_\sigma(x-S_t)]}{q_t(x)},
\qquad
s_t(x)=\frac{\nabla_xq_t(x)}{q_t(x)}
\end{align*}
are jointly continuous, and therefore so is $b_t$.

We next verify the Fokker--Planck equation satisfied by $(q_t)_{t\in[0,1]}$. For every $f\in C_c^\infty(\mathbb R^d)$,
\begin{align*}
\frac{d}{dt}\int_{\mathbb R^d}f(x)q_t(x)\,dx
&=\frac{d}{dt}\mathbb E[f(X_t)]\\
&=\mathbb E[\nabla_xf(X_t)^\top D]\\
&=\int_{\mathbb R^d}\nabla_xf(x)^\top v_t(x)q_t(x)\,dx,
\end{align*}
where differentiation under the expectation is justified by
$\|\nabla_xf(X_t)^\top D\|\leq\|\nabla_xf\|_\infty\|D\|$.
Since $q_ts_t=\nabla_xq_t$, integration by parts gives
\begin{align*}
\int_{\mathbb R^d}\nabla_xf(x)^\top s_t(x)q_t(x)\,dx
=-\int_{\mathbb R^d}\Delta_xf(x)q_t(x)\,dx.
\end{align*}
Thus
\begin{align*}
\frac{d}{dt}\int_{\mathbb R^d}f(x)q_t(x)\,dx
=
\int_{\mathbb R^d}
\left(
\nabla_xf(x)^\top b_t(x)
+\frac{\tau^2}{2}\Delta_xf(x)
\right)q_t(x)\,dx.
\end{align*}

It remains to verify the integrability required by the superposition principle. By conditional Jensen's inequality,
\begin{align*}
\int_{\mathbb R^d}\|v_t(x)\|q_t(x)\,dx
&=\mathbb E\|\mathbb E[D\mid X_t]\|
\leq\mathbb E\|D\|,
\end{align*}
while the Gaussian score identity gives
\begin{align*}
\int_{\mathbb R^d}\|s_t(x)\|q_t(x)\,dx
&=\frac{1}{\sigma}\mathbb E\|\mathbb E[G\mid X_t]\|
\leq\frac{\mathbb E\|G\|}{\sigma}.
\end{align*}
Consequently,
\begin{align*}
\int_0^1\int_{\mathbb R^d}\|b_t(x)\|q_t(x)\,dx\,dt
\leq
\mathbb E\|D\|+\frac{\tau^2}{2\sigma}\mathbb E\|G\|<\infty.
\end{align*}
Moreover, $t\mapsto q_t(x)\,dx$ is narrowly continuous: for every bounded continuous $f$, almost-sure continuity of $t\mapsto X_t$ and dominated convergence imply that $t\mapsto\mathbb E[f(X_t)]$ is continuous.

Therefore, \cite[Theorem~2.5]{trevisan2016wellposedness} yields a probability measure on $C([0,1];\mathbb R^d)$ solving the martingale problem associated with drift $b_t$ and diffusion matrix $\tau^2I_d$, with time-$t$ marginal $q_t$. For the canonical process $Z$, its martingale part has quadratic covariation $\tau^2tI_d$; hence, by the Lévy characterization, it can be written as $\tau B_t$ for a Brownian motion $B$. Thus $Z$ is a weak solution of the displayed SDE and satisfies $\mathcal L(Z_t)=q_t$ for every $t\in[0,1]$.
\end{proof}

\subsection{Path-Space Comparison via Girsanov}

We now introduce a framework \cite{chen2023sampling}, \cite{chen2023the} that controls the KL divergence between the laws of two SDEs.

Let $T > 0$ (for our flow matching setting, $T=1$). Let $0=t_0<t_1<\cdots<t_N=T$ be a uniform time grid with step size $h=t_{k+1}-t_k$. 
For every \(t\in[t_k,t_{k+1})\), define $\eta_h(t)=t_k$. Recall that our sampling SDE is written as $dZ_t=b_t(Z_t)\,dt+\tau\,dB_t$. For the continuous-time Euler discretization, we have 
\begin{align*}
d\widehat Z_t
&=
\widehat b_{\eta_h(t)}(\widehat Z_{\eta_h(t)})\,dt
+
\tau\,dB_t.
\end{align*}
Equivalently, for \(t\in[t_k,t_{k+1})\),
\begin{align*}
d\widehat Z_t
&=
\widehat b_{t_k}(\widehat Z_{t_k})\,dt
+
\tau\,dB_t.
\end{align*}
Assume both processes start from the same initial law: $Z_0\sim \rho_0$ and $\widehat Z_0 \sim \rho_0$.

\begin{lemma}[Girsanov Bound for the Euler Process via Localization]
\label{Lemma_Girsanov_Euler_KL}
Let \(\tau>0\), and let \(b\) and
\(\{\widehat b_{t_k}\}_{k=0}^{N-1}\) be continuous drift fields.
Fix a weak solution \(Z\) of the exact SDE on \([0,T]\),
with \(Z_0\sim\rho_0\), and let \(\widehat Z\) be the full
continuous-time Euler process initialized from \(\rho_0\).
Write
\(\mathbb P=\mathcal L(Z)\) and
\(\widehat{\mathbb P}=\mathcal L(\widehat Z)\)
for their laws on \(C([0,T];\mathbb R^d)\).


Define
$\Delta_t(Z)=\widehat b_{\eta_h(t)}(Z_{\eta_h(t)})-b_t(Z_t)$, and assume that $\int_0^T \mathbb E\left[\|\Delta_t(Z)\|^2\,\right] dt<\infty$. Also assume \begin{align*}
    \int_{0}^{T} \| \Delta_{t} (\widehat{Z})\|^{2} dt < \infty, \quad \widehat{\mathbb P}\text{-a.s.}
\end{align*}
Then
    \begin{align*}
\operatorname{KL}
\left(
\mathcal L(Z_T)
\|
\mathcal L(\widehat Z_T)
\right) \leq
\frac{1}{2\tau^2}
\int_0^T
\mathbb E
\left[
\left\|
b_t(Z_t)
-
\widehat b_{\eta_h(t)}(Z_{\eta_h(t)})
\right\|^2
\right]dt.
\end{align*}
\end{lemma}

The proof has four steps. First, we define a path-dependent stopping-time
construction \(\zeta_m(x)\), which gives the stopping times
\(\zeta_m(Z)\) along the exact process and \(\zeta_m(\widehat Z)\) along
the full Euler process. Second, for \(\zeta_m(Z)\), Novikov's condition
holds for each localized approximation, so we can apply Girsanov's theorem
and obtain a KL bound for each of them. Third, the additional integrability condition on $\widehat Z$ is used to show that \(\zeta_m(\widehat Z)\uparrow T\). 
Finally, we show that the localized Euler laws converge
to the full Euler law on every truncated interval \([0,T-\epsilon]\); lower
semicontinuity of relative entropy and then \(\epsilon\downarrow0\) complete
the argument.

\begin{proof}

Let $(\Omega,\mathcal F,(\mathcal F_t)_{0\leq t\leq T},\mathbf P)$
be a filtered probability space carrying the exact weak solution \(Z\) discussed in Lemma \ref{lem:exact_sampling_conditions} and
its driving Brownian motion \(W\), so that
\begin{align*}
dZ_t=b_t(Z_t)\,dt+\tau\,dW_t.
\end{align*}
Let
\(\mathbb P=\mathcal L_{\mathbf P}(Z)\) be the corresponding path law on
\(C([0,T];\mathbb R^d)\), and let \(\widehat{\mathbb P}\) denote the path
law of the full continuous-time Euler process.

For every integer \(m\geq 1\), define the stopping time \begin{align*} \zeta_m (x) &= \inf \left\{ t\in[0,T]: \frac{1}{\tau^2} \int_0^t \|\Delta_s(x)\|^2ds \geq m \right\} \wedge T. \end{align*}
Since
\(
t\mapsto \tau^{-2}\int_0^t\|\Delta_s(Z)\|^2ds
\)
is adapted and continuous, its first hitting time of \([m,\infty)\) is an
\((\mathcal F_t)\)-stopping time. Thus \(\zeta_m\) is a stopping time.
When no confusion arises, we write \(\zeta_m\) for \(\zeta_m(Z)\) in the Girsanov construction, and for \(\zeta_m(\widehat Z)\) in the synchronous coupling that will be discussed later.
By construction, $\frac{1}{\tau^2} \int_0^{\zeta_m} \|\Delta_t(Z)\|^2dt = \frac{1}{\tau^2} \int_0^T \| \Delta_t(Z)\mathbf 1_{\{t\leq \zeta_m\}} \|^2dt \leq m$. Therefore Novikov's condition holds for the stopped drift perturbation: \begin{align*} \mathbb E_{\mathbf P} \left[ \exp\left( \frac{1}{2\tau^2} \int_0^T \left\| \Delta_t(Z)\mathbf 1_{\{t\leq \zeta_m\}} \right\|^2dt \right) \right] &\leq \exp\left( \frac{m}{2} \right) < \infty. \end{align*} 
Hence the stopped stochastic exponential \begin{align*} 
M_{m}^{T} := \exp\left( \frac{1}{\tau} \int_0^T \left( \Delta_t(Z)\mathbf 1_{\{t\leq \zeta_m\}} \right)^\top dW_t - \frac{1}{2\tau^2} \int_0^T \left\| \Delta_t(Z)\mathbf 1_{\{t\leq \zeta_m\}} \right\|^2dt \right) \end{align*} 
is a true martingale with expectation $\mathbb E_{\mathbf P}[M_{m}^{T}] = \mathbb E_{\mathbf P}[M_{m}^{0}] = 1$. 

Thus it defines a probability measure \(\widehat{\mathbf P}^{(m)}\) on
\((\Omega,\mathcal F)\) by $\frac{d\widehat{\mathbf P}^{(m)}}{d\mathbf P}=M_m^T$ (note that $\int d \widehat{\mathbf P}^{(m)} = \int M_{m}^{T} d\mathbf P = \mathbb E_{\mathbf P}[M_{m}^{T}] = 1 $ so $\widehat{\mathbf P}^{(m)}$ is a probability measure).
By Girsanov's theorem, under \(\widehat{\mathbf P}^{(m)}\),
\begin{align*}
W_t^{(m)}
=
W_t-\frac1\tau\int_0^t
\Delta_s(Z)\mathbf 1_{\{s\leq\zeta_m\}}\,ds
\end{align*}
is a Brownian motion.

Therefore, \begin{align*} dW_t &= dW_t^{(m)} + \frac{1}{\tau} \Delta_t(Z)\mathbf 1_{\{t\leq \zeta_m\}}dt. \end{align*} Since, under \(\mathbf P\), \begin{align*} dZ_t &= b_t(Z_t)dt+\tau dW_t, \end{align*} we may rewrite the same process under \(\widehat{\mathbf P}^{(m)}\) as 
\begin{align*} 
dZ_t &= b_t(Z_t)dt + \tau \left( dW_t^{(m)} + \frac{1}{\tau} \Delta_t(Z)\mathbf 1_{\{t\leq \zeta_m\}}dt \right) \\ &= \left( b_t(Z_t) + \Delta_t(Z)\mathbf 1_{\{t\leq \zeta_m\}} \right)dt + \tau dW_t^{(m)}. 
\end{align*} 

Thus, before the stopping time \(\zeta_m\), the drift is \begin{align*} b_t(Z_t)+\Delta_t(Z) &= \widehat b_{\eta_h(t)}(Z_{\eta_h(t)}), \end{align*} and after the stopping time, the drift returns to \(b_t(Z_t)\). Define
$\widehat{\mathbb P}^{(m)}
:=
\mathcal L_{\widehat{\mathbf P}^{(m)}}(Z)$.
Thus \(\widehat{\mathbb P}^{(m)}\) is the path law of the localized Euler
approximation. For each \(m\), the path-space KL divergence is well-defined. Since \begin{align*}
\log\frac{d\mathbf P}{d\widehat{\mathbf P}^{(m)}}
=
-\frac1\tau
\int_0^T
\left(
\Delta_t(Z)\mathbf 1_{\{t\leq\zeta_m\}}
\right)^\top dW_t+
\frac1{2\tau^2}
\int_0^T
\left\|
\Delta_t(Z)\mathbf 1_{\{t\leq\zeta_m\}}
\right\|^2dt.
\end{align*} 

Since \(Z\) is continuous and adapted, \(b_t(Z_t)\) is predictable, while
the grid-frozen term admits a predictable modification differing only at the
deterministic grid points. Hence
\(\Delta_t(Z)\mathbf 1_{\{t\leq\zeta_m\}}\) admits a predictable version.
Moreover, by the definition of \(\zeta_m\),
\[
\int_0^T
\|\Delta_t(Z)\mathbf 1_{\{t\leq\zeta_m\}}\|^2dt
\leq m\tau^2,
\]
so it is square-integrable and the stochastic integral has mean zero.

Taking expectation under \(\mathbf P\), we obtain \begin{align*}
\operatorname{KL}
\left(
\mathbf P
\middle\|
\widehat{\mathbf P}^{(m)}
\right)
&=
\frac1{2\tau^2}
\mathbb E_{\mathbf P}
\left[
\int_0^{\zeta_m}
\|\Delta_t(Z)\|^2dt
\right]\\
&\leq
\frac1{2\tau^2}
\mathbb E_{\mathbf P}
\left[
\int_0^T
\|\Delta_t(Z)\|^2dt
\right].
\end{align*}
Consequently, by the data-processing inequality under the map
\(\omega\mapsto Z(\omega)\),
\begin{align*}
\operatorname{KL}
\left(
\mathbb P
\middle\|
\widehat{\mathbb P}^{(m)}
\right)
\leq
\operatorname{KL}
\left(
\mathbf P
\middle\|
\widehat{\mathbf P}^{(m)}
\right).
\end{align*}
Therefore, 
\begin{align*} \operatorname{KL} \left( \mathbb P \middle\| \widehat{\mathbb P}^{(m)} \right) &\le \frac{1}{2\tau^2} \mathbb E_{\mathbf P} \left[ \int_0^T \left\| \Delta_t(Z)\mathbf 1_{\{t\leq \zeta_m\}} \right\|^2dt \right] = \frac{1}{2\tau^2} \mathbb E_{\mathbf P} \left[ \int_0^{\zeta_m} \|\Delta_t(Z)\|^2dt \right] \\ &\leq \frac{1}{2\tau^2} \mathbb E_{\mathbf P} \left[ \int_0^T \|\Delta_t(Z)\|^2dt \right]. 
\end{align*}

It remains to pass from the localized Euler approximation to the full continuous-time Euler approximation. For each fixed \(m\), since \(M_m^t\) is a martingale with \(M_m^0=1\),
\[
\mathbb E_{\mathbf P}[M_m^T\mid\mathcal F_0]=1,
\]
so the change of measure preserves the initial law and \(Z_0\sim\rho_0\) under \(\widehat{\mathbf P}^{(m)}\).

Under \(\widehat{\mathbf P}^{(m)}\), on the same probability space, construct the continuous time Euler process
\(\widehat Z\) using the same initial value \(Z_0\) and Brownian motion
\(W^{(m)}\):
\begin{align*}
\widehat Z_t
&=
\widehat Z_{t_k}
+(t-t_k)\widehat b_{t_k}(\widehat Z_{t_k})
+\tau\left(W_t^{(m)}-W_{t_k}^{(m)}\right),
\qquad t\in[t_k,t_{k+1}],
\end{align*}
with \(\widehat Z_0=Z_0\).
Because \(W^{(m)}\) is Brownian with respect to
\((\mathcal F_t)\) and \(Z_0\) is \(\mathcal F_0\)-measurable,
\(W^{(m)}\) is independent of \(Z_0\). The change of measure preserves the
law \(\rho_0\) of \(Z_0\), and the grid-by-grid Euler recursion is a
measurable function of \((Z_0,W^{(m)})\). Hence
$\mathcal L_{\widehat{\mathbf P}^{(m)}}(\widehat Z)
=
\widehat{\mathbb P}$.

Let \(\rho_m=\zeta_m(Z)\wedge\zeta_m(\widehat Z)\). Before \(\rho_m\), both \(Z\) and \(\widehat Z\) satisfy the same continuous-time Euler update with the same initial value and Brownian motion, so they agree grid-by-grid on \([0,\rho_m]\). Hence \(\int_0^t\|\Delta_s(Z)\|^2ds=\int_0^t\|\Delta_s(\widehat Z)\|^2ds\) for every \(t\leq\rho_m\).

Suppose \(\zeta_m(Z)<\zeta_m(\widehat Z)\). Then \(\rho_m=\zeta_m(Z)<T\), so continuity of the defining integral gives \(\tau^{-2}\int_0^{\rho_m}\|\Delta_s(Z)\|^2ds=m\). The preceding equality implies the same for \(\widehat Z\), hence \(\zeta_m(\widehat Z)\leq\rho_m=\zeta_m(Z)\), a contradiction. The reverse strict inequality is ruled out in the same way. Therefore \(\zeta_m(Z)=\zeta_m(\widehat Z)\).

Fix \(\varepsilon\in(0,T)\) and define the path truncation map \(\pi_\varepsilon:C([0,T];\mathbb R^d)\to C([0,T];\mathbb R^d)\) by
\[
\left(\pi_\varepsilon(\omega)\right)(t)=\omega(t\wedge(T-\varepsilon)),\qquad 0\leq t\leq T.
\]
On the event \(\{\zeta_m(Z)>T-\varepsilon\}\), we have $\pi_\varepsilon(Z)=\pi_\varepsilon(\widehat Z).$
Moreover,
\begin{align*}
\widehat{\mathbf P}^{(m)}\left(\zeta_m(Z)\leq T-\varepsilon\right)
&=\widehat{\mathbf P}^{(m)}\left(\zeta_m(\widehat Z)\leq T-\varepsilon\right)=\widehat{\mathbb P}\left(\zeta_m(\widehat Z)\leq T-\varepsilon\right)\longrightarrow0 \text{ as } m \to \infty.
\end{align*}
Indeed, \(\int_0^T\|\Delta_t(\widehat Z)\|^2dt<\infty\), \(\widehat{\mathbb P}\)-almost surely, implies that \(\zeta_m(\widehat Z)=T\) for all sufficiently large \(m\), almost surely.

Therefore, for every bounded continuous \(F:C([0,T];\mathbb R^d)\to\mathbb R\),
\begin{align*}
&\left|
\mathbb E_{\widehat{\mathbf P}^{(m)}}\left[F(\pi_\varepsilon(Z))\right]
-\mathbb E_{\widehat{\mathbf P}^{(m)}}\left[F(\pi_\varepsilon(\widehat Z))\right]
\right|\\
&\qquad\leq
2\|F\|_\infty\,
\widehat{\mathbf P}^{(m)}\left(\zeta_m(Z)\leq T-\varepsilon\right)
\longrightarrow0.
\end{align*}
Since
\(\mathcal L_{\widehat{\mathbf P}^{(m)}}(\widehat Z)
=\widehat{\mathbb P}\), we obtain
\begin{align*}
\mathcal L_{\widehat{\mathbf P}^{(m)}}(\pi_\varepsilon(Z))
\overset{w}{\longrightarrow}
(\pi_\varepsilon)_\#\widehat{\mathbb P}.
\end{align*}
By lower semicontinuity of KL divergence and the data-processing inequality,
\begin{align*}
\operatorname{KL}
\left(
(\pi_\varepsilon)_\#\mathbb P
\middle\|
(\pi_\varepsilon)_\#\widehat{\mathbb P}
\right)
&\leq
\liminf_{m\to\infty}
\operatorname{KL}
\left(
\mathcal L_{\mathbf P}(\pi_\varepsilon(Z))
\middle\|
\mathcal L_{\widehat{\mathbf P}^{(m)}}(\pi_\varepsilon(Z))
\right)\\
&\leq
\liminf_{m\to\infty}
\operatorname{KL}
\left(
\mathbf P
\middle\|
\widehat{\mathbf P}^{(m)}
\right)\\
&\leq
\frac1{2\tau^2}
\mathbb E_{\mathbf P}
\left[
\int_0^T\|\Delta_t(Z)\|^2dt
\right].
\end{align*}

Finally, since $\mathcal{L} (\pi_{\varepsilon}(Z)) \overset{w}{\to} \mathcal{L} (Z)$ and $\mathcal{L} (\pi_{\varepsilon}(\widehat{Z})) \overset{w}{\to} \mathcal{L} (\widehat{Z}) $ as $\varepsilon \to 0$,  applying lower semicontinuity of KL divergence once more gives 
\begin{align*} \operatorname{KL} \left( \mathbb P \middle\| \widehat{\mathbb P} \right) \leq \liminf_{\varepsilon\downarrow0} \operatorname{KL} \left( \mathcal{L} (\pi_{\varepsilon}(Z)) \middle\| \mathcal{L} (\pi_{\varepsilon}(\widehat{Z})) \right) \leq \frac{1}{2\tau^2} \mathbb E_{\mathbf P} \left[ \int_0^T \|\Delta_t(Z)\|^2dt \right]. 
\end{align*} 
Again, by the data-processing inequality, 
\begin{align*} 
\operatorname{KL} \left( \mathcal L(Z_T) \middle\| \mathcal L(\widehat Z_T) \right) 
\leq \operatorname{KL} \left( \mathbb P \middle\| \widehat{\mathbb P} \right) 
\leq \frac{1}{2\tau^2} \mathbb E_{\mathbf P} \left[ \int_0^T \|\Delta_t(Z)\|^2dt \right]. 
\end{align*} Equivalently, \begin{align*} \operatorname{KL} \left( \mathcal L(Z_T) \middle\| \mathcal L(\widehat Z_T) \right) &\leq \frac{1}{2\tau^2} \int_0^T \mathbb E_{\mathbf P} \left[ \left\| b_t(Z_t) - \widehat b_{\eta_h(t)}(Z_{\eta_h(t)}) \right\|^2 \right]dt. \end{align*}

\end{proof}

\paragraph{A point concerning the localization limit.}
The approximation argument in \citet[Section~5.2]{chen2023sampling} places the localized and full Euler processes on a common probability space and subsequently invokes convergence of the localizing times. The Lipschitz-SDE result used there, cited as Corollary~8.5 in \citet[Section~5.2]{chen2023sampling}, comes from the chapter on stochastic differential equations in \citet{legall2016stochasticcalc} and provides a measurable strong-solution representation. Although this result can support a synchronous realization when its hypotheses hold, it does not by itself show that localizing times constructed under the reference law converge to \(T\) under the full Euler law. This Euler-side exhaustion is required to obtain agreement of the localized and full Euler paths on every interval \([0,T-\varepsilon]\), and hence the weak convergence of the localized laws. We avoid this issue by using the explicit energy stopping times \(\zeta_m\) defined above. It remains to verify that

$$
    \int_0^T \bigl\|\Delta_t(\widehat Z)\bigr\|^2\,dt<\infty,
    \qquad \widehat{\mathbb P}\text{-a.s.},
$$

because this condition implies that \(\zeta_m(\widehat Z)\uparrow T\) almost surely; indeed, \(\zeta_m(\widehat Z)=T\) for all sufficiently large \(m\). The following proposition establishes this property under our continuity assumptions.

\begin{proposition}[Almost-Sure Finiteness of the Pathwise Drift Mismatch]
\label{Prop_Pathwise_Drift_Finiteness}
Let \(0=t_0<t_1<\cdots<t_N=T\) be a finite time grid, and define
\(\eta_h(t)=t_k\) for \(t\in[t_k,t_{k+1})\). Let \(\widehat Z=(\widehat Z_t)_{0\leq t\leq T}\) be a continuous-path process under \(\widehat{\mathbb P}\).

Let \(b:[0,T]\times\mathbb R^d\to\mathbb R^d\) be continuous. For every \(k\in\{0,\ldots,N-1\}\), let \(\widehat b_{t_k}:\mathbb R^d\to\mathbb R^d\) be continuous. Define
\begin{align*}
\Delta_t(\widehat Z)
&=
\widehat b_{\eta_h(t)}(\widehat Z_{\eta_h(t)})
- b_t(\widehat Z_t),
\qquad
0\leq t\leq T.
\end{align*}
Then
\begin{align*}
\int_0^T
\left\|
\Delta_t(\widehat Z)
\right\|^2dt
&<
\infty,
\qquad
\widehat{\mathbb P}\text{-a.s.}
\end{align*}
\end{proposition}

\begin{proof}

It suffices to show, there exists some $\Omega_{0}$ with $\widehat{ \mathbb P} (\Omega_{0}) = 1$ s.t. $\sup_{t \in [0, T]}\|\Delta_{t}(\widehat Z(\omega))\| < \infty, \forall \omega \in \Omega_{0}$.
Since \(\widehat Z\) has continuous paths under \(\widehat{\mathbb P}\), there exists \(\Omega_0\) such that $\widehat{\mathbb P}(\Omega_0)=1$ and for every \(\omega\in\Omega_0\), the map \(t\mapsto \widehat Z_t(\omega)\) is continuous on \([0,T]\).

Fix \(\omega\in\Omega_0\). Since \([0,T]\) is compact and \(t\mapsto \widehat Z_t(\omega)\) is continuous, the image set $K_\omega := \left\{\widehat Z_t(\omega):0\leq t\leq T\right\}$
is a compact subset of \(\mathbb R^d\).

Because \(b\) and $\hat{b}_{t_{k}}$ are continuous, and because \([0,T]\times K_\omega\) is compact, we have
\begin{align*}
M_b(\omega)
&:=
\sup_{0\leq t\leq T}
\sup_{x\in K_\omega}
\left\|
b_t(x)
\right\|
<
\infty, \\
M_{\widehat b,k}(\omega) &:=\sup_{x\in K_\omega}\left\|\widehat b_{t_k}(x) \right\|<\infty, \\
M_{\widehat b}(\omega) &:=
\max_{0\leq k\leq N-1}
M_{\widehat b,k}(\omega)
<
\infty.
\end{align*}

Therefore $\left\|b_t(\widehat Z_t(\omega))
\right\|\leq M_b(\omega), \forall t\in[0,T]$ and $\left\|\widehat b_{\eta_h(t)}\left(\widehat Z_{\eta_h(t)}(\omega)\right)\right\|\leq M_{\widehat b}(\omega)$.
Hence, for every \(t\in[0,T]\),
\begin{align*}
\left\|
\Delta_t(\widehat Z(\omega))
\right\|
&=
\left\|
b_t(\widehat Z_t(\omega))
-
\widehat b_{\eta_h(t)}
\left(
\widehat Z_{\eta_h(t)}(\omega)
\right)
\right\| 
\leq \left\| b_t(\widehat Z_t(\omega)) \right\|
+ \left\| \widehat b_{\eta_h(t)}
\left(
\widehat Z_{\eta_h(t)}(\omega)
\right)
\right\|
\\
&\leq
M_b(\omega)+M_{\widehat b}(\omega).
\end{align*}
Therefore,
\begin{align*}
\int_0^T
\left\|
\Delta_t(\widehat Z(\omega))
\right\|^2dt
&\leq
\int_0^T
\left(
M_b(\omega)+M_{\widehat b}(\omega)
\right)^2dt = T
\left(
M_b(\omega)+M_{\widehat b}(\omega)
\right)^2 < \infty.
\end{align*}
Since this holds for every \(\omega\in\Omega_0\), and since \(\widehat{\mathbb P}(\Omega_0)=1\), we conclude that
\begin{align*}
\int_0^T
\left\|
\Delta_t(\widehat Z)
\right\|^2dt
&<
\infty,
\qquad
\widehat{\mathbb P}\text{-a.s.}
\end{align*}
This proves the claim.
\end{proof}

\subsection{Posterior Derivative Identities}

This subsection collects the derivative identities and integrated bounds needed for the combined It\^{o} analysis of the score increment
\begin{align*}
s_t(Z_t)-s_{t_k}(Z_{t_k}).
\end{align*}
The key point is that the term
\begin{align*}
\nabla_x(\nabla_x\cdot v_t)
\end{align*}
requires control of the full Hessian tensor \(D_x^2v_t\), not merely the componentwise Laplacian \(\Delta_xv_t\).

\begin{lemma}[Score Identity Under Gaussian Smoothing]\label{Lemma_Score_Identity}
Let \(Y_0\) and \(Y_1\) be random variables taking values in \(\mathbb R^d\), and fix \(t\in[0,T]\). Let \(D=Y_1-Y_0\). Define \(S_t=Y_0+tD\). Let \(G\sim {\mathcal N}(0,I_d)\) be independent of \((Y_0,D)\), and define \(X_t=S_t+\sigma G\) with \(\sigma>0\).

Let \(q_t\) be the density of \(X_t\), and define the score \(s_t:\mathbb R^d\to\mathbb R^d\) by \(s_t(x)=\nabla_x\log q_t(x)\). Then
\begin{align*}
s_t(x)
&=
-\frac{1}{\sigma}
\mathbb E
\left[
G\mid X_t=x
\right].
\end{align*}
\end{lemma}

\begin{proof}
Let \(\mathsf P_{Y_0,Y_1}\) be the law of \((Y_0,Y_1)\). For fixed \(t\in[0,T]\), define
\begin{align*}
s_t^{\mathrm{path}}(y_0,y_1)
&=
(1-t)y_0+ty_1.
\end{align*}
The superscript ``path'' is only used in this proof to avoid confusing the interpolation point \(s_t^{\mathrm{path}}(y_0,y_1)\) with the score \(s_t(x)\).

Since \(X_t=S_t+\sigma G\), conditional on \((Y_0,Y_1)=(y_0,y_1)\), the random variable \(X_t\) is Gaussian with mean \(s_t^{\mathrm{path}}(y_0,y_1)\) and covariance matrix \(\sigma^2 I_d\). Therefore the conditional density of \(X_t=x\) given \((Y_0,Y_1)=(y_0,y_1)\) is
\begin{align*}
\frac{1}{(2\pi\sigma^2)^{d/2}}
\exp\left(
-\frac{
\left\|
x-s_t^{\mathrm{path}}(y_0,y_1)
\right\|^2
}{
2\sigma^2
}
\right).
\end{align*}
Hence the marginal density of \(X_t\) is
\begin{align*}
q_t(x)
&=
\int_{\mathbb R^d\times\mathbb R^d}
\frac{1}{(2\pi\sigma^2)^{d/2}}
\exp\left(
-\frac{
\left\|
x-s_t^{\mathrm{path}}(y_0,y_1)
\right\|^2
}{
2\sigma^2
}
\right)
\,d\mathsf P_{Y_0,Y_1}(y_0,y_1).
\end{align*}
Because the Gaussian density is strictly positive, \(q_t(x)>0\) for every \(x\in\mathbb R^d\).

We now compute \(\nabla_x q_t(x)\). For each coordinate \(j\in\{1,\ldots,d\}\),
\begin{align*}
\partial_{x_j}
\exp\left(
-\frac{
\left\|
x-s_t^{\mathrm{path}}(y_0,y_1)
\right\|^2
}{
2\sigma^2
}
\right)
&=
-\frac{
x_j-s_{t,j}^{\mathrm{path}}(y_0,y_1)
}{
\sigma^2
}
\exp\left(
-\frac{
\left\|
x-s_t^{\mathrm{path}}(y_0,y_1)
\right\|^2
}{
2\sigma^2
}
\right).
\end{align*}
Therefore,
\begin{align*}
\partial_{x_j}q_t(x)
&=
\int_{\mathbb R^d\times\mathbb R^d}
-\frac{
x_j-s_{t,j}^{\mathrm{path}}(y_0,y_1)
}{
\sigma^2
}
\frac{1}{(2\pi\sigma^2)^{d/2}}
\exp\left(
-\frac{
\left\|
x-s_t^{\mathrm{path}}(y_0,y_1)
\right\|^2
}{
2\sigma^2
}
\right)
\,d\mathsf P_{Y_0,Y_1}(y_0,y_1).
\end{align*}
Since \(q_t(x)>0\),
\begin{align*}
\partial_{x_j}\log q_t(x)
&=
\frac{\partial_{x_j}q_t(x)}{q_t(x)}
\\
&=
-\frac{1}{\sigma^2}
\frac{
\int_{\mathbb R^d\times\mathbb R^d}
\left(
x_j-s_{t,j}^{\mathrm{path}}(y_0,y_1)
\right)
\frac{1}{(2\pi\sigma^2)^{d/2}}
\exp\left(
-\frac{
\left\|
x-s_t^{\mathrm{path}}(y_0,y_1)
\right\|^2
}{
2\sigma^2
}
\right)
\,d\mathsf P_{Y_0,Y_1}(y_0,y_1)
}{
q_t(x)
}.
\end{align*}

By Bayes' rule, the conditional law of \((Y_0,Y_1)\) given \(X_t=x\) is obtained by weighting the prior law \(\mathsf P_{Y_0,Y_1}\) by the same Gaussian likelihood. Thus
\begin{align*}
\partial_{x_j}\log q_t(x)
&=
-\frac{1}{\sigma^2}
\mathbb E
\left[
x_j-S_{t,j}
\mid X_t=x
\right].
\end{align*}
Recall by definition $x = S_t+\sigma G$, so we have $x_j-S_{t,j}=\sigma G_j$ inside the conditional expectation. Therefore,
\begin{align*}
\partial_{x_j}\log q_t(x) = -\frac{1}{\sigma^2}
\mathbb E
\left[
\sigma G_j
\mid X_t=x
\right] =
-\frac{1}{\sigma}
\mathbb E
\left[
G_j
\mid X_t=x
\right].
\end{align*}
Since this holds for every \(j\in\{1,\ldots,d\}\), we obtain
\begin{align*}
\nabla_x\log q_t(x)
&=
-\frac{1}{\sigma}
\mathbb E
\left[
G\mid X_t=x
\right].
\end{align*}
By definition \(s_t(x)=\nabla_x\log q_t(x)\), so
\begin{align*}
s_t(x)
&=
-\frac{1}{\sigma}
\mathbb E
\left[
G\mid X_t=x
\right].
\end{align*}
This proves the lemma.
\end{proof}

\begin{lemma}[Posterior Derivative Identities for the Smoothed Velocity and Score]\label{Lemma_Derivative_Identity}
Let \(Y_0\) and \(Y_1\) be random variables taking values in \(\mathbb R^d\), and fix \(t\in[0,T]\). Let \(D=Y_1-Y_0\). Define \(S_t=Y_0+tD\). Let \(G\sim {\mathcal N}(0,I_d)\) be independent of \((Y_0,D)\), and define \(X_t=S_t+\sigma G\) with \(\sigma>0\).
Assume \(\mathbb E\|D\|^3<\infty\).

Let \(q_t\) be the density of \(X_t\). Define the smoothed velocity \(v_t:\mathbb R^d\to\mathbb R^d\) by
\begin{align*}
v_t(x)
&=
\mathbb E
\left[
D\mid X_t=x
\right],
\end{align*}
and define the score \(s_t:\mathbb R^d\to\mathbb R^d\) by \(s_t(x)=\nabla_x\log q_t(x)\).
Under the stated moment condition, all derivatives below exist and the
required differentiations under the Gaussian-mixture integral are
justified.

Then the score satisfies
\begin{align*}
s_t(x)
&=
-\frac{1}{\sigma}
\mathbb E
\left[
G\mid X_t=x
\right].
\end{align*}
The velocity derivatives satisfy
\begin{align*}
D_xv_t(x)
&=
-\frac{1}{\sigma}
\operatorname{Cov}
\left(
D,G\mid X_t=x
\right),
\\
\partial_t v_t(x)
&=
\frac{1}{\sigma}
\operatorname{Cov}
\left(
D,D^\top G\mid X_t=x
\right),
\end{align*}
and for every \(i,j,\ell\in\{1,\ldots,d\}\),
\begin{align*}
\partial_{x_jx_\ell}v_{t,i}(x)
&=
\frac{1}{\sigma^2}
\mathbb E
\left[
\left(
D_i-\mathbb E[D_i\mid X_t=x]
\right)
\left(
G_j-\mathbb E[G_j\mid X_t=x]
\right)
\left(
G_\ell-\mathbb E[G_\ell\mid X_t=x]
\right)
\mid X_t=x
\right].
\end{align*}
In particular,
\begin{align*}
\Delta_xv_t(x)
&=
\frac{1}{\sigma^2}
\mathbb E
\left[
\left(
D-\mathbb E[D\mid X_t=x]
\right)
\left\|
G-\mathbb E[G\mid X_t=x]
\right\|^2
\mid X_t=x
\right].
\end{align*}

The score derivatives satisfy
\begin{align*}
D_xs_t(x)
&=
\frac{1}{\sigma^2}
\left(
\operatorname{Cov}
\left(
G,G\mid X_t=x
\right)
-
I_d
\right),
\end{align*}
and
\begin{align*}
\Delta_xs_t(x)
&=
-\frac{1}{\sigma^3}
\mathbb E
\left[
\left(
G-\mathbb E[G\mid X_t=x]
\right)
\left\|
G-\mathbb E[G\mid X_t=x]
\right\|^2
\mid X_t=x
\right].
\end{align*}
\end{lemma}

\begin{proof}

We first record the posterior derivative identity used throughout the
proof. Let
\(F:\mathbb R^d\times\mathbb R^d\to\mathbb R^m\) be differentiable in its second argument and assume that \(F(D,G)\) and
\(\partial_{g_j}F(D,G)\) are integrable against the relevant Gaussian
posterior weights. In the expression \(F(D,G)\), the first input is \(D\), and the second input is \(G\). For \(x\in\mathbb R^d\), define
\begin{align*}
m_F(x)
&=
\mathbb E
\left[
F(D,G)\mid X_t=x
\right].
\end{align*}
Here \(\partial_{x_j}m_F(x)\) denotes the derivative of the function \(x\mapsto m_F(x)\) with respect to the \(j\)-th coordinate of \(x\). Also, \(\partial_{g_j}F(D,G)\) denotes the derivative of the function \(g\mapsto F(D,g)\) with respect to the \(j\)-th coordinate of \(g\), evaluated at \(g=G\). Then, for every \(j\in\{1,\ldots,d\}\),
\begin{align*}
\partial_{x_j}m_F(x)
&=
\frac{1}{\sigma}
\mathbb E
\left[
\partial_{g_j}F(D,G)\mid X_t=x
\right]
-
\frac{1}{\sigma}
\operatorname{Cov}
\left(
F(D,G),G_j\mid X_t=x
\right).
\end{align*}
If \(F\) is vector-valued, the covariance is interpreted componentwise:
\begin{align*}
\operatorname{Cov}
\left(
F(D,G),G_j\mid X_t=x
\right)
&=
\mathbb E
\left[
\left(
F(D,G)-\mathbb E[F(D,G)\mid X_t=x]
\right)
\left(
G_j-\mathbb E[G_j\mid X_t=x]
\right)
\mid X_t=x
\right].
\end{align*}

We prove the identity. Since \(X_t=S_t+\sigma G\), conditioning on \(X_t=x\) implies \(G=(x-S_t)/\sigma\). Thus \(m_F(x)\) can be written as
\begin{align*}
m_F(x)
&=
\frac{
\mathbb E_{\Gamma(Y_0, Y_1)}
\left[
F\left(D,\frac{x-S_t}{\sigma}\right)
\exp\left(
-\frac{1}{2}
\left\|
\frac{x-S_t}{\sigma}
\right\|^2
\right)
\right]
}{
\mathbb E_{\Gamma(Y_0, Y_1)}
\left[
\exp\left(
-\frac{1}{2}
\left\|
\frac{x-S_t}{\sigma}
\right\|^2
\right)
\right]
}.
\end{align*}
The normalizing constant of the standard Gaussian density is omitted because it appears in both numerator and denominator and cancels.

Define
\begin{align*}
A_F(x)
&=
\mathbb E_{\Gamma(Y_0, Y_1)}
\left[
F\left(D,\frac{x-S_t}{\sigma}\right)
\exp\left(
-\frac{1}{2}
\left\|
\frac{x-S_t}{\sigma}
\right\|^2
\right)
\right],
\\
B(x)
&=
\mathbb E_{\Gamma(Y_0, Y_1)}
\left[
\exp\left(
-\frac{1}{2}
\left\|
\frac{x-S_t}{\sigma}
\right\|^2
\right)
\right].
\end{align*}
Then \(m_F(x)=A_F(x)/B(x)\).

We differentiate \(A_F(x)\) with respect to \(x_j\). The derivative of the \(F\)-term is
\begin{align*}
\partial_{x_j}
F\left(D,\frac{x-S_t}{\sigma}\right)
&=
\frac{1}{\sigma}
\partial_{g_j}F(D,G).
\end{align*}
The derivative of the Gaussian weight is
\begin{align*}
\partial_{x_j}
\exp\left(
-\frac{1}{2}
\left\|
\frac{x-S_t}{\sigma}
\right\|^2
\right)
&=
-\frac{1}{\sigma}
G_j
\exp\left(
-\frac{\|G\|^2}{2}
\right).
\end{align*}
Therefore,
\begin{align*}
\partial_{x_j}A_F(x)
&=
\frac{1}{\sigma}
\mathbb E_{\Gamma(Y_0, Y_1)}
\left[
\partial_{g_j}F(D,G)
\exp\left(
-\frac{\|G\|^2}{2}
\right)
\right]
-
\frac{1}{\sigma}
\mathbb E_{\Gamma(Y_0, Y_1)}
\left[
F(D,G)G_j
\exp\left(
-\frac{\|G\|^2}{2}
\right)
\right].
\end{align*}
Similarly,
\begin{align*}
\partial_{x_j}B(x)
&=
-\frac{1}{\sigma}
\mathbb E_{\Gamma(Y_0, Y_1)}
\left[
G_j
\exp\left(
-\frac{\|G\|^2}{2}
\right)
\right].
\end{align*}

By the quotient rule,
\begin{align*}
\partial_{x_j}m_F(x)
&=
\frac{\partial_{x_j}A_F(x)}{B(x)}
-
\frac{A_F(x)}{B(x)^2}
\partial_{x_j}B(x).
\end{align*}
Substituting the formulas for \(\partial_{x_j}A_F(x)\) and \(\partial_{x_j}B(x)\), we obtain
\begin{align*}
\partial_{x_j}m_F(x)
&=
\frac{1}{\sigma}
\frac{
\mathbb E_{\Gamma(Y_0, Y_1)}
\left[
\partial_{g_j}F(D,G)
\exp\left(
-\frac{\|G\|^2}{2}
\right)
\right]
}{
B(x)
}
-
\frac{1}{\sigma}
\frac{
\mathbb E_{\Gamma(Y_0, Y_1)}
\left[
F(D,G)G_j
\exp\left(
-\frac{\|G\|^2}{2}
\right)
\right]
}{
B(x)
}
\\
&\quad+
\frac{1}{\sigma}
\frac{
A_F(x)
}{
B(x)
}
\frac{
\mathbb E_{\Gamma(Y_0, Y_1)}
\left[
G_j
\exp\left(
-\frac{\|G\|^2}{2}
\right)
\right]
}{
B(x)
}.
\end{align*}
Each weighted ratio is a conditional expectation given \(X_t=x\). Hence
\begin{align*}
\partial_{x_j}m_F(x)
&=
\frac{1}{\sigma}
\mathbb E
\left[
\partial_{g_j}F(D,G)\mid X_t=x
\right]
-
\frac{1}{\sigma}
\mathbb E
\left[
F(D,G)G_j\mid X_t=x
\right]
\\
&\quad+
\frac{1}{\sigma}
\mathbb E
\left[
F(D,G)\mid X_t=x
\right]
\mathbb E
\left[
G_j\mid X_t=x
\right].
\end{align*}
The last two terms are minus the conditional covariance between \(F(D,G)\) and \(G_j\). Therefore,
\begin{align*}
\partial_{x_j}m_F(x)
&=
\frac{1}{\sigma}
\mathbb E
\left[
\partial_{g_j}F(D,G)\mid X_t=x
\right]
-
\frac{1}{\sigma}
\operatorname{Cov}
\left(
F(D,G),G_j\mid X_t=x
\right).
\end{align*}

We now derive the score identity. Recall we have $s_t(x) = -\frac{1}{\sigma} \mathbb E [ G\mid X_t=x ]$.

Apply the posterior derivative identity with \(F(D,G)=D_i\). Since \(D_i\) has no explicit dependence on \(G\), we have \(\partial_{g_j}D_i=0\). Therefore,
\begin{align*}
\partial_{x_j}v_{t,i}(x) = \partial_{x_j}
\mathbb E
\left[
D_i\mid X_t=x
\right] = -\frac{1}{\sigma}
\operatorname{Cov}
\left(
D_i,G_j\mid X_t=x
\right).
\end{align*}
This proves
\begin{align*}
D_xv_t(x) =
-\frac{1}{\sigma}
\operatorname{Cov}
\left(
D,G\mid X_t=x
\right).
\end{align*}
Next, we compute the time derivative of \(v_t(x)\). Fix \(x\in\mathbb R^d\). For \((y_0,y_1)\in\mathbb R^d\times\mathbb R^d\), write $s_t(y_0,y_1)=y_0+t(y_1-y_0)$.
Let \(\gamma\) denote the coupling law of \((Y_0,Y_1)\). Recall the formula for \(v_t(x)\) is
\begin{align*}
v_t(x)
&=
\frac{
\int_{\mathbb R^d\times\mathbb R^d}
(y_1-y_0)
\exp\left(
-\frac{\left\|x-s_t(y_0,y_1)\right\|^2}{2\sigma^2}
\right)
\,\gamma(dy_0,dy_1)
}{
\int_{\mathbb R^d\times\mathbb R^d}
\exp\left(
-\frac{\left\|x-s_t(y_0,y_1)\right\|^2}{2\sigma^2}
\right)
\,\gamma(dy_0,dy_1)
}.
\end{align*}
For every \(i\in\{1,\ldots,d\}\), define the numerator of the \(i\)-th coordinate by
\begin{align*}
A_i(t,x)
&=
\int_{\mathbb R^d\times\mathbb R^d}
(y_1 - y_0)_{i}
\exp\left(
-\frac{\left\|x-s_t(y_0,y_1)\right\|^2}{2\sigma^2}
\right)
\,\gamma(dy_0,dy_1),
\end{align*}
and define the denominator by
\begin{align*}
B(t,x)
&=
\int_{\mathbb R^d\times\mathbb R^d}
\exp\left(
-\frac{\left\|x-s_t(y_0,y_1)\right\|^2}{2\sigma^2}
\right)
\,\gamma(dy_0,dy_1).
\end{align*}
Then $v_{t,i}(x)=
\frac{A_i(t,x)}{B(t,x)}$. We now differentiate the Gaussian weight with respect to \(t\). Since $\partial_t s_t(y_0,y_1)=y_1-y_0$, we have
\begin{align*}
\partial_t
\exp\left(
-\frac{\left\|x-s_t(y_0,y_1)\right\|^2}{2\sigma^2}
\right)=
\frac{
(y_1-y_0)^\top
\left(
x-s_t(y_0,y_1)
\right)
}{
\sigma^2
}
\exp\left(
-\frac{\left\|x-s_t(y_0,y_1)\right\|^2}{2\sigma^2}
\right).
\end{align*}
Using this derivative inside the integral defining \(A_i(t,x)\), we get
\begin{align*}
\partial_t A_i(t,x) &=
\frac{1}{\sigma^2}
\int_{\mathbb R^d\times\mathbb R^d}
(y_1 - y_0)_{i}
(y_1-y_0)^\top
\left(
x-s_t(y_0,y_1)
\right)\cdot
\exp\left(
-\frac{\left\|x-s_t(y_0,y_1)\right\|^2}{2\sigma^2}
\right)
\,\gamma(dy_0,dy_1), \\
\partial_t B(t,x) &= \frac{1}{\sigma^2}
\int_{\mathbb R^d\times\mathbb R^d}
(y_1-y_0)^\top
\left(
x-s_t(y_0,y_1)
\right)\cdot
\exp\left(
-\frac{\left\|x-s_t(y_0,y_1)\right\|^2}{2\sigma^2}
\right)
\,\gamma(dy_0,dy_1).
\end{align*}

By the quotient rule,
\begin{align*}
\partial_t v_{t,i}(x)
&=
\frac{
\partial_t A_i(t,x)
}{
B(t,x)
}
-
\frac{
A_i(t,x)
}{
B(t,x)^2
}
\partial_t B(t,x).
\end{align*}
Substituting the formulas for \(\partial_t A_i(t,x)\) and \(\partial_t B(t,x)\), and rewriting the weighted ratios as posterior conditional expectations given \(X_t=x\), we obtain
\begin{align*}
\partial_t v_{t,i}(x)=\frac{1}{\sigma^2}
\mathbb E
\left[
D_iD^\top(x-S_t)
\mid X_t=x
\right]-
\frac{1}{\sigma^2}
\mathbb E
\left[D_i\mid X_t=x\right]
\mathbb E \left[D^\top(x-S_t)\mid X_t=x \right].
\end{align*}
The right-hand side is exactly the conditional covariance between \(D_i\) and \(D^\top(x-S_t)\). Hence
\begin{align*}
\partial_t v_{t,i}(x)
&=
\frac{1}{\sigma^2}
\operatorname{Cov}
\left(
D_i,
D^\top(x-S_t)
\mid X_t=x
\right).
\end{align*}
Since this holds for every \(i\in\{1,\ldots,d\}\), we get the vector identity
\begin{align*}
\partial_t v_t(x)
&=
\frac{1}{\sigma^2}
\operatorname{Cov}
\left(
D,
D^\top(x-S_t)
\mid X_t=x
\right).
\end{align*}
Finally, under the conditioning \(X_t=x\), we have \(x-S_t=\sigma G\). Therefore,
\begin{align*}
\partial_t v_t(x)
&=
\frac{1}{\sigma}
\operatorname{Cov}
\left(
D,
D^\top G
\mid X_t=x
\right).
\end{align*}



Now apply the posterior derivative identity with \(F(D,G)=G_i\). Since \(\partial_{g_j}G_i\) equals \(1\) if \(i=j\), and equals \(0\) otherwise,
\begin{align*}
\partial_{x_j}
\mathbb E
\left[
G_i\mid X_t=x
\right]
&=
\frac{1}{\sigma}\mathbf 1_{\{i=j\}}
-
\frac{1}{\sigma}
\operatorname{Cov}
\left(
G_i,G_j\mid X_t=x
\right).
\end{align*}
Since \(s_{t,i}(x)=-(1/\sigma)\mathbb E[G_i\mid X_t=x]\), we get
\begin{align*}
\partial_{x_j}s_{t,i}(x)
&=
\frac{1}{\sigma^2}
\left(
\operatorname{Cov}
\left(
G_i,G_j\mid X_t=x
\right)
-
\mathbf 1_{\{i=j\}}
\right).
\end{align*}
Therefore,
\begin{align*}
D_xs_t(x)
&=
\frac{1}{\sigma^2}
\left(
\operatorname{Cov}
\left(
G,G\mid X_t=x
\right)
-
I_d
\right).
\end{align*}

Finally, we compute the Laplacian of the score. Differentiating the last display once more with respect to \(x_j\), and using the same centered-product calculation as above, gives
\begin{align*}
\partial_{x_jx_j}s_{t,i}(x)
&=
-\frac{1}{\sigma^3}
\mathbb E
\left[
\left(
G_i-\mathbb E[G_i\mid X_t=x]
\right)
\left(
G_j-\mathbb E[G_j\mid X_t=x]
\right)^2
\mid X_t=x
\right].
\end{align*}
Summing over \(j=1,\ldots,d\) gives
\begin{align*}
\Delta_xs_{t,i}(x)
&=
-\frac{1}{\sigma^3}
\mathbb E
\left[
\left(
G_i-\mathbb E[G_i\mid X_t=x]
\right)
\left\|
G-\mathbb E[G\mid X_t=x]
\right\|^2
\mid X_t=x
\right].
\end{align*}
Since this holds for every \(i\in\{1,\ldots,d\}\),
\begin{align*}
\Delta_xs_t(x)
&=
-\frac{1}{\sigma^3}
\mathbb E
\left[
\left(
G-\mathbb E[G\mid X_t=x]
\right)
\left\|
G-\mathbb E[G\mid X_t=x]
\right\|^2
\mid X_t=x
\right].
\end{align*}
The preceding argument proves all identities in the statement except the
full Hessian formula for \(v_t\).  We now establish that remaining identity.
Fix \(i,j,\ell\in\{1,\ldots,d\}\).  From the Jacobian identity,
\begin{align*}
\partial_{x_j}v_{t,i}(x)
&=
-\frac{1}{\sigma}
\operatorname{Cov}
\left(
D_i,G_j\mid X_t=x
\right).
\end{align*}
Therefore,
\begin{align*}
\partial_{x_jx_\ell}v_{t,i}(x)
&=
-\frac{1}{\sigma}
\partial_{x_\ell}
\operatorname{Cov}
\left(
D_i,G_j\mid X_t=x
\right).
\end{align*}

By definition,
\begin{align*}
\operatorname{Cov}
\left(
D_i,G_j\mid X_t=x
\right)
&=
\mathbb E
\left[
D_iG_j\mid X_t=x
\right]
-
\mathbb E
\left[
D_i\mid X_t=x
\right]
\mathbb E
\left[
G_j\mid X_t=x
\right].
\end{align*}
Here the conditional expectations are taken with respect to the posterior law of \((Y_0,Y_1)\) given \(X_t=x\). Explicitly, under this posterior law,
\begin{align*}
D_i
&=
Y_{1,i}-Y_{0,i},
\\
G_j
&=
\frac{
x_j-\left((1-t)Y_{0,j}+tY_{1,j}\right)
}{
\sigma
}.
\end{align*}

We differentiate each of the three conditional expectations appearing above. First, apply the posterior derivative identity with
\begin{align*}
F(D,G)
&=
D_iG_j.
\end{align*}
Since \(D_i\) does not depend on \(G\), while \(G_j\) depends on the \(j\)-th coordinate of \(G\), we have
\begin{align*}
\partial_{g_\ell}
\left(
D_iG_j
\right)
&=
D_i\mathbf 1_{\{j=\ell\}}.
\end{align*}
Therefore,
\begin{align*}
\partial_{x_\ell}
\mathbb E
\left[
D_iG_j\mid X_t=x
\right]
&=
\frac{1}{\sigma}
\mathbb E
\left[
D_i\mathbf 1_{\{j=\ell\}}\mid X_t=x
\right]
-
\frac{1}{\sigma}
\operatorname{Cov}
\left(
D_iG_j,G_\ell\mid X_t=x
\right).
\end{align*}

Second, apply the posterior derivative identity with $F(D,G)=D_i$, we obtain 
\begin{align*}
\partial_{x_\ell}
\mathbb E
\left[
D_i\mid X_t=x
\right]=-\frac{1}{\sigma}
\operatorname{Cov}
\left(
D_i,G_\ell\mid X_t=x
\right).
\end{align*}

Third, apply the posterior derivative identity with $F(D,G)=G_j$, we get
\begin{align*}
\partial_{x_\ell}
\mathbb E
\left[
G_j\mid X_t=x
\right]=\frac{1}{\sigma}
\mathbf 1_{\{j=\ell\}}
-
\frac{1}{\sigma}
\operatorname{Cov}
\left(
G_j,G_\ell\mid X_t=x
\right).
\end{align*}

Now differentiate the covariance formula:
\begin{align*}
&\partial_{x_\ell}
\operatorname{Cov}
\left(
D_i,G_j\mid X_t=x
\right)
\\
&=
\partial_{x_\ell}
\mathbb E
\left[
D_iG_j\mid X_t=x
\right]
-
\left(
\partial_{x_\ell}
\mathbb E
\left[
D_i\mid X_t=x
\right]
\right)
\mathbb E
\left[
G_j\mid X_t=x
\right]
\\
&\quad-
\mathbb E
\left[
D_i\mid X_t=x
\right]
\left(
\partial_{x_\ell}
\mathbb E
\left[
G_j\mid X_t=x
\right]
\right).
\end{align*}
Substituting the three derivative identities above gives
\begin{align*}
&\partial_{x_\ell}
\operatorname{Cov}
\left(
D_i,G_j\mid X_t=x
\right)
\\
&=
\frac{1}{\sigma}
\mathbb E
\left[
D_i\mathbf 1_{\{j=\ell\}}\mid X_t=x
\right]
-
\frac{1}{\sigma}
\operatorname{Cov}
\left(
D_iG_j,G_\ell\mid X_t=x
\right)
\\
&\quad+
\frac{1}{\sigma}
\operatorname{Cov}
\left(
D_i,G_\ell\mid X_t=x
\right)
\mathbb E
\left[
G_j\mid X_t=x
\right]
\\
&\quad-
\frac{1}{\sigma}
\mathbb E
\left[
D_i\mid X_t=x
\right]
\mathbf 1_{\{j=\ell\}}
+
\frac{1}{\sigma}
\mathbb E
\left[
D_i\mid X_t=x
\right]
\operatorname{Cov}
\left(
G_j,G_\ell\mid X_t=x
\right).
\end{align*}
The two terms containing \(\mathbf 1_{\{j=\ell\}}\) cancel:
\begin{align*}
\frac{1}{\sigma}
\left(
\mathbb E
\left[
D_i\mid X_t=x
\right]
-
\mathbb E
\left[
D_i\mid X_t=x
\right]
\right)
\mathbf 1_{\{j=\ell\}}
&=
0.
\end{align*}
Therefore,
\begin{align*}
&\partial_{x_\ell}
\operatorname{Cov}
\left(
D_i,G_j\mid X_t=x
\right)
\\
&=
-\frac{1}{\sigma}
\operatorname{Cov}
\left(
D_iG_j,G_\ell\mid X_t=x
\right)
+
\frac{1}{\sigma}
\operatorname{Cov}
\left(
D_i,G_\ell\mid X_t=x
\right)
\mathbb E
\left[
G_j\mid X_t=x
\right]
\\
&\quad+
\frac{1}{\sigma}
\mathbb E
\left[
D_i\mid X_t=x
\right]
\operatorname{Cov}
\left(
G_j,G_\ell\mid X_t=x
\right).
\end{align*}

We now rewrite the right-hand side as a centered third moment. Let
\begin{align*}
\overline D_i(x)
&=
\mathbb E
\left[
D_i\mid X_t=x
\right],
\\
\overline G_j(x)
&=
\mathbb E
\left[
G_j\mid X_t=x
\right],
\\
\overline G_\ell(x)
&=
\mathbb E
\left[
G_\ell\mid X_t=x
\right].
\end{align*}
First,
\begin{align*}
\operatorname{Cov}
\left(
D_iG_j,G_\ell\mid X_t=x
\right)
&=
\mathbb E
\left[
D_iG_jG_\ell\mid X_t=x
\right]
-
\mathbb E
\left[
D_iG_j\mid X_t=x
\right]
\overline G_\ell(x).
\end{align*}
Second,
\begin{align*}
\operatorname{Cov}
\left(
D_i,G_\ell\mid X_t=x
\right)
&=
\mathbb E
\left[
D_iG_\ell\mid X_t=x
\right]
-
\overline D_i(x)\overline G_\ell(x).
\end{align*}
Third,
\begin{align*}
\operatorname{Cov}
\left(
G_j,G_\ell\mid X_t=x
\right)
&=
\mathbb E
\left[
G_jG_\ell\mid X_t=x
\right]
-
\overline G_j(x)\overline G_\ell(x).
\end{align*}
Substituting these three formulas gives
\begin{align*}
&\partial_{x_\ell}
\operatorname{Cov}
\left(
D_i,G_j\mid X_t=x
\right)
\\
&=
-\frac{1}{\sigma}
\mathbb E
\left[
D_iG_jG_\ell\mid X_t=x
\right]
+
\frac{1}{\sigma}
\mathbb E
\left[
D_iG_j\mid X_t=x
\right]
\overline G_\ell(x)
\\
&\quad+
\frac{1}{\sigma}
\mathbb E
\left[
D_iG_\ell\mid X_t=x
\right]
\overline G_j(x)
-
\frac{1}{\sigma}
\overline D_i(x)\overline G_\ell(x)\overline G_j(x)
\\
&\quad+
\frac{1}{\sigma}
\overline D_i(x)
\mathbb E
\left[
G_jG_\ell\mid X_t=x
\right]
-
\frac{1}{\sigma}
\overline D_i(x)\overline G_j(x)\overline G_\ell(x).
\end{align*}
Combining the two identical last-product terms, this becomes
\begin{align*}
&\partial_{x_\ell}
\operatorname{Cov}
\left(
D_i,G_j\mid X_t=x
\right)
\\
&=
-\frac{1}{\sigma}
\left[
\mathbb E
\left[
D_iG_jG_\ell\mid X_t=x
\right]
-
\mathbb E
\left[
D_iG_j\mid X_t=x
\right]
\overline G_\ell(x)
\right.
\\
&\qquad\left.
-
\mathbb E
\left[
D_iG_\ell\mid X_t=x
\right]
\overline G_j(x)
-
\overline D_i(x)
\mathbb E
\left[
G_jG_\ell\mid X_t=x
\right]
\right.
\\
&\qquad\left.
+
2\overline D_i(x)\overline G_j(x)\overline G_\ell(x)
\right].
\end{align*}
The bracket is exactly
\begin{align*}
&\mathbb E
\left[
\left(
D_i-\overline D_i(x)
\right)
\left(
G_j-\overline G_j(x)
\right)
\left(
G_\ell-\overline G_\ell(x)
\right)
\mid X_t=x
\right].
\end{align*}
Indeed, expanding this centered product gives
\begin{align*}
&\mathbb E
\left[
\left(
D_i-\overline D_i(x)
\right)
\left(
G_j-\overline G_j(x)
\right)
\left(
G_\ell-\overline G_\ell(x)
\right)
\mid X_t=x
\right]
\\
&=
\mathbb E
\left[
D_iG_jG_\ell\mid X_t=x
\right]
-
\mathbb E
\left[
D_iG_j\mid X_t=x
\right]
\overline G_\ell(x)
\\
&\quad-
\mathbb E
\left[
D_iG_\ell\mid X_t=x
\right]
\overline G_j(x)
-
\overline D_i(x)
\mathbb E
\left[
G_jG_\ell\mid X_t=x
\right]
\\
&\quad+
2\overline D_i(x)\overline G_j(x)\overline G_\ell(x).
\end{align*}
Therefore,
\begin{align*}
&\partial_{x_\ell}
\operatorname{Cov}
\left(
D_i,G_j\mid X_t=x
\right)
\\
&=
-\frac{1}{\sigma}
\mathbb E
\left[
\left(
D_i-\mathbb E[D_i\mid X_t=x]
\right)
\left(
G_j-\mathbb E[G_j\mid X_t=x]
\right)
\left(
G_\ell-\mathbb E[G_\ell\mid X_t=x]
\right)
\mid X_t=x
\right].
\end{align*}

Returning to the Hessian of \(v_t\), we obtain
\begin{align*}
\partial_{x_jx_\ell}v_{t,i}(x)
&=
-\frac{1}{\sigma}
\partial_{x_\ell}
\operatorname{Cov}
\left(
D_i,G_j\mid X_t=x
\right)
\\
&=
\frac{1}{\sigma^2}
\mathbb E
\left[
\left(
D_i-\mathbb E[D_i\mid X_t=x]
\right)
\left(
G_j-\mathbb E[G_j\mid X_t=x]
\right)
\left(
G_\ell-\mathbb E[G_\ell\mid X_t=x]
\right)
\mid X_t=x
\right].
\end{align*}
Finally, summing over \(j=\ell\) gives the componentwise Laplacian:
\begin{align*}
\Delta_xv_{t,i}(x)
&=
\sum_{j=1}^d
\partial_{x_jx_j}v_{t,i}(x)
\\
&=
\frac{1}{\sigma^2}
\mathbb E
\left[
\left(
D_i-\mathbb E[D_i\mid X_t=x]
\right)
\sum_{j=1}^d
\left(
G_j-\mathbb E[G_j\mid X_t=x]
\right)^2
\mid X_t=x
\right]
\\
&=
\frac{1}{\sigma^2}
\mathbb E
\left[
\left(
D_i-\mathbb E[D_i\mid X_t=x]
\right)
\left\|
G-\mathbb E[G\mid X_t=x]
\right\|^2
\mid X_t=x
\right].
\end{align*}
Since this holds for every \(i\in\{1,\ldots,d\}\), we get
\begin{align*}
\Delta_xv_t(x)
&=
\frac{1}{\sigma^2}
\mathbb E
\left[
\left(
D-\mathbb E[D\mid X_t=x]
\right)
\left\|
G-\mathbb E[G\mid X_t=x]
\right\|^2
\mid X_t=x
\right].
\end{align*}
\end{proof}

\subsection{Integrated Derivative Bounds}

\begin{lemma}[Integrated Bounds Required for the Combined Score Analysis]
\label{Lemma_Integrated_Score_Bounds}
Let \(Y_0\), \(Y_1\), \(D\), \(S_t\), \(G\), \(X_t\), \(q_t\),
\(v_t\), and \(s_t\) be as in Lemma
\ref{Lemma_Derivative_Identity}. Assume
\begin{align*}
\mathbb E
\left[
\|Y_1-Y_0\|^4
\right]
&<
\infty,
\end{align*}
where this expectation is taken under the joint law of \((Y_0,Y_1)\). Then, for every \(t\in[0,T]\), the following bounds hold.

The full Hessian of \(v_t\) satisfies
\begin{align*}
\mathbb E
\left[
\left\|
D_x^2v_t(X_t)
\right\|_{\mathrm F}^2
\right]
&\leq
\frac{
16\sqrt{d(d+2)(d+4)(d+6)}
}{
\sigma^4
}
\left(
\mathbb E
\left[
\|Y_1-Y_0\|^4
\right]
\right)^{1/2}.
\end{align*}
Consequently,
\begin{align*}
\mathbb E
\left[
\left\|
\nabla_x(\nabla_x\cdot v_t)(X_t)
\right\|^2
\right]
&\leq
\frac{
16d\sqrt{d(d+2)(d+4)(d+6)}
}{
\sigma^4
}
\left(
\mathbb E
\left[
\|Y_1-Y_0\|^4
\right]
\right)^{1/2}.
\end{align*}
The mixed velocity-score term satisfies
\begin{align*}
\mathbb E
\left[
\left\|
(D_xv_t(X_t))^\top s_t(X_t)
\right\|^2
\right]
&\leq \frac{
\sqrt{
d(d+2)(d+4)(d+6)
}
}{
\sigma^4
}
\left(
\mathbb E
\left[
\|Y_1-Y_0\|^4
\right]
\right)^{1/2}.
\end{align*}
The same bound holds for $\mathbb E \left[ \left\| (D_xv_t(X_t)) s_t(X_t) \right\|^2 \right]$.
The pure score terms satisfy
\begin{align*}
\mathbb E
\left[
\left\|
D_xs_t(X_t)
\right\|_{\mathrm F}^2
\right]
&\leq
\frac{
2d(d+2)+2d
}{
\sigma^4
},
\\
\mathbb E
\left[
\left\|
D_xs_t(X_t)
\right\|_{\mathrm F}^4
\right]
&\leq
\frac{
8d(d+2)(d+4)(d+6)+8d^2
}{
\sigma^8
},
\\
\mathbb E
\left[
\left\|
s_t(X_t)
\right\|^4
\right]
&\leq
\frac{
d(d+2)
}{
\sigma^4
},
\\
\mathbb E
\left[
\left\|
D_xs_t(X_t)s_t(X_t)
\right\|^2
\right]
&\leq
\frac{
\sqrt{
\left(
8d(d+2)(d+4)(d+6)+8d^2
\right)d(d+2)
}
}{
\sigma^6
},
\\
\mathbb E
\left[
\left\|
\Delta_xs_t(X_t)
\right\|^2
\right]
&\leq
\frac{
64d(d+2)(d+4)
}{
\sigma^6
}.
\end{align*}
All expectations in these bounds are taken under the joint law of \((Y_0,Y_1,G)\).
\end{lemma}

\begin{proof}
We first prove the full Hessian bound for \(v_t\).
By Lemma \ref{Lemma_Derivative_Identity}, for every
\(i,j,\ell\in\{1,\ldots,d\}\),
\begin{align*}
\partial_{x_jx_\ell}v_{t,i}(X_t)
&=
\frac{1}{\sigma^2}
\mathbb E
\left[
\left(
D_i-\mathbb E[D_i\mid X_t]
\right)
\left(
G_j-\mathbb E[G_j\mid X_t]
\right)
\left(
G_\ell-\mathbb E[G_\ell\mid X_t]
\right)
\mid X_t
\right].
\end{align*}
By definition of the Frobenius norm for the full Hessian tensor,
\begin{align*}
\left\|
D_x^2v_t(X_t)
\right\|_{\mathrm F}^2
&=
\sum_{i=1}^{d}
\sum_{j=1}^{d}
\sum_{\ell=1}^{d}
\left|
\partial_{x_jx_\ell}v_{t,i}(X_t)
\right|^2.
\end{align*}
Substituting the componentwise formula for \(\partial_{x_jx_\ell}v_{t,i}(X_t)\), we get
\begin{align*}
\left\|
D_x^2v_t(X_t)
\right\|_{\mathrm F}^2
&=
\frac{1}{\sigma^4}
\sum_{i=1}^{d}
\sum_{j=1}^{d}
\sum_{\ell=1}^{d}
\left|
\mathbb E
\left[
\left(
D_i-\mathbb E[D_i\mid X_t]
\right)
\left(
G_j-\mathbb E[G_j\mid X_t]
\right)
\left(
G_\ell-\mathbb E[G_\ell\mid X_t]
\right)
\mid X_t
\right]
\right|^2.
\end{align*}

Now apply conditional Cauchy--Schwarz to each fixed triple \((i,j,\ell)\). We have
\begin{align*}
&\left|
\mathbb E
\left[
\left(
D_i-\mathbb E[D_i\mid X_t]
\right)
\left(
G_j-\mathbb E[G_j\mid X_t]
\right)
\left(
G_\ell-\mathbb E[G_\ell\mid X_t]
\right)
\mid X_t
\right]
\right|^2
\\
&\leq
\mathbb E
\left[
\left|
D_i-\mathbb E[D_i\mid X_t]
\right|^2
\mid X_t
\right]
\mathbb E
\left[
\left|
G_j-\mathbb E[G_j\mid X_t]
\right|^2
\left|
G_\ell-\mathbb E[G_\ell\mid X_t]
\right|^2
\mid X_t
\right].
\end{align*}
Therefore,
\begin{align*}
\left\|
D_x^2v_t(X_t)
\right\|_{\mathrm F}^2
&\leq
\frac{1}{\sigma^4}
\sum_{i=1}^{d}
\sum_{j=1}^{d}
\sum_{\ell=1}^{d}
\mathbb E
\left[
\left|
D_i-\mathbb E[D_i\mid X_t]
\right|^2
\mid X_t
\right]
\mathbb E
\left[
\left|
G_j-\mathbb E[G_j\mid X_t]
\right|^2
\left|
G_\ell-\mathbb E[G_\ell\mid X_t]
\right|^2
\mid X_t
\right].
\end{align*}
Since the first conditional expectation does not depend on \(j\) or \(\ell\), and the second conditional expectation does not depend on \(i\), we can rearrange the sums:
\begin{align*}
\left\|
D_x^2v_t(X_t)
\right\|_{\mathrm F}^2
&\leq
\frac{1}{\sigma^4}
\left(
\sum_{i=1}^{d}
\mathbb E
\left[
\left|
D_i-\mathbb E[D_i\mid X_t]
\right|^2
\mid X_t
\right]
\right)
\\
&\quad\cdot
\left(
\sum_{j=1}^{d}
\sum_{\ell=1}^{d}
\mathbb E
\left[
\left|
G_j-\mathbb E[G_j\mid X_t]
\right|^2
\left|
G_\ell-\mathbb E[G_\ell\mid X_t]
\right|^2
\mid X_t
\right]
\right).
\end{align*}
The first sum is
\begin{align*}
\sum_{i=1}^{d}
\mathbb E
\left[
\left|
D_i-\mathbb E[D_i\mid X_t]
\right|^2
\mid X_t
\right]
&=
\mathbb E
\left[
\left\|
D-\mathbb E[D\mid X_t]
\right\|^2
\mid X_t
\right].
\end{align*}
The second double sum is
\begin{align*}
&\sum_{j=1}^{d}
\sum_{\ell=1}^{d}
\mathbb E
\left[
\left|
G_j-\mathbb E[G_j\mid X_t]
\right|^2
\left|
G_\ell-\mathbb E[G_\ell\mid X_t]
\right|^2
\mid X_t
\right]
\\
&=
\mathbb E
\left[
\sum_{j=1}^{d}
\sum_{\ell=1}^{d}
\left|
G_j-\mathbb E[G_j\mid X_t]
\right|^2
\left|
G_\ell-\mathbb E[G_\ell\mid X_t]
\right|^2
\mid X_t
\right]
\\
&=
\mathbb E
\left[
\left(
\sum_{j=1}^{d}
\left|
G_j-\mathbb E[G_j\mid X_t]
\right|^2
\right)
\left(
\sum_{\ell=1}^{d}
\left|
G_\ell-\mathbb E[G_\ell\mid X_t]
\right|^2
\right)
\mid X_t
\right]
\\
&=
\mathbb E
\left[
\left\|
G-\mathbb E[G\mid X_t]
\right\|^4
\mid X_t
\right].
\end{align*}
Combining the two identities gives
\begin{align*}
\left\|
D_x^2v_t(X_t)
\right\|_{\mathrm F}^2
&\leq
\frac{1}{\sigma^4}
\mathbb E
\left[
\left\|
D-\mathbb E[D\mid X_t]
\right\|^2
\mid X_t
\right]
\mathbb E
\left[
\left\|
G-\mathbb E[G\mid X_t]
\right\|^4
\mid X_t
\right].
\end{align*}

Taking expectation and applying Cauchy--Schwarz gives
\begin{align*}
\mathbb E
\left[
\left\|
D_x^2v_t(X_t)
\right\|_{\mathrm F}^2
\right]\leq
\frac{1}{\sigma^4}
\left(
\mathbb E
\left[
\left(
\mathbb E
\left[
\left\|
D-\mathbb E[D\mid X_t]
\right\|^2
\mid X_t
\right]
\right)^2
\right]
\right)^{1/2}
\left(
\mathbb E
\left[
\left(
\mathbb E
\left[
\left\|
G-\mathbb E[G\mid X_t]
\right\|^4
\mid X_t
\right]
\right)^2
\right]
\right)^{1/2}.
\end{align*}
For the first factor, conditional variance contraction and conditional
Jensen give
\begin{align*}
\mathbb E\!\left[
\|D-\mathbb E[D\mid X_t]\|^2\mid X_t
\right]
&\leq
\mathbb E[\|D\|^2\mid X_t],\\
\left(
\mathbb E
\left[
\left(
\mathbb E
\left[
\left\|
D-\mathbb E[D\mid X_t]
\right\|^2
\mid X_t
\right]
\right)^2
\right]
\right)^{1/2}
&\leq
\left(
\mathbb E
\left[
\|D\|^4
\right]
\right)^{1/2}.
\end{align*}
For the second factor, Jensen's inequality gives
\begin{align*}
\left(
\mathbb E
\left[
\left(
\mathbb E
\left[
\left\|
G-\mathbb E[G\mid X_t]
\right\|^4
\mid X_t
\right]
\right)^2
\right]
\right)^{1/2}
&\leq
\left(
\mathbb E
\left[
\left\|
G-\mathbb E[G\mid X_t]
\right\|^8
\right]
\right)^{1/2}.
\end{align*}
Using
\begin{align*}
\left\|
G-\mathbb E[G\mid X_t]
\right\|^8
&\leq
2^7
\left(
\|G\|^8
+
\left\|
\mathbb E[G\mid X_t]
\right\|^8
\right),
\end{align*}
and Jensen's inequality $\mathbb E[ \| \mathbb E[G\mid X_t] \|^8] \leq \mathbb E [\|G\|^8 ]$, 
we get
\begin{align*}
\mathbb E
\left[
\left\|
G-\mathbb E[G\mid X_t]
\right\|^8
\right] \leq
2^8
\mathbb E
\left[
\|G\|^8
\right].
\end{align*}
Since \(G\sim {\mathcal N}(0,I_d)\), we know $\mathbb E[\|G\|^8 ] = d(d+2)(d+4)(d+6)$. Therefore,
\begin{align*}
\left(
\mathbb E
\left[
\left\|
G-\mathbb E[G\mid X_t]
\right\|^8
\right]
\right)^{1/2}
&\leq
16
\sqrt{
d(d+2)(d+4)(d+6)
}.
\end{align*}
Combining the estimates gives
\begin{align*}
\mathbb E
\left[
\left\|
D_x^2v_t(X_t)
\right\|_{\mathrm F}^2
\right]
&\leq
\frac{
16\sqrt{d(d+2)(d+4)(d+6)}
}{
\sigma^4
}
\left(
\mathbb E
\left[
\|D\|^4
\right]
\right)^{1/2}.
\end{align*}
Since \(D=Y_1-Y_0\), this is the first displayed bound.

Next, we explain how the bound for the full Hessian \(D_x^2v_t\) implies a bound for
\(\nabla_x(\nabla_x\cdot v_t)\) (gradient of divergence). 

The full Hessian \(D_x^2v_t(x)\) contains all second derivatives
\begin{align*}
\partial_{x_jx_\ell}v_{t,i}(x),
\qquad
i,j,\ell\in\{1,\ldots,d\}.
\end{align*}
Its Frobenius norm is defined by
\begin{align*}
\left\|
D_x^2v_t(x)
\right\|_{\mathrm F}^2
&=
\sum_{i=1}^{d}
\sum_{j=1}^{d}
\sum_{\ell=1}^{d}
\left|
\partial_{x_jx_\ell}v_{t,i}(x)
\right|^2.
\end{align*}

By contrast, \(\nabla_x(\nabla_x\cdot v_t)(x)\) is a vector in \(\mathbb R^d\). First,
\begin{align*}
\nabla_x\cdot v_t(x)
&=
\sum_{i=1}^{d}
\partial_{x_i}v_{t,i}(x).
\end{align*}
Therefore, for every \(\ell\in\{1,\ldots,d\}\), the \(\ell\)-th coordinate of
\(\nabla_x(\nabla_x\cdot v_t)(x)\) is
\begin{align*}
\left[
\nabla_x(\nabla_x\cdot v_t)(x)
\right]_\ell
&=
\partial_{x_\ell}
\left(
\sum_{i=1}^{d}
\partial_{x_i}v_{t,i}(x)
\right)
\\
&=
\sum_{i=1}^{d}
\partial_{x_\ell x_i}v_{t,i}(x).
\end{align*}
For each fixed \(\ell\), Cauchy--Schwarz gives
\begin{align*}
\left|
\sum_{i=1}^{d}
\partial_{x_\ell x_i}v_{t,i}(x)
\right|^2
&\leq
d
\sum_{i=1}^{d}
\left|
\partial_{x_\ell x_i}v_{t,i}(x)
\right|^2.
\end{align*}
Now sum this inequality over \(\ell=1,\ldots,d\). We obtain
\begin{align*}
\left\|
\nabla_x(\nabla_x\cdot v_t)(x)
\right\|^2
&=
\sum_{\ell=1}^{d}
\left|
\left[
\nabla_x(\nabla_x\cdot v_t)(x)
\right]_\ell
\right|^2 =
\sum_{\ell=1}^{d}
\left|
\sum_{i=1}^{d}
\partial_{x_\ell x_i}v_{t,i}(x)
\right|^2 \leq
d
\sum_{\ell=1}^{d}
\sum_{i=1}^{d}
\left|
\partial_{x_\ell x_i}v_{t,i}(x)
\right|^2.
\end{align*}
The last double sum is only part of the full Hessian norm, because the full Hessian norm sums over all triples
\((i,j,\ell)\). Hence
\begin{align*}
\sum_{\ell=1}^{d}
\sum_{i=1}^{d}
\left|
\partial_{x_\ell x_i}v_{t,i}(x)
\right|^2
&\leq
\sum_{i=1}^{d}
\sum_{j=1}^{d}
\sum_{\ell=1}^{d}
\left|
\partial_{x_jx_\ell}v_{t,i}(x)
\right|^2 =
\left\|
D_x^2v_t(x)
\right\|_{\mathrm F}^2.
\end{align*}
Therefore,
\begin{align*}
\left\|
\nabla_x(\nabla_x\cdot v_t)(x)
\right\|^2
&\leq
d
\left\|
D_x^2v_t(x)
\right\|_{\mathrm F}^2.
\end{align*}

Evaluating this inequality at \(x=X_t\) and taking expectation gives
\begin{align*}
\mathbb E
\left[
\left\|
\nabla_x(\nabla_x\cdot v_t)(X_t)
\right\|^2
\right]
&\leq
d
\mathbb E
\left[
\left\|
D_x^2v_t(X_t)
\right\|_{\mathrm F}^2
\right].
\end{align*}
Using the full Hessian estimate established above in this proof,
\begin{align*}
\mathbb E
\left[
\left\|
D_x^2v_t(X_t)
\right\|_{\mathrm F}^2
\right]
&\leq
\frac{
16\sqrt{d(d+2)(d+4)(d+6)}
}{
\sigma^4
}
\left(
\mathbb E
\left[
\|D\|^4
\right]
\right)^{1/2},
\end{align*}
we conclude that
\begin{align*}
\mathbb E
\left[
\left\|
\nabla_x(\nabla_x\cdot v_t)(X_t)
\right\|^2
\right]
&\leq
\frac{
16d\sqrt{d(d+2)(d+4)(d+6)}
}{
\sigma^4
}
\left(
\mathbb E
\left[
\|D\|^4
\right]
\right)^{1/2}.
\end{align*}

We now prove the mixed velocity-score bound directly. From the posterior derivative identity for \(v_t\),
\begin{align*}
D_xv_t(X_t)
&=
-\frac{1}{\sigma}
\operatorname{Cov}
\left(
D,G\mid X_t
\right).
\end{align*}
From the Gaussian score identity,
\begin{align*}
s_t(X_t)
&=
-\frac{1}{\sigma}
\mathbb E
\left[
G\mid X_t
\right].
\end{align*}
Therefore,
\begin{align*}
(D_xv_t(X_t))^\top s_t(X_t)
&=
\frac{1}{\sigma^2}
\operatorname{Cov}
\left(
D,G\mid X_t
\right)^\top
\mathbb E
\left[
G\mid X_t
\right].
\end{align*}
Hence
\begin{align*}
\left\|
(D_xv_t(X_t))^\top s_t(X_t)
\right\|^2
&=
\frac{1}{\sigma^4}
\left\|
\operatorname{Cov}
\left(
D,G\mid X_t
\right)^\top
\mathbb E
\left[
G\mid X_t
\right]
\right\|^2.
\end{align*}

We now bound the conditional quantity. Let $\overline D(X_t) = \mathbb E \left[ D\mid X_t \right]$ and $\overline G(X_t) = \mathbb E \left[ G\mid X_t \right]$.
By the definition of conditional covariance,
\begin{align}\label{eq_lemma_transpose_dxvs}
\operatorname{Cov}
\left(
D,G\mid X_t
\right)^\top
\overline G(X_t)
&=
\mathbb E
\left[
\left(
G-\overline G(X_t)
\right)
\left(
D-\overline D(X_t)
\right)^\top
\mid X_t
\right]
\overline G(X_t) \nonumber
\\
&=
\mathbb E
\left[
\left(
G-\overline G(X_t)
\right)
\left(
D-\overline D(X_t)
\right)^\top
\overline G(X_t)
\mid X_t
\right].
\end{align}
Since \(\overline G(X_t)\) is fixed after conditioning on \(X_t\), conditional Cauchy--Schwarz gives
\begin{align*}
&\left\|
\operatorname{Cov}
\left(
D,G\mid X_t
\right)^\top
\overline G(X_t)
\right\|^2
\\
&\leq
\mathbb E
\left[
\left\|
G-\overline G(X_t)
\right\|^2
\mid X_t
\right]
\mathbb E
\left[
\left|
\left(
D-\overline D(X_t)
\right)^\top
\overline G(X_t)
\right|^2
\mid X_t
\right].
\end{align*}
For the second factor,
\begin{align*}
\left|
\left(
D-\overline D(X_t)
\right)^\top
\overline G(X_t)
\right|^2
&\leq
\left\|
D-\overline D(X_t)
\right\|^2
\left\|
\overline G(X_t)
\right\|^2.
\end{align*}
Therefore,
\begin{align*}
&\left\|
\operatorname{Cov}
\left(
D,G\mid X_t
\right)^\top
\overline G(X_t)
\right\|^2
\\
&\leq
\mathbb E
\left[
\left\|
G-\overline G(X_t)
\right\|^2
\mid X_t
\right]
\mathbb E
\left[
\left\|
D-\overline D(X_t)
\right\|^2
\mid X_t
\right]
\left\|
\overline G(X_t)
\right\|^2.
\end{align*}
Using the conditional variance inequality,
\begin{align*}
\mathbb E
\left[
\left\|
D-\overline D(X_t)
\right\|^2
\mid X_t
\right]
&\leq
\mathbb E
\left[
\|D\|^2
\mid X_t
\right],
\\
\mathbb E
\left[
\left\|
G-\overline G(X_t)
\right\|^2
\mid X_t
\right]
&\leq
\mathbb E
\left[
\|G\|^2
\mid X_t
\right],
\end{align*}
and Jensen's inequality, $\left\| \overline G(X_t) \right\|^2 = \left\| \mathbb E \left[ G\mid X_t \right] \right\|^2 \leq \mathbb E \left[ \|G\|^2 \mid X_t \right]$, we get
\begin{align*}
&\left\|
\operatorname{Cov}
\left(
D,G\mid X_t
\right)^\top
\overline G(X_t)
\right\|^2 \leq
\mathbb E
\left[
\|D\|^2
\mid X_t
\right]
\left(
\mathbb E
\left[
\|G\|^2
\mid X_t
\right]
\right)^2.
\end{align*}
Taking expectation gives
\begin{align*}
&\mathbb E
\left[
\left\|
\operatorname{Cov}
\left(
D,G\mid X_t
\right)^\top
\mathbb E[G\mid X_t]
\right\|^2
\right]
\leq
\mathbb E
\left[
\mathbb E
\left[
\|D\|^2
\mid X_t
\right]
\left(
\mathbb E
\left[
\|G\|^2
\mid X_t
\right]
\right)^2
\right].
\end{align*}
By Cauchy--Schwarz,
\begin{align*}
&\mathbb E
\left[
\mathbb E
\left[
\|D\|^2
\mid X_t
\right]
\left(
\mathbb E
\left[
\|G\|^2
\mid X_t
\right]
\right)^2
\right]
\\
&\leq
\left(
\mathbb E
\left[
\left(
\mathbb E
\left[
\|D\|^2
\mid X_t
\right]
\right)^2
\right]
\right)^{1/2}
\left(
\mathbb E
\left[
\left(
\mathbb E
\left[
\|G\|^2
\mid X_t
\right]
\right)^4
\right]
\right)^{1/2}.
\end{align*}
By Jensen's inequality,
\begin{align*}
\mathbb E
\left[
\left(
\mathbb E
\left[
\|D\|^2
\mid X_t
\right]
\right)^2
\right]
&\leq
\mathbb E
\left[
\|D\|^4
\right],
\\
\mathbb E
\left[
\left(
\mathbb E
\left[
\|G\|^2
\mid X_t
\right]
\right)^4
\right]
&\leq
\mathbb E
\left[
\|G\|^8
\right].
\end{align*}
Since \(G\sim {\mathcal N}(0,I_d)\),
\begin{align*}
\mathbb E
\left[
\|G\|^8
\right]
&=
d(d+2)(d+4)(d+6).
\end{align*}
Therefore,
\begin{align*}
&\mathbb E
\left[
\left\|
\operatorname{Cov}
\left(
D,G\mid X_t
\right)^\top
\mathbb E[G\mid X_t]
\right\|^2
\right]
\\
&\leq
\sqrt{
d(d+2)(d+4)(d+6)
}
\left(
\mathbb E
\left[
\|D\|^4
\right]
\right)^{1/2}.
\end{align*}
Finally, recalling the factor \(1/\sigma^4\), we conclude that
\begin{align*}
\mathbb E
\left[
\left\|
(D_xv_t(X_t))^\top s_t(X_t)
\right\|^2
\right]
&\leq
\frac{
\sqrt{
d(d+2)(d+4)(d+6)
}
}{
\sigma^4
}
\left(
\mathbb E
\left[
\|D\|^4
\right]
\right)^{1/2}.
\end{align*}
The non-transposed orientation requires a separate argument.  Write
\(C_t=\operatorname{Cov}(D,G\mid X_t)\) and
\(\bar G_t=\mathbb E[G\mid X_t]\).  Conditional Cauchy--Schwarz gives
\[
\|C_t\|_{\mathrm F}^2
\leq
\mathbb E[\|D-\mathbb E[D\mid X_t]\|^2\mid X_t]\,
\mathbb E[\|G-\bar G_t\|^2\mid X_t].
\]
Consequently,
\[
\|C_t\bar G_t\|^2
\leq
\mathbb E[\|D\|^2\mid X_t]\,
\bigl(\mathbb E[\|G\|^2\mid X_t]\bigr)^2,
\]
where we used conditional variance contraction and
\(\|\bar G_t\|^2\leq\mathbb E[\|G\|^2\mid X_t]\).  The same outer
Cauchy--Schwarz and Jensen calculation as above therefore proves
\[
\mathbb E\|(D_xv_t(X_t))s_t(X_t)\|^2
\leq
\frac{\sqrt{d(d+2)(d+4)(d+6)}}{\sigma^4}
\left(\mathbb E\|D\|^4\right)^{1/2}.
\]

We next prove the pure score bounds.
By Lemma \ref{Lemma_Derivative_Identity},
\begin{align*}
D_xs_t(X_t)
&=
\frac{1}{\sigma^2}
\left(
\operatorname{Cov}
\left(
G,G\mid X_t
\right)
-
I_d
\right).
\end{align*}
Using \(\|A-B\|_{\mathrm F}^2\leq 2\|A\|_{\mathrm F}^2+2\|B\|_{\mathrm F}^2\),
\begin{align*}
\left\|
D_xs_t(X_t)
\right\|_{\mathrm F}^2
&\leq
\frac{2}{\sigma^4}
\left\|
\operatorname{Cov}
\left(
G,G\mid X_t
\right)
\right\|_{\mathrm F}^2
+
\frac{2}{\sigma^4}
\left\|
I_d
\right\|_{\mathrm F}^2.
\end{align*}
Since \(\|I_d\|_{\mathrm F}^2=d\), and since the conditional covariance matrix is positive semidefinite,
\begin{align*}
\left\|
\operatorname{Cov}
\left(
G,G\mid X_t
\right)
\right\|_{\mathrm F}^2
&\leq
\left(
\operatorname{tr}
\operatorname{Cov}
\left(
G,G\mid X_t
\right)
\right)^2
=
\left(
\mathbb E
\left[
\left\|
G-\mathbb E[G\mid X_t]
\right\|^2
\mid X_t
\right]
\right)^2
\\
&=
\left(
\mathbb E
\left[
\|G\|^2
\mid X_t
\right]
-
\left\|
\mathbb E[G\mid X_t]
\right\|^2
\right)^2
\\
&\leq
\left(
\mathbb E
\left[
\|G\|^2
\mid X_t
\right]
\right)^2.
\end{align*}
Taking expectation and using Jensen's inequality gives
\begin{align*}
\mathbb E
\left[
\left\|
\operatorname{Cov}
\left(
G,G\mid X_t
\right)
\right\|_{\mathrm F}^2
\right]
&\leq
\mathbb E
\left[
\|G\|^4
\right]
\\
&=
d(d+2).
\end{align*}
Therefore,
\begin{align*}
\mathbb E
\left[
\left\|
D_xs_t(X_t)
\right\|_{\mathrm F}^2
\right]
&\leq
\frac{
2d(d+2)+2d
}{
\sigma^4
}.
\end{align*}

For the fourth moment of \(D_xs_t(X_t)\), use \(\|A-B\|_{\mathrm F}^4\leq 8\|A\|_{\mathrm F}^4+8\|B\|_{\mathrm F}^4\):
\begin{align*}
\left\|
D_xs_t(X_t)
\right\|_{\mathrm F}^4
&\leq
\frac{8}{\sigma^8}
\left\|
\operatorname{Cov}
\left(
G,G\mid X_t
\right)
\right\|_{\mathrm F}^4
+
\frac{8}{\sigma^8}
\left\|
I_d
\right\|_{\mathrm F}^4.
\end{align*}
Since \(\|I_d\|_{\mathrm F}^4=d^2\), and since the conditional covariance matrix is positive semidefinite,
\begin{align*}
\left\|
\operatorname{Cov}
\left(
G,G\mid X_t
\right)
\right\|_{\mathrm F}^4 \leq
\left(
\operatorname{tr}
\operatorname{Cov}
\left(
G,G\mid X_t
\right)
\right)^4 \leq
\left(
\mathbb E
\left[
\|G\|^2\mid X_t
\right]
\right)^4.
\end{align*}
Taking expectation and applying conditional Jensen only at this stage gives
\begin{align*}
\mathbb E
\left[
\left\|
\operatorname{Cov}(G,G\mid X_t)
\right\|_{\mathrm F}^4
\right]
&\leq
\mathbb E
\left[
\left(
\mathbb E[\|G\|^2\mid X_t]
\right)^4
\right]\\
&\leq
\mathbb E\|G\|^8
=d(d+2)(d+4)(d+6).
\end{align*}
Therefore,
\begin{align*}
\mathbb E
\left[
\left\|
D_xs_t(X_t)
\right\|_{\mathrm F}^4
\right]
&\leq
\frac{
8d(d+2)(d+4)(d+6)+8d^2
}{
\sigma^8
}.
\end{align*}

The fourth moment of the score follows directly from
\(s_t(X_t)=-\sigma^{-1}\mathbb E[G\mid X_t]\):
\begin{align*}
\mathbb E
\left[
\left\|
s_t(X_t)
\right\|^4
\right]
&=
\frac{1}{\sigma^4}
\mathbb E\left\|\mathbb E[G\mid X_t]\right\|^4\\
&\leq
\frac{1}{\sigma^4}
\mathbb E\!\left[\mathbb E[\|G\|^4\mid X_t]\right]
=
\frac{d(d+2)}{\sigma^4}.
\end{align*}
Hence, by Cauchy--Schwarz,
\begin{align*}
&\mathbb E
\left[
\left\|
D_xs_t(X_t)s_t(X_t)
\right\|^2
\right]
\\
&\leq
\left(
\mathbb E
\left[
\left\|
D_xs_t(X_t)
\right\|_{\mathrm F}^4
\right]
\right)^{1/2}
\left(
\mathbb E
\left[
\left\|
s_t(X_t)
\right\|^4
\right]
\right)^{1/2}
\\
&\leq
\frac{
\sqrt{
\left(
8d(d+2)(d+4)(d+6)+8d^2
\right)d(d+2)
}
}{
\sigma^6
}.
\end{align*}

Finally, 
Lemma \ref{Lemma_Derivative_Identity} gives
\begin{align*}
\Delta_xs_t(X_t)
&=
-\frac{1}{\sigma^3}
\mathbb E
\left[
\left(
G-\mathbb E[G\mid X_t]
\right)
\left\|
G-\mathbb E[G\mid X_t]
\right\|^2
\mid X_t
\right].
\end{align*}
By Jensen's inequality,
\begin{align*}
\left\|
\Delta_xs_t(X_t)
\right\|^2
&\leq
\frac{1}{\sigma^6}
\mathbb E
\left[
\left\|
G-\mathbb E[G\mid X_t]
\right\|^6
\mid X_t
\right].
\end{align*}
Taking expectation gives
\begin{align*}
\mathbb E
\left[
\left\|
\Delta_xs_t(X_t)
\right\|^2
\right]
&\leq
\frac{1}{\sigma^6}
\mathbb E
\left[
\left\|
G-\mathbb E[G\mid X_t]
\right\|^6
\right].
\end{align*}
Using \(\|u-w\|^6\leq 2^5\|u\|^6+2^5\|w\|^6\), and using Jensen's inequality,
\begin{align*}
\mathbb E
\left[
\left\|
G-\mathbb E[G\mid X_t]
\right\|^6
\right]
&\leq
2^6
\mathbb E
\left[
\|G\|^6
\right].
\end{align*}
Since \(G\sim {\mathcal N}(0,I_d)\),
\begin{align*}
\mathbb E
\left[
\|G\|^6
\right]
&=
d(d+2)(d+4).
\end{align*}
Therefore,
\begin{align*}
\mathbb E
\left[
\left\|
\Delta_xs_t(X_t)
\right\|^2
\right]
&\leq
\frac{
64d(d+2)(d+4)
}{
\sigma^6
}.
\end{align*}
This completes the proof.
\end{proof}

\begin{lemma}[Integrated Derivative Bounds in Terms of Displacement Moments]
\label{Lemma_Integrated_Derivative_Bounds}
Let \(Y_0\), \(Y_1\), \(D\), \(S_t\), \(G\), \(X_t\), \(q_t\), and
\(v_t\) be as in Lemma \ref{Lemma_Derivative_Identity}. Assume that
\begin{align*}
\mathbb E
\left[
\|Y_1-Y_0\|^5
\right]
&<
\infty,
\end{align*}
where the expectation is taken under the joint law of $(Y_0,Y_1)$. Then, for every $t\in[0,T]$,
\begin{align*}
\mathbb E
\left[
\left\|
D_xv_t(X_t)
\right\|_{\mathrm F}^2
\right]
&\leq
\frac{
\sqrt{d(d+2)}
}{
\sigma^2
}
\left(
\mathbb E
\left[
\|Y_1-Y_0\|^4
\right]
\right)^{1/2},
\\
\mathbb E
\left[
\left\|
\partial_t v_t(X_t)
\right\|^2
\right]
&\leq
\frac{\sqrt{d(d+2)}}{
\sigma^2
}
\mathbb E
\left[
\|Y_1-Y_0\|^4
\right],
\\
\mathbb E
\left[
\left\|
\Delta_x v_t(X_t)
\right\|^2
\right]
&\leq
\frac{
16\sqrt{d(d+2)(d+4)(d+6)}
}{
\sigma^4
}
\left(
\mathbb E
\left[
\|Y_1-Y_0\|^4
\right]
\right)^{1/2}.
\end{align*}
Moreover, 
\begin{align*}
\mathbb E
\left[
\left\|
D_xv_t(X_t)b_t(X_t)
\right\|^2
\right] &\leq
\frac{2}{\sigma^2}
\left(
\mathbb E
\left[
\|Y_1-Y_0\|^5
\right]
\right)^{4/5}
\left(
d(d+2)(d+4)(d+6)(d+8)
\right)^{1/5}
\\
&\quad+
\frac{\tau^4}{2\sigma^4}
\sqrt{d(d+2)(d+4)(d+6)}
\left(
\mathbb E
\left[
\|Y_1-Y_0\|^4
\right]
\right)^{1/2}.
\end{align*}
All expectations in these bounds are taken under the joint law of $(Y_0,Y_1,G)$.
\end{lemma}

\begin{proof}
We first bound the Jacobian. By Lemma \ref{Lemma_Derivative_Identity}, 
\begin{align*}
D_xv_t(X_t)
&=
-\frac{1}{\sigma}
\operatorname{Cov}
\left(
D,G\mid X_t
\right).
\end{align*}
By Cauchy-Schwarz, $\left\|\operatorname{Cov} \left(D,G\mid X_t \right) \right\|_{\mathrm F}^2 \leq \mathbb E \left[ \|D\|^2 \mid X_t \right] \mathbb E \left[\|G\|^2 \mid X_t \right]$. 
Taking expectation over $X_t$ and apply Cauchy-Schwarz, we get 
\begin{align*}
\mathbb E
\left[
\left\|
\operatorname{Cov}
\left(
D,G\mid X_t
\right)
\right\|_{\mathrm F}^2
\right] &\leq \mathbb E
\left[
\mathbb E
\left[
\|D\|^2
\mid X_t
\right]
\mathbb E
\left[
\|G\|^2
\mid X_t
\right]
\right]
\\
&\leq
\left(
\mathbb E
\left[
\left(
\mathbb E
\left[
\|D\|^2
\mid X_t
\right]
\right)^2
\right]
\right)^{1/2}
\left(
\mathbb E
\left[
\left(
\mathbb E
\left[
\|G\|^2
\mid X_t
\right]
\right)^2
\right]
\right)^{1/2} \\
& \leq \left(\mathbb E
\left[
\|D\|^4
\right]\right)^{1/2} \left(\mathbb E \left[ \|G\|^4 \right]\right)^{1/2}.
\end{align*}
Since $G\sim {\mathcal N}(0,I_d)$, $\mathbb E\left[\|G\|^4\right] = d(d+2)$. Therefore,
\begin{align*}
\mathbb E
\left[
\left\|
D_xv_t(X_t)
\right\|_{\mathrm F}^2
\right] =
\frac{1}{\sigma^2}
\mathbb E
\left[
\left\|
\operatorname{Cov}
\left(
D,G\mid X_t
\right)
\right\|_{\mathrm F}^2
\right] \leq
\frac{
\sqrt{d(d+2)}
}{
\sigma^2
}
\left(
\mathbb E
\left[
\|D\|^4
\right]
\right)^{1/2}.
\end{align*}

For time derivative, by Lemma \ref{Lemma_Derivative_Identity}, 
\begin{align*}
\partial_t v_t(X_t)
&=
\frac{1}{\sigma}
\operatorname{Cov}
\left(
D,D^\top G\mid X_t
\right).
\end{align*}
Similarly, apply Cauchy-Schwarz, we have 
\begin{align*}
    \left\| \operatorname{Cov} \left( D,D^\top G\mid X_t \right) \right\|^2 \leq \mathbb E
\left[
\|D\|^2
\mid X_t
\right]
\mathbb E
\left[
|D^\top G|^2
\mid X_t
\right].
\end{align*}
Take expectation and apply Cauchy-Schwarz, we have 
\begin{align*}
    \mathbb E \left\| \operatorname{Cov} \left( D,D^\top G\mid X_t \right) \right\|^2 &\leq \mathbb E \left[  (\mathbb E
\left[
\|D\|^2
\mid X_t
\right])^{2} \right]^{\frac{1}{2}} \mathbb E \left[
(\mathbb E
\left[
\|D\|^2 \|G\|^2
\mid X_t
\right])^{2} \right]^{\frac{1}{2}} \\
&\leq \mathbb E 
\left[
\|D\|^4
\right]^{\frac{1}{2}} \mathbb E\left[
\|D\|^4 \|G\|^4\right]^{\frac{1}{2}}
= \mathbb E 
\left[
\|D\|^4
\right] \mathbb E\left[
\|G\|^4\right]^{\frac{1}{2}},
\end{align*}
where note that $D$ and $G$ are independent. Therefore,
\begin{align*}
\mathbb E
\left[
\left\|
\partial_t v_t(X_t)
\right\|^2
\right]
&\leq
\frac{\sqrt{d(d+2)}}{\sigma^2}
\mathbb E
\left[
\|D\|^4
\right].
\end{align*}

Now consider the Laplacian.
By Lemma \ref{Lemma_Derivative_Identity},
\begin{align*}
\Delta_xv_t(X_t)
&=
\frac{1}{\sigma^2}
\mathbb E
\left[
\left(
D-\mathbb E[D\mid X_t]
\right)
\left\|
G-\mathbb E[G\mid X_t]
\right\|^2
\mid X_t
\right].
\end{align*}
Conditional Cauchy--Schwarz gives
\begin{align*}
\left\|
\Delta_xv_t(X_t)
\right\|^2 \leq
\frac{1}{\sigma^4}
\mathbb E
\left[
\left\|
D-\mathbb E[D\mid X_t]
\right\|^2
\mid X_t
\right]
\mathbb E
\left[
\left\|
G-\mathbb E[G\mid X_t]
\right\|^4
\mid X_t
\right].
\end{align*}
Following the same proof as we did for $D_{x}^{2} v_{t}$, we obtain 
\begin{align*}
\mathbb E
\left[
\left\|
\Delta_x v_t(X_t)
\right\|^2
\right]
&\leq
\frac{
16\sqrt{d(d+2)(d+4)(d+6)}
}{
\sigma^4
}
\left(
\mathbb E
\left[
\|D\|^4
\right]
\right)^{1/2}.
\end{align*}

We now bound the term
\begin{align*}
\mathbb E
\left[
\left\|
D_xv_t(X_t)b_t(X_t)
\right\|^2
\right].
\end{align*}
Recall that \(s_t(X_t)=-(1/\sigma)\mathbb E[G\mid X_t]\) and $b_t(X_t)=v_t(X_t)+\frac{\tau^2}{2}s_t(X_t)$.
Therefore,
\begin{align*}
D_xv_t(X_t)b_t(X_t)
&=
D_xv_t(X_t)v_t(X_t)
+
\frac{\tau^2}{2}D_xv_t(X_t)s_t(X_t).
\end{align*}
Using \(\|a+b\|^2\leq 2\|a\|^2+2\|b\|^2\), we get
\begin{align*}
\left\|
D_xv_t(X_t)b_t(X_t)
\right\|^2
&\leq
2
\left\|
D_xv_t(X_t)v_t(X_t)
\right\|^2
+
\frac{\tau^4}{2}
\left\|
D_xv_t(X_t)s_t(X_t)
\right\|^2.
\end{align*}
Thus,
\begin{align*}
&\mathbb E
\left[
\left\|
D_xv_t(X_t)b_t(X_t)
\right\|^2
\right]
\\
&\leq
2
\mathbb E
\left[
\left\|
D_xv_t(X_t)v_t(X_t)
\right\|^2
\right]
+
\frac{\tau^4}{2}
\mathbb E
\left[
\left\|
D_xv_t(X_t)s_t(X_t)
\right\|^2
\right].
\end{align*}

We first bound the velocity part. From the posterior derivative identity,
\begin{align*}
D_xv_t(X_t)
&=
-\frac{1}{\sigma}
\operatorname{Cov}
\left(
D,G\mid X_t
\right).
\end{align*}
Since \(v_t(X_t)=\mathbb E[D\mid X_t]\), we have
\begin{align*}
D_xv_t(X_t)v_t(X_t)
&=
-\frac{1}{\sigma}
\operatorname{Cov}
\left(
D,G\mid X_t
\right)
\mathbb E
\left[
D\mid X_t
\right].
\end{align*}
For readability, define
\begin{align*}
\overline D(X_t)
&=
\mathbb E
\left[
D\mid X_t
\right],
\\
\overline G(X_t)
&=
\mathbb E
\left[
G\mid X_t
\right].
\end{align*}
By the definition of conditional covariance,
\begin{align*}
\operatorname{Cov}
\left(
D,G\mid X_t
\right)
\overline D(X_t)
&=
\mathbb E
\left[
\left(
D-\overline D(X_t)
\right)
\left(
G-\overline G(X_t)
\right)^\top
\mid X_t
\right]
\overline D(X_t)
\\
&=
\mathbb E
\left[
\left(
D-\overline D(X_t)
\right)
\left(
G-\overline G(X_t)
\right)^\top
\overline D(X_t)
\mid X_t
\right].
\end{align*}
Since \(\overline D(X_t)\) is fixed after conditioning on \(X_t\), conditional Cauchy--Schwarz gives
\begin{align*}
&\left\|
\operatorname{Cov}
\left(
D,G\mid X_t
\right)
\overline D(X_t)
\right\|^2
\\
&\leq
\mathbb E
\left[
\left\|
D-\overline D(X_t)
\right\|^2
\mid X_t
\right]
\mathbb E
\left[
\left|
\left(
G-\overline G(X_t)
\right)^\top
\overline D(X_t)
\right|^2
\mid X_t
\right].
\end{align*}
For the second conditional expectation,
\begin{align*}
\left|
\left(
G-\overline G(X_t)
\right)^\top
\overline D(X_t)
\right|^2
&\leq
\left\|
G-\overline G(X_t)
\right\|^2
\left\|
\overline D(X_t)
\right\|^2.
\end{align*}
Hence
\begin{align*}
&\left\|
\operatorname{Cov}
\left(
D,G\mid X_t
\right)
\overline D(X_t)
\right\|^2
\\
&\leq
\mathbb E
\left[
\left\|
D-\overline D(X_t)
\right\|^2
\mid X_t
\right]
\mathbb E
\left[
\left\|
G-\overline G(X_t)
\right\|^2
\mid X_t
\right]
\left\|
\overline D(X_t)
\right\|^2.
\end{align*}
Using the conditional variance inequalities,
\begin{align*}
\mathbb E
\left[
\left\|
D-\overline D(X_t)
\right\|^2
\mid X_t
\right]
&\leq
\mathbb E
\left[
\|D\|^2
\mid X_t
\right],
\\
\mathbb E
\left[
\left\|
G-\overline G(X_t)
\right\|^2
\mid X_t
\right]
&\leq
\mathbb E
\left[
\|G\|^2
\mid X_t
\right],
\end{align*}
and Jensen's inequality,
\begin{align*}
\left\|
\overline D(X_t)
\right\|^2
&=
\left\|
\mathbb E
\left[
D\mid X_t
\right]
\right\|^2
\\
&\leq
\mathbb E
\left[
\|D\|^2
\mid X_t
\right],
\end{align*}
we obtain
\begin{align*}
\left\|
\operatorname{Cov}
\left(
D,G\mid X_t
\right)
\overline D(X_t)
\right\|^2
&\leq
\left(
\mathbb E
\left[
\|D\|^2
\mid X_t
\right]
\right)^2
\mathbb E
\left[
\|G\|^2
\mid X_t
\right].
\end{align*}
Therefore,
\begin{align*}
&\mathbb E
\left[
\left\|
D_xv_t(X_t)v_t(X_t)
\right\|^2
\right]
\\
&\leq
\frac{1}{\sigma^2}
\mathbb E
\left[
\left(
\mathbb E
\left[
\|D\|^2
\mid X_t
\right]
\right)^2
\mathbb E
\left[
\|G\|^2
\mid X_t
\right]
\right].
\end{align*}
Apply Hölder's inequality with conjugate exponents \(5/4\) and \(5\). Then
\begin{align*}
&\mathbb E
\left[
\left(
\mathbb E
\left[
\|D\|^2
\mid X_t
\right]
\right)^2
\mathbb E
\left[
\|G\|^2
\mid X_t
\right]
\right]
\\
&\leq
\left(
\mathbb E
\left[
\left(
\mathbb E
\left[
\|D\|^2
\mid X_t
\right]
\right)^{5/2}
\right]
\right)^{4/5}
\left(
\mathbb E
\left[
\left(
\mathbb E
\left[
\|G\|^2
\mid X_t
\right]
\right)^5
\right]
\right)^{1/5}.
\end{align*}
For the first factor, Jensen's inequality gives
\begin{align*}
\mathbb E
\left[
\left(
\mathbb E
\left[
\|D\|^2
\mid X_t
\right]
\right)^{5/2}
\right]
&\leq
\mathbb E
\left[
\|D\|^5
\right].
\end{align*}
For the second factor, Jensen's inequality gives
\begin{align*}
\mathbb E
\left[
\left(
\mathbb E
\left[
\|G\|^2
\mid X_t
\right]
\right)^5
\right]
&\leq
\mathbb E
\left[
\|G\|^{10}
\right].
\end{align*}
Since \(G\sim {\mathcal N}(0,I_d)\),
\begin{align*}
\mathbb E
\left[
\|G\|^{10}
\right]
&=
d(d+2)(d+4)(d+6)(d+8).
\end{align*}
Thus,
\begin{align*}
\mathbb E
\left[
\left\|
D_xv_t(X_t)v_t(X_t)
\right\|^2
\right]
&\leq
\frac{1}{\sigma^2}
\left(
\mathbb E
\left[
\|D\|^5
\right]
\right)^{4/5}
\\
&\quad\cdot
\left(
d(d+2)(d+4)(d+6)(d+8)
\right)^{1/5}.
\end{align*}

Also recall 
\begin{align*}
\mathbb E
\left[
\left\|
D_xv_t(X_t)s_t(X_t)
\right\|^2
\right]
&\leq
\frac{
\sqrt{d(d+2)(d+4)(d+6)}
}{
\sigma^4
}
\left(
\mathbb E
\left[
\|D\|^4
\right]
\right)^{1/2}.
\end{align*}

Combining the velocity part and the score part, we obtain
\begin{align*}
&\mathbb E
\left[
\left\|
D_xv_t(X_t)b_t(X_t)
\right\|^2
\right]
\\
&\leq
\frac{2}{\sigma^2}
\left(
\mathbb E
\left[
\|D\|^5
\right]
\right)^{4/5}
\left(
d(d+2)(d+4)(d+6)(d+8)
\right)^{1/5}
\\
&\quad+
\frac{\tau^4}{2\sigma^4}
\sqrt{d(d+2)(d+4)(d+6)}
\left(
\mathbb E
\left[
\|D\|^4
\right]
\right)^{1/2}.
\end{align*}
Since \(D=Y_1-Y_0\), this can be written as
\begin{align*}
&\mathbb E
\left[
\left\|
D_xv_t(X_t)b_t(X_t)
\right\|^2
\right]
\\
&\leq
\frac{2}{\sigma^2}
\left(
\mathbb E
\left[
\|Y_1-Y_0\|^5
\right]
\right)^{4/5}
\left(
d(d+2)(d+4)(d+6)(d+8)
\right)^{1/5}
\\
&\quad+
\frac{\tau^4}{2\sigma^4}
\sqrt{d(d+2)(d+4)(d+6)}
\left(
\mathbb E
\left[
\|Y_1-Y_0\|^4
\right]
\right)^{1/2}.
\end{align*}

\end{proof}

\subsection{Euler Increment Bounds}

\begin{lemma}[It\^{o} Control of the Euler Velocity Term]\label{Lemma_Vel_regularity_sum}
Let \(Y_0\), \(Y_1\), \(D\), \(S_t\), \(G\), \(X_t\), \(q_t\), \(v_t\),
and \(s_t\) be as in Lemma \ref{Lemma_Derivative_Identity}, assume
\(\mathbb E\|Y_1-Y_0\|^5<\infty\), and set
\(b_t=v_t+\frac{\tau^2}{2}s_t\).
Let $Z_t$ solve
\begin{align*}
dZ_t
&=
b_t(Z_t)dt+\tau dB_t,
\end{align*}
and assume that $Z_t\sim q_t$ for every $t\in[0,T]$.
Let \(0=t_0<t_1<\cdots<t_N=T\) be a uniform grid with step size
\(h\), so \(t_{k+1}-t_k=h\) and \(Nh=T\). Then
\begin{align*}
&\sum_{k=0}^{N-1}
\int_{t_k}^{t_{k+1}}
\mathbb E
\left[
\left\|
v_t(Z_t)-v_{t_k}(Z_{t_k})
\right\|^2
\right]dt
\\
&\leq
\frac{2T h^2}{3}
\Bigg[
\frac{3\sqrt{d(d+2)}}{\sigma^2}
\mathbb E
\left[
\|Y_1-Y_0\|^4
\right]
\\
&\quad+
\frac{6}{\sigma^2}
\left(
\mathbb E
\left[
\|Y_1-Y_0\|^5
\right]
\right)^{4/5}
\left(
d(d+2)(d+4)(d+6)(d+8)
\right)^{1/5}
\\
&\quad+
\frac{14\tau^4\sqrt{d(d+2)(d+4)(d+6)}}{\sigma^4}
\left(
\mathbb E
\left[
\|Y_1-Y_0\|^4
\right]
\right)^{1/2}
\Bigg]
\\
&\quad+
\frac{\tau^2T h\sqrt{d(d+2)}}{\sigma^2}
\left(
\mathbb E
\left[
\|Y_1-Y_0\|^4
\right]
\right)^{1/2}.
\end{align*}
All expectations in the displayed bound are taken under the coupling law of $(Y_0,Y_1)$, except the expectation on the left-hand side, which is taken under the law of the exact SDE path $(Z_t)_{0\leq t\leq T}$.
\end{lemma}

\begin{proof}
For $t\in[t_k,t_{k+1})$, It\^{o}'s formula applied to $v_t(Z_t)$ gives
\begin{align*}
v_t(Z_t)-v_{t_k}(Z_{t_k})
&=
\int_{t_k}^{t}
\left(
\partial_r v_r(Z_r)
+
D_xv_r(Z_r)b_r(Z_r)
+
\frac{\tau^2}{2}\Delta_xv_r(Z_r)
\right)dr
\\
&\quad+
\tau
\int_{t_k}^{t}
D_xv_r(Z_r)dB_r.
\end{align*}
Using $\|u+w\|^2\leq 2\|u\|^2+2\|w\|^2$, Jensen's inequality, and It\^{o} isometry, we get
\begin{align*}
&\mathbb E
\left[
\left\|
v_t(Z_t)-v_{t_k}(Z_{t_k})
\right\|^2
\right]
\\
&\leq
2(t-t_k)
\int_{t_k}^{t}
\mathbb E
\left[
\left\|
\partial_r v_r(Z_r)
+
D_xv_r(Z_r)b_r(Z_r)
+
\frac{\tau^2}{2}\Delta_xv_r(Z_r)
\right\|^2
\right]dr
\\
&\quad+
2\tau^2
\int_{t_k}^{t}
\mathbb E
\left[
\left\|
D_xv_r(Z_r)
\right\|_{\mathrm F}^2
\right]dr.
\end{align*}
Since $Z_r\sim q_r$ and $X_r\sim q_r$,
Lemma \ref{Lemma_Integrated_Derivative_Bounds} applies with \(X_r\)
replaced by \(Z_r\). Using $\|u+w+z\|^2\leq 3\|u\|^2+3\|w\|^2+3\|z\|^2$, we have
\begin{align*}
&\mathbb E
\left[
\left\|
\partial_r v_r(Z_r)
+
D_xv_r(Z_r)b_r(Z_r)
+
\frac{\tau^2}{2}\Delta_xv_r(Z_r)
\right\|^2
\right]
\\
&\leq
3
\mathbb E
\left[
\left\|
\partial_r v_r(Z_r)
\right\|^2
\right]
+
3
\mathbb E
\left[
\left\|
D_xv_r(Z_r)b_r(Z_r)
\right\|^2
\right]
+
3
\mathbb E
\left[
\left\|
\frac{\tau^2}{2}\Delta_xv_r(Z_r)
\right\|^2
\right]
\\
&\leq
\frac{3\sqrt{d(d+2)}}{\sigma^2}
\mathbb E
\left[
\|Y_1-Y_0\|^4
\right]
\\
&\quad+
\frac{6}{\sigma^2}
\left(
\mathbb E
\left[
\|Y_1-Y_0\|^5
\right]
\right)^{4/5}
\left(
d(d+2)(d+4)(d+6)(d+8)
\right)^{1/5}
\\
&\quad+
\frac{3\tau^4}{2\sigma^4}
\sqrt{d(d+2)(d+4)(d+6)}
\left(
\mathbb E
\left[
\|Y_1-Y_0\|^4
\right]
\right)^{1/2}
\\
&\quad+
\frac{12\tau^4\sqrt{d(d+2)(d+4)(d+6)}}{\sigma^4}
\left(
\mathbb E
\left[
\|Y_1-Y_0\|^4
\right]
\right)^{1/2}.
\end{align*}
Also,
\begin{align*}
\mathbb E
\left[
\left\|
D_xv_r(Z_r)
\right\|_{\mathrm F}^2
\right]
&\leq
\frac{
\sqrt{d(d+2)}
}{
\sigma^2
}
\left(
\mathbb E
\left[
\|Y_1-Y_0\|^4
\right]
\right)^{1/2}.
\end{align*}
Therefore, for $t\in[t_k,t_{k+1})$,
\begin{align*}
&\mathbb E
\left[
\left\|
v_t(Z_t)-v_{t_k}(Z_{t_k})
\right\|^2
\right]
\\
&\leq
2(t-t_k)^2
\Big[
\frac{3\sqrt{d(d+2)}}{\sigma^2}
\mathbb E
\left[
\|Y_1-Y_0\|^4
\right]
\\
&\quad+
\frac{6}{\sigma^2}
\left(
\mathbb E
\left[
\|Y_1-Y_0\|^5
\right]
\right)^{4/5}
\left(
d(d+2)(d+4)(d+6)(d+8)
\right)^{1/5}
\\
&\quad+
\frac{3\tau^4}{2\sigma^4}
\sqrt{d(d+2)(d+4)(d+6)}
\left(
\mathbb E
\left[
\|Y_1-Y_0\|^4
\right]
\right)^{1/2}
\\
&\quad+
\frac{12\tau^4\sqrt{d(d+2)(d+4)(d+6)}}{\sigma^4}
\left(
\mathbb E
\left[
\|Y_1-Y_0\|^4
\right]
\right)^{1/2}
\Big]
\\
&\quad+
2\tau^2(t-t_k)
\frac{
\sqrt{d(d+2)}
}{
\sigma^2
}
\left(
\mathbb E
\left[
\|Y_1-Y_0\|^4
\right]
\right)^{1/2}.
\end{align*}
Integrating this estimate over $t\in[t_k,t_{k+1}]$ gives
\begin{align*}
&\int_{t_k}^{t_{k+1}}
\mathbb E
\left[
\left\|
v_t(Z_t)-v_{t_k}(Z_{t_k})
\right\|^2
\right]dt
\\
&\leq
\frac{2h^3}{3}
\Bigg[
\frac{3\sqrt{d(d+2)}}{\sigma^2}
\mathbb E
\left[
\|Y_1-Y_0\|^4
\right]
\\
&\quad+
\frac{6}{\sigma^2}
\left(
\mathbb E
\left[
\|Y_1-Y_0\|^5
\right]
\right)^{4/5}
\left(
d(d+2)(d+4)(d+6)(d+8)
\right)^{1/5}
\\
&\quad+
\frac{14\tau^4\sqrt{d(d+2)(d+4)(d+6)}}{\sigma^4}
\left(
\mathbb E
\left[
\|Y_1-Y_0\|^4
\right]
\right)^{1/2}
\Bigg]
\\
&\quad+
\frac{\tau^2h^2\sqrt{d(d+2)}}{\sigma^2}
\left(
\mathbb E
\left[
\|Y_1-Y_0\|^4
\right]
\right)^{1/2}.
\end{align*}
Finally, summing over 
\(k=0,\ldots,N-1\) and using
\(Nh=T\), we obtain
\begin{align*}
&\sum_{k=0}^{N-1}
\int_{t_k}^{t_{k+1}}
\mathbb E
\left[
\left\|
v_t(Z_t)-v_{t_k}(Z_{t_k})
\right\|^2
\right]dt
\\
&\leq
\frac{2T h^2}{3}
\Bigg[
\frac{3\sqrt{d(d+2)}}{\sigma^2}
\mathbb E
\left[
\|Y_1-Y_0\|^4
\right]
\\
&\quad+
\frac{6}{\sigma^2}
\left(
\mathbb E
\left[
\|Y_1-Y_0\|^5
\right]
\right)^{4/5}
\left(
d(d+2)(d+4)(d+6)(d+8)
\right)^{1/5}
\\
&\quad+
\frac{14\tau^4\sqrt{d(d+2)(d+4)(d+6)}}{\sigma^4}
\left(
\mathbb E
\left[
\|Y_1-Y_0\|^4
\right]
\right)^{1/2}
\Bigg]
\\
&\quad+
\frac{\tau^2T h\sqrt{d(d+2)}}{\sigma^2}
\left(
\mathbb E
\left[
\|Y_1-Y_0\|^4
\right]
\right)^{1/2}.
\end{align*}
This proves the lemma.
\end{proof}

\begin{lemma}[Combined Score Increment Regularity]\label{Lemma_Score_regularity_sum}
Let $Y_0$ and $Y_1$ be random variables taking values in $\mathbb R^d$. Let $D=Y_1-Y_0$. Assume that
\begin{align*}
\mathbb E
\left[
\|Y_1-Y_0\|^4
\right]
&<
\infty,
\end{align*}
where the expectation is taken under the joint law of $(Y_0,Y_1)$.

For $t\in[0,T]$, define $S_t=Y_0+tD$. Let $G\sim {\mathcal N}(0,I_d)$ be independent of $(Y_0,D)$, and define $X_t=S_t+\sigma G$ with $\sigma>0$. Let $q_t$ be the density of $X_t$. Define
$s_t(x)=\nabla_x\log q_t(x)$ and
\begin{align*}
v_t(x)
&=
\mathbb E
\left[
Y_1-Y_0\mid X_t=x
\right].
\end{align*}
Let
\begin{align*}
b_t(x)
&=
v_t(x)+\frac{\tau^2}{2}s_t(x),
\end{align*}
where $\tau>0$. Let $Z_t$ solve the exact SDE
\begin{align*}
dZ_t
&=
b_t(Z_t)dt+\tau dB_t,
\end{align*}
and assume that $Z_t\sim q_t$ for every $t\in[0,T]$.

Let $0=t_0<t_1<\cdots<t_N=T$ be a uniform grid with step size $h$, so $t_{k+1}-t_k=h$ and $Nh=T$. Then
\begin{align*}
&\sum_{k=0}^{N-1}
\int_{t_k}^{t_{k+1}}
\mathbb E
\left[
\left\|
s_t(Z_t)-s_{t_k}(Z_{t_k})
\right\|^2
\right]dt
\\
&\leq
\frac{
\tau^2T h\left(2d(d+2)+2d\right)
}{
\sigma^4
}
\\
&\quad+
\frac{2T h^2}{3}
\left[
\frac{
64d\sqrt{d(d+2)(d+4)(d+6)}
}{
\sigma^4
}
\left(
\mathbb E
\left[
\|Y_1-Y_0\|^4
\right]
\right)^{1/2}
\right.
\\
&\qquad\left.
+
\frac{4\sqrt{
d(d+2)(d+4)(d+6)
}
}{
\sigma^4
}
\left(
\mathbb E
\left[
\|Y_1-Y_0\|^4
\right]
\right)^{1/2}
\right.
\\
&\qquad\left.
+
\frac{
\tau^4
\sqrt{
\left(8d(d+2)(d+4)(d+6)+8d^2\right)d(d+2)
}
}{
\sigma^6
}
\right.
\\
&\qquad\left.
+
\frac{
64\tau^4d(d+2)(d+4)
}{
\sigma^6
}
\right].
\end{align*}
The expectation on the left-hand side is taken under the law of the exact SDE path $(Z_t)_{0\leq t\leq T}$.
\end{lemma}

\begin{proof}
We control the combined score increment directly. This is different from decomposing it into score-time and score-space pieces.

Fix $k\in\{0,\ldots,N-1\}$ and $t\in[t_k,t_{k+1}]$. Applying It\^{o}'s formula to the process $s_r(Z_r)$ over $r\in[t_k,t]$ gives
\begin{align*}
s_t(Z_t)-s_{t_k}(Z_{t_k})
&=
\int_{t_k}^{t}
\left(
\partial_r s_r(Z_r)
+
D_xs_r(Z_r)b_r(Z_r)
+
\frac{\tau^2}{2}\Delta_xs_r(Z_r)
\right)dr
\\
&\quad+
\tau
\int_{t_k}^{t}
D_xs_r(Z_r)dB_r.
\end{align*}
Since $b_r(x)=v_r(x)+\frac{\tau^2}{2}s_r(x)$, the drift term inside the integral is
\begin{align*}
&\partial_r s_r(Z_r)
+
D_xs_r(Z_r)b_r(Z_r)
+
\frac{\tau^2}{2}\Delta_xs_r(Z_r)
\\
&=
\partial_r s_r(Z_r)
+
D_xs_r(Z_r)v_r(Z_r)
+
\frac{\tau^2}{2}D_xs_r(Z_r)s_r(Z_r)
+
\frac{\tau^2}{2}\Delta_xs_r(Z_r).
\end{align*}

Now use the continuity equation for the Gaussian-smoothed interpolation. Since $q_r$ is the density of $X_r=S_r+\sigma G$, and $v_r(x)=\mathbb E[D\mid X_r=x]$, the curve $(q_r)_{r\in[0,T]}$ satisfies
\begin{align*}
\partial_r q_r(x)
+
\nabla_x\cdot
\left(
q_r(x)v_r(x)
\right)
&=
0.
\end{align*}
Dividing by $q_r(x)$ gives
\begin{align*}
\partial_r\log q_r(x)
&=
-\nabla_x\cdot v_r(x)
-
v_r(x)^\top s_r(x).
\end{align*}
Taking $\nabla_x$ on both sides gives
\begin{align*}
\partial_r \nabla_{x} \log q_{r}(x)
= \partial_r s_r(x)
&=
-\nabla_x
\left(
\nabla_x\cdot v_r(x)
\right)
-
(D_xv_r(x))^\top s_r(x)
-
D_xs_r(x)v_r(x).
\end{align*}
Substituting $x=Z_r$ gives
\begin{align*}
\partial_r s_r(Z_r)
&=
-\nabla_x
\left(
\nabla_x\cdot v_r
\right)(Z_r)
-
(D_xv_r(Z_r))^\top s_r(Z_r)
-
D_xs_r(Z_r)v_r(Z_r).
\end{align*}
Therefore, the term $D_xs_r(Z_r)v_r(Z_r)$ cancels and we have 
\begin{align*}
&\partial_r s_r(Z_r)
+
D_xs_r(Z_r)v_r(Z_r)
+
\frac{\tau^2}{2}D_xs_r(Z_r)s_r(Z_r)
+
\frac{\tau^2}{2}\Delta_xs_r(Z_r)
\\
=& 
-\nabla_x
\left(
\nabla_x\cdot v_r
\right)(Z_r)
-
(D_xv_r(Z_r))^\top s_r(Z_r) + \frac{\tau^2}{2}D_xs_r(Z_r)s_r(Z_r) + \frac{\tau^2}{2}\Delta_xs_r(Z_r).
\end{align*}

Hence
\begin{align*}
s_t(Z_t)-s_{t_k}(Z_{t_k})
&=
\int_{t_k}^{t}
\left[
-\nabla_x
\left(
\nabla_x\cdot v_r
\right)(Z_r)
-
(D_xv_r(Z_r))^\top s_r(Z_r)
\right.
\\
&\qquad\left.
+
\frac{\tau^2}{2}D_xs_r(Z_r)s_r(Z_r)
+
\frac{\tau^2}{2}\Delta_xs_r(Z_r)
\right]dr
\\
&\quad+
\tau
\int_{t_k}^{t}
D_xs_r(Z_r)dB_r.
\end{align*}

Using $\|u+w\|^2\leq 2\|u\|^2+2\|w\|^2$, Jensen's inequality, and It\^{o} isometry, we obtain
\begin{align*}
&\mathbb E
\left[
\left\|
s_t(Z_t)-s_{t_k}(Z_{t_k})
\right\|^2
\right]
\\
&\leq
2(t-t_k)
\int_{t_k}^{t}
\mathbb E
\left[
\left\|
-\nabla_x
\left(
\nabla_x\cdot v_r
\right)(Z_r)
-
(D_xv_r(Z_r))^\top s_r(Z_r)
\right.
\right.
\\
&\qquad\left.
\left.
+
\frac{\tau^2}{2}D_xs_r(Z_r)s_r(Z_r)
+
\frac{\tau^2}{2}\Delta_xs_r(Z_r)
\right\|^2
\right]dr
\\
&\quad+
2\tau^2
\int_{t_k}^{t}
\mathbb E
\left[
\left\|
D_xs_r(Z_r)
\right\|_{\mathrm F}^2
\right]dr.
\end{align*}
Using $\|u_1+u_2+u_3+u_4\|^2\leq 4\|u_1\|^2+4\|u_2\|^2+4\|u_3\|^2+4\|u_4\|^2$, we have
\begin{align*}
&\mathbb E
\left[
\left\|
-\nabla_x
\left(
\nabla_x\cdot v_r
\right)(Z_r)
-
(D_xv_r(Z_r))^\top s_r(Z_r)
+
\frac{\tau^2}{2}D_xs_r(Z_r)s_r(Z_r)
+
\frac{\tau^2}{2}\Delta_xs_r(Z_r)
\right\|^2
\right]
\\
&\leq
4
\mathbb E
\left[
\left\|
\nabla_x
\left(
\nabla_x\cdot v_r
\right)(Z_r)
\right\|^2
\right]
+
4
\mathbb E
\left[
\left\|
(D_xv_r(Z_r))^\top s_r(Z_r)
\right\|^2
\right]
\\
&\quad+
\tau^4
\mathbb E
\left[
\left\|
D_xs_r(Z_r)s_r(Z_r)
\right\|^2
\right]
+
\tau^4
\mathbb E
\left[
\left\|
\Delta_xs_r(Z_r)
\right\|^2
\right].
\end{align*}

We now bound each term. 

First, 
by the posterior-velocity Hessian bound in Lemma
\ref{Lemma_Integrated_Score_Bounds},
\begin{align*}
\mathbb E
\left[
\left\|
\nabla_x
\left(
\nabla_x\cdot v_r
\right)(Z_r)
\right\|^2
\right]
&\leq
\frac{
16d\sqrt{d(d+2)(d+4)(d+6)}
}{
\sigma^4
}
\left(
\mathbb E
\left[
\|Y_1-Y_0\|^4
\right]
\right)^{1/2}.
\end{align*}
Also, 
\begin{align*}
\mathbb E
\left[
\left\|
(D_xv_r(Z_r))^\top s_r(Z_r)
\right\|^2
\right]
&\leq
\frac{
\sqrt{
d(d+2)(d+4)(d+6)
}
}{
\sigma^4
}
\left(
\mathbb E
\left[
\|Y_1-Y_0\|^4
\right]
\right)^{1/2}.
\end{align*}
Third, 
\begin{align*}
\mathbb E
\left[
\left\|
D_xs_r(Z_r)s_r(Z_r)
\right\|^2
\right]
&\leq
\frac{
\sqrt{
\left(8d(d+2)(d+4)(d+6)+8d^2\right)d(d+2)
}
}{
\sigma^6
}.
\end{align*}

Fourth, we use the already proved Gaussian score-geometry bound
\begin{align*}
\mathbb E
\left[
\left\|
\Delta_xs_r(Z_r)
\right\|^2
\right]
&\leq
\frac{
64d(d+2)(d+4)
}{
\sigma^6
}.
\end{align*}
Finally, using the already proved Gaussian score-geometry bound
\begin{align*}
\mathbb E
\left[
\left\|
D_xs_r(Z_r)
\right\|_{\mathrm F}^2
\right]
&\leq
\frac{
2d(d+2)+2d
}{
\sigma^4
}.
\end{align*}

Therefore, for every $r\in[0,T]$,
\begin{align*}
&\mathbb E
\left[
\left\|
-\nabla_x
\left(
\nabla_x\cdot v_r
\right)(Z_r)
-
(D_xv_r(Z_r))^\top s_r(Z_r)
+
\frac{\tau^2}{2}D_xs_r(Z_r)s_r(Z_r)
+
\frac{\tau^2}{2}\Delta_xs_r(Z_r)
\right\|^2
\right]
\\
&\leq
\frac{
64d\sqrt{d(d+2)(d+4)(d+6)}
}{
\sigma^4
}
\left(
\mathbb E
\left[
\|Y_1-Y_0\|^4
\right]
\right)^{1/2}
\\
&\quad+
\frac{4\sqrt{
d(d+2)(d+4)(d+6)
}
}{
\sigma^4
}
\left(
\mathbb E
\left[
\|Y_1-Y_0\|^4
\right]
\right)^{1/2}
\\
&\quad+
\frac{
\tau^4
\sqrt{
\left(8d(d+2)(d+4)(d+6)+8d^2\right)d(d+2)
}
}{
\sigma^6
}
\\
&\quad+
\frac{
64\tau^4d(d+2)(d+4)
}{
\sigma^6
}.
\end{align*}

Plugging this and the martingale bound into the It\^{o} estimate gives, for every $t\in[t_k,t_{k+1}]$,
\begin{align*}
&\mathbb E
\left[
\left\|
s_t(Z_t)-s_{t_k}(Z_{t_k})
\right\|^2
\right]
\\
&\leq
2(t-t_k)^2
\left[
\frac{
64d\sqrt{d(d+2)(d+4)(d+6)}
}{
\sigma^4
}
\left(
\mathbb E
\left[
\|Y_1-Y_0\|^4
\right]
\right)^{1/2}
\right.
\\
&\qquad\left.
+
\frac{4\sqrt{
d(d+2)(d+4)(d+6)
}
}{
\sigma^4
}
\left(
\mathbb E
\left[
\|Y_1-Y_0\|^4
\right]
\right)^{1/2}
\right.
\\
&\qquad\left.
+
\frac{
\tau^4
\sqrt{
\left(8d(d+2)(d+4)(d+6)+8d^2\right)d(d+2)
}
}{
\sigma^6
}
\right.
\\
&\qquad\left.
+
\frac{
64\tau^4d(d+2)(d+4)
}{
\sigma^6
}
\right]
\\
&\quad+
2\tau^2(t-t_k)
\frac{
2d(d+2)+2d
}{
\sigma^4
}.
\end{align*}

Integrating over $t\in[t_k,t_{k+1}]$ gives
\begin{align*}
&\int_{t_k}^{t_{k+1}}
\mathbb E
\left[
\left\|
s_t(Z_t)-s_{t_k}(Z_{t_k})
\right\|^2
\right]dt
\\
&\leq
\frac{2h^3}{3}
\left[
\frac{
64d\sqrt{d(d+2)(d+4)(d+6)}
}{
\sigma^4
}
\left(
\mathbb E
\left[
\|Y_1-Y_0\|^4
\right]
\right)^{1/2}
\right.
\\
&\qquad\left.
+
\frac{4
\sqrt{
d(d+2)(d+4)(d+6)
}
}{
\sigma^4
}
\left(
\mathbb E
\left[
\|Y_1-Y_0\|^4
\right]
\right)^{1/2}
\right.
\\
&\qquad\left.
+
\frac{
\tau^4
\sqrt{
\left(8d(d+2)(d+4)(d+6)+8d^2\right)d(d+2)
}
}{
\sigma^6
}
\right.
\\
&\qquad\left.
+
\frac{
64\tau^4d(d+2)(d+4)
}{
\sigma^6
}
\right]
\\
&\quad+
\tau^2h^2
\frac{
2d(d+2)+2d
}{
\sigma^4
}.
\end{align*}

Summing over $k=0,\ldots,N-1$ and using $Nh=T$, we get
\begin{align*}
&\sum_{k=0}^{N-1}
\int_{t_k}^{t_{k+1}}
\mathbb E
\left[
\left\|
s_t(Z_t)-s_{t_k}(Z_{t_k})
\right\|^2
\right]dt
\\
&\leq
\frac{2T h^2}{3}
\left[
\frac{
64d\sqrt{d(d+2)(d+4)(d+6)}
}{
\sigma^4
}
\left(
\mathbb E
\left[
\|Y_1-Y_0\|^4
\right]
\right)^{1/2}
\right.
\\
&\qquad\left.
+
\frac{4\sqrt{
d(d+2)(d+4)(d+6)
}
}{
\sigma^4
}
\left(
\mathbb E
\left[
\|Y_1-Y_0\|^4
\right]
\right)^{1/2}
\right.
\\
&\qquad\left.
+
\frac{
\tau^4
\sqrt{
\left(8d(d+2)(d+4)(d+6)+8d^2\right)d(d+2)
}
}{
\sigma^6
}
\right.
\\
&\qquad\left.
+
\frac{
64\tau^4d(d+2)(d+4)
}{
\sigma^6
}
\right]
\\
&\quad+
\frac{
\tau^2T h\left(2d(d+2)+2d\right)
}{
\sigma^4
}.
\end{align*}
This proves the lemma.

\end{proof}

\subsection{Completion of the KL Sampling Bound}

\begin{proof}[Proof of Theorem~\ref{Theorem_Sampling_Error}]
Set \(T=t_N=1\), and let \(Z\) be the exact weak solution with \(\mathcal L(Z_t)=q_t\) supplied by Lemma~\ref{lem:exact_sampling_conditions}. Let \(\widehat Z\) denote the continuous-time Euler interpolation, so that \(\widehat Z_{t_k}=\widehat Z_k\).

Write
\begin{align*}
V_h
&=
\sum_{k=0}^{N-1}\int_{t_k}^{t_{k+1}}
\mathbb E\|v_t(Z_t)-v_{t_k}(Z_{t_k})\|^2\,dt,\\
S_h
&=
\sum_{k=0}^{N-1}\int_{t_k}^{t_{k+1}}
\mathbb E\|s_t(Z_t)-s_{t_k}(Z_{t_k})\|^2\,dt.
\end{align*}
We prove the theorem in four steps.

\emph{Step 1: reduction to \(V_h\), \(S_h\), and the estimation errors.}
For \(t\in[t_k,t_{k+1})\), decompose the drift mismatch as
\begin{align*}
b_t(Z_t)-\widehat b_{t_k}(Z_{t_k})
&=
\{v_t(Z_t)-v_{t_k}(Z_{t_k})\}
+\{v_{t_k}(Z_{t_k})-\widehat v_{t_k}(Z_{t_k})\}\\
&\quad+
\frac{\tau^2}{2}\{s_t(Z_t)-s_{t_k}(Z_{t_k})\}
+\frac{\tau^2}{2}
\{s_{t_k}(Z_{t_k})-\widehat s_{t_k}(Z_{t_k})\}.
\end{align*}
Using the inequality
\(\|a_1+a_2+a_3+a_4\|^2\leq4\sum_{j=1}^4\|a_j\|^2\),
and Assumption~\ref{Assumption_Integrated_Estimation_Error} we obtain
\begin{align*}
\int_0^1
\mathbb E\|b_t(Z_t)-\widehat b_{\eta_h(t)}(Z_{\eta_h(t)})\|^2\,dt
\leq
4V_h+4\varepsilon_{\mathrm{vel}}^2
+\tau^4S_h+\tau^4\varepsilon_{\mathrm{score}}^2.
\end{align*}
Steps 2 and 3 below show that \(V_h,S_h<\infty\); hence the condition $\int_0^T \mathbb E_{\mathbb P}\left[\|\Delta_t(Z)\|^2\,\right] dt<\infty$ in Lemma~\ref{Lemma_Girsanov_Euler_KL} holds. The corresponding condition for $\widehat{Z}$ is justified by Proposition \ref{Prop_Pathwise_Drift_Finiteness}. Applying that lemma and dividing by
\(2\tau^2\) then gives
\begin{align}
\operatorname{KL}\!\left(
\mathcal L(Z_{t_N})\,\middle\|\,\mathcal L(\widehat Z_N)
\right)
\leq
\frac{2}{\tau^2}V_h
+\frac{2}{\tau^2}\varepsilon_{\mathrm{vel}}^2
+\frac{\tau^2}{2}S_h
+\frac{\tau^2}{2}\varepsilon_{\mathrm{score}}^2.
\label{eq:KL_reduction_Vh_Sh}
\end{align}

\emph{Step 2: explicit velocity-increment bound.}
Lyapunov's inequality yields
\[
\mathbb E\|Y_1-Y_0\|^4\leq D_5^4,
\qquad
\left(\mathbb E\|Y_1-Y_0\|^4\right)^{1/2}\leq D_5^2.
\]
For \(d\geq1\),
\[
\sqrt{d(d+2)}\leq d+1,\qquad
\{d(d+2)(d+4)(d+6)(d+8)\}^{1/5}\leq d+8,
\]
and
\[
\sqrt{d(d+2)(d+4)(d+6)}\leq(d+6)^2.
\]
Applying these inequalities to
Lemma~\ref{Lemma_Vel_regularity_sum} with \(T=1\) gives
\begin{align}
V_h
&\leq
\frac{2(9d+51)}{3}\frac{h^2D_5^4}{\sigma^2}
+\frac{28}{3}\frac{h^2\tau^4(d+6)^2D_5^2}{\sigma^4}
+\frac{h\tau^2(d+1)D_5^2}{\sigma^2}.
\label{eq:explicit_Vh}
\end{align}

\emph{Step 3: explicit score-increment bound.}
The two displacement-dependent terms in
Lemma~\ref{Lemma_Score_regularity_sum} satisfy
\[
(64d+4)\sqrt{d(d+2)(d+4)(d+6)}
\leq68(d+6)^3.
\]
For the two pure-score drift terms, letting
\(g_8=d(d+2)(d+4)(d+6)\), we have \(g_8\geq d^2\) and hence
\begin{align*}
\sqrt{(8g_8+8d^2)d(d+2)}
&\leq4(d+6)^3,\\
64d(d+2)(d+4)&\leq64(d+6)^3.
\end{align*}
Finally, \(2d(d+2)+2d=2d(d+3)\).  Therefore
\begin{align}
S_h
&\leq
\frac{136}{3}\frac{h^2(d+6)^3D_5^2}{\sigma^4}
+\frac{2h\tau^2d(d+3)}{\sigma^4}
+\frac{136}{3}\frac{h^2\tau^4(d+6)^3}{\sigma^6}.
\label{eq:explicit_Sh}
\end{align}

\emph{Step 4: collect constants and balance the leading terms.}
Substituting \eqref{eq:explicit_Vh} and \eqref{eq:explicit_Sh} into
\eqref{eq:KL_reduction_Vh_Sh} gives
\begin{align*}
\operatorname{KL}\!\left(
\mathcal L(Z_{t_N})\,\middle\|\,\mathcal L(\widehat Z_N)
\right)
&\leq
\frac{1}{\tau^2}
\left\{
\frac{4(9d+51)}{3}\frac{h^2D_5^4}{\sigma^2}
+2\varepsilon_{\mathrm{vel}}^2
\right\}\\
&\quad+
\tau^2
\left\{
\frac{56}{3}\frac{h^2(d+6)^2D_5^2}{\sigma^4}
+\frac{68}{3}\frac{h^2(d+6)^3D_5^2}{\sigma^4}
+\frac12\varepsilon_{\mathrm{score}}^2
\right\}\\
&\quad+
\frac{h\tau^4d(d+3)}{\sigma^4}
+\frac{68}{3}\frac{h^2\tau^6(d+6)^3}{\sigma^6}
+\frac{2h(d+1)D_5^2}{\sigma^2}.
\end{align*}
Because \((d+6)^2\leq(d+6)^3\), the coefficient in braces on the
second line is at most \(B_h\).  The first line is \(A_h/\tau^2\), and
the remaining three terms are exactly the last three terms on the
right-hand side of \eqref{eq:explicit_KL_sampling_bound}.  This proves
that bound.

Since \(A_h,B_h>0\), the choice \(\tau=(A_h/B_h)^{1/4}\) satisfies
\[
\frac{A_h}{\tau^2}+B_h\tau^2=2\sqrt{A_hB_h}.
\]
Thus this choice balances the first two terms, as claimed.
\end{proof}

\subsection[Wasserstein Conversion to the Unsmoothed Target]
{Wasserstein Conversion to the Unsmoothed Target}

\rev{The following two corollaries convert the terminal KL guarantee in
Theorem~\ref{Theorem_Sampling_Error} into \(W_1\) bounds relative to the
unsmoothed target. Corollary~\ref{cor:counterfactual_W1_from_KL} applies to the
unprojected Euler output, whereas
Corollary~\ref{cor:terminal_projection_W1} gives a sharper conversion for the
terminally projected output when the target is supported on \([-1,1]^d\).}

Throughout this subsection, set \(T=1\) and condition on the auxiliary data, the EOT target reservoir, and the resulting coupling kernels.

By Lemma \ref{lem:existence_population_valid_Q}, the population-valid comparison coupling satisfies
\begin{align*}
\int Q(\cdot\mid y_1)\int \mathbb P_{Y\mid C=c,A=a}\,d\mathbb P_C(c) (y_1)
=
\widehat{\mathbb P}_{Y\mid A=a}^{\mathrm{aux}}.
\end{align*}
Consequently, the exact interpolation associated with \(Q\) has \(Y_0\sim\widehat{\mathbb P}_{Y\mid A=a}^{\mathrm{aux}}\). Algorithm~\ref{alg:score_corrected_sde_sampler} is initialized from the same empirical source distribution. Therefore,
\begin{align*}
\mathcal L(Z_0)
&=
\widehat{\mathbb P}_{Y\mid A=a}^{\mathrm{aux}}*N(0,\sigma^2I_d)
=
\mathcal L(\widehat Z_0).
\end{align*}
Thus $\operatorname{KL}\left(\mathcal L(Z_0)\,\middle\|\,\mathcal L(\widehat Z_0)\right)=0$.

At terminal time, \(\mathcal L(Z_{t_N})=\mathcal L(Z_1)=\mathcal L(Y_1+\sigma G)=\int \mathbb P_{Y\mid C=c,A=a}\,d\mathbb P_C(c)*N(0,\sigma^2I_d)\). We treat the discrepancy between \(\mathcal L(Y_1+\sigma G)\) and \(\int \mathbb P_{Y\mid C=c,A=a}\,d\mathbb P_C(c)\) solely as Gaussian-smoothing error. No stopping time \(T<1\) is introduced.

\begin{corollary}[Wasserstein Error Relative to the Unsmoothed Target]
\label{cor:counterfactual_W1_from_KL}
Assume the conditions of Theorem \ref{Theorem_Sampling_Error} and
Assumption \ref{ass:clipping_bounded_support}.  Write the learned score as
\(\widehat s_t(x)=\{\widehat u_t(x)-x\}/\sigma^2\), and assume the network
class bounds
\[
\sup_{x,k}\|\widehat v_{t_k}(x)\|\leq2\sqrt d,
\qquad
\sup_{x,k}\|\widehat u_{t_k}(x)\|\leq\sqrt d.
\]
Assume also \(h\tau^2/(2\sigma^2)\leq1\).
Then
\begin{align*}
&W_1\left(
\int \mathbb P_{Y\mid C=c,A=a}\,d\mathbb P_C(c),
\mathcal L(\widehat Z_N)
\right) \\
&\leq
\sigma\sqrt d +
2^{1/4}\sqrt d
\left[
1+\sigma^2+
\left(
\sqrt{1+\sigma^2}
+3
+\min\{\tau,\sqrt2\sigma\}
\right)^2
\right]^{1/2} \cdot
\operatorname{KL}\left(
\mathcal L(Y_1+\sigma G)
\,\middle\|\,
\mathcal L(\widehat Z_N)
\right)^{1/4}.
\end{align*}
\end{corollary}

\begin{proof}
Let \(Y_1\sim\int \mathbb P_{Y\mid C=c,A=a}\,d\mathbb P_C(c)\) and \(G\sim {\mathcal N}(0,I_d)\) be independent. Coupling \(Y_1\) with \(Y_1+\sigma G\) gives
\begin{align*}
W_1(\int \mathbb P_{Y\mid C=c,A=a}\,d\mathbb P_C(c),\mathcal L(Y_1+\sigma G))
&\leq
\sigma\mathbb E\|G\|
\leq
\sigma\sqrt d.
\end{align*}

Let \((X,\widetilde X)\) be a maximal coupling of
\(\mathcal L(Y_1+\sigma G)\) and \(\mathcal L(\widehat Z_N)\).  Then
\[
\mathbb P(X\neq\widetilde X)
=
\operatorname{TV}\!\left(
\mathcal L(Y_1+\sigma G),\mathcal L(\widehat Z_N)
\right).
\]
Therefore
\begin{align*}
W_1\left(
\mathcal L(Y_1+\sigma G),
\mathcal L(\widehat Z_N)
\right)
&\leq
\mathbb E\|X-\widetilde X\|
=
\mathbb E\left[\|X-\widetilde X\|\mathbf 1\{X\neq\widetilde X\}\right]\\
&\leq
\left(
\mathbb E\|X-\widetilde X\|^2
\right)^{1/2}
\operatorname{TV}\left(
\mathcal L(Y_1+\sigma G),
\mathcal L(\widehat Z_N)
\right)^{1/2}\\
&\leq
\left(
2\mathbb E\|X\|^2
+
2\mathbb E\|\widetilde X\|^2
\right)^{1/2}
\left[
\frac{1}{2}
\operatorname{KL}\left(
\mathcal L(Y_1+\sigma G)
\,\middle\|\,
\mathcal L(\widehat Z_N)
\right)
\right]^{1/4}\\
&=
2^{1/4}
\left(
\mathbb E\|X\|^2
+
\mathbb E\|\widetilde X\|^2
\right)^{1/2}
\operatorname{KL}\left(
\mathcal L(Y_1+\sigma G)
\,\middle\|\,
\mathcal L(\widehat Z_N)
\right)^{1/4}.
\end{align*}
Since \(Y_1\in[-1,1]^d\),
\begin{align*}
\mathbb E\|X\|^2=\mathbb E\|Y_1+\sigma G\|^2=\mathbb E\|Y_1\|^2+\sigma^2\mathbb E\|G\|^2 \leq d(1+\sigma^2).
\end{align*}

It remains to bound the second moment of \(\widehat Z_N\). By the definition of the learned score class,
\begin{align*}
\widehat s_t(x)
=
\frac{\widehat u_t(x)-x}{\sigma^2},
\end{align*}
the Euler recursion can be written as
\[
a_h
:=
1-\frac{h\tau^2}{2\sigma^2}\in[0,1],
\qquad
\xi_k
:=
\widehat v_{t_k}(\widehat Z_k)
+\frac{\tau^2}{2\sigma^2}\widehat u_{t_k}(\widehat Z_k).
\]
\begin{align*}
\widehat Z_{k+1}
&=
a_h\widehat Z_k+h\xi_k+\tau\sqrt h\,G_k.
\end{align*}
Iterating this recursion gives
\begin{align*}
\widehat Z_N
&=
a_h^N\widehat Z_0+
h\sum_{k=0}^{N-1}
a_h^{N-1-k}\xi_k
\\
&\quad+
\tau\sqrt h
\sum_{k=0}^{N-1}
a_h^{N-1-k}G_k.
\end{align*}

Applying Minkowski's inequality in \(L^2\), we obtain
\begin{align*}
\left(\mathbb E[\|\widehat Z_N\|^2]\right)^{1/2}
\leq\;&
a_h^N\left(\mathbb E[\|\widehat Z_0\|^2]\right)^{1/2}+
h\sum_{k=0}^{N-1}
a_h^{N-1-k}
\left(\mathbb E\|\xi_k\|^2\right)^{1/2}+
\tau\sqrt h
\left(
\mathbb E\left[
\left\|
\sum_{k=0}^{N-1}
a_h^{N-1-k}G_k
\right\|^2
\right]
\right)^{1/2}.
\end{align*}
Since we have
\begin{align*}
\left\|
\widehat v_{t_k}(\widehat Z_k)
+\frac{\tau^2}{2\sigma^2}\widehat u_{t_k}(\widehat Z_k)
\right\|
\leq
\left(2+\frac{\tau^2}{2\sigma^2}\right)\sqrt d, \quad \mathbb E\left[
\left\|
\sum_{k=0}^{N-1}
\left(1-\frac{h\tau^2}{2\sigma^2}\right)^{N-1-k}G_k
\right\|^2
\right]
=
d\sum_{k=0}^{N-1}
\left(1-\frac{h\tau^2}{2\sigma^2}\right)^{2k}
\end{align*}
where \(G_k\) are independent, \(\mathbb E[\|G_k\|^2]=d\), and all cross terms vanish.
Therefore,
\begin{align*}
\left(\mathbb E[\|\widehat Z_N\|^2]\right)^{1/2}
\leq &
\left(\mathbb E[\|\widehat Z_0\|^2]\right)^{1/2}
+h\left(2+\frac{\tau^2}{2\sigma^2}\right)\sqrt d
\sum_{k=0}^{N-1}
\left(1-\frac{h\tau^2}{2\sigma^2}\right)^k +\tau
\left[
hd\sum_{k=0}^{N-1}
\left(1-\frac{h\tau^2}{2\sigma^2}\right)^{2k}
\right]^{1/2}.
\end{align*}

Since \(Nh=1\), $h\sum_{k=0}^{N-1}
\left(1-\frac{h\tau^2}{2\sigma^2}\right)^k \le h N = 1$, so we have
\begin{align*}
h\sum_{k=0}^{N-1}
\left(1-\frac{h\tau^2}{2\sigma^2}\right)^k
&=
h\frac{2\sigma^2}{h\tau^2}
\left[
1-\left(1-\frac{h\tau^2}{2\sigma^2}\right)^N
\right]
\leq
\min\left\{1,\frac{2\sigma^2}{\tau^2}\right\},
\end{align*}
and
\begin{align*}
h\sum_{k=0}^{N-1}
\left(1-\frac{h\tau^2}{2\sigma^2}\right)^{2k}
\leq
h\sum_{k=0}^{N-1}
\left(1-\frac{h\tau^2}{2\sigma^2}\right)^k
\leq
\min\left\{1,\frac{2\sigma^2}{\tau^2}\right\}.
\end{align*}
Furthermore, observe that $ \left(2+\frac{\tau^2}{2\sigma^2}\right) \min\left\{1,\frac{2\sigma^2}{\tau^2}\right\} \leq 3$ and $\tau \left( \min\left\{1,\frac{2\sigma^2}{\tau^2}\right\} \right)^{1/2}= \min\{\tau,\sqrt{2}\sigma\}$.
It follows that
\begin{align*}
\left(\mathbb E[\|\widehat Z_N\|^2]\right)^{1/2}
\leq
\left(\mathbb E[\|\widehat Z_0\|^2]\right)^{1/2}
+3\sqrt d+\sqrt d\min\{\tau,\sqrt{2}\sigma\}.
\end{align*}
Since $\mathbb E[\|\widehat Z_0\|^2] = \mathbb E[\|\widehat Y_0\|^2]+\sigma^2d \leq d(1+\sigma^2)$,
we obtain
\begin{align*}
\mathbb E[\|\widehat Z_N\|^2]
\leq
d\left(
\sqrt{1+\sigma^2}+3+\min\{\tau,\sqrt{2}\sigma\}
\right)^2.
\end{align*}

It follows that
\begin{align*}
\mathbb E\|\widehat Z_N\|^2
&\leq
d\left(
\sqrt{1+\sigma^2}
+3
+\min\{\tau,\sqrt2\sigma\}
\right)^2.
\end{align*}
Combining the preceding moment bounds yields
\begin{align*}
W_1\left(
\mathcal L(Y_1+\sigma G),
\mathcal L(\widehat Z_N)
\right)
&\leq
2^{1/4}\sqrt d
\left[
1+\sigma^2+
\left(
\sqrt{1+\sigma^2}
+3
+\min\{\tau,\sqrt2\sigma\}
\right)^2
\right]^{1/2}\\
&\qquad\cdot
\operatorname{KL}\left(
\mathcal L(Y_1+\sigma G)
\,\middle\|\,
\mathcal L(\widehat Z_N)
\right)^{1/4}.
\end{align*}
The triangle inequality through \(\mathcal L(Y_1+\sigma G)\) proves the result.
\end{proof}

\begin{revision}
\begin{corollary}[Wasserstein Error After Output Clipping]
\label{cor:terminal_projection_W1}
Let \(Y_1\sim\mu_1\) be supported on \(K=[-1,1]^d\), let
\(G\sim\mathcal N(0,I_d)\) be independent of \(Y_1\), and let
\(\widehat Z_N\) be any \(\mathbb R^d\)-valued random vector. Define the
Euclidean projection
\[
\Pi=\operatorname{clip}_1:\mathbb R^d\to K,
\qquad
\widetilde Z_N=\Pi(\widehat Z_N).
\]
Then
\begin{align}
\label{eq:terminal_projection_W1}
W_1\!\left(\mu_1,\mathcal L(\widetilde Z_N)\right)
\leq \sigma\sqrt d+
\sqrt{
2d\,\operatorname{KL}\!\left(
\mathcal L(Y_1+\sigma G)
\middle\|
\mathcal L(\widehat Z_N)
\right)
}.
\end{align}
Under the conditions of Theorem~\ref{Theorem_Sampling_Error}, its KL bound
may be substituted into the right-hand side of
\eqref{eq:terminal_projection_W1}. Clipping is applied only after the final Euler step.
\end{corollary}

\begin{proof}
Write \(P_\sigma=\mathcal L(Y_1+\sigma G)\) and
\(\widehat P=\mathcal L(\widehat Z_N)\). Since
\(\Pi(Y_1)=Y_1\) and Euclidean projection onto \(K\) is nonexpansive,
\begin{align*}
W_1(\mu_1,\Pi_\#P_\sigma)
&\leq
\mathbb E\|\Pi(Y_1)-\Pi(Y_1+\sigma G)\|
\leq \sigma\mathbb E\|G\|
\leq \sigma\sqrt d.
\end{align*}
Both \(\Pi_\#P_\sigma\) and \(\Pi_\#\widehat P\) are supported on \(K\),
whose diameter is \(2\sqrt d\). A maximal coupling, Pinsker's inequality,
and data processing for KL divergence therefore give
\begin{align*}
W_1(\Pi_\#P_\sigma,\Pi_\#\widehat P)
&\leq
2\sqrt d\,\operatorname{TV}(\Pi_\#P_\sigma,\Pi_\#\widehat P)\\
&\leq
\sqrt{
2d\,\operatorname{KL}(\Pi_\#P_\sigma\|\Pi_\#\widehat P)
}\\
&\leq
\sqrt{
2d\,\operatorname{KL}(P_\sigma\|\widehat P)
}.
\end{align*}
The triangle inequality proves \eqref{eq:terminal_projection_W1}.
\end{proof}

This is a guarantee for the clipped output
\(\widetilde Z_N\); the KL bound in
Theorem~\ref{Theorem_Sampling_Error} continues to concern the Euler output \(\widehat Z_N\) without clipping.
\end{revision}

%% file: sect_Q.tex
\section[Coupling Construction and Replacement]
{Coupling Construction and Replacement}
\label{Sect_EOT}

We first introduce some notations.
For any distribution \(\zeta\) and any conditional distribution \(K(dy_0\mid y_1)\), write
\begin{align*}
\Gamma_{\zeta,K}(dy_0,dy_1)=K(dy_0\mid y_1)\zeta(dy_1).
\end{align*}
This is the joint law obtained by drawing \(Y_1\sim\zeta\), and then drawing \(Y_0\) from \(K(\cdot\mid Y_1)\).
On \(\mathbb R^d\times\mathbb R^d\), define
\[
\rho\bigl((y_0,y_1),(\widetilde y_0,\widetilde y_1)\bigr)
=
\|y_0-\widetilde y_0\|+\|y_1-\widetilde y_1\|,
\]
and let \(W_{1,\rho}\) denote the \(1\)-Wasserstein distance induced by
\(\rho\).

We first construct \(\widehat Q_n\) using finitely many samples from the prediction oracle. 
Define the empirical auxiliary covariate distribution by
\(\widehat{\mathbb P}_C := n^{-1}\sum_{i=n+1}^{2n}\delta_{C_i}\).
Conditional on \(\mathcal D_{\mathrm{aux}}\), for each \(i=n+1,\ldots,2n\), draw
\begin{align*}
\widehat Y_{i,1}^{Q},\ldots,\widehat Y_{i,M_Q}^{Q}\overset{\mathrm{ind.}}{\sim}\predorac\ \text{evaluated at }c=C_i.
\end{align*}
We then define
\begin{align*}
\widehat{\mathbb P}_{Y_1}^{Q}=\frac{1}{nM_Q}\sum_{i=n+1}^{2n}\sum_{r=1}^{M_Q}\delta_{\widehat Y_{i,r}^{Q}},
\end{align*}
which approximates \(\int\predorac\,d\widehat{\mathbb P}_C(c)\).
To keep the coupling arguments readable, write
\begin{align*}
\mu_1
&:=
\int \mathbb P_{Y\mid C=c,A=a}\,d\mathbb P_C(c),
&
\overline\mu_Q
&:=
\int\predorac\,d\widehat{\mathbb P}_C(c),
&
\mu_Q^{\mathrm{MC}}
&:=
\widehat{\mathbb P}_{Y_1}^{Q}.
\end{align*}
Thus \(\mu_1\) is the population counterfactual target,
\(\overline\mu_Q\) is the fitted target averaged over auxiliary covariates,
and \(\mu_Q^{\mathrm{MC}}\) is its finite-\(M_Q\) empirical approximation.

In this section, we focus on the entropic optimal transport kernel between \(\widehat{\mathbb P}_{Y\mid A=a}^{\mathrm{aux}}\) and \(\mu_Q^{\mathrm{MC}}\):
\begin{align*}
\eotsol\in\operatorname*{arg\,min}_{\Gamma\in\Pi\left(\widehat{\mathbb P}_{Y\mid A=a}^{\mathrm{aux}},\mu_Q^{\mathrm{MC}}\right)}
\bigg\{\int\|y_1-y_0\|^2\,\Gamma(dy_0,dy_1)
+\varepsilon_{\mathrm{OT}}\operatorname{KL}\left(\Gamma\,\middle\|\,\widehat{\mathbb P}_{Y\mid A=a}^{\mathrm{aux}}\otimes\mu_Q^{\mathrm{MC}}\right)\bigg\}.
\end{align*}
When \(\varepsilon_{\mathrm{OT}}\) is close to zero, the EOT objective is close to the unregularized optimal transport problem between \(\widehat{\mathbb P}_{Y\mid A=a}^{\mathrm{aux}}\) and \(\mu_Q^{\mathrm{MC}}\).

For a generic kernel, one would need to assume regularity of \(y_1\mapsto \widehat Q_n(\cdot\mid y_1)\). This regularity is needed because we need to control the coupling replacement error. The proof evaluates \(\widehat Q_n\) not only at samples from \(\mu_Q^{\mathrm{MC}}\), where the EOT marginal identity \(\int \widehat Q_n(\cdot\mid y_1)\mu_Q^{\mathrm{MC}}(dy_1)=\widehat{\mathbb P}_{Y\mid A=a}^{\mathrm{aux}}\) holds, but also at \(y_1\)'s from the population target. 
For the EOT kernel, this needed regularity can be proved directly: Lemma \ref{Lemma_EOT_Stability} shows that \(y_1\mapsto \widehat Q_n(\cdot\mid y_1)\) is Lipschitz in \(W_1\), with constant proportional to \(1/\varepsilon_{\mathrm{OT}}\). We then establish a coupling replacement error bound, see Lemma \ref{lem:existence_population_valid_Q} and Corollary \ref{Cor_Coupling_Replacement}. 

\subsection[The Learned EOT Kernel]
{The Learned EOT Kernel}

The entropic optimal transport problem defining \(\eotsol\) has a unique optimizer. By \cite[Proposition~2.1]{genevay2016stochastic}, there exist \(\widehat\varphi_n\) and \(\widehat\psi_n\) such that
\begin{align*}
d\eotsol(y_0,y_1)
=\exp\left(\frac{\widehat\varphi_n(y_0)+\widehat\psi_n(y_1)-\|y_0-y_1\|^2}{\varepsilon_{\mathrm{OT}}}\right) d\widehat{\mathbb P}_{Y\mid A=a}^{\mathrm{aux}}(y_0) d\mu_Q^{\mathrm{MC}}(y_1)
\end{align*}
for almost every \((y_0,y_1)\). 
Consequently, for \(\mu_Q^{\mathrm{MC}}\)-almost every
\(y\), the conditional distribution of \(Y_0\) given \(Y_1=y\) under
\(\eotsol\) is
\begin{align*}
\widehat Q_n(dy_0\mid y)
=
\frac{\exp\left(\frac{\widehat\varphi_n(y_0)-\|y_0-y\|^2}{\varepsilon_{\mathrm{OT}}}\right)}
{\int\exp\left(\frac{\widehat\varphi_n(z)-\|z-y\|^2}{\varepsilon_{\mathrm{OT}}}\right)\widehat{\mathbb P}_{Y\mid A=a}^{\mathrm{aux}}(dz)}
\widehat{\mathbb P}_{Y\mid A=a}^{\mathrm{aux}}(dy_0).
\end{align*}
Motivated by the out-of-sample entropic extension discussed before
Equation~(13) of \cite{pooladian2021entropic}, we use the displayed Gibbs
formula to define \(\widehat Q_n(\cdot\mid y)\) for every
\(y\in\mathbb R^d\). 
We remark that this is an extension of the conditional kernel.
The denominator is finite and strictly positive, so the right-hand side is a
probability distribution for every \(y\).
It holds that 
\begin{align*}
\eotsol(dy_0,dy_1)=\widehat Q_n(dy_0\mid y_1)\mu_Q^{\mathrm{MC}}(dy_1).
\end{align*}
Moreover, under Assumption \ref{ass:clipping_bounded_support}, we have $\widehat Q_n([-1,1]^d\mid y)=1,\forall y\in\mathbb R^d$.
The same kernel can be evaluated at target queries drawn from the training-fold fitted distribution or from the population target distribution. The resulting joint distribution remains well defined, but its \(Y_0\)-marginal need not equal \(\widehat{\mathbb P}_{Y\mid A=a}^{\mathrm{aux}}\), because the exact source-marginal identity holds only when \(\widehat Q_n\) is integrated against \(\mu_Q^{\mathrm{MC}}\).

By construction, we have the marginal identity
\begin{align*}
\int\widehat Q_n(\cdot\mid y_1)\mu_Q^{\mathrm{MC}}(dy_1)=\widehat{\mathbb P}_{Y\mid A=a}^{\mathrm{aux}}.
\end{align*}
This identity is not an identity under \(\int\predorac\,d\widehat{\mathbb P}_C(c)\), nor under the training-fold fitted target distribution \(\frac{1}{n}\sum_{i=1}^n\widehat{\mathbb P}_{Y\mid C=C_i,A=a}\).

\begin{lemma}[Stability of the Finite-Sample Entropic Coupling]
\label{Lemma_EOT_Stability}
Assume Assumption \ref{ass:clipping_bounded_support}. Then, for every \(y_1,y_1'\in\mathbb R^d\),
\begin{align*}
W_1\left(\widehat Q_n(\cdot\mid y_1),\widehat Q_n(\cdot\mid y_1')\right)\leq\frac{2d}{\varepsilon_{\mathrm{OT}}}\|y_1-y_1'\|.
\end{align*}
\end{lemma}

\begin{proof}
Recall that 
\begin{align*}
\widehat Q_n(dy_0\mid y)
=
\frac{\exp\left(\frac{\widehat\varphi_n(y_0)-\|y_0-y\|^2}{\varepsilon_{\mathrm{OT}}}\right)}
{\int\exp\left(\frac{\widehat\varphi_n(z)-\|z-y\|^2}{\varepsilon_{\mathrm{OT}}}\right)\widehat{\mathbb P}_{Y\mid A=a}^{\mathrm{aux}}(dz)}
\widehat{\mathbb P}_{Y\mid A=a}^{\mathrm{aux}}(dy_0).
\end{align*}
Fix \(y_1,y_1'\in\mathbb R^d\) and a \(1\)-Lipschitz function \(f:\mathbb R^d\to\mathbb R\). For \(t\in[0,1]\), substituting \(y=(1-t)y_1+ty_1'\) into the preceding display gives
\begin{align*}
\int f(y_0)\widehat Q_n\left(dy_0\mid(1-t)y_1+ty_1'\right) =
\frac{\int f(y_0)\exp\left(\frac{\widehat\varphi_n(y_0)-\|y_0-((1-t)y_1+ty_1')\|^2}{\varepsilon_{\mathrm{OT}}}\right)\widehat{\mathbb P}_{Y\mid A=a}^{\mathrm{aux}}(dy_0)}
{\int\exp\left(\frac{\widehat\varphi_n(z)-\|z-((1-t)y_1+ty_1')\|^2}{\varepsilon_{\mathrm{OT}}}\right)\widehat{\mathbb P}_{Y\mid A=a}^{\mathrm{aux}}(dz)}.
\end{align*}
For every \(y_0\in[-1,1]^d\),
\begin{align*}
\frac{d}{dt}\left[-\frac{\|y_0-((1-t)y_1+ty_1')\|^2}{\varepsilon_{\mathrm{OT}}}\right]
=
\frac{2}{\varepsilon_{\mathrm{OT}}}
\left\langle y_0-((1-t)y_1+ty_1'),y_1'-y_1\right\rangle.
\end{align*}
Applying the quotient rule to the preceding ratio yields
\begin{align*}
&\frac{d}{dt}\int f(y_0)\widehat Q_n\left(dy_0\mid(1-t)y_1+ty_1'\right)\\
&=
\frac{2}{\varepsilon_{\mathrm{OT}}}
\int f(y_0)
\left\langle y_0-((1-t)y_1+ty_1'),y_1'-y_1\right\rangle
\widehat Q_n\left(dy_0\mid(1-t)y_1+ty_1'\right)\\
&\quad-
\frac{2}{\varepsilon_{\mathrm{OT}}}
\int f(y_0)\widehat Q_n\left(dy_0\mid(1-t)y_1+ty_1'\right)
\int
\left\langle z-((1-t)y_1+ty_1'),y_1'-y_1\right\rangle
\widehat Q_n\left(dz\mid(1-t)y_1+ty_1'\right)\\
&=
\frac{2}{\varepsilon_{\mathrm{OT}}}
\operatorname{Cov}_{Y_0\sim\widehat Q_n(\cdot\mid(1-t)y_1+ty_1')}
\left(f(Y_0),\left\langle Y_0,y_1'-y_1\right\rangle\right).
\end{align*}
Assumption \ref{ass:clipping_bounded_support} implies that \(\widehat Q_n(\cdot\mid(1-t)y_1+ty_1')\) is supported on \([-1,1]^d\). Since \(f\) is \(1\)-Lipschitz,
\begin{align*}
\sup_{y_0,z\in[-1,1]^d}|f(y_0)-f(z)|\leq2\sqrt d, \quad \text{and} \quad \sup_{y_0,z\in[-1,1]^d}\left|\left\langle y_0-z,y_1'-y_1\right\rangle\right|\leq2\sqrt d\,\|y_1'-y_1\|.
\end{align*}
Therefore, Popoviciu's variance inequality gives
\begin{align*}
\operatorname{Var}_{Y_0\sim\widehat Q_n(\cdot\mid(1-t)y_1+ty_1')}\left(f(Y_0)\right)&\leq d,\\
\operatorname{Var}_{Y_0\sim\widehat Q_n(\cdot\mid(1-t)y_1+ty_1')}\left(\left\langle Y_0,y_1'-y_1\right\rangle\right)&\leq d\|y_1'-y_1\|^2.
\end{align*}
By Cauchy--Schwarz,
\begin{align*}
\left|\operatorname{Cov}_{Y_0\sim\widehat Q_n(\cdot\mid(1-t)y_1+ty_1')}
\left(f(Y_0),\left\langle Y_0,y_1'-y_1\right\rangle\right)\right|
\leq d\|y_1'-y_1\|.
\end{align*}
Hence, for every \(t\in[0,1]\),
\begin{align*}
\left|\frac{d}{dt}\int f(y_0)\widehat Q_n\left(dy_0\mid(1-t)y_1+ty_1'\right)\right|
\leq\frac{2d}{\varepsilon_{\mathrm{OT}}}\|y_1'-y_1\|.
\end{align*}
Integrating over \(t\in[0,1]\) yields
\begin{align*}
\left|\int f(y_0)\widehat Q_n(dy_0\mid y_1')-\int f(y_0)\widehat Q_n(dy_0\mid y_1)\right|
\leq\frac{2d}{\varepsilon_{\mathrm{OT}}}\|y_1'-y_1\|.
\end{align*}
Taking the supremum over all \(1\)-Lipschitz \(f\) and applying Kantorovich--Rubinstein duality gives
\begin{align*}
W_1\left(\widehat Q_n(\cdot\mid y_1),\widehat Q_n(\cdot\mid y_1')\right)
\leq\frac{2d}{\varepsilon_{\mathrm{OT}}}\|y_1-y_1'\|.
\end{align*}
\end{proof}

\begin{proposition}[Target-Query Discrepancy Under Empirical \(C\)]
\label{prop:target_query_discrepancy_empirical_C}
Assume Assumptions \ref{ass:clipping_bounded_support} and \ref{ass:psi_empirical_training_W1_accuracy}. Then
\begin{align*}
&W_1\left(\int \mathbb P_{Y\mid C=c,A=a}\,d\mathbb P_C(c),\frac{1}{n}\sum_{i=n+1}^{2n}\widehat{\mathbb P}_{Y\mid C=C_i,A=a}\right)\\
&\leq W_1\left(\int \mathbb P_{Y\mid C=c,A=a}\,d\mathbb P_C(c),\frac{1}{n}\sum_{i=n+1}^{2n}\mathbb P_{Y\mid C=C_i,A=a}\right)+\epspsiemp.
\end{align*}
Consequently, on the event of Lemma \ref{lem:empirical_covariate_standardization_error},
\begin{align*}
W_1\left(\int \mathbb P_{Y\mid C=c,A=a}\,d\mathbb P_C(c),\frac{1}{n}\sum_{i=n+1}^{2n}\widehat{\mathbb P}_{Y\mid C=C_i,A=a}\right)\lesssim n^{-1/(d+2)}+\epspsiemp.
\end{align*}
\end{proposition}

\begin{proof}
The triangle inequality gives
\begin{align*}
&W_1\left(\int \mathbb P_{Y\mid C=c,A=a}\,d\mathbb P_C(c),\frac{1}{n}\sum_{i=n+1}^{2n}\widehat{\mathbb P}_{Y\mid C=C_i,A=a}\right)\\
&\leq W_1\left(\int \mathbb P_{Y\mid C=c,A=a}\,d\mathbb P_C(c),\frac{1}{n}\sum_{i=n+1}^{2n}\mathbb P_{Y\mid C=C_i,A=a}\right)\\
&\quad+W_1\left(\frac{1}{n}\sum_{i=n+1}^{2n}\mathbb P_{Y\mid C=C_i,A=a},\frac{1}{n}\sum_{i=n+1}^{2n}\widehat{\mathbb P}_{Y\mid C=C_i,A=a}\right).
\end{align*}
For each \(i=n+1,\ldots,2n\), take an optimal coupling of \(\mathbb P_{Y\mid C=C_i,A=a}\) and \(\widehat{\mathbb P}_{Y\mid C=C_i,A=a}\), and mix these couplings uniformly over \(i\). This gives
\begin{align*}
&W_1\left(\frac{1}{n}\sum_{i=n+1}^{2n}\mathbb P_{Y\mid C=C_i,A=a},\frac{1}{n}\sum_{i=n+1}^{2n}\widehat{\mathbb P}_{Y\mid C=C_i,A=a}\right)\\
&\leq\frac{1}{n}\sum_{i=n+1}^{2n}W_1\left(\mathbb P_{Y\mid C=C_i,A=a},\widehat{\mathbb P}_{Y\mid C=C_i,A=a}\right)\\
&\leq\left[\frac{1}{n}\sum_{i=n+1}^{2n}W_1\left(\mathbb P_{Y\mid C=C_i,A=a},\widehat{\mathbb P}_{Y\mid C=C_i,A=a}\right)^2\right]^{1/2}\leq\epspsiemp.
\end{align*}
Combining this bound with Lemma \ref{lem:empirical_covariate_standardization_error} proves the claim.
\end{proof}

Conditional on \(\mathcal D_{\mathrm{aux}}\), \(\mathcal D_{\mathrm{tr}}\), and \(\{\widehat Y_{i,r}^{Q}\}\), the offline final-flow Monte Carlo array is drawn independently of the coupling reservoir:
\begin{align}\label{eq_offline_sample_1}
\widehat Y_{i,1},\ldots,\widehat Y_{i,M}\overset{\mathrm{ind.}}{\sim}\predorac\ \text{evaluated at }c=C_i,\qquad i=1,\ldots,n.
\end{align}
For each \(i\) and \(k\), draw independently
\begin{align}\label{eq_offline_sample_2}
\widehat Y_{0,i,k}\sim\widehat Q_n(\cdot\mid\widehat Y_{i,k}),\qquad K_{i,k}\sim\operatorname{Unif}\{0,\ldots,N-1\},\qquad G_{i,k}\sim {\mathcal N}(0,I_d),
\end{align}
and set \(t_{i,k}=t_{K_{i,k}}\). For every \(i\) with \(A_i=a\), draw independently
\begin{align}\label{eq_offline_sample_3}
\widehat Y_{0,i}^{\mathrm{obs}}\sim\widehat Q_n(\cdot\mid Y_i),\qquad K_i^{\mathrm{obs}}\sim\operatorname{Unif}\{0,\ldots,N-1\},\qquad G_i^{\mathrm{obs}}\sim {\mathcal N}(0,I_d),
\end{align}
and set \(t_i^{\mathrm{obs}}=t_{K_i^{\mathrm{obs}}}\). These variables are drawn once before minimizing the final empirical objective and are held fixed during optimization. The \(M\)-sample empirical-process error for this fixed objective is separate from the \(M_Q\)-sample coupling-reservoir error above.

\begin{assumption}[Nonempty Treatment-\(a\) Auxiliary Sample]
\label{ass:nonempty_treatment_a_auxiliary_sample}
The auxiliary split contains at least one observation with treatment \(A=a\):
\begin{align*}
\sum_{i=n+1}^{2n}\mathbf 1\{A_i=a\}\geq1.
\end{align*}
\end{assumption}

Under the overlap condition
\(\mathbb P(A=a\mid C=c)\geq\pi_{\min}\), the event in this assumption has
probability at least
\[
1-(1-\pi_{\min})^n
\geq
1-e^{-n\pi_{\min}}.
\]
Thus the assumption may equivalently be replaced by this additional
exponentially small failure probability.

\begin{lemma}[Existence of an Ideal Coupling \(Q\) With the Correct Population Marginals]
\label{lem:existence_population_valid_Q}
Under Assumptions \ref{ass:clipping_bounded_support} and
\ref{ass:nonempty_treatment_a_auxiliary_sample}, let \(\lambda^\star\) be an
optimal coupling of \(\mu_1\) and \(\mu_Q^{\mathrm{MC}}\), and write
\(\lambda^\star(d\widehat y_1\mid y_1)\) for the conditional law of its
second coordinate given the first. Define
\[
Q(dy_0\mid y_1)
:=
\int
\widehat Q_n(dy_0\mid\widehat y_1)\,
\lambda^\star(d\widehat y_1\mid y_1).
\]
Then \(Q\) satisfies the marginal identity

\begin{align*}
\int Q(\cdot\mid y_1)\,\mu_1(dy_1)
&=
\widehat{\mathbb P}_{Y\mid A=a}^{\mathrm{aux}}
\end{align*}
and the coupling-replacement bound
\begin{align*}
W_{1,\rho}\!\left(
\Gamma_{\mu_1,Q},
\Gamma_{\mu_1,\widehat Q_n}
\right)
&\leq
\frac{2d}{\varepsilon_{\mathrm{OT}}}
W_1\!\left(\mu_1,\mu_Q^{\mathrm{MC}}\right).
\end{align*}
\end{lemma}

\begin{proof}
Under the construction in the statement, draw
\((Y_1,\widehat Y_1)\sim\lambda^\star\), followed by
\[
Y_0\sim\widehat Q_n(\cdot\mid\widehat Y_1).
\]
The EOT marginal identity implies
\[
Y_0\sim
\int\widehat Q_n(\cdot\mid y)\mu_Q^{\mathrm{MC}}(dy)
=\widehat{\mathbb P}_{Y\mid A=a}^{\mathrm{aux}}.
\]
Thus \(Q\) has the displayed marginal identity.

Conditional on \((Y_1,\widehat Y_1)\), couple
\(\widetilde Y_0\sim\widehat Q_n(\cdot\mid Y_1)\) and
\(Y_0\sim\widehat Q_n(\cdot\mid\widehat Y_1)\) optimally.  Then
\(\bigl((Y_0,Y_1),(\widetilde Y_0,Y_1)\bigr)\) couples
\(\Gamma_{\mu_1,Q}\) and \(\Gamma_{\mu_1,\widehat Q_n}\), so
\begin{align*}
W_{1,\rho}(\Gamma_{\mu_1,Q},\Gamma_{\mu_1,\widehat Q_n})
&\leq
\mathbb E\|Y_0-\widetilde Y_0\|\\
&=
\mathbb E\,W_1\!\left(
\widehat Q_n(\cdot\mid\widehat Y_1),
\widehat Q_n(\cdot\mid Y_1)
\right)\\
&\leq
\frac{2d}{\varepsilon_{\mathrm{OT}}}
\mathbb E\|\widehat Y_1-Y_1\|\\
&=
\frac{2d}{\varepsilon_{\mathrm{OT}}}
W_1\!\left(\mu_1,\mu_Q^{\mathrm{MC}}\right),
\end{align*}
where the penultimate line uses Lemma~\ref{Lemma_EOT_Stability}.
\end{proof}

\begin{remark}\label{remark_Q}
Under Assumption~\ref{ass:clipping_bounded_support}, the coupling constructed
in Lemma~\ref{lem:existence_population_valid_Q} is supported on
\([-1,1]^d\times[-1,1]^d\). Hence its conditional kernel \(Q(dy_0\mid y_1)\)
may be chosen so that
\[
Q([-1,1]^d\mid y_1)=1,
\qquad y_1\in[-1,1]^d.
\]
\end{remark}

\begin{corollary}[Coupling-Replacement Error]\label{Cor_Coupling_Replacement}
Under Assumptions \ref{ass:clipping_bounded_support}, \ref{ass:psi_empirical_training_W1_accuracy}, and \ref{ass:nonempty_treatment_a_auxiliary_sample}, with probability at least \(1-\frac{2}{n}\),
\begin{align*}
W_{1,\rho}\!\left(
\Gamma_{\mu_1,Q},
\Gamma_{\mu_1,\widehat Q_n}
\right)
&\lesssim
\frac{1}{\varepsilon_{\mathrm{OT}}}
\left(
n^{-1/(d+2)}
+\epspsiemp
+(nM_Q)^{-1/(d+2)}
\right),
\end{align*}
where \(d\) is treated as fixed.
\end{corollary}

\begin{proof}
Apply triangle inequality and Lemma \ref{lem:existence_population_valid_Q},
\begin{align*}
W_{1,\rho}\!\left(
\Gamma_{\mu_1,Q},
\Gamma_{\mu_1,\widehat Q_n}
\right)
&\leq
\frac{2d}{\varepsilon_{\mathrm{OT}}}
W_1\!\left(\mu_1,\mu_Q^{\mathrm{MC}}\right)
\\
&\leq
\frac{2d}{\varepsilon_{\mathrm{OT}}}
\left\{
W_1\!\left(\mu_1,\overline\mu_Q\right)
+
W_1\!\left(\overline\mu_Q,\mu_Q^{\mathrm{MC}}\right)
\right\}.
\end{align*}
By Proposition \ref{prop:target_query_discrepancy_empirical_C}, with probability at least \(1-\frac{1}{n}\),
\begin{align*}
W_1\!\left(\mu_1,\overline\mu_Q\right)
&\lesssim
n^{-1/(d+2)}+\epspsiemp.
\end{align*}
Taking \(\delta_Q=1/n\) in Lemma \ref{lem:finite_MQ_approximation} gives, with probability at least \(1-\frac{1}{n}\),
\begin{align*}
W_1\!\left(\overline\mu_Q,\mu_Q^{\mathrm{MC}}\right)
&\lesssim(nM_Q)^{-1/(d+2)},
\end{align*}
For fixed \(d\), substituting the preceding two bounds gives
\begin{align*}
W_{1,\rho}\!\left(
\Gamma_{\mu_1,Q},
\Gamma_{\mu_1,\widehat Q_n}
\right)
&\lesssim
\frac{1}{\varepsilon_{\mathrm{OT}}}
\left(
n^{-1/(d+2)}
+\epspsiemp
+(nM_Q)^{-1/(d+2)}
\right).
\end{align*}
A union bound over the two events gives probability at least \(1-\frac{2}{n}\).
\end{proof}

\subsection[Finite-Sample Target Approximation]
{Finite-Sample Target Approximation}

In this subsection, we establish two high-probability \(W_1\) approximation
bounds used in the coupling analysis. Bounds of this type are standard in
empirical Wasserstein theory; see, for example, \cite{weed2019sharp}. We provide a proof for completeness.

\begin{lemma}[Finite-\(M_Q\) Approximation of the Fitted Target Distribution]
\label{lem:finite_MQ_approximation}
Assume Assumption \ref{ass:clipping_bounded_support}. Conditional on \(\mathcal D_{\mathrm{aux}}\), for every \(\delta_Q\in(0,1)\), with probability at least \(1-\delta_Q\),
\begin{align*}
W_1\!\left(\overline\mu_Q,\mu_Q^{\mathrm{MC}}\right)
\leq\sqrt d\left[\left(2+2^{d/2}\right)(nM_Q)^{-1/(d+2)}+\sqrt{\frac{2\log(1/\delta_Q)}{nM_Q}}\right].
\end{align*}
In particular, for fixed \(d\) and \(\delta_Q=1/n\), the right-hand
side is \(O_d((nM_Q)^{-1/(d+2)})\).
\end{lemma}

\begin{proof}
Set \(m=nM_Q\) and \(K=\lceil m^{1/(d+2)}\rceil\).  Partition
\([-1,1]^d\) into \(K^d\) cubes \(A_1,\ldots,A_{K^d}\) of side length
\(2/K\), and let \(a_j\) be the center of \(A_j\).  If
\(\zeta^\sharp\) denotes the law obtained by moving the mass in each \(A_j\)
to \(a_j\), then
\[
W_1(\zeta,\zeta^\sharp)\leq\frac{\sqrt d}{K}
\]
for every \(\zeta\) supported on the cube.

Define the conditional cell probabilities and their empirical counterparts by
\begin{align*}
p_j
&:=
\overline\mu_Q(A_j)
=
\frac1n\sum_{i=n+1}^{2n}
\widehat{\mathbb P}_{Y\mid C=C^i,A=a}(A_j),
\\
\widehat p_j
&:=
\mu_Q^{\mathrm{MC}}(A_j)
=
\frac1m
\sum_{i=n+1}^{2n}\sum_{r=1}^{M_Q}
\mathbf 1\{\widehat Y_{i,r}^Q\in A_j\}.
\end{align*}
The discretized measures have masses \(p_j\) and \(\widehat p_j\) at
\(a_j\).  Since the cube has diameter \(2\sqrt d\),
\[
W_1(\overline\mu_Q^\sharp,{\mu_Q^{\mathrm{MC}}}^\sharp)
\leq
2\sqrt d\,\operatorname{TV}
(\overline\mu_Q^\sharp,{\mu_Q^{\mathrm{MC}}}^\sharp)
=
\sqrt d\sum_{j=1}^{K^d}|p_j-\widehat p_j|.
\]
Consequently, with \(F:=\sum_{j=1}^{K^d}|p_j-\widehat p_j|\),
\begin{equation}
\label{eq:finite_MQ_quantization_reduction}
W_1(\overline\mu_Q,\mu_Q^{\mathrm{MC}})
\leq
\frac{2\sqrt d}{K}+\sqrt d\,F.
\end{equation}

We next control \(F\) conditionally on \(\mathcal D_{\mathrm{aux}}\).
Replacing one of the \(m\) independent draws moves empirical mass \(1/m\)
from at most one cell to at most one other cell.  Hence the empirical
histogram changes by at most \(2/m\) in \(\ell_1\), and therefore \(F\)
changes by at most \(2/m\).  McDiarmid's inequality gives
\[
\mathbb P\!\left(
F-\mathbb E[F\mid\mathcal D_{\mathrm{aux}}]\geq t
\;\middle|\;
\mathcal D_{\mathrm{aux}}
\right)
\leq
\exp\!\left(-\frac{mt^2}{2}\right).
\]
Thus, with conditional probability at least \(1-\delta_Q\),
\begin{equation}
\label{eq:finite_MQ_mcd}
F
\leq
\mathbb E[F\mid\mathcal D_{\mathrm{aux}}]
+
\sqrt{\frac{2\log(1/\delta_Q)}{m}}.
\end{equation}

It remains to bound the conditional mean.  Cauchy--Schwarz followed by
Jensen gives
\begin{align*}
\mathbb E[F\mid\mathcal D_{\mathrm{aux}}]
&\leq
\sqrt{K^d}
\left[
\sum_{j=1}^{K^d}
\mathbb E\!\left[
(\widehat p_j-p_j)^2
\mid\mathcal D_{\mathrm{aux}}
\right]
\right]^{1/2}.
\end{align*}
For
\(p_{i,j}:=\widehat{\mathbb P}_{Y\mid C=C^i,A=a}(A_j)\),
conditional independence and unbiasedness imply
\begin{align*}
\sum_{j=1}^{K^d}
\mathbb E\!\left[
(\widehat p_j-p_j)^2
\mid\mathcal D_{\mathrm{aux}}
\right]
&=
\frac1{m^2}
\sum_{i=n+1}^{2n}\sum_{r=1}^{M_Q}
\sum_{j=1}^{K^d}p_{i,j}(1-p_{i,j})
\\
&\leq
\frac1{m^2}
\sum_{i=n+1}^{2n}\sum_{r=1}^{M_Q}
\sum_{j=1}^{K^d}p_{i,j}
=
\frac1m.
\end{align*}
Therefore
\[
\mathbb E[F\mid\mathcal D_{\mathrm{aux}}]
\leq
\sqrt{\frac{K^d}{m}}.
\]
Combining this inequality with
\eqref{eq:finite_MQ_quantization_reduction} and
\eqref{eq:finite_MQ_mcd} yields
\[
W_1(\overline\mu_Q,\mu_Q^{\mathrm{MC}})
\leq
\sqrt d\left[
\frac2K+\sqrt{\frac{K^d}{m}}
+\sqrt{\frac{2\log(1/\delta_Q)}{m}}
\right].
\]
Finally, \(m\geq1\) implies
\[
\frac1K\leq m^{-1/(d+2)}
\quad\text{and}\quad
K\leq2m^{1/(d+2)},
\]
so
\(\sqrt{K^d/m}\leq2^{d/2}m^{-1/(d+2)}\).  Substituting
\(m=nM_Q\) proves the stated bound.
\end{proof}

\begin{lemma}[Empirical Covariate Standardization Error]
\label{lem:empirical_covariate_standardization_error}
Assume Assumption \ref{ass:clipping_bounded_support} and \(n\geq2\).  Define
\[
\mu_{1,n}
:=
\frac1n\sum_{i=n+1}^{2n}
\mathbb P_{Y\mid C=C^i,A=a},
\qquad
K_n:=\left\lceil n^{1/(d+2)}\right\rceil.
\]
Then, with probability at least \(1-1/n\) over
\(C^{n+1},\ldots,C^{2n}\),
\begin{align*}
W_1(\mu_1,\mu_{1,n})
&\leq
\frac{4\sqrt d}{K_n}
+2\sqrt d\left(\frac{K_n^d}{n}\right)^{1/2}
+2\sqrt d\left(\frac{2\log n}{n}\right)^{1/2}.
\end{align*}
Consequently, for fixed \(d\), this error is
\(O_d(n^{-1/(d+2)})\).
\end{lemma}

\begin{proof}
Partition \([-1,1]^d\) into \(K_n^d\) cubes
\(A_1,\ldots,A_{K_n^d}\) of side length \(2/K_n\).  Moving every point to
the center of its cell changes any law supported on the cube by at most
\(2\sqrt d/K_n\) in \(W_1\).

For each cell, define
\begin{align*}
q_j(c)
&:=
\mathbb P(Y\in A_j\mid C=c,A=a),
&
p_j
&:=
\mathbb E_C[q_j(C)],
&
\widehat p_j
&:=
\frac1n\sum_{i=n+1}^{2n}q_j(C^i).
\end{align*}
The discretizations of \(\mu_1\) and \(\mu_{1,n}\) put masses \(p_j\) and
\(\widehat p_j\) on the same cell centers.  Therefore, with
\[
F:=\sum_{j=1}^{K_n^d}|p_j-\widehat p_j|,
\]
the triangle inequality and the diameter bound give
\begin{equation}
\label{eq:empirical_covariate_quantization_reduction}
W_1(\mu_1,\mu_{1,n})
\leq
\frac{4\sqrt d}{K_n}+2\sqrt d\,F.
\end{equation}

Changing one covariate \(C^i\) changes the probability vector
\((q_1(C^i),\ldots,q_{K_n^d}(C^i))\) to another probability vector.
Its \(\ell_1\)-change is at most \(2\), so \(F\) changes by at most \(2/n\).
McDiarmid's inequality and \(\delta=1/n\) yield, with probability at least
\(1-1/n\),
\begin{equation}
\label{eq:empirical_covariate_F_mcd}
F
\leq
\mathbb E F+\sqrt{\frac{2\log n}{n}}.
\end{equation}

For the expectation, Cauchy--Schwarz and Jensen imply
\[
\mathbb E F
\leq
\sqrt{K_n^d}
\left[
\sum_{j=1}^{K_n^d}
\operatorname{Var}(\widehat p_j)
\right]^{1/2}.
\]
Because the covariates are i.i.d.,
\[
\sum_{j=1}^{K_n^d}\operatorname{Var}(\widehat p_j)
=
\frac1n\sum_{j=1}^{K_n^d}\operatorname{Var}(q_j(C))
\leq
\frac1n\,
\mathbb E_C\!\left[\sum_{j=1}^{K_n^d}q_j(C)^2\right]
\leq
\frac1n,
\]
where the last inequality uses \(0\leq q_j(C)\leq1\) and
\(\sum_jq_j(C)=1\).  Hence
\begin{equation}
\label{eq:empirical_covariate_F_expectation}
\mathbb E F
\leq
\sqrt{\frac{K_n^d}{n}}.
\end{equation}
Combining
\eqref{eq:empirical_covariate_quantization_reduction},
\eqref{eq:empirical_covariate_F_mcd}, and
\eqref{eq:empirical_covariate_F_expectation} proves the explicit bound.

Finally, \(K_n\leq2n^{1/(d+2)}\), and for fixed \(d\geq1\),
\(\sqrt{\log n/n}=O_d(n^{-1/(d+2)})\).  This proves the stated order.
\end{proof}

%% file: Field_Error_Common.tex
\section{Proof of Proposition \ref{thm:final_velocity_L2_bound_DR}}
\label{Appendix_Field_Error_Proof}


In this section, we work on $\mathcal E_a$, conditioning on the auxiliary split, coupling reservoir, and fitted nuisance objects. We use the kernels $Q,\widehat Q_n$ and the offline sampling scheme from \eqref{eq_offline_sample_1}, \eqref{eq_offline_sample_2} and \eqref{eq_offline_sample_3} in Appendix~\ref{Sect_EOT}. Population expectations below are taken over fresh observations or interpolation variables, with the fitted objects fixed.

To align the training analysis with the discrete-time sampling analysis, we
evaluate the loss on the same time grid. Let
\(0=t_0<t_1<\cdots<t_N=1\), with \(h=1/N\), and let
\(K\sim\operatorname{Unif}\{0,\ldots,N-1\}\).

Let $\Theta$ be a parameter class with metric $\mathrm d_\Theta$ and finite proper covering numbers $\mathcal N(r,\Theta,\mathrm d_\Theta)$; every cover center belongs to $\Theta$. The class and loss family are fixed under the preceding conditioning. For $\theta\in\Theta$, let $\Phi_\theta(y_0,y_1,k,g)$ be a measurable nonnegative loss on $[-1,1]^d\times[-1,1]^d\times\{0,\ldots,N-1\}\times\mathbb R^d$. With $K\sim\operatorname{Unif}\{0,\ldots,N-1\}$ and $G\sim\mathcal N(0,I_d)$ mutually independent and independent of the endpoints, define
\begin{align*}
H_\theta(y_0,y_1)&=\mathbb E_{K,G}[\Phi_\theta(y_0,y_1,K,G)],\\
\ell_{\widehat Q_n,\theta}(y_1)&=\int H_\theta(y_0,y_1)\,\widehat Q_n(dy_0\mid y_1),\\
L_{a,Q}^\Phi(\theta)&=\int H_\theta(y_0,y_1)\,\Gamma_{\mu_1,Q}(dy_0,dy_1).
\end{align*}
Define $L_{a,\widehat Q_n}^\Phi$ by replacing $Q$ with $\widehat Q_n$. For $O=(C,A,Y)$ and $O_i=(C_i,A_i,Y_i)$, set
\begin{align*}
m_\theta(c)&=\mathbb E_{Y\sim\mathbb P_{Y\mid C=c,A=a}}[\ell_{\widehat Q_n,\theta}(Y)],\\
\widehat m_\theta(c)&=\mathbb E_{\widehat Y\sim\predorac}[\ell_{\widehat Q_n,\theta}(\widehat Y)],\\
\widehat\ell_{\mathrm{DR},\theta}(O)
&=\mathbf 1\{A=a\}\widehat\alpha(C)\{\ell_{\widehat Q_n,\theta}(Y)-\widehat m_\theta(C)\}+\widehat m_\theta(C).
\end{align*}
For the offline draws, write
\begin{align*}
\Phi_{i,\theta}^{\mathrm{obs}}
&=\Phi_\theta(\widehat Y_{0,i}^{\mathrm{obs}},Y_i,K_i^{\mathrm{obs}},G_i^{\mathrm{obs}}),\\
\overline\Phi_{i,\theta}^{\mathrm{aux}}
&=\frac1M\sum_{m=1}^M\Phi_\theta(\widehat Y_{0,i,m},\widehat Y_{i,m},K_{i,m},G_{i,m}),
\end{align*}
with $\Phi_{i,\theta}^{\mathrm{obs}}=0$ for $A_i\neq a$. The implemented empirical loss is
\begin{align}
\widehat L_{\mathrm{actual}}^\Phi(\theta)
&=\frac1n\sum_{i=1}^n\left[\mathbf 1\{A_i=a\}\widehat\alpha(C_i)\{\Phi_{i,\theta}^{\mathrm{obs}}-\overline\Phi_{i,\theta}^{\mathrm{aux}}\}+\overline\Phi_{i,\theta}^{\mathrm{aux}}\right].
\label{eq:generic_actual_dr_loss}
\end{align}

The following lemma shows concretely how a covering number enters the concentration argument: we first approximate every function in the class by one of \(J\) cover elements, apply Hoeffding's inequality and a union bound to these \(J\) representatives, and then incur at most an additional \(2r\) error when passing back to the full class.

\begin{lemma}[A finite-cover empirical-process bound]
\label{lem:finite_cover_empirical_process}
Conditional on a sigma-field $\mathcal F$, let $W_1,\ldots,W_m$ be independent. Let $\mathcal Z$ be a separable class of measurable functions taking values in a common interval of length $B_{\mathrm{env}}>0$, with a proper uniform $r$-cover $\{z_1,\ldots,z_J\}$. 
Conditional on \(\mathcal F\), the function class \(\mathcal Z\) and the
chosen cover \(\{z_1,\ldots,z_J\}\) are treated as fixed. For $0<\delta<1$, with conditional probability at least $1-\delta$,
\begin{align*}
\sup_{z\in\mathcal Z}\left|\frac1m\sum_{i=1}^m\{z(W_i)-\mathbb E[z(W_i)\mid\mathcal F]\}\right|
\leq 2r+B_{\mathrm{env}}\sqrt{\frac{\log(2J/\delta)}{2m}}.
\end{align*}
\end{lemma}
\begin{proof}
For each $z_j$, conditional Hoeffding's inequality gives
\begin{align*}
\mathbb P\left(\left|\frac1m\sum_{i=1}^m\{z_j(W_i)-\mathbb E[z_j(W_i)\mid\mathcal F]\}\right|>u\mid\mathcal F\right)
\leq 2\exp\left(-\frac{2mu^2}{B_{\mathrm{env}}^2}\right).
\end{align*}
Take $u=B_{\mathrm{env}}\sqrt{\log(2J/\delta)/(2m)}$ and apply a union bound over the cover. Replacing a function by its cover center changes its empirical average and its conditional expected average by at most $r$ each, giving the additional $2r$.
\end{proof}

We now turn minimization of the empirical doubly robust loss into a high-probability bound on the population excess risk, in the following proposition. The error separates into three sources: replacing the ideal coupling \(Q\) by the fitted coupling \(\widehat Q_n\), estimation error in the nuisance functions, and finite-sample error from the observed data and offline Monte Carlo draws. The last part is controlled through the covering-number concentration argument above.

\begin{proposition}[Doubly Robust ERM Bound]
\label{prop:shared_DR_field_bound}
Under Assumptions~\ref{ass:clipping_bounded_support}, \ref{ass:alpha_l2_accuracy}, and \ref{ass:prediction_oracle_accuracy}, let the loss satisfy the following bounds for constants $B_\Phi,L_Y,L_\Theta>0$ and all admissible arguments:
\begin{align}
0\leq\Phi_\theta(y_0,y_1,k,g)&\leq B_\Phi,
\label{eq:generic_loss_envelope}\\
|H_\theta(y_0,y_1)-H_\theta(y_0',y_1')|
&\leq L_Y(\|y_0-y_0'\|+\|y_1-y_1'\|),
\label{eq:generic_endpoint_lipschitz}\\
|\Phi_\theta(y_0,y_1,k,g)-\Phi_{\theta'}(y_0,y_1,k,g)|
&\leq L_\Theta\,\mathrm d_\Theta(\theta,\theta').
\label{eq:generic_parameter_lipschitz}
\end{align}
Define
\begin{align*}
N_{\mathrm{DR}}^\Phi
=\mathcal N\left(\frac{B_\Phi(1+2\pi_{\min}^{-1})}{L_\Theta(1+4\pi_{\min}^{-1})\sqrt n},\Theta,\mathrm d_\Theta\right),
N_{\mathrm{obs}}^\Phi
=\mathcal N\left(\frac{B_\Phi}{L_\Theta\sqrt n},\Theta,\mathrm d_\Theta\right),
N_{\mathrm{aux}}^\Phi
=\mathcal N\left(\frac{B_\Phi}{L_\Theta\sqrt{nM}},\Theta,\mathrm d_\Theta\right).
\end{align*}
For $0<\delta<1$, set
\begin{align}
\mathfrak R_\Phi(\delta)
&=4B_\Phi(1+2\pi_{\min}^{-1})\left[\frac1{\sqrt n}+\sqrt{\frac{2\log(6N_{\mathrm{DR}}^\Phi/\delta)}{n}}\right]\nonumber\\
&\quad+8B_\Phi\pi_{\min}^{-1}\left[\frac1{\sqrt n}+\sqrt{\frac{2\log(6N_{\mathrm{obs}}^\Phi/\delta)}{n}}\right]\nonumber\\
&\quad+4B_\Phi(1+2\pi_{\min}^{-1})\left[\frac1{\sqrt{nM}}+\sqrt{\frac{2\log(6N_{\mathrm{aux}}^\Phi/\delta)}{nM}}\right].
\label{eq:generic_statistical_remainder}
\end{align}
Write $\Delta_Q=W_{1,\rho}(\Gamma_{\mu_1,Q},\Gamma_{\mu_1,\widehat Q_n})$. Let $\widehat\theta_M$ be an exact minimizer of $\widehat L_{\mathrm{actual}}^\Phi$ over $\Theta$. With conditional probability at least $1-\delta$ over the training observations and offline Monte Carlo draws,
\begin{align}
L_{a,Q}^\Phi(\widehat\theta_M)-\inf_{\theta\in\Theta}L_{a,Q}^\Phi(\theta)
&\leq 2L_Y\Delta_Q+2L_Y\left(1+\frac{2d}{\varepsilon_{\mathrm{OT}}}\right)\varepsilon_\alpha\varepsilon_{\mathrm{pred}}
+\mathfrak R_\Phi(\delta).
\label{eq:shared_DR_field_bound}
\end{align}
\end{proposition}

\begin{proof}
Empirical optimality gives
\begin{align*}
L_{a,Q}^\Phi(\widehat\theta_M)-\inf_{\theta\in\Theta}L_{a,Q}^\Phi(\theta)
\leq 2\sup_{\theta\in\Theta}|L_{a,Q}^\Phi(\theta)-\widehat L_{\mathrm{actual}}^\Phi(\theta)|.
\end{align*}
Inserting the fitted-kernel population loss, the population DR loss, and its empirical average gives
\begin{align}
L_{a,Q}^\Phi(\widehat\theta_M)-\inf_{\theta\in\Theta}L_{a,Q}^\Phi(\theta)
&\leq 2\sup_{\theta\in\Theta}|L_{a,Q}^\Phi(\theta)-L_{a,\widehat Q_n}^\Phi(\theta)|\nonumber\\
&\quad+2\sup_{\theta\in\Theta}|L_{a,\widehat Q_n}^\Phi(\theta)-\mathbb E[\widehat\ell_{\mathrm{DR},\theta}(O)]|\nonumber\\
&\quad+2\sup_{\theta\in\Theta}\left|\mathbb E[\widehat\ell_{\mathrm{DR},\theta}(O)]-\frac1n\sum_{i=1}^n\widehat\ell_{\mathrm{DR},\theta}(O_i)\right|\nonumber\\
&\quad+2\sup_{\theta\in\Theta}\left|\frac1n\sum_{i=1}^n\widehat\ell_{\mathrm{DR},\theta}(O_i)-\widehat L_{\mathrm{actual}}^\Phi(\theta)\right|.
\label{eq:generic_master_decomposition}
\end{align}

\emph{Coupling replacement and nuisance error.}
By \eqref{eq:generic_endpoint_lipschitz} and Kantorovich--Rubinstein duality,
\begin{align*}
\sup_{\theta\in\Theta}|L_{a,Q}^\Phi(\theta)-L_{a,\widehat Q_n}^\Phi(\theta)|\leq L_Y\Delta_Q.
\end{align*}
Conditioning on $C$ and using $\mathbb P(A=a\mid C)=1/\alpha_a(C)$ gives
\begin{align*}
\mathbb E[\widehat\ell_{\mathrm{DR},\theta}(O)\mid C]
=\frac{\widehat\alpha(C)}{\alpha_a(C)}\{m_\theta(C)-\widehat m_\theta(C)\}+\widehat m_\theta(C).
\end{align*}
Since $L_{a,\widehat Q_n}^\Phi(\theta)=\mathbb E_C[m_\theta(C)]$,
\begin{align}
L_{a,\widehat Q_n}^\Phi(\theta)-\mathbb E[\widehat\ell_{\mathrm{DR},\theta}(O)]
=\mathbb E_C\left[\left(1-\frac{\widehat\alpha(C)}{\alpha_a(C)}\right)\{m_\theta(C)-\widehat m_\theta(C)\}\right].
\label{eq:generic_DR_identity}
\end{align}
Coupling $\widehat Q_n(\cdot\mid y)$ and $\widehat Q_n(\cdot\mid y')$ optimally, \eqref{eq:generic_endpoint_lipschitz} and Lemma~\ref{Lemma_EOT_Stability} yield
\begin{align*}
|\ell_{\widehat Q_n,\theta}(y)-\ell_{\widehat Q_n,\theta}(y')|
&\leq L_Y\{W_1(\widehat Q_n(\cdot\mid y),\widehat Q_n(\cdot\mid y'))+\|y-y'\|\}\\
&\leq L_Y\left(1+\frac{2d}{\varepsilon_{\mathrm{OT}}}\right)\|y-y'\|.
\end{align*}
Kantorovich--Rubinstein duality and Assumption~\ref{ass:prediction_oracle_accuracy} therefore give
\begin{align*}
\sup_{\theta\in\Theta}\|m_\theta-\widehat m_\theta\|_{L^2(\mathbb P_C)}
\leq L_Y\left(1+\frac{2d}{\varepsilon_{\mathrm{OT}}}\right)\varepsilon_{\mathrm{pred}}.
\end{align*}
Cauchy--Schwarz in \eqref{eq:generic_DR_identity}, followed by Assumption~\ref{ass:alpha_l2_accuracy}, bounds the second term of \eqref{eq:generic_master_decomposition} by $2L_Y(1+2d/\varepsilon_{\mathrm{OT}})\varepsilon_\alpha\varepsilon_{\mathrm{pred}}$.

\emph{Generalization.}
The training observations remain independent under the initial conditioning. Integration of \eqref{eq:generic_loss_envelope} and \eqref{eq:generic_parameter_lipschitz} gives
\begin{align*}
|\widehat\ell_{\mathrm{DR},\theta}(O)|&\leq B_\Phi(1+2\pi_{\min}^{-1}),\\
|\widehat\ell_{\mathrm{DR},\theta}(O)-\widehat\ell_{\mathrm{DR},\theta'}(O)|
&\leq L_\Theta(1+4\pi_{\min}^{-1})\mathrm d_\Theta(\theta,\theta').
\end{align*}
Thus the parameter cover defining $N_{\mathrm{DR}}^\Phi$ induces a proper uniform cover of the DR loss class with radius $B_\Phi(1+2\pi_{\min}^{-1})/\sqrt n$. Lemma~\ref{lem:finite_cover_empirical_process}, with $m=n$, interval length $2B_\Phi(1+2\pi_{\min}^{-1})$, and failure probability $\delta/3$, bounds the third term of \eqref{eq:generic_master_decomposition} by the first term of \eqref{eq:generic_statistical_remainder}.

\emph{Monte Carlo approximation.}
Condition additionally on $O_1,\ldots,O_n$. All training observations and fitted objects are now fixed. The observed tuples are independent across $i$, and the auxiliary tuples are independent across $(i,m)$, as specified in Appendix~\ref{Sect_EOT}. Dependence within each tuple, including the conditional draw from $\widehat Q_n$, is retained. Under this additional conditioning,
\begin{align*}
\mathbb E[\mathbf 1\{A_i=a\}\widehat\alpha(C_i)\Phi_{i,\theta}^{\mathrm{obs}}\mid O_1,\ldots,O_n]
&=\mathbf 1\{A_i=a\}\widehat\alpha(C_i)\ell_{\widehat Q_n,\theta}(Y_i),\\
\mathbb E[\Phi_\theta(\widehat Y_{0,i,m},\widehat Y_{i,m},K_{i,m},G_{i,m})\mid O_1,\ldots,O_n]
&=\widehat m_\theta(C_i).
\end{align*}
Expanding \eqref{eq:generic_actual_dr_loss} gives
\begin{align*}
\frac1n\sum_{i=1}^n\widehat\ell_{\mathrm{DR},\theta}(O_i)-\widehat L_{\mathrm{actual}}^\Phi(\theta)
&=\frac1n\sum_{i=1}^n\mathbf 1\{A_i=a\}\widehat\alpha(C_i)\{\ell_{\widehat Q_n,\theta}(Y_i)-\Phi_{i,\theta}^{\mathrm{obs}}\}\\
&\quad+\frac1n\sum_{i=1}^n[1-\mathbf 1\{A_i=a\}\widehat\alpha(C_i)]\{\widehat m_\theta(C_i)-\overline\Phi_{i,\theta}^{\mathrm{aux}}\}.
\end{align*}
Apply Lemma~\ref{lem:finite_cover_empirical_process} to these two averages, including $i$ as a deterministic argument of each integrand. For the observed term, the weighted losses lie in $[0,2B_\Phi\pi_{\min}^{-1}]$ and are $2\pi_{\min}^{-1}L_\Theta$-Lipschitz in $\theta$. The cover defining $N_{\mathrm{obs}}^\Phi$ therefore has induced radius $2B_\Phi\pi_{\min}^{-1}/\sqrt n$. With $n$ draws and failure probability $\delta/3$, twice this discrepancy is bounded by the second term of \eqref{eq:generic_statistical_remainder}.

For the auxiliary term, apply the lemma to the $nM$ individual weighted losses, not to their $n$ within-observation averages. Their absolute values are at most $B_\Phi(1+2\pi_{\min}^{-1})$, and their parameter Lipschitz constant is $(1+2\pi_{\min}^{-1})L_\Theta$. The cover defining $N_{\mathrm{aux}}^\Phi$ has induced radius $B_\Phi(1+2\pi_{\min}^{-1})/\sqrt{nM}$. With interval length $2B_\Phi(1+2\pi_{\min}^{-1})$ and failure probability $\delta/3$, twice this discrepancy is bounded by the third term of \eqref{eq:generic_statistical_remainder}. The loss classes used above are separable because they are Lipschitz images of the totally bounded class $\Theta$.

Each Monte Carlo failure probability is at most $\delta/3$ conditional on the training observations and the initially fixed objects. Taking conditional expectations removes the additional conditioning. A union bound with the generalization event therefore gives total conditional failure probability at most $\delta$. Substituting all four bounds into \eqref{eq:generic_master_decomposition} proves \eqref{eq:shared_DR_field_bound}.
\end{proof}

We remark that Proposition~\ref{prop:shared_DR_field_bound} applies to a generic loss class. Its conditions will be verified separately for the velocity and score fields below.

%% file: Vel_new_Q.tex
\subsection{Velocity-Field Bound}
\label{Appendix_Vel_Error}

We now specialize the generic doubly robust framework above to velocity-field
estimation. Define the squared velocity loss
\[
\Phi_v^{\mathrm{vel}}(y_0,y_1,k,g)
=
\left\|
v\!\left(y_0+t_k(y_1-y_0)+\sigma g,t_k\right)
-(y_1-y_0)
\right\|^2.
\]
Each parameter \(\theta\) induces a velocity field \(v_\theta\). To simplify
notation in this specialization, we write \(v=v_\theta\), suppress the
abstract index \(\theta\), and use the velocity field itself to index the loss,
writing \(\Phi_v^{\mathrm{vel}}\) in place of the generic \(\Phi_\theta\).

The generic population risks \(L_{a,Q}^\Phi\) and
\(L_{a,\widehat Q_n}^\Phi\) then become the velocity losses below; we write
them as \(L_{a,Q}(v)\) and \(L_{a,\widehat Q_n}(v)\) for brevity. 
Then the previous abstract loss becomes
\begin{align*}
\ell_{Q}(v;y_1) &= \mathbb E_{K,Y_0\sim Q(\cdot\mid y_1),G}\!\left[\left\|v\!\left(Y_0+t_{K}(y_1-Y_0)+\sigma G,t_{K}\right)-(y_1-Y_0)\right\|^2\right], \\
L_{a,Q}(v) &= \mathbb E_{Y_1\sim\int \mathbb P_{Y\mid C=c,A=a}\,d\mathbb P_C(c)}\!\left[\ell_{Q}(v;Y_1)\right], \\
\ell_{\widehat Q_n}(v;y_1) &= \mathbb E_{K,Y_0\sim\widehat Q_n(\cdot\mid y_1),G}\!\left[\left\|v\!\left(Y_0+t_{K}(y_1-Y_0)+\sigma G,t_{K}\right)-(y_1-Y_0)\right\|^2\right], \\
L_{a,\widehat Q_n}(v) &= \mathbb E_{Y_1\sim\int \mathbb P_{Y\mid C=c,A=a}\,d\mathbb P_C(c)}\!\left[\ell_{\widehat Q_n}(v;Y_1)\right].
\end{align*}

Since
\(v_Q^\star(X_{t_K},t_K)=\mathbb E[D\mid X_{t_K},K]\), the residual
\(D-v_Q^\star(X_{t_K},t_K)\) is orthogonal in \(L^2\) to every measurable
function of \((X_{t_K},K)\).  Expanding the two squared losses gives
\begin{align}
L_{a,Q}(v)-L_{a,Q}(v_Q^\star)
&=
\mathbb E\!\left[
\|v(X_{t_K},t_K)-v_Q^\star(X_{t_K},t_K)\|^2
\right]\nonumber\\
&=
\frac1N\sum_{k=0}^{N-1}
\mathbb E\!\left[
\|v(X_{t_k},t_k)-v_Q^\star(X_{t_k},t_k)\|^2
\right].
\label{eq:velocity_exact_excess_risk_identity}
\end{align}
The right-hand side corresponds to the quantity used in the sampling analysis.

\subsubsection[Velocity Network Class and Approximation]
{Velocity Network Class and Approximation}

Throughout this subsection, \(Q\) denotes the ideal coupling kernel, and we write \(v_Q^{\star}\) for the population velocity field associated with this coupling.

We approximate the velocity field by the following scalar neural network.
For integers \(L\geq2\) and \(p,J\geq1\), and constants
\(K,\kappa,\gamma>0\), define the scalar sparse ReLU class
\begin{align*}
&\mathfrak N_{\mathrm{scal}}(q,L,p,J,K,\kappa,\gamma)
\\
&=
\Bigg\{
h:[0,1]^q\to\mathbb R:
\exists\,p_0,\ldots,p_L\in\mathbb N,\;
p_0=q,\;
p_L=1,\;
\max_{1\leq \ell\leq L-1}p_\ell\leq p,
\\
&\qquad
\exists\,W_\ell\in\mathbb R^{p_\ell\times p_{\ell-1}},
\;b_\ell\in\mathbb R^{p_\ell},
\;\ell=1,\ldots,L,
\text{ such that }h_0(z)=z,
\\
&\qquad
h_\ell(z)
=
\operatorname{ReLU}
\left(
W_\ell h_{\ell-1}(z)+b_\ell
\right),
\qquad
\ell=1,\ldots,L-1,
\\
&\qquad
h(z)
=
W_Lh_{L-1}(z)+b_L,
\\
&\qquad
\sum_{\ell=1}^{L}
\left(
\|W_\ell\|_0+\|b_\ell\|_0
\right)
\leq
J,
\max_{1\leq \ell\leq L}
\left(
\|W_\ell\|_{\max}
\vee
\|b_\ell\|_{\max}
\right)
\leq
\kappa, 
\sup_{z\in[0,1]^q}|h(z)|
\leq
K,
\\
&\qquad
|h(z)-h(z')|
\leq
\gamma\|z-z'\|_\infty
\text{ for every }z,z'\in[0,1]^q
\Bigg\}.
\end{align*}
Thus \(L\) counts all affine layers, including the output layer.
Here \(\|W_\ell\|_0\) and \(\|b_\ell\|_0\) denote the number of nonzero entries,
and \(\|W_\ell\|_{\max}\) and \(\|b_\ell\|_{\max}\) denote the largest absolute
value of their entries.

For \(u, \omega\in\mathbb R\), define
\begin{align*}
\operatorname{clip}_{\omega}(u)
&=
-\omega+\operatorname{ReLU}(u+\omega)-\operatorname{ReLU}(u-\omega).
\end{align*}
For \(x=(x_1,\ldots,x_d)\), define
\begin{align*}
\operatorname{clip}_{R}(x)
=
\left(
\operatorname{clip}_{R}(x_1),\ldots,\operatorname{clip}_{R}(x_d)
\right), \qquad
S_R(x,t) = \left(
\frac{\operatorname{clip}_{R}(x_1)+R}{2R},
\ldots,
\frac{\operatorname{clip}_{R}(x_d)+R}{2R}, t \right).
\end{align*}

We first define some constants that will be used to construct a network class. 
Let \(R>1\), and define $\Lambda=\frac{16d(2R+1)}{\sigma^2}$
and
\begin{align*}
T_R(z)=\left(
-R+2Rz_1,\ldots,-R+2Rz_d,z_{d+1}
\right),
\qquad z\in[0,1]^{d+1}.
\end{align*}
Lemma~\ref{lem:noisy_oracle_lipschitz_on_cube} shows that
\(v_Q^\star\circ T_R\) is \(\Lambda\)-Lipschitz jointly in the rescaled
spatial and time coordinates.

\begin{lemma}[Lipschitz Regularity of the True Vector Field]
\label{lem:noisy_oracle_lipschitz_on_cube}
Assume Assumption \ref{ass:clipping_bounded_support}.
Then \(v_{Q}^{\star}\) admits a version such that, for every
\(i=1,\ldots,d\), the function
\begin{align*}
g_i(z)
&=
v_{Q,i}^{\star}(T_R(z)),
\qquad z\in[0,1]^{d+1},
\end{align*}
satisfies
\begin{align*}
|g_i(z)-g_i(z')|
&\leq
\Lambda
\|z-z'\|_\infty,
\qquad z,z'\in[0,1]^{d+1}.
\end{align*}
\end{lemma}

\begin{proof}[Proof of
Lemma~\ref{lem:noisy_oracle_lipschitz_on_cube}]
All expectations and conditional expectations below are taken under the joint law induced by \(Q\). By Lemma \ref{Lemma_Derivative_Identity},  for every \(i,j\in\{1,\ldots,d\}\),
\begin{align*}
\partial_{x_j}v_{Q,i}^{\star}(x,t)
&=
-\frac{1}{\sigma}
\operatorname{Cov}
\left(
D_i,G_j\mid X_t=x
\right),
\end{align*}
and
\begin{align*}
\partial_t v_{Q,i}^{\star}(x,t)
&=
\frac{1}{\sigma}
\operatorname{Cov}
\left(
D_i,D^\top G\mid X_t=x
\right).
\end{align*}
Under the conditional law given \(X_t=x\), we have
\begin{align*}
G
&=
\frac{x-S_t}{\sigma}.
\end{align*}
Since \(x\) is fixed under this conditional law,
\begin{align*}
\partial_{x_j}v_{Q,i}^{\star}(x,t)
&=
\frac{1}{\sigma^2}
\operatorname{Cov}
\left(
D_i,S_{t,j}\mid X_t=x
\right),
\end{align*}
and
\begin{align*}
\partial_t v_{Q,i}^{\star}(x,t)
&=
\frac{1}{\sigma^2}
\operatorname{Cov}
\left(
D_i,D^\top(x-S_t)\mid X_t=x
\right).
\end{align*}

Since \(Y_0,Y_1\in[-1,1]^d\), we have
\begin{align*}
D_i&\in[-2,2],
&
S_{t,j}&\in[-1,1].
\end{align*}
Hence, conditional on \(X_t=x\),
\begin{align*}
\left|
D_i-\mathbb E[D_i\mid X_t=x]
\right|
&\leq
4,
\\
\left|
S_{t,j}-\mathbb E[S_{t,j}\mid X_t=x]
\right|
&\leq
2.
\end{align*}
Therefore,
\begin{align*}
\left|
\partial_{x_j}v_{Q,i}^{\star}(x,t)
\right|
&\leq
\frac{8}{\sigma^2}.
\end{align*}

For the time derivative, if \(x\in[-R,R]^d\), then $\|D\|\leq 2\sqrt d$ and $\|x-S_t\| \leq (R+1)\sqrt d$.
Thus
\begin{align*}
\left|
D^\top(x-S_t)
\right| \leq
2d(R+1),
\end{align*}
and so, conditional on \(X_t=x\),
\begin{align*}
\left|
D^\top(x-S_t)
-
\mathbb E[D^\top(x-S_t)\mid X_t=x]
\right|
&\leq
4d(R+1).
\end{align*}
Together with $\left| D_i-\mathbb E[D_i\mid X_t=x] \right| \leq 4$, 
this gives
\begin{align*}
\left|
\partial_t v_{Q,i}^{\star}(x,t)
\right|
&\leq
\frac{16d(R+1)}{\sigma^2}.
\end{align*}

For \(g_i(z)=v_{Q,i}^{\star}(T_R(z))\), the preceding derivative bounds imply
\begin{align*}
\left|
\partial_{z_j}g_i(z)
\right|
&\leq
\frac{16R}{\sigma^2},
\qquad j=1,\ldots,d,
\\
\left|
\partial_{z_{d+1}}g_i(z)
\right|
&\leq
\frac{16d(R+1)}{\sigma^2}.
\end{align*}
Therefore, for every \(z,z'\in[0,1]^{d+1}\),
\begin{align*}
|g_i(z)-g_i(z')|
&\leq
\left[
\frac{16Rd}{\sigma^2}
+
\frac{16d(R+1)}{\sigma^2}
\right]
\|z-z'\|_\infty =
\Lambda
\|z-z'\|_\infty.
\end{align*}
\end{proof}

Based on the scalar neural network defined above, we use it to construct a neural network class that we use to approximate our vector field. 
Let $\varepsilon > 0$, recall $\Lambda=\frac{16d(2R+1)}{\sigma^2}$. Define $\eta=\frac{\varepsilon}{\sqrt d\,\Lambda}$ as an auxiliary variable that control the approximation accuracy of each scalar network. 
\begin{proposition}[Approximation Error Under the Noisy Interpolation Law]
\label{prop:relu_l2_approximation_noisy_oracle}
Assume Assumption \ref{ass:clipping_bounded_support}. 
Let \(R>1\), \(q=d+1\) and
\(\varepsilon>0\) satisfy \(0<\eta=\varepsilon/(\sqrt d\,\Lambda)<1\) and
\begin{align*}
32d^2
\exp\left(
-\frac{(R-1)^2}{2\sigma^2}
\right)
&\leq
\varepsilon^2.
\end{align*}
Choose the width and parameter bounds so that
\(p_{\mathrm{Lip}}(q,\eta)\geq q\) and
\(\kappa_{\mathrm{Lip}}(q,\eta)\geq2\). The scalar-network architecture can
then be chosen with
\begin{align*}
K_{\mathrm{Lip}}(q,\eta)&=
\sqrt q, \quad
L_{\mathrm{Lip}}(q,\eta)=
O\left(
\log\frac{1}{\eta}+q
\right), \quad
p_{\mathrm{Lip}}(q,\eta)
=
O\left(
q+\eta^{-q}
\right),
\\
J_{\mathrm{Lip}}(q,\eta)
&=
O\left(
\eta^{-q}
\left[
\log\frac{1}{\eta}+q
\right]
\right), \quad
\kappa_{\mathrm{Lip}}(q,\eta)=
O(1).
\end{align*}
These are the scalar approximation parameters from
\cite[Lemma 10]{chen2020distribution}, with \(L_{\mathrm{Lip}}\) denoting the
total number of affine layers. Enlarging the width and parameter bounds to
enforce \(p_{\mathrm{Lip}}\geq q\) and \(\kappa_{\mathrm{Lip}}\geq2\) does not
alter these orders for fixed \(q\).
We construct a vector valued neural network coordinate-wisely by
\begin{align*}
v_i^{\mathrm{app}}(x,t)
&=
\operatorname{clip}_{2}
\left(
\Lambda\widehat h_i(S_R(x,t))+g_i(0)
\right),
\qquad i=1,\ldots,d.
\end{align*}
where $\widehat h_i\in
\mathfrak N_{\mathrm{scal}}
\left(q,
L_{\mathrm{Lip}}(q,\eta),
p_{\mathrm{Lip}}(q,\eta),
J_{\mathrm{Lip}}(q,\eta),
K_{\mathrm{Lip}}(q,\eta),
\kappa_{\mathrm{Lip}}(q,\eta),
10q
\right)$ and $g_{i}(0)$ are constants that will be specified in the proof. Then we have 
\begin{align*}
\frac{1}{N}\sum_{k=0}^{N-1}
\mathbb E
\left[
\left\|
v^{\mathrm{app}}(X_{t_{k}},t_{k})
-
v_Q^\star(X_{t_k},t_k)
\right\|^2
\right]
&\leq
2\varepsilon^2.
\end{align*}

\end{proposition}

\begin{proof}
By Lemma \ref{lem:noisy_oracle_lipschitz_on_cube}, for each \(i=1,\ldots,d\),
the function $g_i(z)=v_{Q,i}^\star(T_R(z)), z\in[0,1]^{d+1}$ is \(\Lambda\)-Lipschitz with respect to \(\|\cdot\|_\infty\). Define
\begin{align*}
f_i(z)=
\frac{g_i(z)-g_i(0)}{\Lambda}.
\end{align*}
Then \(f_i(0)=0\) and \(f_i\) is \(1\)-Lipschitz with respect to
\(\|\cdot\|_\infty\). Since
\(\|z-z'\|_\infty\leq\|z-z'\|_2\), it is also \(1\)-Lipschitz with respect to
the Euclidean norm required by \cite[Lemma 10]{chen2020distribution}.
Consequently, that lemma provides
\begin{align*}
\widehat h_i
&\in
\mathfrak N_{\mathrm{scal}}
\left(
q,
L_{\mathrm{Lip}}(q,\eta),
p_{\mathrm{Lip}}(q,\eta),
J_{\mathrm{Lip}}(q,\eta),
K_{\mathrm{Lip}}(q,\eta),
\kappa_{\mathrm{Lip}}(q,\eta),
10q
\right)
\end{align*}
such that
\begin{align*}
\sup_{z\in[0,1]^{d+1}}
\left|
\widehat h_i(z)-f_i(z)
\right|
&\leq
\eta.
\end{align*}
Now recall we constructed
\begin{align*}
v_i^{\mathrm{app}}(x,t)
&=
\operatorname{clip}_{2}
\left(
\Lambda\widehat h_i(S_R(x,t))+g_i(0)
\right),
\qquad i=1,\ldots,d.
\end{align*}
If \(x\in[-R,R]^d\), then \(T_R(S_R(x,t))=(x,t)\). Since
\(\eta=\varepsilon/(\sqrt d\,\Lambda)\),
\begin{align*}
\left|
\Lambda\widehat h_i(S_R(x,t))+g_i(0)
-
v_{Q,i}^\star(x,t)
\right|
&\leq
\frac{\varepsilon}{\sqrt d}.
\end{align*}
Because \(v_{Q,i}^\star(x,t)\in[-2,2]\) and
\(\operatorname{clip}_{2}\) is the projection onto \([-2,2]\), clipping does
not increase the error. Therefore our constructed vector field approximation satisfies
\begin{align*}
\sup_{x\in[-R,R]^d,\;t\in[0,1]}
\left\|
v^{\mathrm{app}}(x,t)-v_Q^\star(x,t)
\right\|^2
&\leq
\varepsilon^2.
\end{align*}

It remains to control the event \(X_t\notin[-R,R]^d\). Since
\(S_t\in[-1,1]^d\), for every \(t\in[0,1]\),
\begin{align*}
\mathbb P
\left(
X_t\notin[-R,R]^d
\right)
&\leq
2d
\exp\left(
-\frac{(R-1)^2}{2\sigma^2}
\right).
\end{align*}
On this event, $\left\| v^{\mathrm{app}}(X_t,t)-v_Q^\star(X_t,t) \right\|^2 \leq 16d$.
Thus, for every \(t\in[0,1]\),
\begin{align*}
\mathbb E
\left[
\left\|
v^{\mathrm{app}}(X_t,t)-v_Q^\star(X_t,t)
\right\|^2
\right]
&\leq
\varepsilon^2
+
32d^2
\exp\left(
-\frac{(R-1)^2}{2\sigma^2}
\right)\leq
2\varepsilon^2.
\end{align*}
 The bound holds at every grid time \(t_k\); averaging over
\(k=0,\ldots,N-1\) proves the claim.
\end{proof}

The preceding proposition motivates the following coordinatewise and
vector-valued network classes:
\begin{align*}
\mathcal F_{\mathrm{coord}}
\left(
R,\Lambda,L,p,J,K,\kappa,\gamma
\right)&=
\Bigg\{
f:\mathbb R^d\times[0,1]\to\mathbb R:
f(x,t)
=
\operatorname{clip}_{2}
\left(
\Lambda h(S_R(x,t))+b
\right),
\\
&\qquad
h\in
\mathfrak N_{\mathrm{scal}}
\left(
d+1,L,p,J,K,\kappa,\gamma
\right),
\qquad
|b|\leq 2
\Bigg\}
\end{align*}
and
\begin{align*}
\mathcal V
\left(
R,\Lambda,L,p,J,K,\kappa,\gamma
\right)
&=
\Bigg\{
v:\mathbb R^d\times[0,1]\to\mathbb R^d:
v(x,t)
=
\left(
v_1(x,t),\ldots,v_d(x,t)
\right),
\\
&\qquad
v_i
\in
\mathcal F_{\mathrm{coord}}
\left(
R,\Lambda,L,p,J,K,\kappa,\gamma
\right),
\qquad
i=1,\ldots,d
\Bigg\}.
\end{align*}
Concretely, we instantiate the class $\mathcal{V}$ as
\begin{align*}
\mathcal V_n
&=
\Bigg\{
v:\mathbb R^d\times[0,1]\to\mathbb R^d:
v(x,t)=\left(v_1(x,t),\ldots,v_d(x,t)\right),
\\
&\qquad
v_i(x,t)
=
\operatorname{clip}_{2}
\left(
\Lambda
h_i(S_R(x,t))
+
b_i
\right),
\qquad i=1,\ldots,d,
\\
&\qquad
h_i
\in
\mathfrak N_{\mathrm{scal}}
\left(
q,
L_{\mathrm{Lip}}(q,\eta),
p_{\mathrm{Lip}}(q,\eta),
J_{\mathrm{Lip}}(q,\eta),
K_{\mathrm{Lip}}(q,\eta),
\kappa_{\mathrm{Lip}}(q,\eta),
10q
\right),
\qquad i=1,\ldots,d,
\\
&\qquad
|b_i|\leq 2,
\qquad i=1,\ldots,d
\Bigg\}.
\end{align*}
Then Proposition \ref{prop:relu_l2_approximation_noisy_oracle} implies
\begin{align*}
\inf_{v\in\mathcal V_{n}}
\frac{1}{N}\sum_{k=0}^{N-1}
\mathbb E
\left[
\left\|
v(X_{t_k},t_k)
-
v_Q^\star(X_{t_k},t_k)
\right\|^2
\right]
&\leq
2\varepsilon^2.
\end{align*}

We transfer a cover of the scalar subnetworks directly to the clipped vector class. 

\begin{lemma}[Direct Covering-Number Transfer for Clipped Vector Networks]
\label{lem:covering_number_from_para}
Let
\(\mathcal H=\mathfrak N_{\mathrm{scal}}(q,L,p,J,K,\kappa,\gamma)\), where
\(L\) is the total number of affine layers, \(p\geq q\), and \(\kappa>1\).
Let
\[
\mathcal V=
\left\{
v:\
v_i(x,t)=\operatorname{clip}_{\omega}
\bigl(\Lambda h_i(S_R(x,t))+b_i\bigr),\
h_i\in\mathcal H,\ |b_i|\leq\omega,\
i=1,\ldots,d
\right\}.
\]
Then, for every \(\delta_{\mathrm{cov}}>0\),
\begin{align}
\log\mathcal N\!\left(
\delta_{\mathrm{cov}},\mathcal V,
\|\cdot\|_{\infty,[-R,R]^d\times[0,1]}
\right)
&\leq
dJ\log\!\left(
1+
\frac{
8L^2(p+2)\kappa^Lp^{L+1}\Lambda\sqrt d
}{
\delta_{\mathrm{cov}}
}
\right)\nonumber\\
&\quad+
d\log\!\left(
1+\frac{4\omega\sqrt d}{\delta_{\mathrm{cov}}}
\right).
\label{eq:direct_vector_cover_transfer}
\end{align}
\end{lemma}

\begin{proof}
Set
\[
\eta_h=\frac{\delta_{\mathrm{cov}}}{2\Lambda\sqrt d},
\qquad
\eta_b=\frac{\delta_{\mathrm{cov}}}{2\sqrt d}.
\]
After zero-padding hidden layers when necessary, \(\mathcal H\) is a subclass
of the sparse-ReLU class in \cite[Lemma~5.3]{chen2022nonparametric} with total
depth \(L\), maximum layer width \(p\), input bound \(B=1\), sparsity \(J\),
and parameter bound \(\kappa>1\). The cited result therefore supplies an
\(\eta_h/2\)-cover, not necessarily with centers in \(\mathcal H\), whose
log-cardinality is at most
\[
J\log\!\left(
1+
\frac{8L^2(p+2)\kappa^Lp^{L+1}\Lambda\sqrt d}
{\delta_{\mathrm{cov}}}
\right).
\]
Discard balls that do not intersect \(\mathcal H\), and select one element of
\(\mathcal H\) from each remaining ball. By the triangle inequality, the
selected elements form a proper \(\eta_h\)-cover with no larger cardinality.
The interval \([-\omega,\omega]\) has an \(\eta_b\)-cover of cardinality at
most \(1+4\omega\sqrt d/\delta_{\mathrm{cov}}\) (this is by definition of cover).

For a coordinate \(v_i\), choose \(h_i'\) and \(b_i'\) from these covers.
Because \(\operatorname{clip}_\omega\) is \(1\)-Lipschitz,
\[
\sup_{x,t}|v_i(x,t)-v_i'(x,t)|
\leq
\Lambda\|h_i-h_i'\|_{\infty,[0,1]^q}+|b_i-b_i'|
\leq\frac{\delta_{\mathrm{cov}}}{\sqrt d}.
\]
Taking the product cover over the \(d\) coordinates gives
\(\sup_{x,t}\|v(x,t)-v'(x,t)\|\leq\delta_{\mathrm{cov}}\).
Multiplying the scalar and bias cover cardinalities and then taking
logarithms proves \eqref{eq:direct_vector_cover_transfer}.
\end{proof}

For \(\mathcal V_n\), take \(\omega=2\), \(\Lambda\) as above, and substitute
the scalar approximation parameters
\(L_{\mathrm{Lip}}(q,\eta)\), \(p_{\mathrm{Lip}}(q,\eta)\),
\(J_{\mathrm{Lip}}(q,\eta)\), and \(\kappa_{\mathrm{Lip}}(q,\eta)\) into
\eqref{eq:direct_vector_cover_transfer}.  To keep the later bounds readable,
define, for \(r>0\),
\begin{align}
\mathfrak C_v(r)
&:=
dJ_{\mathrm{Lip}}(q,\eta)
\log\!\left(
1+
\frac{
8L_{\mathrm{Lip}}(q,\eta)^2
(p_{\mathrm{Lip}}(q,\eta)+2)
\kappa_{\mathrm{Lip}}(q,\eta)^{L_{\mathrm{Lip}}(q,\eta)}
p_{\mathrm{Lip}}(q,\eta)^{L_{\mathrm{Lip}}(q,\eta)+1}
\Lambda\sqrt d
}{
r
}
\right)\nonumber\\
&\quad+
d\log\!\left(1+\frac{8\sqrt d}{r}\right).
\label{eq:velocity_cover_complexity}
\end{align}
Thus, for every \(\delta_{\mathrm{cov}}>0\),
\begin{align}
\log\mathcal N\!\left(
\delta_{\mathrm{cov}},\mathcal V_n,
\|\cdot\|_{\infty,[-R,R]^d\times[0,1]}
\right)
&\leq
\mathfrak C_v(\delta_{\mathrm{cov}}).
\label{eq:velocity_network_covering_number}
\end{align}
The function \(r\mapsto\mathfrak C_v(r)\) is nonincreasing, a fact used below
when a loss-class cover is transferred to a cover of \(\mathcal V_n\).

\begin{lemma}[Uniform Bounds for the Vector Network Class]
\label{lem:vector_network_class_uniform_bounds}
Every \(v\in\mathcal V_{n}\) satisfies
\begin{align}
\label{eq:network_class_output_bound}
\sup_{x\in\mathbb R^d,\;t\in[0,1]}
\|v(x,t)\|
&\leq
2\sqrt d,
\\
\label{eq:network_class_x_lipschitz_bound}
\|v(x,t)-v(x',t)\|
&\leq
\gamma_{x,R,\varepsilon}\|x-x'\|_\infty,
\qquad x,x'\in\mathbb R^d,\quad t\in[0,1],
\\
\label{eq:network_class_t_lipschitz_bound}
\|v(x,t)-v(x,s)\|
&\leq
\gamma_{t,R,\varepsilon}|t-s|,
\qquad x\in\mathbb R^d,\quad t,s\in[0,1],
\end{align}
where
\begin{align*}
\gamma_{x,R,\varepsilon}
&=
\frac{5(d+1)\sqrt d\,\Lambda}{R},
&
\gamma_{t,R,\varepsilon}
&=
10(d+1)\sqrt d\,\Lambda.
\end{align*}
\end{lemma}

\begin{proof}
From the coordinatewise clipping, we have $|v_i(x,t)|\leq 2, \forall i=1,\ldots,d$.
Hence $\|v(x,t)\|\leq 2\sqrt d$.

Next, fix \(x,x'\in\mathbb R^d\) and \(t\in[0,1]\). Since
\(\operatorname{clip}_{2}\) is \(1\)-Lipschitz and each \(h_i\) is
\(10q\)-Lipschitz with respect to \(\|\cdot\|_\infty\),
\begin{align*}
|v_i(x,t)-v_i(x',t)|
&\leq
\Lambda
\left|
h_i(S_R(x,t))-h_i(S_R(x',t))
\right|\leq
10q\Lambda
\|S_R(x,t)-S_R(x',t)\|_\infty.
\end{align*}
By the definition of \(S_R\),
$\|S_R(x,t)-S_R(x',t)\|_\infty \leq \frac{1}{2R} \|x-x'\|_\infty$.
Therefore, $|v_i(x,t)-v_i(x',t)|\leq \frac{5(d+1)\Lambda}{R}\|x-x'\|_\infty$.
Summing over coordinates gives
\begin{align*}
\|v(x,t)-v(x',t)\|
&\leq
\frac{5(d+1)\sqrt d\,\Lambda}{R}
\|x-x'\|_\infty.
\end{align*}

For variation in the time coordinate, \(t,s\in[0,1]\),
\begin{align*}
|v_i(x,t)-v_i(x,s)|
&\leq
\Lambda
\left|
h_i(S_R(x,t))-h_i(S_R(x,s))
\right|\leq
10q\Lambda
\|S_R(x,t)-S_R(x,s)\|_\infty=
10(d+1)\Lambda |t-s|.
\end{align*}
Summing over coordinates gives $\|v(x,t)-v(x,s)\|\leq 10(d+1)\sqrt d\,\Lambda |t-s|$.
\end{proof}

\subsubsection{Completion of the Velocity Bound}

Recall that for \(v\in\mathcal V_n\), we defined the pointwise velocity loss
\begin{align}
\Phi_v^{\mathrm{vel}}(y_0,y_1,k,g)
&=
\left\|
v\left(
y_0+t_k(y_1-y_0)+\sigma g,t_k
\right)
-(y_1-y_0)
\right\|^2.
\label{eq:velocity_pointwise_loss}
\end{align}
Equip \(\mathcal V_n\) with
\[
\mathrm d_v(v,\widetilde v)
=
\sup_{x\in[-R,R]^d,\;t\in[0,1]}
\|v(x,t)-\widetilde v(x,t)\|.
\]

\begin{lemma}[Constants for the velocity loss]
\label{lem:velocity_loss_constants}
The loss in \eqref{eq:velocity_pointwise_loss} satisfies
\eqref{eq:generic_loss_envelope}--\eqref{eq:generic_parameter_lipschitz}
with
\begin{align}
B_{\mathrm{vel}}
&=16d,\nonumber\\
L_{Y,\mathrm{vel}}
&=
8\sqrt d
\left(
1+\gamma_{x,R,\varepsilon}
\right),\nonumber\\
L_{\Theta,\mathrm{vel}}
&=8\sqrt d.
\label{eq:velocity_loss_constants}
\end{align}
Moreover,
\begin{align}
\log\mathcal N(r,\mathcal V_n,\mathrm d_v)
&\leq\mathfrak C_v(r).
\label{eq:velocity_metric_entropy}
\end{align}
\end{lemma}

\begin{proof}
Write
\[
S_t=(1-t)y_0+ty_1,\qquad D=y_1-y_0,
\]
and define \(S_t'\) and \(D'\) from \(y_0',y_1'\).  Bounded support and
\eqref{eq:network_class_output_bound} give
\[
\|v(S_t+\sigma g,t)-D\|
\leq
\|v(S_t+\sigma g,t)\|+\|D\|
\leq4\sqrt d,
\]
so the squared loss is at most \(16d\).

For two endpoint pairs,
\[
\|S_t-S_t'\|_\infty
\leq
\|y_0-y_0'\|+\|y_1-y_1'\|,
\qquad
\|D-D'\|
\leq
\|y_0-y_0'\|+\|y_1-y_1'\|.
\]
Using
\(\left|\|a\|^2-\|b\|^2\right|
\leq(\|a\|+\|b\|)\|a-b\|\) and
\eqref{eq:network_class_x_lipschitz_bound}, we obtain
\begin{align*}
&\left|
\Phi_v^{\mathrm{vel}}(y_0,y_1,k,g)
-\Phi_v^{\mathrm{vel}}(y_0',y_1',k,g)
\right|\\
&\leq
8\sqrt d
\left[
\gamma_{x,R,\varepsilon}\|S_{t_k}-S_{t_k}'\|_\infty
+\|D-D'\|
\right]\\
&\leq
L_{Y,\mathrm{vel}}
\left(
\|y_0-y_0'\|+\|y_1-y_1'\|
\right).
\end{align*}
Averaging over \(K,G\) proves the required endpoint bound for
\(H_v^{\mathrm{vel}}\).

For \(v,\widetilde v\in\mathcal V_n\), both residuals have norm at most
\(4\sqrt d\), and hence
\[
\left|
\Phi_v^{\mathrm{vel}}(y_0,y_1,k,g)
-\Phi_{\widetilde v}^{\mathrm{vel}}(y_0,y_1,k,g)
\right|
\leq
8\sqrt d\,
\|v(S_{t_k}+\sigma g,t_k)-\widetilde v(S_{t_k}+\sigma g,t_k)\|.
\]
Every network in \(\mathcal V_n\) depends on \(x\) through
\(\operatorname{clip}_R(x)\), so the last norm is bounded by
\(\mathrm d_v(v,\widetilde v)\).  This proves the parameter bound.
Finally, \eqref{eq:velocity_metric_entropy} is exactly
\eqref{eq:velocity_network_covering_number} under the metric
\(\mathrm d_v\).
\end{proof}

\begin{proof}[Proof of the velocity bound in Proposition~\ref{thm:final_velocity_L2_bound_DR}]
Choose
\[
\varepsilon=n^{-1/(d+5)},
\qquad
R=
1+\sigma
\left[
2\log\left(\frac{32d^2}{\varepsilon^2}\right)
\right]^{1/2}.
\]
This choice satisfies the tail condition in
Proposition~\ref{prop:relu_l2_approximation_noisy_oracle}.  By
\eqref{eq:velocity_exact_excess_risk_identity} and that proposition,
\begin{align}
\inf_{v\in\mathcal V_n}
\left\{
L_{a,Q}(v)-L_{a,Q}(v_Q^\star)
\right\}
&\leq2\varepsilon^2.
\label{eq:velocity_approximation_for_shared_bound}
\end{align}

For \(\Phi_v^{\mathrm{vel}}\),
\(L_{a,Q}^{\Phi}(v)=L_{a,Q}(v)\), and
\(\widehat L_{\mathrm{actual}}^\Phi(v)\) is the velocity training objective.
Lemma~\ref{lem:velocity_loss_constants} verifies every loss hypothesis of
Proposition~\ref{prop:shared_DR_field_bound}, with
\[
\Theta=\mathcal V_n,\quad
\mathrm d_\Theta=\mathrm d_v,\quad
B_\Phi=B_{\mathrm{vel}},\quad
L_Y=L_{Y,\mathrm{vel}},\quad
L_\Theta=L_{\Theta,\mathrm{vel}}.
\]

We now evaluate the quantities in that proposition.  Since
\(\Lambda=16d(2R+1)/\sigma^2\),
\begin{align}
\eta^{-1}
&=
\frac{\sqrt d\,\Lambda}{\varepsilon}
=
\frac{
16d^{3/2}
\left[
3+2\sigma
\sqrt{2\log(32d^2/\varepsilon^2)}
\right]
}{
\sigma^2\varepsilon
}.
\label{eq:velocity_inverse_tolerance_streamlined}
\end{align}
The architecture orders in
Proposition~\ref{prop:relu_l2_approximation_noisy_oracle}, substituted into
\(\mathfrak C_v\), imply for the three radii in
Proposition~\ref{prop:shared_DR_field_bound}
\begin{align}
\log N_{\mathrm{DR}}^{\mathrm{vel}}
+\log N_{\mathrm{obs}}^{\mathrm{vel}}
+\log N_{\mathrm{aux}}^{\mathrm{vel}}
&=
\widetilde O\left(\eta^{-(d+1)}\right).
\label{eq:velocity_shared_entropy_order}
\end{align}
Here \(\widetilde O\) hides logarithms in
\(n,M,\sigma^{-1},\varepsilon_{\mathrm{OT}}^{-1}\), and
\(\pi_{\min}^{-1}\), and constants depending on fixed \(d\).
Combining \eqref{eq:generic_statistical_remainder},
\eqref{eq:velocity_inverse_tolerance_streamlined}, and
\eqref{eq:velocity_shared_entropy_order}, with \(\delta=n^{-1}\), give
\begin{align}
\mathfrak R_{\mathrm{vel}}(n^{-1})
&\lesssim
\frac{1+2\pi_{\min}^{-1}}{\sigma^{d+1}}
n^{-2/(d+5)}.
\label{eq:velocity_shared_statistical_rate}
\end{align}
Indeed, before substituting \(\varepsilon\), the leading term is
\[
\widetilde O\left(
(1+2\pi_{\min}^{-1})
\sigma^{-(d+1)}
\varepsilon^{-(d+1)/2}n^{-1/2}
\right),
\]
and
\(\varepsilon^{-(d+1)/2}n^{-1/2}=n^{-2/(d+5)}\).

The endpoint constant has the explicit order
\begin{align}
L_{Y,\mathrm{vel}}
&=
8\sqrt d
\left[
1+
\frac{5(d+1)\sqrt d\,\Lambda}{R}
\right]
=O(\sigma^{-2}),
\label{eq:velocity_endpoint_constant_order}
\end{align}
because \(R\geq1\) and
\(\Lambda/R=16d(2+R^{-1})/\sigma^2\).
On the event in Corollary~\ref{Cor_Coupling_Replacement},
\[
\Delta_Q
\lesssim
\varepsilon_{\mathrm{OT}}^{-1}
\left[
n^{-1/(d+2)}
+\epspsiemp
+(nM_Q)^{-1/(d+2)}
\right].
\]
Substitute this display,
\eqref{eq:velocity_approximation_for_shared_bound},
\eqref{eq:velocity_shared_statistical_rate}, and
\eqref{eq:velocity_endpoint_constant_order} into
\eqref{eq:shared_DR_field_bound}.  Since \(M_Q\geq1\),
\((nM_Q)^{-1/(d+2)}\leq n^{-1/(d+2)}\), and since
\(0<\sigma\leq1\), the approximation term is absorbed by
\eqref{eq:velocity_shared_statistical_rate}.  Thus
\begin{align*}
L_{a,Q}(\widehat v_M)-L_{a,Q}(v_Q^\star)
&\lesssim
\frac{\varepsilon_{\mathrm{OT}}^{-1}}{\sigma^2}
\left[
n^{-1/(d+2)}+\epspsiemp
\right]\\
&\quad+
\frac{1+\varepsilon_{\mathrm{OT}}^{-1}}{\sigma^2}
\varepsilon_\alpha\varepsilon_{\mathrm{pred}}
+
\frac{1+2\pi_{\min}^{-1}}{\sigma^{d+1}}
n^{-2/(d+5)}.
\end{align*}
Applying \eqref{eq:velocity_exact_excess_risk_identity} with
\(v=\widehat v_M\) proves
\eqref{eq:final_velocity_L2_bound_DR_theorem}.  The coupling event has
probability at least \(1-2/n\), and the conditional ERM bound with
\(\delta=1/n\) has probability at least \(1-1/n\).  Finally,
\(\mathbb P(\mathcal E_a^c)\leq e^{-n\pi_{\min}}\), which gives the stated
failure probability.
\end{proof}

%% file: Score_Error.tex
\subsection{Score-Field Bound}
\label{Appendix_Score_Error}

Recall that \(S_t=Y_0+t(Y_1-Y_0)\) and \(X_t=S_t+\sigma G\). We define the
 oracle by
\(u_Q^\star(x,t)=\mathbb E[S_t\mid X_t=x]\). By the Gaussian score identity,
\begin{align*}
s_Q^\star(x,t) &= \frac{u_Q^\star(x,t)-x}{\sigma^2}.
\end{align*}
Instead of approximating \(s_Q^\star\) directly, we approximate
\(u_Q^\star\). For every
\(u:\mathbb R^d\times[0,1]\to\mathbb R^d\), define the induced score field by
\begin{align*}
s_u(x,t) &= \frac{u(x,t)-x}{\sigma^2}.
\end{align*}

\subsubsection[Score Loss and Exact Risk Identity]
{Score Loss and Exact Risk Identity}

We now specialize the generic doubly robust bound to the score estimator.
Here the abstract parameter \(\theta\) indexes a conditional-mean field
\(u_\theta\in\mathcal U_n\), where $\mathcal U_{n}$ is defined later in \eqref{eq_u_n}.
This in turn induces the score field
\[
s_{u_\theta}(x,t)=\frac{u_\theta(x,t)-x}{\sigma^2}.
\]
For brevity, we suppress the abstract index \(\theta\) and write \(u\) for
\(u_\theta\) and \(s_u\) for the induced score field.
For \(u\in\mathcal U_n\), define
\[
\Phi_u^{\mathrm{score}}(y_0,y_1,k,g)
=
\left\|
s_u\!\left(y_0+t_k(y_1-y_0)+\sigma g,t_k\right)
+\frac{g}{\sigma}
\right\|^2.
\]

Recall population minimizer of the denoising score-matching loss satisfies
\begin{align*}
s_Q^\star(x,t) = \nabla_x\log q_{t,Q}(x)
=
\mathbb E
\left[
-\frac{G}{\sigma}
\mid
Y_0+t(Y_1-Y_0)+\sigma G=x
\right],
\end{align*}
where \(Y_1\sim\int \mathbb P_{Y \mid C=c,A=a}d\mathbb P_C(c)\) and
\(Y_0\sim Q(\cdot\mid Y_1)\).

For a score field \(s:\mathbb R^d\times[0,1]\to\mathbb R^d\), let
\(K\sim\operatorname{Unif}\{0,\ldots,N-1\}\), independently of all endpoint
and Gaussian variables, we have 
\begin{align*}
r_Q(s;y_1)
&=
\mathbb E_{K,
Y_0\sim Q(\cdot\mid y_1),\,
G\sim {\mathcal N}(0,I_d)
}
\left[
\left\|
s\left(
Y_0+t_K(y_1-Y_0)+\sigma G,t_K
\right)
+
\frac{G}{\sigma}
\right\|^2
\right], \\
L^{\mathrm{score}}_{a,Q}(s)
&=
\mathbb E_{Y_1\sim\int \mathbb P_{Y \mid C=c, A=a} d \mathbb P_{C}}
\left[
r_Q(s;Y_1)
\right], \\
r_{\widehat Q_n}(s;y_1)
&=
\mathbb E_{K,
Y_0\sim \widehat Q_n(\cdot\mid y_1),\,
G\sim {\mathcal N}(0,I_d)
}
\left[
\left\|
s\left(
Y_0+t_K(y_1-Y_0)+\sigma G,t_K
\right)
+
\frac{G}{\sigma}
\right\|^2
\right],
\\
L^{\mathrm{score}}_{a,\widehat Q_n}(s)
&=
\mathbb E_{Y_1\sim\int \mathbb P_{Y \mid C=c, A=a} d \mathbb P_{C}}
\left[
r_{\widehat Q_n}(s;Y_1)
\right].
\end{align*}

\paragraph{Exact Excess-Risk Identity.}
Because
\(s_Q^\star(X_{t_K},t_K)
=\mathbb E[-G/\sigma\mid X_{t_K},K]\), conditional projection gives, for
every score field \(s\) with finite population loss,
\begin{align}
\label{eq:score_exact_excess_risk_identity}
L^{\mathrm{score}}_{a,Q}(s)-L^{\mathrm{score}}_{a,Q}(s_Q^\star)
&=
\mathbb E\!\left[
\left\|s(X_{t_K},t_K)-s_Q^\star(X_{t_K},t_K)\right\|^2
\right]\nonumber\\
&=
\frac1N\sum_{k=0}^{N-1}
\mathbb E\!\left[
\left\|s(X_{t_k},t_k)-s_Q^\star(X_{t_k},t_k)\right\|^2
\right].
\end{align}
Thus it remains to control the excess risk of the empirical minimizer.

\subsubsection[Score Network Class and Approximation]
{Score Network Class and Approximation}

The denoising score-matching loss is evaluated under the joint sampling
procedure in which the same Gaussian draw \(G\) is used both to construct
\(X_t=S_t+\sigma G\) and to form the denoising target \(-G/\sigma\).
Consequently,
\begin{align}
\label{eq:denoiser_score_loss_identity}
\left\|s_u(S_t+\sigma G,t)+\frac{G}{\sigma}\right\|^2
= \left\|\frac{u(S_t+\sigma G,t)-S_t-\sigma G}{\sigma^2} +\frac{G}{\sigma}\right\|^2 = \frac{1}{\sigma^4}\left\|u(S_t+\sigma G,t)-S_t\right\|^2.
\end{align}
This identity is not a statement about arbitrary independently chosen
\(x\in\mathbb R^d\) and \(g\in\mathbb R^d\). It applies to the jointly
constructed pair \(x=S_t+\sigma g\), with the same \(g\) appearing in the
denoising target. Consequently, the population score loss can be written as
\begin{align*}
L^{\mathrm{score}}_{a,Q}(s_u)
&= \frac{1}{\sigma^4} \mathbb E_{K}
\mathbb E\left[\left\|u(X_{t_K},t_K)-S_{t_K}\right\|^2\right].
\end{align*}
Since \(u_Q^\star(X_t,t)=\mathbb E[S_t\mid X_t]\), the conditional projection
identity gives
\begin{align}
\label{eq:denoiser_excess_risk_identity}
L^{\mathrm{score}}_{a,Q}(s_u)-L^{\mathrm{score}}_{a,Q}(s_Q^\star)
&= \frac{1}{\sigma^4} \mathbb E_{K}
\mathbb E\left[\left\|u(X_{t_K},t_K)-u_Q^\star(X_{t_K},t_K)\right\|^2\right].
\end{align}
Moreover,
\(s_u(X_t,t)-s_Q^\star(X_t,t)
=\{u(X_t,t)-u_Q^\star(X_t,t)\}/\sigma^2\), so
\begin{align}
\label{eq:denoiser_excess_risk_score_error}
L^{\mathrm{score}}_{a,Q}(s_u)-L^{\mathrm{score}}_{a,Q}(s_Q^\star)
&= \mathbb E_{K}
\mathbb E\left[\left\|s_u(X_{t_K},t_K)-s_Q^\star(X_{t_K},t_K)\right\|^2\right].
\end{align}

Under Assumption \ref{ass:clipping_bounded_support}, \(S_t\in[-1,1]^d\)
almost surely. Therefore,
\begin{align}
\label{eq:denoiser_oracle_output_bound}
u_Q^\star(x,t) &= \mathbb E[S_t\mid X_t=x]\in[-1,1]^d
\end{align}
for every \(x\in\mathbb R^d\) and \(t\in[0,1]\). Thus \(u_Q^\star\) is
globally bounded, although the induced score
\(s_Q^\star(x,t)=\{u_Q^\star(x,t)-x\}/\sigma^2\) is not globally bounded.

Recall that for \(R>1\), we defined
\begin{align*}
T_R(z) &= \left(-R+2Rz_1,\ldots,-R+2Rz_d,z_{d+1}\right),
\qquad z\in[0,1]^{d+1}.
\end{align*}

\begin{lemma}[Lipschitz Regularity of the Conditional Mean Oracle]
\label{lem:noisy_conditional_mean_lipschitz_on_cube}
Assume Assumption \ref{ass:clipping_bounded_support}. For \(i=1,\ldots,d\), define
\(g_i(z)=u_{Q,i}^\star(T_R(z))\). Then
\begin{align*}
|g_i(z)-g_i(z')| &\leq \Lambda_u\|z-z'\|_\infty,\qquad z,z'\in[0,1]^{d+1},
\end{align*}
where
\begin{align}
\label{eq:conditional_mean_rescaled_lipschitz_constant}
\Lambda_u &= 2+\frac{8d(2R+1)}{\sigma^2}.
\end{align}
\end{lemma}

\begin{proof}
By the score identity, $u_Q^\star(x,t) = x+\sigma^2s_Q^\star(x,t)$. By Lemma \ref{Lemma_Derivative_Identity},
\begin{align*}
D_xs_Q^\star(x,t) &= \frac{1}{\sigma^2}\left(\operatorname{Cov}(G,G\mid X_t=x)-I_d\right).
\end{align*}
Therefore,
\begin{align*}
D_xu_Q^\star(x,t) &= I_d+\sigma^2D_xs_Q^\star(x,t)=\operatorname{Cov}(G,G\mid X_t=x).
\end{align*}
Under the conditioning \(X_t=x\), we have \(G=(x-S_t)/\sigma\), and hence
\begin{align*}
G-\mathbb E[G\mid X_t=x] &= -\frac{S_t-\mathbb E[S_t\mid X_t=x]}{\sigma}.
\end{align*}
Consequently, for every \(i,j\in\{1,\ldots,d\}\),
\begin{align*}
\partial_{x_j}u_{Q,i}^\star(x,t)
&= \operatorname{Cov}(G_i,G_j\mid X_t=x) = \frac{1}{\sigma^2}\operatorname{Cov}(S_{t,i},S_{t,j}\mid X_t=x).
\end{align*}
Since \(S_t\in[-1,1]^d\),
\begin{align*}
\left|S_{t,i}-\mathbb E[S_{t,i}\mid X_t=x]\right| &\leq 2,\qquad
\left|S_{t,j}-\mathbb E[S_{t,j}\mid X_t=x]\right|\leq 2.
\end{align*}
It follows that
\begin{align}
\label{eq:conditional_mean_spatial_derivative_bound}
\left|\partial_{x_j}u_{Q,i}^\star(x,t)\right|
&\leq \frac{1}{\sigma^2}\mathbb E\left[
\left|S_{t,i}-\mathbb E[S_{t,i}\mid X_t=x]\right|
\left|S_{t,j}-\mathbb E[S_{t,j}\mid X_t=x]\right|
\mid X_t=x\right]
\leq \frac{4}{\sigma^2}.
\end{align}

We next control the time derivative.  Write \(D=Y_1-Y_0\) and, for fixed
\(x\), define the unnormalized posterior weight
\[
w_t(Y_0,Y_1;x)
=
\exp\!\left(-\frac{\|x-S_t\|^2}{2\sigma^2}\right).
\]
Bounded support permits differentiation under the endpoint expectation, and
\[
\partial_t w_t(Y_0,Y_1;x)
=
\frac{D^\top(x-S_t)}{\sigma^2}
w_t(Y_0,Y_1;x).
\]
Therefore, for any differentiable scalar \(F_t=F_t(Y_0,Y_1)\), the quotient
rule applied to
\[
\mathbb E[F_t\mid X_t=x]
=
\frac{\mathbb E[F_t w_t(Y_0,Y_1;x)]}
{\mathbb E[w_t(Y_0,Y_1;x)]}
\]
gives the posterior differentiation identity
\begin{align*}
\partial_t\mathbb E[F_t\mid X_t=x]
&=
\mathbb E[\partial_tF_t\mid X_t=x]
+
\frac{1}{\sigma^2}
\operatorname{Cov}\!\left(
F_t,D^\top(x-S_t)\mid X_t=x
\right).
\end{align*}
Taking \(F_t=S_{t,i}\), for which \(\partial_tF_t=D_i\), yields
\begin{align*}
\partial_tu_{Q,i}^\star(x,t)
&= \mathbb E[Y_{1,i}-Y_{0,i}\mid X_t=x] +\frac{1}{\sigma^2}\operatorname{Cov}\left(
S_{t,i},(Y_1-Y_0)^\top(x-S_t)\mid X_t=x
\right).
\end{align*}
The first term satisfies
\begin{align*}
\left|\mathbb E[Y_{1,i}-Y_{0,i}\mid X_t=x]\right| &\leq 2.
\end{align*}
If \(x\in[-R,R]^d\), then
\begin{align*}
\left|(Y_1-Y_0)^\top(x-S_t)\right| \leq \sum_{j=1}^d|Y_{1,j}-Y_{0,j}||x_j-S_{t,j}| \leq 2d(R+1).
\end{align*}
Therefore,
\begin{align*}
&\left|(Y_1-Y_0)^\top(x-S_t)
-\mathbb E\left[(Y_1-Y_0)^\top(x-S_t)\mid X_t=x\right]\right| \\
&\leq \left|(Y_1-Y_0)^\top(x-S_t)\right|
+\mathbb E\left[\left|(Y_1-Y_0)^\top(x-S_t)\right|\mid X_t=x\right] \\
&\leq 4d(R+1).
\end{align*}
Together with
\(\left|S_{t,i}-\mathbb E[S_{t,i}\mid X_t=x]\right|\leq2\), this gives
\begin{align}
\label{eq:conditional_mean_time_derivative_bound}
\left|\partial_tu_{Q,i}^\star(x,t)\right|
&\leq 2+\frac{1}{\sigma^2}\mathbb E\left[
\left|S_{t,i}-\mathbb E[S_{t,i}\mid X_t=x]\right|
\right. \nonumber\\
&\hspace{3.4cm}\left.
\cdot\left|(Y_1-Y_0)^\top(x-S_t)
-\mathbb E\left[(Y_1-Y_0)^\top(x-S_t)\mid X_t=x\right]\right|
\mid X_t=x\right] \nonumber\\
&\leq 2+\frac{8d(R+1)}{\sigma^2}.
\end{align}

For \(j=1,\ldots,d\), the chain rule and
\eqref{eq:conditional_mean_spatial_derivative_bound} give $|\partial_{z_j}g_i(z)| = 2R\left|\partial_{x_j}u_{Q,i}^\star(T_R(z))\right| \leq \frac{8R}{\sigma^2}$, so we get 
\begin{align*}
|g_i(z)-g_i(z')|
&\leq \left[\sum_{j=1}^{d+1}\sup_{\widetilde z\in[0,1]^{d+1}}
|\partial_{z_j}g_i(\widetilde z)|\right]\|z-z'\|_\infty \leq \left[\frac{8Rd}{\sigma^2}+2+\frac{8d(R+1)}{\sigma^2}\right]\|z-z'\|_\infty \\
&= \left[2+\frac{8d(2R+1)}{\sigma^2}\right]\|z-z'\|_\infty
=\Lambda_u\|z-z'\|_\infty.
\end{align*}
\end{proof}

\paragraph{Network Construction and Approximation Guarantee.}

Let \(\varepsilon>0\), and define
\begin{align}
\label{eq:denoiser_eta_definition}
\eta_u &= \frac{\sigma^2\varepsilon}{\sqrt d\,\Lambda_u}.
\end{align}

\begin{proposition}[Approximation Error for the Score Oracle Under the Conditional Mean Parametrization]
\label{prop:relu_l2_approximation_noisy_score_oracle}
Assume Assumption \ref{ass:clipping_bounded_support}. Let \(R>1\),
\(q=d+1\), and \(\varepsilon>0\) satisfy
\(0<\eta_u=\sigma^2\varepsilon/(\sqrt d\,\Lambda_u)<1\) and
\begin{align}
\label{eq:denoiser_approximation_tail_condition}
\frac{8d^2}{\sigma^4}
\exp\left(-\frac{(R-1)^2}{2\sigma^2}\right)
&\leq \varepsilon^2.
\end{align}
Use the same architecture bounds
\(K_{\mathrm{Lip}}(q,\eta_u)\), \(L_{\mathrm{Lip}}(q,\eta_u)\),
\(p_{\mathrm{Lip}}(q,\eta_u)\), \(J_{\mathrm{Lip}}(q,\eta_u)\), and
\(\kappa_{\mathrm{Lip}}(q,\eta_u)\) as in
Proposition~\ref{prop:relu_l2_approximation_noisy_oracle}, now evaluated at
\(\eta_u\). In particular, \(L_{\mathrm{Lip}}\) counts total affine layers,
\(p_{\mathrm{Lip}}\geq q\), and \(\kappa_{\mathrm{Lip}}\geq2\).
For every \(i=1,\ldots,d\), construct
\begin{align*}
u_i^{\mathrm{app}}(x,t)
&= \operatorname{clip}_1
\left(\Lambda_u\widehat h_i(S_R(x,t))+g_i(0)\right),
\end{align*}
where
\begin{align*}
\widehat h_i
&\in\mathfrak N_{\mathrm{scal}}
\left(
q,L_{\mathrm{Lip}}(q,\eta_u),
p_{\mathrm{Lip}}(q,\eta_u),
J_{\mathrm{Lip}}(q,\eta_u),
K_{\mathrm{Lip}}(q,\eta_u),
\kappa_{\mathrm{Lip}}(q,\eta_u),10q
\right),
\end{align*}
and \(g_i(z)=u_{Q,i}^\star(T_R(z))\). Define
\(u^{\mathrm{app}}=(u_1^{\mathrm{app}},\ldots,u_d^{\mathrm{app}})\) and
\(s^{\mathrm{app}}(x,t)
=\{u^{\mathrm{app}}(x,t)-x\}/\sigma^2\). Then
\begin{align*}
\frac{1}{N}\sum_{k=0}^{N-1}
\mathbb E\left[
\left\|s^{\mathrm{app}}(X_{t_k},t_k)-s_Q^\star(X_{t_k},t_k)\right\|^2
\right]
&\leq 2\varepsilon^2.
\end{align*}
\end{proposition}

\begin{proof}
The proof has two subclaims: uniform approximation on the cube
\([-R,R]^d\times[0,1]\), followed by control of the Gaussian tail outside the
cube.

For the first subclaim, define
\[
f_i(z)=\frac{g_i(z)-g_i(0)}{\Lambda_u},
\qquad z\in[0,1]^{d+1}.
\]
Lemma~\ref{lem:noisy_conditional_mean_lipschitz_on_cube} implies that
\(f_i(0)=0\) and that \(f_i\) is \(1\)-Lipschitz in the
\(\ell_\infty\) norm. Since
\(\|z-z'\|_\infty\leq\|z-z'\|_2\), it is also \(1\)-Lipschitz in the
Euclidean norm required by the scalar approximation result invoked in
Proposition~\ref{prop:relu_l2_approximation_noisy_oracle}. That result provides
\(\widehat h_i\) in the displayed network class such that
\[
\sup_{z\in[0,1]^{d+1}}|\widehat h_i(z)-f_i(z)|\leq\eta_u.
\]
If \(x\in[-R,R]^d\), then \(T_R(S_R(x,t))=(x,t)\), and hence
\[
\left|
\Lambda_u\widehat h_i(S_R(x,t))+g_i(0)-u_{Q,i}^\star(x,t)
\right|
\leq\Lambda_u\eta_u
=\frac{\sigma^2\varepsilon}{\sqrt d}.
\]
Both \(u_{Q,i}^\star(x,t)\) and the clipped approximation lie in
\([-1,1]\), so projection onto this interval cannot increase the error.
Consequently,
\[
\sup_{x\in[-R,R]^d,\ t\in[0,1]}
\|s^{\mathrm{app}}(x,t)-s_Q^\star(x,t)\|^2
\leq\varepsilon^2.
\]

For the second subclaim, \(S_t\in[-1,1]^d\) almost surely, and a coordinatewise
Gaussian tail bound followed by a union bound gives, uniformly in \(t\),
\[
\mathbb P\{X_t\notin[-R,R]^d\}
\leq
2d\exp\!\left(-\frac{(R-1)^2}{2\sigma^2}\right).
\]
On this event, the two conditional mean remain in \([-1,1]^d\), so
\[
\|s^{\mathrm{app}}(X_t,t)-s_Q^\star(X_t,t)\|^2
=
\frac{\|u^{\mathrm{app}}(X_t,t)-u_Q^\star(X_t,t)\|^2}{\sigma^4}
\leq\frac{4d}{\sigma^4}.
\]
The tail contribution is therefore at most
\[
\frac{8d^2}{\sigma^4}
\exp\!\left(-\frac{(R-1)^2}{2\sigma^2}\right)
\leq\varepsilon^2
\]
by \eqref{eq:denoiser_approximation_tail_condition}.  Adding the cube and
tail contributions proves the claimed \(2\varepsilon^2\) bound at every grid
time, and averaging over \(k=0,\ldots,N-1\) completes the proof.
\end{proof}

The preceding proposition motivates 
\begin{align}\label{eq_u_n}
\mathcal U_n
&= \Bigg\{
u:\mathbb R^d\times[0,1]\to\mathbb R^d:
u(x,t)=\left(u_1(x,t),\ldots,u_d(x,t)\right), \\
&\qquad
u_i(x,t)=\operatorname{clip}_1
\left(\Lambda_u h_i(S_R(x,t))+b_i\right),
\qquad i=1,\ldots,d, \nonumber \\
&\qquad
h_i\in\mathfrak N_{\mathrm{scal}}
\left(
q,L_{\mathrm{Lip}}(q,\eta_u),
p_{\mathrm{Lip}}(q,\eta_u),
J_{\mathrm{Lip}}(q,\eta_u),
K_{\mathrm{Lip}}(q,\eta_u),
\kappa_{\mathrm{Lip}}(q,\eta_u),10q
\right), \nonumber \\
&\qquad
|b_i|\leq 1,\qquad i=1,\ldots,d
\Bigg\}. \nonumber
\end{align}
The induced score class is
\begin{align}
\label{eq:induced_score_class}
\mathcal S_n
&= \left\{
s_u:\mathbb R^d\times[0,1]\to\mathbb R^d:
s_u(x,t)=\frac{u(x,t)-x}{\sigma^2},
\quad u\in\mathcal U_n
\right\}.
\end{align}
Proposition \ref{prop:relu_l2_approximation_noisy_score_oracle} implies
\begin{align*}
\inf_{s\in\mathcal S_n}
\frac{1}{N}\sum_{k=0}^{N-1}
\mathbb E\left[
\left\|s(X_{t_k},t_k)-s_Q^\star(X_{t_k},t_k)\right\|^2
\right]
\leq 2\varepsilon^2.
\end{align*}

\paragraph{Covering Number Evaluation.}

The conditional-mean class is an instance of
Lemma~\ref{lem:covering_number_from_para} with clipping level
\(\omega=1\), scale \(\Lambda_u\), and scalar-network accuracy
\(\eta_u=\sigma^2\varepsilon/(\sqrt d\,\Lambda_u)\).  Its scalar architecture
orders are
\begin{equation}
\label{eq:score_network_architecture_orders}
\begin{aligned}
L_{\mathrm{Lip}}(q,\eta_u)
&=
O\!\left(
\log\frac{\sqrt d\,\Lambda_u}{\sigma^2\varepsilon}+d
\right),\\
p_{\mathrm{Lip}}(q,\eta_u)
&=
O\!\left(
d+1+
\left(
\frac{\sqrt d\,\Lambda_u}{\sigma^2\varepsilon}
\right)^{d+1}
\right),\\
J_{\mathrm{Lip}}(q,\eta_u)
&=
O\!\left(
\left(
\frac{\sqrt d\,\Lambda_u}{\sigma^2\varepsilon}
\right)^{d+1}
\left[
\log\frac{\sqrt d\,\Lambda_u}{\sigma^2\varepsilon}+d
\right]
\right).
\end{aligned}
\end{equation}
For \(r>0\), define
\begin{align}
\mathfrak C_u(r)
&:=
dJ_{\mathrm{Lip}}(q,\eta_u)
\log\!\left(
1+
\frac{
8L_{\mathrm{Lip}}(q,\eta_u)^2
(p_{\mathrm{Lip}}(q,\eta_u)+2)
\kappa_{\mathrm{Lip}}(q,\eta_u)^{L_{\mathrm{Lip}}(q,\eta_u)}
p_{\mathrm{Lip}}(q,\eta_u)^{L_{\mathrm{Lip}}(q,\eta_u)+1}
\Lambda_u\sqrt d
}{
r
}
\right)\nonumber\\
&\quad+
d\log\!\left(1+\frac{4\sqrt d}{r}\right).
\label{eq:score_cover_complexity}
\end{align}
Therefore, for every \(\delta_{\mathrm{cov}}>0\),
\begin{align}
\label{eq:denoiser_network_covering_number}
\log\mathcal N\left(
\delta_{\mathrm{cov}},\mathcal U_n,
\|\cdot\|_{\infty,[-R,R]^d\times[0,1]}
\right) \leq
\mathfrak C_u(\delta_{\mathrm{cov}}).
\end{align}

For \(u,\widetilde u\in\mathcal U_n\),
\begin{align*}
s_u(x,t)-s_{\widetilde u}(x,t)
&= \frac{u(x,t)-\widetilde u(x,t)}{\sigma^2}.
\end{align*}
Therefore,
\begin{align*}
&\sup_{x\in[-R,R]^d,\;t\in[0,1]}
\|s_u(x,t)-s_{\widetilde u}(x,t)\| = \frac{1}{\sigma^2}
\sup_{x\in[-R,R]^d,\;t\in[0,1]}
\|u(x,t)-\widetilde u(x,t)\|.
\end{align*}
It follows that
\begin{align*}
&\mathcal N\left(
\delta_{\mathrm{cov}},\mathcal S_n,
\|\cdot\|_{\infty,[-R,R]^d\times[0,1]}
\right) \leq
\mathcal N\left(
\sigma^2\delta_{\mathrm{cov}},\mathcal U_n,
\|\cdot\|_{\infty,[-R,R]^d\times[0,1]}
\right).
\end{align*}
Hence
\begin{align}
\log\mathcal N\left(
\delta_{\mathrm{cov}},\mathcal S_n,
\|\cdot\|_{\infty,[-R,R]^d\times[0,1]}
\right)\leq
\mathfrak C_u(\sigma^2\delta_{\mathrm{cov}}).
\label{eq:induced_score_network_covering_number}
\end{align}
The fixed affine term \(-x/\sigma^2\) does not increase the covering number,
because it is identical for every element of \(\mathcal S_n\) and cancels
when two induced score fields are subtracted.

\paragraph{Uniform Bounds Used in the Error Analysis.}

\begin{lemma}[Uniform Bounds for the Conditional-Mean Class and Induced Score Field]
\label{lem:score_network_class_uniform_bounds}
Every \(u\in\mathcal U_n\) satisfies
\begin{align}
\label{eq:denoiser_network_class_output_bound}
\sup_{x\in\mathbb R^d,\;t\in[0,1]}\|u(x,t)\|
&\leq \sqrt d, \\
\label{eq:denoiser_network_class_x_lipschitz_bound}
\|u(x,t)-u(x',t)\|
&\leq \gamma_{x,R,\varepsilon}^{u}\|x-x'\|_\infty, \\
\label{eq:denoiser_network_class_t_lipschitz_bound}
\|u(x,t)-u(x,r)\|
&\leq \gamma_{t,R,\varepsilon}^{u}|t-r|,
\end{align}
where
\begin{align*}
\gamma_{x,R,\varepsilon}^{u}
&= \frac{5(d+1)\sqrt d\,\Lambda_u}{R},\qquad
\gamma_{t,R,\varepsilon}^{u}
=10(d+1)\sqrt d\,\Lambda_u.
\end{align*}
The induced score field satisfies
\begin{align}
\label{eq:induced_score_class_x_lipschitz_bound}
\|s_u(x,t)-s_u(x',t)\|
&\leq
\frac{\gamma_{x,R,\varepsilon}^{u}+\sqrt d}{\sigma^2}
\|x-x'\|_\infty, \\
\label{eq:induced_score_class_t_lipschitz_bound}
\|s_u(x,t)-s_u(x,r)\|
&\leq
\frac{\gamma_{t,R,\varepsilon}^{u}}{\sigma^2}|t-r|.
\end{align}
No finite uniform bound on
\(\sup_{x\in\mathbb R^d,\;t\in[0,1]}\|s_u(x,t)\|\) is asserted.
\end{lemma}

\begin{proof}
The proof has two parts.  First, clipping each coordinate to \([-1,1]\)
immediately gives \(\|u(x,t)\|\leq\sqrt d\).

For regularity, each scalar subnetwork \(h_i\) is
\(10q=10(d+1)\)-Lipschitz in the \(\ell_\infty\) norm.  The spatial part of
\(S_R\) is \(1/(2R)\)-Lipschitz, while its time coordinate is
\(1\)-Lipschitz.  Since the outer clipping map is \(1\)-Lipschitz,
\begin{align*}
|u_i(x,t)-u_i(x',t)|
&\leq
\frac{5(d+1)\Lambda_u}{R}\|x-x'\|_\infty,\\
|u_i(x,t)-u_i(x,r)|
&\leq
10(d+1)\Lambda_u|t-r|.
\end{align*}
Summing the squared coordinate bounds proves
\eqref{eq:denoiser_network_class_x_lipschitz_bound} and
\eqref{eq:denoiser_network_class_t_lipschitz_bound}.

Finally,
\[
s_u(x,t)-s_u(x',t)
=
\frac{u(x,t)-u(x',t)-(x-x')}{\sigma^2}.
\]
Using \(\|x-x'\|\leq\sqrt d\,\|x-x'\|_\infty\) gives
\eqref{eq:induced_score_class_x_lipschitz_bound}; the time bound follows
because the affine term \(-x/\sigma^2\) is independent of time.  The score
itself is not uniformly bounded because of this affine term, which is why the
subsequent loss analysis uses its exact cancellation with the denoising
target.
\end{proof}

\begin{lemma}[Uniform Bounds and Lipschitz Properties of the Denoising Score Loss]
\label{lemma_uniform_bound_score_loss}
For every \(u\in\mathcal U_n\), \(y_t\in[-1,1]^d\), \(t\in[0,1]\), and \(g\in\mathbb R^d\),
\begin{align}
\label{eq:denoiser_pointwise_loss_envelope}
0\leq \left\|s_u(y_t+\sigma g,t)+\frac{g}{\sigma}\right\|^2 &\leq \frac{4d}{\sigma^4}.
\end{align}
For every \(y_t,y_t'\in[-1,1]^d\), using the same \(g\) in both evaluation points,
\begin{align}
\label{eq:denoiser_endpoint_pointwise_difference}
\left|\left\|s_u(y_t+\sigma g,t)+\frac{g}{\sigma}\right\|^2-\left\|s_u(y_t'+\sigma g,t)+\frac{g}{\sigma}\right\|^2\right|
&\leq \frac{4d}{\sigma^4}\left(1+\frac{5(d+1)\Lambda_u}{R}\right)\|y_t-y_t'\|_\infty.
\end{align}
Finally, for every \(u,\widetilde u\in\mathcal U_n\),
\begin{align}
\label{eq:denoiser_network_pointwise_difference}
\left|\left\|s_u(y_t+\sigma g,t)+\frac{g}{\sigma}\right\|^2-\left\|s_{\widetilde u}(y_t+\sigma g,t)+\frac{g}{\sigma}\right\|^2\right| &\leq \frac{4\sqrt d}{\sigma^4}\left\|u(y_t+\sigma g,t)-\widetilde u(y_t+\sigma g,t)\right\|.
\end{align}
\end{lemma}

\begin{proof}
By \(s_u(x,t)=\{u(x,t)-x\}/\sigma^2\), and because the same \(g\) appears in the evaluation point \(y_t+\sigma g\) and in the denoising target \(g/\sigma\),
\begin{align*}
s_u(y_t+\sigma g,t)+\frac{g}{\sigma} &= \frac{u(y_t+\sigma g,t)-(y_t+\sigma g)}{\sigma^2}+\frac{g}{\sigma}=\frac{u(y_t+\sigma g,t)-y_t}{\sigma^2}.
\end{align*}
Since \(u(y_t+\sigma g,t)\in[-1,1]^d\) and \(y_t\in[-1,1]^d\),
\begin{align*}
\left\|s_u(y_t+\sigma g,t)+\frac{g}{\sigma}\right\| &= \frac{\|u(y_t+\sigma g,t)-y_t\|}{\sigma^2}\leq \frac{2\sqrt d}{\sigma^2}.
\end{align*}
Squaring proves \eqref{eq:denoiser_pointwise_loss_envelope}.

For \(y_t,y_t'\in[-1,1]^d\), applying the cancellation identity at each
evaluation point gives
\begin{align*}
s_u(y_t+\sigma g,t)+\frac{g}{\sigma} &= \frac{u(y_t+\sigma g,t)-y_t}{\sigma^2},\qquad
s_u(y_t'+\sigma g,t)+\frac{g}{\sigma}=\frac{u(y_t'+\sigma g,t)-y_t'}{\sigma^2}.
\end{align*}
Using \(\left|\|a\|^2-\|b\|^2\right|\leq(\|a\|+\|b\|)\|a-b\|\),
\begin{align*}
&\left|\left\|s_u(y_t+\sigma g,t)+\frac{g}{\sigma}\right\|^2-\left\|s_u(y_t'+\sigma g,t)+\frac{g}{\sigma}\right\|^2\right| \\
&\leq \frac{4\sqrt d}{\sigma^4}\left[\left\|u(y_t+\sigma g,t)-u(y_t'+\sigma g,t)\right\|+\|y_t-y_t'\|\right] \\
&\leq \frac{4\sqrt d}{\sigma^4}\left(\gamma_{x,R,\varepsilon}^{u}+\sqrt d\right)\|y_t-y_t'\|_\infty \\
&= \frac{4d}{\sigma^4}\left(1+\frac{5(d+1)\Lambda_u}{R}\right)\|y_t-y_t'\|_\infty.
\end{align*}
where note that
\begin{align*}
\|(y_t+\sigma g)-(y_t'+\sigma g)\|_\infty &= \|y_t-y_t'\|_\infty.
\end{align*}
This proves \eqref{eq:denoiser_endpoint_pointwise_difference}.

Finally, for \(u,\widetilde u\in\mathcal U_n\),
\begin{align*}
s_u(y_t+\sigma g,t)+\frac{g}{\sigma} &= \frac{u(y_t+\sigma g,t)-y_t}{\sigma^2},\qquad
s_{\widetilde u}(y_t+\sigma g,t)+\frac{g}{\sigma}=\frac{\widetilde u(y_t+\sigma g,t)-y_t}{\sigma^2}.
\end{align*}
Therefore,
\begin{align*}
\left|\left\|s_u(y_t+\sigma g,t)+\frac{g}{\sigma}\right\|^2-\left\|s_{\widetilde u}(y_t+\sigma g,t)+\frac{g}{\sigma}\right\|^2\right|
&\leq \frac{4\sqrt d}{\sigma^4}\left\|u(y_t+\sigma g,t)-\widetilde u(y_t+\sigma g,t)\right\|,
\end{align*}
which proves \eqref{eq:denoiser_network_pointwise_difference}.
\end{proof}

Using the same \(g\) in the two evaluation points is valid in the later population comparison because both losses are integrated with respect to the same Gaussian measure. In particular,
\begin{align*}
&\mathbb E_G\left[\left\|s_u(y_t+\sigma G,t)+\frac{G}{\sigma}\right\|^2\right]-\mathbb E_G\left[\left\|s_u(y_t'+\sigma G,t)+\frac{G}{\sigma}\right\|^2\right] \\
&= \mathbb E_G\left[\left\|s_u(y_t+\sigma G,t)+\frac{G}{\sigma}\right\|^2-\left\|s_u(y_t'+\sigma G,t)+\frac{G}{\sigma}\right\|^2\right].
\end{align*}
This is only a common-variable representation of two expectations under the same Gaussian law. It does not impose any additional probabilistic relationship between \(y_t\) and \(y_t'\). For each term, the same Gaussian \(G\) is used in its noisy evaluation point and its denoising target.

\subsubsection{Completion of the Score Bound}

Recall that for \(u\in\mathcal U_n\), we defined
\begin{align}
\Phi_u^{\mathrm{score}}(y_0,y_1,k,g)
&=
\left\|
s_u\left(
y_0+t_k(y_1-y_0)+\sigma g,t_k
\right)
+\frac{g}{\sigma}
\right\|^2.
\label{eq:score_pointwise_loss}
\end{align}
We equip \(\mathcal U_n\) with
\[
\mathrm d_u(u,\widetilde u)
=
\sup_{x\in[-R,R]^d,\;t\in[0,1]}
\|u(x,t)-\widetilde u(x,t)\|.
\]

\begin{lemma}[Constants for the score loss]
\label{lem:score_loss_constants_streamlined}
The loss in \eqref{eq:score_pointwise_loss} satisfies
\eqref{eq:generic_loss_envelope}--\eqref{eq:generic_parameter_lipschitz}
with
\begin{align}
B_{\mathrm{score}}
&=
\frac{4d}{\sigma^4},\nonumber\\
L_{Y,\mathrm{score}}
&=
\frac{4d}{\sigma^4}
\left(
1+\frac{5(d+1)\Lambda_u}{R}
\right),\nonumber\\
L_{\Theta,\mathrm{score}}
&=
\frac{4\sqrt d}{\sigma^4}.
\label{eq:score_loss_constants_streamlined}
\end{align}
Moreover,
\begin{align}
\log\mathcal N(r,\mathcal U_n,\mathrm d_u)
&\leq\mathfrak C_u(r).
\label{eq:score_metric_entropy}
\end{align}
\end{lemma}

\begin{proof}
Let \(S_t=(1-t)y_0+ty_1\).  The cancellation
\[
s_u(S_t+\sigma g,t)+\frac g\sigma
=
\frac{u(S_t+\sigma g,t)-S_t}{\sigma^2}
\]
and the bounds \(u(S_t+\sigma g,t),S_t\in[-1,1]^d\) give
\(\Phi_u^{\mathrm{score}}\leq4d/\sigma^4\).

For a second endpoint pair, write
\(S_t'=(1-t)y_0'+ty_1'\).  The pointwise bound
\eqref{eq:denoiser_endpoint_pointwise_difference} gives
\begin{align*}
&\left|
\Phi_u^{\mathrm{score}}(y_0,y_1,k,g)
-\Phi_u^{\mathrm{score}}(y_0',y_1',k,g)
\right|\\
&\leq
\frac{4d}{\sigma^4}
\left(
1+\frac{5(d+1)\Lambda_u}{R}
\right)
\|S_{t_k}-S_{t_k}'\|_\infty\\
&\leq
L_{Y,\mathrm{score}}
\left(
\|y_0-y_0'\|+\|y_1-y_1'\|
\right).
\end{align*}
Averaging this inequality over \(K,G\) proves the endpoint condition.
The parameter perturbation bound
\eqref{eq:denoiser_network_pointwise_difference} gives
\[
\left|
\Phi_u^{\mathrm{score}}(y_0,y_1,k,g)
-\Phi_{\widetilde u}^{\mathrm{score}}(y_0,y_1,k,g)
\right|
\leq
\frac{4\sqrt d}{\sigma^4}
\|u(S_{t_k}+\sigma g,t_k)-\widetilde u(S_{t_k}+\sigma g,t_k)\|.
\]
Every \(u\in\mathcal U_n\) depends on its spatial argument through
\(\operatorname{clip}_R\), so the last norm is at most
\(\mathrm d_u(u,\widetilde u)\).  Finally,
\eqref{eq:score_metric_entropy} follows from
\eqref{eq:denoiser_network_covering_number}.
\end{proof}

\begin{proof}[Proof of the score bound in Proposition~\ref{thm:final_score_L2_bound_DR}]
Choose
\[
\varepsilon=n^{-1/(d+5)},
\qquad
R=
1+\sqrt2\,\sigma
\left[
\log\left(
\frac{8d^2}{\sigma^4\varepsilon^2}
\right)
\right]^{1/2}.
\]
The tail condition in
Proposition~\ref{prop:relu_l2_approximation_noisy_score_oracle} then holds.
Write \(\widehat s_M=s_{\widehat u_M}\) for the empirical minimizer, with
\(\widehat u_M\in\mathcal U_n\).  By
\eqref{eq:denoiser_excess_risk_score_error} and the approximation
proposition,
\begin{align}
\inf_{u\in\mathcal U_n}
\left\{
L_{a,Q}^{\mathrm{score}}(s_u)
-L_{a,Q}^{\mathrm{score}}(s_Q^\star)
\right\}
&\leq2\varepsilon^2.
\label{eq:score_approximation_for_shared_bound}
\end{align}

For \(\Phi_u^{\mathrm{score}}\),
\(L_{a,Q}^{\Phi}(u)=L_{a,Q}^{\mathrm{score}}(s_u)\), and
\(\widehat L_{\mathrm{actual}}^\Phi(u)\) is the implemented score objective.
Lemma~\ref{lem:score_loss_constants_streamlined} verifies the hypotheses of
Proposition~\ref{prop:shared_DR_field_bound} with
\[
\Theta=\mathcal U_n,\quad
\mathrm d_\Theta=\mathrm d_u,\quad
B_\Phi=B_{\mathrm{score}},\quad
L_Y=L_{Y,\mathrm{score}},\quad
L_\Theta=L_{\Theta,\mathrm{score}}.
\]

For the selected \(R\),
\begin{align}
\frac{\Lambda_u}{R}
&=
\frac2R+\frac{8d(2R+1)}{R\sigma^2}
=O(\sigma^{-2}),\nonumber\\
\Lambda_u
&=\widetilde O(\sigma^{-2}),\qquad
\eta_u^{-1}
=
\frac{\sqrt d\,\Lambda_u}{\sigma^2\varepsilon}
=
\widetilde O\left(
\frac1{\sigma^4\varepsilon}
\right).
\label{eq:score_scales_streamlined}
\end{align}
Substituting the architecture orders
\eqref{eq:score_network_architecture_orders} into \(\mathfrak C_u\), at the
three radii in Proposition~\ref{prop:shared_DR_field_bound}, yields
\begin{align}
\log N_{\mathrm{DR}}^{\mathrm{score}}
+\log N_{\mathrm{obs}}^{\mathrm{score}}
+\log N_{\mathrm{aux}}^{\mathrm{score}}
&=
\widetilde O\left[
\left(
\frac1{\sigma^4\varepsilon}
\right)^{d+1}
\right].
\label{eq:score_shared_entropy_order}
\end{align}
Since \(B_{\mathrm{score}}=4d/\sigma^4\), combining
\eqref{eq:generic_statistical_remainder} and
\eqref{eq:score_shared_entropy_order}, with \(\delta=n^{-1}\), give
\begin{align}
\mathfrak R_{\mathrm{score}}(n^{-1})
&\lesssim
\frac{1+2\pi_{\min}^{-1}}{\sigma^{2d+6}}
n^{-2/(d+5)}.
\label{eq:score_shared_statistical_rate}
\end{align}
Before the substitution \(\varepsilon=n^{-1/(d+5)}\), the leading term is
\[
\widetilde O\left(
(1+2\pi_{\min}^{-1})
\sigma^{-(2d+6)}
\varepsilon^{-(d+1)/2}n^{-1/2}
\right).
\]

The bounds \eqref{eq:score_loss_constants_streamlined} and
\eqref{eq:score_scales_streamlined} give
\begin{align}
L_{Y,\mathrm{score}}
&=O(\sigma^{-6}).
\label{eq:score_endpoint_constant_order}
\end{align}
On the event in Corollary~\ref{Cor_Coupling_Replacement},
\[
\Delta_Q
\lesssim
\varepsilon_{\mathrm{OT}}^{-1}
\left[
n^{-1/(d+2)}
+\epspsiemp
+(nM_Q)^{-1/(d+2)}
\right].
\]
Substitute this display,
\eqref{eq:score_approximation_for_shared_bound},
\eqref{eq:score_shared_statistical_rate}, and
\eqref{eq:score_endpoint_constant_order} into
\eqref{eq:shared_DR_field_bound}.  The inequalities
\((nM_Q)^{-1/(d+2)}\leq n^{-1/(d+2)}\) and \(0<\sigma\leq1\) give
\begin{align*}
L_{a,Q}^{\mathrm{score}}(\widehat s_M)
-L_{a,Q}^{\mathrm{score}}(s_Q^\star)
&\lesssim
\frac{\varepsilon_{\mathrm{OT}}^{-1}}{\sigma^6}
\left[
n^{-1/(d+2)}+\epspsiemp
\right]\\
&\quad+
\frac{1+\varepsilon_{\mathrm{OT}}^{-1}}{\sigma^6}
\varepsilon_\alpha\varepsilon_{\mathrm{pred}}
+
\frac{1+2\pi_{\min}^{-1}}{\sigma^{2d+6}}
n^{-2/(d+5)}.
\end{align*}
Finally, \eqref{eq:score_exact_excess_risk_identity} converts this excess
risk into the grid-averaged score error in
\eqref{eq:final_score_L2_bound_DR_theorem}.  The coupling and conditional
statistical events fail with probability at most \(2/n\) and \(1/n\),
respectively, while
\(\mathbb P(\mathcal E_a^c)\leq e^{-n\pi_{\min}}\).  A final union bound
with the velocity event proves that both conclusions of
Proposition~\ref{thm:final_velocity_L2_bound_DR} hold with probability
\(1-\mathcal O(n^{-1})-e^{-n\pi_{\min}}\).
\end{proof}

%% file: Proof_End_to_End.tex
\section[Proof of Theorem \ref{thm:end_to_end_sample_complexity}]
{Proof of Theorem~\ref{thm:end_to_end_sample_complexity}}

\begin{proof}

Work on the event where the population coupling \(Q\) is defined and the
two bounds in Proposition~\ref{thm:final_velocity_L2_bound_DR} hold.
Since \(D_5\leq\overline D_5<\infty\),
Lemma~\ref{lem:exact_sampling_conditions} supplies a weak solution \(Z\)
satisfying \(Z_t\sim q_t\) for every \(t\in[0,1]\).
Moreover, the grid-frozen learned drifts are continuous by continuity of the
fitted neural networks.
Thus the \hyperref[cond:standing_sampling]{standard sampling conditions} hold.

Since \(h=1/N\), each grid-averaged field bound in
Proposition~\ref{thm:final_velocity_L2_bound_DR} is exactly the corresponding integrated
estimation error in Assumption
\ref{Assumption_Integrated_Estimation_Error}.  Thus that proposition may
be substituted into Theorem~\ref{Theorem_Sampling_Error}.
In the stated regime, the velocity part of
Proposition~\ref{thm:final_velocity_L2_bound_DR} gives
\begin{align*}
\varepsilon_{\mathrm{vel}}^2
&=
\widetilde O\left(
\frac{2+\varepsilon_{\mathrm{OT}}^{-1}}{\sigma^2}
n^{-1/(d+2)}
+
\frac{1+2\pi_{\min}^{-1}}{\sigma^{d+1}}
n^{-2/(d+5)}
\right),
\end{align*}
while its score part gives
\begin{align*}
\varepsilon_{\mathrm{score}}^2
&=
\widetilde O\left(
\frac{1}{\sigma^6\varepsilon_{\mathrm{OT}}}
n^{-1/(d+2)}
+
\frac{1+2\pi_{\min}^{-1}}{\sigma^{2d+6}}
n^{-2/(d+5)}
\right).
\end{align*}
Since \(\tau^2=\sigma \overline{D_5}/d\),
\begin{align}\label{eq_eps_vel_target}
\frac{\varepsilon_{\mathrm{vel}}^2}{\tau^2}
&=
\widetilde O\left(
\frac{
d\left(2+\varepsilon_{\mathrm{OT}}^{-1}\right)
}{
\sigma^3\overline{D_5}
}
n^{-1/(d+2)}
+
\frac{
d\left(1+2\pi_{\min}^{-1}\right)
}{
\sigma^{d+2}\overline{D_5}
}
n^{-2/(d+5)}
\right),
\end{align}
and
\begin{align}\label{eq_eps_score_target}
\tau^2\varepsilon_{\mathrm{score}}^2
&=
\widetilde O\left(
\frac{
\overline{D_5}
}{
d\varepsilon_{\mathrm{OT}}\sigma^5
}
n^{-1/(d+2)}
+
\frac{
\left(1+2\pi_{\min}^{-1}\right)\overline{D_5}
}{
d\sigma^{2d+5}
}
n^{-2/(d+5)}
\right).
\end{align}

By Corollary \ref{cor:counterfactual_W1_from_KL}, under
\(\frac{h\tau^2}{2\sigma^2}\leq1\) and \(\sigma\leq1\),
\begin{align*}
W_1\left(
\int \mathbb P_{Y\mid C=c,A=a}\,d\mathbb P_C(c),
\mathcal L(\widehat Z_N)
\right)
&\lesssim
\sigma+
\operatorname{KL}\left(
\mathcal L(Y_1+\sigma G)
\middle\|
\mathcal L(\widehat Z_N)
\right)^{1/4}.
\end{align*}
Taking \(\sigma=\Theta(\varepsilon_{\mathrm{target}})\), it is sufficient to
require
\begin{align*}
\operatorname{KL}\left(
\mathcal L(Y_1+\sigma G)
\middle\|
\mathcal L(\widehat Z_N)
\right)
&\lesssim
\varepsilon_{\mathrm{target}}^4.
\end{align*}

To make each estimation contribution in
\eqref{eq_eps_vel_target} and \eqref{eq_eps_score_target} at most of order
\(\varepsilon_{\mathrm{target}}^4\), it is sufficient that
\begin{align*}
n
&\gtrsim
\left[
\frac{
\left(2+\varepsilon_{\mathrm{OT}}^{-1}\right)
}{
\overline{D_5}
}
\right]^{d+2}
\varepsilon_{\mathrm{target}}^{-7(d+2)}, \qquad
n
\gtrsim
\left[
\frac{
\left(1+2\pi_{\min}^{-1}\right)
}{
\overline{D_5}
}
\right]^{(d+5)/2}
\varepsilon_{\mathrm{target}}^{-(d+6)(d+5)/2},\\
n
&\gtrsim
\left[
\frac{
\overline{D_5}
}{
\varepsilon_{\mathrm{OT}}
}
\right]^{d+2}
\varepsilon_{\mathrm{target}}^{-9(d+2)}, \qquad
n
\gtrsim
\left[
\left(1+2\pi_{\min}^{-1}\right)\overline{D_5}
\right]^{(d+5)/2}
\varepsilon_{\mathrm{target}}^{-(2d+9)(d+5)/2},
\end{align*}
where \(\sigma=\Theta(\varepsilon_{\mathrm{target}})\) has been substituted.
Note that the last condition determines the asymptotic dependence on
\(\varepsilon_{\mathrm{target}}\).

It remains to control the discretization terms in
Theorem \ref{Theorem_Sampling_Error}. The choice
\begin{align*}
h \le c_{h}
\left(
\frac{d\overline{D_5}^2}{\varepsilon_{\mathrm{target}}^4\sigma^2}
+
\frac{d\overline{D_5}}{\sigma}
\right)^{-1}
\end{align*}
implies
\(\frac{hd\overline{D_5}^2}{\sigma^2}\lesssim\varepsilon_{\mathrm{target}}^4\)
and \(\frac{hd\overline{D_5}}{\sigma}\lesssim1\).
Consequently,
\(\operatorname{KL}\!\left(
\mathcal L(Z_{t_N})\middle\|\mathcal L(\widehat Z_N)
\right)\lesssim\varepsilon_{\mathrm{target}}^4\).
Moreover,
\begin{align*}
\frac{h\tau^2}{2\sigma^2}
&=
\frac{h\overline{D_5}}{2d\sigma}
\leq
\frac{1}{2d^2} \frac{hd\overline{D_5}}{\sigma}
= \frac{h\overline{D_5}}{2d\sigma} \le c_{h}
\frac{\frac{\overline{D_5}}{2d\sigma}}{\frac{d\overline{D_5}^2}{\varepsilon_{\mathrm{target}}^4\sigma^2}
+
\frac{d\overline{D_5}}{\sigma}} \le \frac{c_h }{2d^2}
\le
1,
\end{align*}
so the required stability condition holds.
\end{proof}